%% file: main.tex
\documentclass[twoside,11pt]{article}

\usepackage{blindtext}

\usepackage{jmlr2e}
\usepackage{enumerate}
\usepackage{amsmath}
\usepackage{bbm}
\usepackage{xcolor}
\usepackage{algorithm}
\usepackage{multirow}
\usepackage{adjustbox}
\usepackage{booktabs}
\usepackage{caption}
\usepackage{etoolbox}
\usepackage{wrapfig}



\def\CC {{\mathbb C}}     
\def\EE {{\mathbb E}}  
\def\RR {{\mathbb R}}     
\def\ZZ {{\mathbb Z}}     
\def\NN {{\mathbb N}}     

\def\var {\mathrm{Var}}  
\def\bias {\mathrm{Bias}}  
\def\cov {\mathrm{Cov}}  
\def\supp {\mathrm{Supp}}  

\def\bx {\boldsymbol{x}}
\def\bX {\boldsymbol{X}}
\def\bY {\boldsymbol{Y}}
\def\bt {\boldsymbol{t}}
\def\bw {\boldsymbol{w}}
\def\by {\boldsymbol{y}}
\def\bk {\boldsymbol{k}}
\def\bb {\boldsymbol{b}}

\def\ba {\boldsymbol{a}}
\def\bu {\boldsymbol{u}}
\def\bU {\boldsymbol{U}}
\def\bv {\boldsymbol{v}}
\def\bxi {\boldsymbol{\xi}}

\def\mcN {\mathcal{N}}

\def\mcX {\mathcal{X}}
\def\mcY {\mathcal{Y}}
\def\mcO {\mathcal{O}}
\def\mcR {\mathcal{R}}
\def\mcD {\mathcal{D}}

\def\mcU {\mathcal{U}}

\def\mcS {\mathcal{S}}

\def\tsigma \tilde{\sigma}
\def\tmu \tilde{\mu}
\def\bone {\mathbbm{1}}
\def\onev {\mathbf{1}}

\newenvironment{roadmap}{%
  \par\noindent{\bf Proof Roadmap\ }%
}{%
  \hfill\BlackBox\\[2mm]%
}

\newenvironment{restatedresult}[1]{%
  \begin{trivlist}
  \item[\hskip\labelsep\bfseries #1]\itshape
}{%
  \end{trivlist}
}

\input{theorem_statements}

\usepackage{lastpage}

\ShortHeadings{Theory and Practice of $NNGP$ and $GPnn$}{Allison and Maciazek and Stephenson}
\firstpageno{1}

\begin{document}
\title{The Theory and Practice of Highly Scalable Gaussian Process Regression with Nearest Neighbours}

\author{\name Robert Allison \thanks{All authors contributed equally to this work.} \email marfa@bristol.ac.uk  \\
	\name Tomasz Maciazek \footnotemark[1] \email tomasz.maciazek@bristol.ac.uk \\
    \name Anthony Stephenson \footnotemark[1] \email ant.stephenson@bristol.ac.uk  \\
       \addr School of Mathematics\\
       University of Bristol\\
       Bristol, BS8 1UG, UK}

\maketitle
\thispagestyle{plain}

\begin{abstract}
Gaussian process ($GP$) regression is a widely used non-parametric modeling tool, but its cubic complexity in the training size limits its use on massive data sets. A practical remedy is to predict using only the nearest neighbours of each test point, as in Nearest Neighbour Gaussian Process ($NNGP$) regression for geospatial problems and the related scalable $GPnn$ method for more general machine-learning applications. Despite their strong empirical performance, the large-$n$ theory of $NNGP/GPnn$ remains incomplete. We develop a theoretical framework for $NNGP$ and $GPnn$ regression. Under mild regularity assumptions, we derive almost sure pointwise limits for three key predictive criteria: mean squared error ($MSE$), calibration coefficient ($CAL$), and negative log-likelihood ($NLL$). We then study the $L_2$-risk, prove universal consistency, and show that the risk attains Stone’s minimax rate $n^{-2\alpha/(2p+d)}$, where $\alpha$ and $p$ capture regularity of the regression problem. We also prove uniform convergence of $MSE$ over compact hyper-parameter sets and show that its derivatives with respect to lengthscale, kernel scale, and noise variance vanish asymptotically, with explicit rates. This explains the observed robustness of $GPnn$ to hyper-parameter tuning. These results provide a rigorous statistical foundation for $NNGP/GPnn$ as a highly scalable and principled alternative to full $GP$ models.
\end{abstract}

\section{Introduction}
Gaussian Process ($GP$) regression \citep{RW05} has become a standard tool for statistical modeling, with applications ranging from geo-spatial statistics and \citep{stein1999interpolation} to time-series analysis \citep{Roberts2013GPTimeSeries} and Bayesian optimisation \citep{Jones1998,NIPS2012_05311655}. A key attraction of $GP$ models is that they are analytically tractable and that they provide both accurate point predictions and uncertainty quantification through the predictive mean and variance. However, the computational cost of exact GP inference scales cubically with the number of observations, $\mathcal{O}(n^3)$, due to the need to invert an $n\times n$ covariance matrix (often done via Cholesky decomposition). More modern implementations \citep{gpytorch} calculate matrix-vector multiplications directly and use conjugate gradients to attain better efficiency of near-$\mathcal{O}(n^2)$ for exact $GP$ inference. This complexity severely restricts their use on modern data sets containing millions to billions observations, which are increasingly common in, for example, environmental monitoring \citep{Kays}, climate modeling \citep{esd-12-401-2021}, and large-scale spatio-temporal applications \citep{Heaton2019}.

To address this limitation, a large body of work has proposed approximate $GP$ methods based on inducing points \citep{Snelson2005,pmlr-v5-titsias09a}, low-rank structure \citep{NIPS2000_19de10ad,BanerjeeGelfandFinleySang2008}, sparse precision matrices \citep{LindgrenRueLindstrom2011}, or structured kernel interpolation \citep{kiss-gp}. A particularly simple and practically attractive class of scalable methods is based on locality: predictions at a test point $\bx_*$ are computed using only a small neighbourhood of the training data around $\bx_*$. Among such methods, the recently proposed Gaussian process regression with nearest neighbours $GPnn$ \citep{GPnn23} stands out for its conceptual simplicity and strong empirical performance. For each test point, $GPnn$ selects its $m$ nearest neighbours (with respect to a chosen metric) and applies standard $GP$ regression on this local subset. Training reduces to preprocessing for fast nearest-neighbour search together with a cheap estimation of a small number of global kernel hyperparameters, while prediction and calibration are dominated by nearest-neighbour queries and the at most $\mathcal{O}(m^3)$ cost of inverting a local $m\times m$ covariance matrix. With $m$ fixed, as is often done in practice via validation or cross-validation \citep{JSSv103i05,GPnn23,nngp1,nngp2}, the resulting prediction cost is effectively independent of the full training size, up to nearest-neighbour search overhead, and is well suited to parallel and GPU implementations.

From a statistical perspective, $GPnn$ can be viewed as a kernel analogue of classical $k$-nearest-neighbour ($k$-NN) regression. The theory of $k$-NN methods is by now well developed, with universal consistency and minimax-optimal rates available under suitable smoothness and design assumptions \citep[see, e.g.,][]{Gyorfi2002,Kohler2006}. By contrast, the consistency and convergence properties of $GPnn$ and related nearest-neighbour $GP$ predictors remain much less understood. Empirically, $GPnn$ performs remarkably well in massive-data regimes and appears surprisingly robust to coarse hyperparameter choices \citep{GPnn23}, but several fundamental questions remain open:
\begin{itemize}
\item Is $GPnn$ universally consistent as the training size grows?
\item What are the asymptotic limits of its main predictive criteria, such as mean squared error ($MSE$), calibration ($CAL$), and negative log-likelihood ($NLL$)?
\item Can $GPnn$ attain Stone’s minimax-optimal convergence rates \citep{Stone_rates}?
\item How does the choice and growth of the neighbourhood size $m$ affect these properties?
\item Why does predictive risk appear to become relatively insensitive to the precise values of the kernel hyperparameters in large-data regimes?
\end{itemize}

A closely related line of work in geospatial statistics has led to Nearest Neighbour Gaussian Processes ($NNGP$) \citep{nngp1,nngp2}. These methods start from a spatial mixed-model formulation with a latent Gaussian random field and linear mean structure, and obtain scalable inference by conditioning each observation only on a small neighbour set. $NNGP$ models enable scalable Bayesian inference for large geospatial datasets and have become a standard practical tool in Bayesian spatial statistics \citep{nngp1,nngp2,JSSv103i05,Datta2016-wh}. In their usual form, however, $NNGP$ models treat covariance hyperparameters in a fully Bayesian manner through hierarchical modeling, which can remain computationally demanding at very large scales. In the present paper we take a different, explicitly prediction-focused viewpoint: we study a fixed-hyperparameter response-level $NNGP$ formulation in the same computational spirit as $GPnn$, thereby sacrificing full Bayesian hyperparameter averaging in favour of massive scalability. This perspective is, to our knowledge, new, but it remains closely tied to the established $NNGP$ literature, especially the conjugate $NNGP$ introduced by \cite{nngp2}. It is also strongly motivated by the empirical observation, made precise here, that predictive risk becomes increasingly insensitive to hyperparameter choice in the large-data regime. Thus, alongside a theory for $GPnn$, this work develops a corresponding theory for a practically scalable $NNGP$ predictor and provides evidence that this simplification need not materially harm predictive performance.

In this paper we develop a comprehensive theoretical analysis of $GPnn$ regression and its $NNGP$ counterpart. We study three key performance measures of the predictive distribution at a test location $\bx_*$: the mean squared error $MSE(\bx_*,X_n)$, the calibration coefficient $CAL(\bx_*,X_n)$, and the negative log-likelihood $NLL(\bx_*,X_n)$. Our analysis covers both \emph{pointwise} behaviour, where $X_n$ and $\bx_*$ are fixed, and \emph{integrated} behaviour, where we average over random draws of $X_n$ and $\bx_*$ from the data distribution.

Our main contributions are as follows.
\begin{enumerate}
\item \textbf{Pointwise limits and universal consistency.} Under mild regularity assumptions on the regression function, noise, and covariate distribution, we derive almost-sure pointwise limits for $MSE(\bx_*,X_n)$, $CAL(\bx_*,X_n)$, and $NLL(\bx_*,X_n)$ as $n\to\infty$ with fixed neighbourhood size $m$. In particular,
\[
MSE(\bx_*,X_n)\xrightarrow{n\to\infty}\sigma_\xi^2(\bx_*)\left(1+\frac{1}{m}\right),
\]
with analogous limits for $CAL$ and $NLL$. If $m=m_n$ grows with $n$ so that $m_n/n\to 0$, then the optimal asymptotic limit $MSE(\bx_*,X_n)\to \sigma_\xi^2(\bx_*)$ is recovered. Averaging over $\bx_*$ and $X_n$, we further show convergence of the expected $MSE$, i.e. the $L_2$-risk, thereby establishing universal consistency of $GPnn$ and $NNGP$.

\item \textbf{Minimax-optimal convergence rates.} When the true regression function is bounded and $q$-H\"{o}lder-continuous with $0\le q\le 1$, the $GP$ kernel satisfies a H\"{o}lder-type smoothness condition of order $0\le p\le 1$, and the data distribution satisfies suitable moment and regularity assumptions, we derive upper bounds on the expected $MSE$ of $GPnn$ and $NNGP$. These imply that, for $m_n\propto n^{2p/(2p+d)}$,
\[
\EE[MSE(\bx_*,X_n)]\le C\,n^{-2\alpha/(2p+d)},\qquad \alpha=\min\{p,q\},
\]
matching Stone’s minimax-optimal rate from general regression theory \citep{Stone_rates}.

\item \textbf{Robustness to hyperparameter choice.} We prove uniform convergence of $MSE$ as a function of training data size and $GPnn/NNGP$ hyperparameters $\Theta$ over compact subsets of the hyperparameter space, and show that its derivatives with respect to these hyperparameters converge uniformly to zero, with matching convergence rates. Thus, in the large-data regime, the $MSE$ landscape becomes asymptotically flat in $\Theta$, explaining the empirical robustness of $GPnn$ and $NNGP$ to coarse hyperparameter choice and the limited gains from expensive likelihood-based optimisation.

\item \textbf{Calibration of predictive distributions.} Motivated by the limiting behaviour of $CAL$ and $NLL$, we propose a simple and computationally cheap \emph{post-hoc} calibration procedure that rescales predictive variances while leaving predictive means unchanged. The procedure achieves exact variance calibration on a held-out calibration set and requires only a single scalar adjustment.

\item \textbf{Massively scalable synthetic experiments.} We extend large-scale synthetic simulation and bulk-prediction experiments for $GPnn/NNGP$ to regimes far beyond those considered previously, including sample sizes above $10^{11}$. This allows us to validate empirically both the predicted convergence of the predictive risk and the asymptotic flattening of the hyperparameter landscape.
\end{enumerate}

Taken together, our results provide a rigorous theoretical foundation for $GPnn$ and $NNGP$ regression. They show that these methods can be both highly scalable and theoretically principled: they enjoy universal consistency and minimax optimal rates, while their robustness to hyper-parameter tuning and the availability of simple calibration procedures make them practically attractive for large-scale applications. Subsequent sections formalise our assumptions and notation and present the main theoretical results. Detailed proofs together with auxiliary technical lemmas are presented in \href{https://github.com/tmaciazek/gpnn_synthetic/blob/main/Online_Appendix_1.pdf}{Online Appendix~1}. 

\section{Prior Work}
Nearest-neighbour Gaussian process methods were introduced in spatial statistics as scalable approximations to full $GP$ models, with emphasis on Bayesian inference and efficient computation for large geostatistical datasets \citep{vecchia_estimation_1988,nngp1,nngp2}. More recently, two of the authors of the present paper proposed $GPnn$ \citep{GPnn23} as a simple local $GP$ method for large-scale machine-learning regression, together with a practical calibration procedure and strong empirical results. The present work complements these methodological developments with a substantially broader predictive-risk theory for fixed-hyperparameter nearest-neighbour $GP$ regression.

Our results are related to the classical consistency and minimax-rate theory for nearest-neighbour regression \citep{Gyorfi2002}, as well as to the Bayesian asymptotic theory of $GP$s, including posterior consistency and contraction results for $GP$ regression models \citep{CHOI20071969,vdVvdZ2008,vdVvdZ2009}. Prior work in spatial statistics also shows that, under fixed-domain asymptotics, some covariance parameters in a Mat\'{e}rn full-$GP$ model are not consistently estimable even though the resulting predictions can be asymptotically equivalent \citep{Zhang2004}. Relatedly, the spatial $NNGP$ literature contains empirical evidence that predictive performance can be robust to covariance-hyperparameter choice, e.g., \citet{nngp2} reported similar mean squared prediction error across several $NNGP$ formulations despite notable differences in covariance-parameter estimates. Our results place this robustness on a rigorous footing by proving asymptotic flattening of the predictive-risk landscape with respect to the hyperparameters.

The companion conference paper \citep{GPnn23} introduced $GPnn$, its practical implementation, and a first (pointwise and finite nearest-neighbour set size regime) asymptotic robustness result in a substantially simpler zero-mean setting. By contrast, the present paper develops a unified infinite-regime treatment of both $GPnn$ and fixed-hyperparameter $NNGP$, allows a nontrivial mean structure and more general kernels, and establishes pointwise almost-sure predictive limits, approximate and universal consistency, Stone-type $L_2$-risk rates, uniform convergence over compact hyperparameter sets, and asymptotic vanishing and convergence rates of predictive-risk derivatives. We therefore view \citet{GPnn23} as an important precursor to the present work, but not as a substitute for the substantially broader theory developed here.

\section{Preliminaries}
\label{sec:preliminaries}
\paragraph{Notation for Random Variables} We denote the covariate domain space by $\Omega_\mcX\subset\RR^{d_\mcX}$ and a single covariate (random variable) by calligraphic $\mcX$. Similarly, single response variable is denoted by calligraphic $\mcY\in\RR$. The covariate/response distributions are denoted by $P_\mcX$ and $P_\mcY$ and their joint distribution is $P_{\mcX,\mcY}$. The random variables defined as i.i.d. samples of size $n$ of covariate-response pairs are denoted by uppercase boldface letters $(\bX_n,\bY_n)$, where $\bX_n=(\mcX_1,\dots,\mcX_n)$ and $\bY_n=(\mcY_1,\dots,\mcY_n)$. Single data realisations are denoted by lowercase letters. A realisation of $\mcX$ is $\bx\in\RR^{d_\mcX}$ and a realisation of $\mcY$ is $y$. An observed covariate sample is $X_n=(\bx_1,\dots,\bx_n)$ (a matrix of size $n\times d$) and an observed response sample is the vector $\by_n=(y_1,\dots,y_n)$. Then, the regression function can be written as $f(\bx)=\EE[\mcY|\mcX=\bx]$. Similarly, we denote the noise random variable as $\Xi$, it's single realisation as $\xi$ and a sample vector of length $n$ is $\bxi_n$. Any lowercase boldface characters will always denote vectors.

$GPnn$ \citep{GPnn23} is designed to tackle typical Machine Learning regression problems where the objective is to estimate sample values of an unknown function $f:\ \RR^d\to \RR$ and noise variance given a finite number of noisy measurements of the values of $f$. More specifically, we assume that the response variables are generated as
\begin{equation}\label{eq:gpnn_responses}
\mcY_i = f\left(\mcX_i\right) + \Xi_i,\quad i=1,\dots, n.
\end{equation}
An example of a $GPnn$ regression task would be to use the House Electric data \citep{houseelectric} to determine power consumption of a household based on its given characteristics. The covariates $\left\{\mcX_i\right\}_{i=1}^n$, the function $f$ and the mean-zero noise random variables $\Xi_i$ are assumed to satisfy certain assumptions that vary throughout the different sections of this paper. The most general set of assumptions (AC.\ref{a_X}-\ref{a_xi}) concerns the pointwise performance results of $GPnn$ regression presented in Section \ref{sec:consistency}. Results concerning different types of convergence rates of the $GPnn$ method that are presented in later sections require stricter assumptions which are specified in each Theorem. 

 Nearest Neighbour Gaussian Process ($NNGP$) has been designed for geo-spatial applications where the responses are assumed to be generated from a slightly more complex model which is a spatial linear mixed model. In this work we use the linear mixed model described in \cite{nngp1,nngp2}. There, the spatial locations are elements of $\Omega_\mcX\subset\RR^{d_\mcX}$ and to each location we (deterministically) assign a vector of regressors $\bt(\bx)\in \RR^{d_T}$. The responses $\mcY$ are assumed to be generated according to
\begin{equation}\label{eq:nngp_responses}
\mcY_i = \bt\left(\mcX_i\right)^T.\bb+w\left(\mcX_i\right)+\Xi_i,
\end{equation}
where $\bb\in\RR^{d_T}$ is the vector of regression coefficients, $\Xi_i$ is the additive noise and $w(\bx)$ is a sample path drawn from a $GP$ with mean-zero and covariance function $\tilde{k}:\, \RR^{d_\mcX}\times \RR^{d_\mcX}\to\RR$. The role of $w(\bx)$ is to model the effect of unknown/unobserved spatially-dependent covariates.  An example of a $NNGP$ regression task would be to determine the forest canopy height in a certain region $\Omega_\mcX\subset\RR^{2}$ on Earth given past fire history and tree cover that play the role of the regressors $\bt(\bx)$, see \cite{nngp2}.

Both $GPnn$ and $NNGP$ make predictions using only the training data in the nearest-neighbour set of a given test point. The notion of nearest neighbour depends on the underlying metric. In this work, for maximal generality, we formulate the pointwise consistency theory using the kernel-induced metric associated with the chosen $GP$ kernel (Definition~\ref{def:kernel_metric}), which in particular allows for periodic kernels and hence periodic nearest-neighbour structure. Most of the remaining results are proved under stronger assumptions, typically that the kernel-induced metric is a non-decreasing function of the Euclidean metric. Under this condition, the same conclusions also hold when nearest neighbours are chosen according to the Euclidean metric.

Let us next define the $GPnn$ and $NNGP$ predictive distributions. In what follows, we fix a continuous symmetric and positive definite kernel function $c:\, \RR^d\times \RR^d\to \RR$ normalised so that $c(\bx,\bx)=1$ and which determines the exact form of the $GPnn$ and $NNGP$ estimators. Consider a sequence of $n$ training points $X_n=(\bx_1,\dots,\bx_n)$ together with their response values $\by_n=(y_1,\dots,y_n)$, and a test point $\bx_*$. Let $\mcN_m(\bx_*,X_n)$ be the set of $m$-nearest neighbours of $\bx_*$ in $X_n$. Let $X_{\mcN}(\bx_*)=(\bx_{n,1}(\bx_*),\dots,\bx_{n,m}(\bx_*))$ be the sequence of the $m$-nearest neighbours of $\bx_*$ ordered increasingly according to their distance from $\bx_*$ (we assume that ties occur with probability zero) and let $\by_{\mcN}$ be their corresponding responses. Given the  hyperparameters: $\hat\sigma_f^2>0$ (the kernelscale), $\hat\sigma_\xi^2\geq 0$ (the noise variance) and $\hat\ell>0$ (the lengthscale) we define the (shifted) Gram matrix for $m$-nearest neighbours of $\bx_*$ as
\begin{equation}\label{eq:kernel_def}
\left[K_\mcN\right]_{ij}:= \hat\sigma_f^2\, c(\bx_{n,i}/\hat\ell,\bx_{n,j}/\hat\ell)+\hat\sigma_\xi^2\,\delta_{ij},\quad \left[\bk_{\mcN}^*\right]_j:=\hat\sigma_f^2\, c(\bx_*/\hat\ell,\bx_{n,j}/\hat\ell),\quad 1\leq i,j\leq m,
\end{equation}
where $\delta_{ij}$ is the Kronecker delta. 

\paragraph{$\mathbf{GPnn}$ Predictive Mean and Variance.} In $GPnn$, the predicted response at $\bx_*$ is characterized by the standard local $GP$ predictive mean and variance \citep{GP_book}
\[
\mu_{GPnn}={\bk_{\mcN}^*}^T K_\mcN^{-1}\by_{n,m},
\qquad
{\sigma_\mcN^*}^2=\hat\sigma_\xi^2+\hat\sigma_f^2-{\bk_{\mcN}^*}^T K_\mcN^{-1}\bk_{\mcN}^*.
\]
Our analysis relies only on these first two moments and does not require the full predictive distribution to be Gaussian, nor the data-generating mechanism to satisfy the Gaussian process assumptions underlying these formulas. As we explain in \href{https://github.com/tmaciazek/gpnn_synthetic/blob/main/Online_Appendix_1.pdf}{Online Appendix~1} (Lemma \ref{lemma:unbiased_estimator}, Section C), when $m$ is fixed the above defined estimator $\mu_{GPnn}$ is biased (after taking expectation over the noise) in the limit $n\to\infty$. We therefore replace it with its asymptotically unbiased counterpart that reads
\begin{equation}\label{eq:gpnn_unbiased_est}
{\tilde\mu}_{GPnn}(\bx_*)=\Gamma\,{\bk_{\mcN}^*}^T\,K_\mcN^{-1}\,\by_{\mcN},\quad \Gamma=\frac{\hat\sigma_\xi^2+m\hat\sigma_f^2}{m\hat\sigma_f^2}.
\end{equation}
We emphasize that the $\Gamma$-correction is relevant only in the fixed-$m$ regime. If $m=m_n$ grows with $n$, then $\Gamma\to 1$ as $n\to\infty$, so all of our results for growing neighbourhood size continue to hold for the standard non-$\Gamma$-corrected predictors from the literature, and these standard predictors are then asymptotically unbiased.

\paragraph{$\mathbf{NNGP}$ Predictive Mean and Variance.} \cite{nngp2} distinguishes three common $NNGP$ formulations according to how prediction is performed: collapsed $NNGP$, response $NNGP$ and conjugate $NNGP$ described in \citet[Algorithm~2, Algorithm~4, Algorithm~5 respectively]{nngp2}. The three formulations treat the regression hyperparameters in a Bayesian way by imposing suitable hyperparameter priors and then propagating the resulting hyperparameter posterior uncertainty into prediction, either through posterior sampling or, in conjugate $NNGP$, by analytic integration. Crucially, conditional on fixed hyperparameter values, the response predictor has a Gaussian predictive distribution in all three $NNGP$ formulations. In response and conjugate $NNGP$ this distribution has the following mean \citep[see][Algorithm 4]{nngp2}:
\begin{equation}\label{eq:nngp_unbiased_est}
{\tilde\mu}_{NNGP}(\bx_*)=\bt_*^T.\hat\bb+\Gamma\, {{k}_{\mcN}^*}^T\, K_{\mcN}^{-1}\left(\by_\mcN-T_\mcN.\hat\bb\right)
\end{equation}
and it's variance coincides with that of $GPnn$, see ${\sigma_\mcN^*}^2$ in Equation \eqref{eq:gpnn_unbiased_est}. In Equation \eqref{eq:nngp_unbiased_est} we have slightly adjusted the original version given in \cite{nngp2} by incorporating the factor $\Gamma$ thereby ensuring asymptotic unbiasedness when $m$ is fixed. $T_\mcN$ is the $m\times d_T$-matrix of regressors at the nearest-neighbours $\left(\bt\left(\bx_{n,1}(\bx_*)\right),\dots, \bt\left(\bx_{n,m}(\bx_*)\right)\right)$ and $\bt_* := \bt\left(\bx_*\right)$. The above form of $NNGP$ estimators is what we adapt for the purpose of this paper. This explicitly excludes collapsed $NNGP$ due to its reliance on posterior recovery of the latent spatial field $w(\cdotp)$ at all the observed locations, rather than on a direct response-level predictive formula involving only nearest-neighbours.


The fully Bayesian posterior predictive distributions in the three $NNGP$ formulations are obtained by averaging the respective hyperparameter-conditional predictive distributions over posterior uncertainty in the model parameters. In this paper, however, we do not adopt such a fully Bayesian treatment of the hyperparameters. Instead, our analysis is carried out under the assumption that the hyperparameters are fixed or pre-selected, for example using auxiliary or set-aside data. A related partial simplification is also adopted in the conjugate $NNGP$, where the lengthscale $\hat\ell$ and the ratio $\hat\sigma_\xi^2/\hat\sigma_f^2$ are fixed in advance, and only the remaining parameters $\hat\bb$, $\hat\sigma_f^2$ are integrated out in the posterior predictive distribution. Accordingly, the present results apply to the conditional predictive distribution associated with a single fixed choice of \emph{all} hyperparameters, and should be interpreted in that sense for the response and conjugate $NNGP$.


\begin{table}
\begin{center}
\begin{tabular}{ c  | c c c }
 & Response Model & Predictive Mean & Predictive Variance \\ 
 \hline
 $GPnn$ & \footnotesize{$f(\bx) + \xi(\bx)$} &  \footnotesize{$\Gamma\,{\bk_{\mcN}^*}^T\,K_\mcN^{-1}\,\by_{\mcN}$} & \multirow{ 2}{*}{\footnotesize{$\hat\sigma_\xi^2 +  \hat\sigma_f^2-{\bk_{\mcN}^*}^T\,K_\mcN^{-1}\bk_{\mcN}^*$} } \\  
 $NNGP$ &  \footnotesize{$\bt(\bx)^T.\bb+w(\bx)+\xi(\bx)$} &  \footnotesize{$\bt_*^T.\hat\bb+\Gamma\, {{k}_{\mcN}^*}^T\, K_{\mcN}^{-1}\left(\by_\mcN-T_\mcN.\hat\bb\right)$} &  
\end{tabular}
\end{center}
\caption{Summary of response models, predictive mean and variance in $GPnn$ and $NNGP$. See the main text for the explanation of the symbols. As explained in the main text, for $NNGP$ these formulas apply only in the conditional Gaussian predictive distribution which fixes the hyperparameters. The predictive means are corrected by the coefficient $\Gamma$ making them unbiased when $m$ is fixed and in the limit $n\to\infty$, as opposed to standard formulas used in the literature.}
\label{tab:predictive_summary}
\end{table}

\subsection{$L_2$-risk, Universal Consistency and Stone's Optimal Convergence Rates}
Consider the task of estimating (noiseless) latent regression function $f(\bx_*)$ in the generative model \eqref{eq:gpnn_responses} given noisy data $(X_n,\by_n)$. Denote the estimated value of $f$ at test point $\bx_*$ as $\hat f_n(\bx_*)$, where the subscript $n$ refers to the size of the training dataset. Assume that the training data are i.i.d. samples from the distribution $P_{\mcX,\mcY}$.

The $L_2(P_\mcX)$-risk (which we simply call \textit{risk} throughout the paper) is defined as
\begin{equation}\label{def:risk}
\mcR\left(\hat f_n\right):=\EE\left[\int\left(\hat f_n(\bx)-f(\bx)\right)^2dP_\mcX(\bx)\right],
\end{equation}
where the inner integral is taken over the test data given a training sample $X_n,\by_n$ and it can be viewed as the squared $L_2(P_\mcX)$-distance between $\hat f_n$ and $f$. The outer expectation is taken over all the training samples of size $n$ coming from $P_{\mcX,\mcY}^n$. Similarly, we can define an $L_2(P_\mcX)$-risk directly using the observed noisy responses (rather than the exact values of $f$) which is more applicable to the $GPnn$ and $NNGP$ response models \eqref{eq:gpnn_responses} and \eqref{eq:nngp_responses} as follows.
\begin{equation*}
\mcR_Y\left(\hat f_n\right):=\EE\left[\int\left(\hat f_n(\bx)-y\right)^2dP_{\mcX,\mcY}(\bx,\by)\right].
\end{equation*}
In our noise model specified in the assumption (AC.\ref{a_xi}) in Section \ref{sec:consistency} the above two $L_2(P_\mcX)$-risk measures differ by an additive constant, i.e.,
\[
\mcR_Y\left(\hat f_n\right)=\mcR\left(\hat f_n\right)+\int\sigma_\xi^2\left(\mcX\right)dP_\mcX(\bx),
\]
where $\sigma_\xi^2(\mcX)$ is the variance of the noise variable $\Xi$ at $\mcX$.

We say that the estimator $\hat f_n(\bx_*)$ is \textit{universally consistent} with respect to a family of training data distributions $\mcD$ if it satisfies the following conditions.
\begin{definition}[Universal Consistency]
    A sequence of regression function estimates $(\hat f_n)$ is universally consistent with respect to $\mcD$ if for all distributions $P_{\mcX,\mcY}\in \mcD$ we have 
    \[
    \mcR\left(\hat f_n\right)\xrightarrow{n\to\infty}0.
    \]
\end{definition}
The above definition of universal consistency is standard in the regression theory literature. For instance, \cite{Gyorfi2002} call this property the {\it{weak universal consistency}}, whereas other works often drop the ``weak" qualifier. In this work, we study nearest-neighbour-based estimators which are indexed by $n$ (the training data size) and $m$ (the number of nearest-neighbours). There, we also distinguish the notion of approximate universal consistency.
\begin{definition}[Approximate Universal Consistency]
    A sequence of nearest - neighbour regression function estimates $(\hat f_{n,m})$ is approximately universally consistent with respect to $\mcD$ if for all distributions $P_{\mcX,\mcY}\in \mcD$ we have 
    \[
    \inf_{m\in\NN}\lim_{n\to\infty}\mcR\left(\hat f_{n,m}\right)=0.
    \]
\end{definition}
\begin{example}
The $m$-NN estimator where $\hat f_{n,m}^{(NN)}(\bx_*)$ is the arithmetic mean of the responses of the $m$ nearest neigbours of a given test point $\bx_*$ is approximately universally consistent when $m$ is fixed. This is because for $m$-NN we have \citep{Gyorfi2002}
\[
\mcR\left(\hat f_{n,m}^{(NN)}\right)\xrightarrow{n\to\infty}\frac{1}{m}\int\sigma_\xi^2\left(\mcX\right)dP_\mcX(\bx),
\]
which can be made arbitrarily small by making $m$ large enough. When $m$ is increased with $n$ such that $m_n\xrightarrow{n\to\infty}\infty$ and $m_n/n\xrightarrow{n\to\infty}0$, the $m_n$-NN estimator is also known to be universally consistent \citep[Theorem~6.1]{Gyorfi2002}.
\end{example}
\cite{Stone_rates} found the best possible minimax rate at which the risk of a universally consistent estimator $\hat f_n$ can tend to zero with $n$. More precisely, denote $\mathcal{D}_q$ the class of distributions of $(\mcX,\mcY)$ where $\mcX$ is uniformly distributed on the unit hypercube $[0,1]^d$ and $\mcY = f(\mcX)+\Xi$ with some $q$-smooth function $f:\,\RR^d\to\RR$ and the noise variable $\Xi$ is drawn from the standard normal distribution independently of $\mcX$. Function $f$ is $q$-smooth if all its partial derivatives of the order $\lfloor q\rfloor$ exist and are $\beta$-H\"{o}lder continuous with $\beta = q - \lfloor q\rfloor$ with respect to the Euclidean metric on $\RR^d$. Stone showed that there exists a positive constant $\mathcal{C}>0$ such that
\[
\lim_{n\to\infty}\,\inf_{\hat f_n}\,\sup_{P\in \mathcal{D}_q} \mathcal{P}_P\left[\int \left(\hat f_n(\bx)-f(\bx)\right)^2dP_\mcX >\mathcal{C}\, n^{-\frac{2q}{2q+d}}\right]=1,
\]
where the outer probability is taken with respect to the training data samples coming from the product distribution $P^n$. This means that the best universally achievable risk cannot decay faster than $\mathcal{O}\left(n^{-\frac{2q}{2q+d}}\right)$.  In this work, we prove that $GPnn$ and $NNGP$ achieve the optimal convergence rates when $0<q\leq 1$ and provide experimental evidence that $GPnn$ and $NNGP$ can achieve these rates also when $q>1$.

\section{Consistency of $GPnn$ and $NNGP$}
\label{sec:consistency}
In this section we study the performance of $GPnn$ and $NNGP$ in terms of the following three metrics: the mean squared error ($MSE$, also called the $L_2$-error), the calibration coefficient ($CAL$) and the negative log-likelihood ($NLL$) when the size of the training set tends to infinity. The calibration coefficient is designed to provide a measure of how well-behaved the variance of the predictive distribution is, see \cite{GPnn23} and Sections \ref{sec:real_world} and \ref{sec:synthetic_nngp_rates}. Let us start by defining these metrics for a given data responses $\by_n$ and test response $y_*$ (and thus implicitly for a given $X_n,\, \bx_*$) as follows. Let $\hat f_n$ be equal to ${\tilde\mu}_{GPnn}$ or ${\tilde\mu}_{NNGP}$ defined in \eqref{eq:gpnn_unbiased_est} and \eqref{eq:nngp_unbiased_est}.
\begin{align}
\begin{split}
se(y_*,\by_n) & :=\left(y_* - \hat f_n(\bx_*)\right)^2, \quad cal(y_*,\by_n) := \frac{\left(y_* - \hat f_n(\bx_*)\right)^2}{ {\sigma_\mcN^*}^2},\\
nll(y_*,\by_n) & := \frac{1}{2}\left(\log\left( {\sigma_\mcN^*}^2\right)+\frac{\left(\by_* - \hat f_n(\bx_*)\right)^2}{ {\sigma_\mcN^*}^2}+\log2\pi \right).
\end{split}
\end{align}
We focus on the above performance metrics averaged over the noise component i.e. we treat the training set $\bX_n$ and the test point $\mcX_*$ as given and define the respective conditional expectations over the test response $\mcY_*$ and the training responses $\bY_n$ as follows.
\begin{align}
& MSE := \EE\left[ se(\mcY_*,\bY_n)\mid \mcX_*,\bX_n \right], \label{def:mse}
\\ 
& CAL :=  \EE\left[ cal(\mcY_*,\bY_n)\mid \mcX_*,\bX_n \right] = \frac{MSE}{ {\sigma_\mcN^*}^2}, \label{def:cal}
\\
& NLL :=  \EE\left[ nll(\mcY_*,\bY_n)\mid \mcX_*,\bX_n  \right] = \frac{1}{2}\left(\log\left( {\sigma_\mcN^*}^2\right)+CAL+\log2\pi \right). \label{def:nll}
\end{align}
Note that for $GPnn$ the conditional expectation $\EE\left[*\mid \mcX_*,\bX_n \right]$ means taking the expectation with respect to the noise $\Xi$ at the given training- and test-points, while for $NNGP$ one also needs to take the expectation over the random field $w(\cdot)\sim GP(0,\tilde k)$. We subsequently study the expectations of the above performance metrics with respect to $\bX_n$ and $\mcX_*$. 

We study the $n\to\infty$ limits in a broad setting which we specify in the assumptions (AC.\ref{a_X}-\ref{a_metr}) below. We first define some preliminary notions necessary for stating the assumptions.
\begin{definition}[Kernel-Induced (pseudo)Metric \citep{kernel_metric}]\label{def:kernel_metric}
Let $c:\, \RR^d\times \RR^d\to \RR$ be a positive definite symmetric kernel function such that $c(\bx,\bx)=1$ for all $\bx\in \RR^d$. The following function defines the (pseudo)metric $\rho_c:\, \RR^d\times \RR^d\to \RR_{\geq 0}$ associated with the kernel $c$ and the lengthscale parameter $\hat\ell>0$
\begin{equation}\label{eq:def_kernel_metric}
\rho_c(\bx,\bx'):=\sqrt{1-c(\bx/\hat\ell,\bx'/\hat\ell)},\quad \bx,\,\bx'\in \RR^d.
\end{equation}
\end{definition}
Note that for the kernel-induced metric defined above the maximum distance between two points is at most $1$. We will use this fact throughout the paper. If the kernel function satisfies $c(\bx,\bx')<1$ whenever $\bx\neq\bx'$, then \eqref{eq:def_kernel_metric} defines a metric. However, if $c(\bx,\bx')=1$ for some $\bx\neq\bx'$ then \eqref{eq:def_kernel_metric} defines a pseudometric. This is particularly relevant when modelling periodic functions using periodic kernel functions.
\begin{definition}[Equivalent (pseudo)metrics]
Let $\Omega$ be a set and let $\rho_1,\rho_2$ be (pseudo)metrics on $\Omega$. We say that $\rho_1$ and $\rho_2$
are equivalent if there exist constants $0<c\le C<\infty$
such that for all $\bx,\bx'\in \Omega$,
\[
c\,\rho_1\left(\bx,\bx'\right)\;\le\; \rho_2\left(\bx,\bx'\right)\;\le\; C\,\rho_1\left(\bx,\bx'\right).
\]
Consequently, convergent sequences are the same for $\rho_1$ and $\rho_2$.
\end{definition}
\begin{definition}[Function Continuous Almost Everywhere]
Let $P$ be a probability measure on $\Omega$ being a (pseudo)metric space. A function $f:\,\Omega\to \RR$ is continuous almost everywhere if there exists a set $E\subset \Omega$ such that $P(E)=0$ and the restriction $f|_{\Omega-E}$ is continuous with respect to the (pseudo)metric on $\Omega$.
\end{definition}
We are now ready to state the assumptions.

\begin{enumerate}[({AC.}1)]
\item The training covariates $\{\mcX_i\}_{i=1}^n$ and the test covariate $\mcX_*$ are i.i.d. distributed according to the probability measure $P_\mcX$ on $\RR^{d_\mcX}$. \label{a_X}
\item The nearest neighbours are chosen according to the kernel-induced metric $\rho_c$. \label{a_nn}
\item The function $f$ in the $GPnn$ response model \eqref{eq:gpnn_responses} and the functions $t_i$, $i=1,\dots,d_T$ in the $NNGP$ response model \eqref{eq:nngp_responses} are continuous almost everywhere with respect to the kernel-induced (pseudo)metric $\rho_c$ and are integrable i.e., they are measurable and satisfy
\[\int f(\bx) dP_\mcX(\bx)<\infty,\quad \int t_i(\bx) dP_\mcX(\bx)<\infty.\]
\label{a_f}
\item The noise $\Xi$ is heteroscedastic with mean zero and, 
\[
  \EE[\Xi_i \mid \mcX_i] = 0, \qquad
  \EE[\Xi_i^2 \mid \mcX_i] = \sigma_\xi^2(\mcX_i),
\]
for some function $\sigma_\xi^2 : \Omega_\mcX \to \RR_{>0}$ and the noise random variables are uncorrelated given the covariates, i.e.,
\[
\cov\left[\Xi_i,\Xi_j\mid\mcX_i,\mcX_j\right]=0\quad \mathrm{for}\quad i\neq j,\quad \cov\left[\Xi_*,\Xi_i\mid\mcX_*,\mcX_i\right]=0.
\]
In the $NNGP$ response model \eqref{eq:nngp_responses}  we also assume that each of the random variables $\Xi_1,\dots,\Xi_n,\Xi_*$ are independent of the sample path $w(\cdot)$. We further assume that the variance function $\sigma_\xi^2(\cdot)$ is almost continuous with respect to the kernel metric $\rho_c$ and is an integrable function of $\bx$ i.e.,
\[\int \sigma_\xi^2(\bx)dP_\mcX(\bx) < \infty.\]
\label{a_xi}
\item The covariance function of the $GP$ sample paths generating the $NNGP$ responses \eqref{eq:nngp_responses} satisfies $\tilde k(\bx,\bx)=\sigma_w^2$ for all $\bx\in\RR^{d_\mcX}$. Define $\tilde c(\cdot,\cdot) := \tilde k(\cdot,\cdot)/\sigma_w^2$. The (pseudo)metrics $\rho_{c}$ (associated with the $GP$-regression kernel as in Equation \ref{eq:kernel_def}) and $\rho_{\tilde c}$ are equivalent.
\label{a_metr}
\end{enumerate} 

\begin{definition}[Support of a Probability Measure]\label{def:supp}
Let $P_\mcX$ be the probability measure of $\bx\in\RR^d$ and $B_\rho(\bx, \epsilon)$ be the closed ball of radius $\epsilon$ under the (pseudo)metric $\rho$ centred at $\bx$. Then we define
\[
\supp_\rho(P_\mcX):=\left\{\bx\in \RR^d\mid P_\mcX\left(B_\rho(\bx, \epsilon)\right)>0\quad\mathrm{for\,\,all}\quad \epsilon>0\right\}.
\]
\end{definition}
When the metric is not explicitly mentioned, $\supp(P_\mcX)$ will denote the probability measure support under the Euclidean metric. 
\begin{theorem}[Universal Point-Wise Consistency]
\label{thm:ptwise_consistency}
Assume (AC.\ref{a_X}-\ref{a_metr}). If the number of nearest neighbours $m$ is fixed, the following limits hold for $GPnn$ and $NNGP$ with probability one (with respect to $\bX_n\sim P_{\mcX}^n$) and for any test point $\bx_*\in \supp_{\rho_c}(P_\mcX)$ (see Definition \ref{def:supp}).
\begin{align}
& MSE(\bx_*,\bX_n)\xrightarrow{n\to\infty} \sigma_\xi^2(\bx_*)\left(1+\frac{1}{m}\right)
\\
& CAL(\bx_*,\bX_n)\xrightarrow{n\to\infty}\frac{\sigma_\xi^2(\bx_*)}{\hat\sigma_\xi^2}\, \left(1+\mcO\left(m^{-2}\right)\right), 
\\
& NLL(\bx_*,\bX_n)\xrightarrow{n\to\infty} \frac{1}{2}\left(\log\left(2\pi\, \hat\sigma_\xi^2\right)+\frac{\sigma_\xi^2(\bx_*)}{\hat\sigma_\xi^2}+\frac{1}{m}\right)+\mcO\left(m^{-2}\right). 
\end{align}
What is more, if $m$ grows with $n$ so that $\lim_{n\to\infty} m_n/n=0$, the following limits hold with probability one and for any text point $\bx_*\in \supp_{\rho_c}(P_\mcX)$.
\begin{align}
& MSE(\bx_*,\bX_n)\xrightarrow{n\to\infty} \sigma_\xi^2(\bx_*),\quad CAL(\bx_*,\bX_n)\xrightarrow{n\to\infty}\frac{\sigma_\xi^2(\bx_*)}{\hat\sigma_\xi^2}  
\\
& NLL(\bx_*,\bX_n)\xrightarrow{n\to\infty} \frac{1}{2}\left(\log\left(2\pi\, \hat\sigma_\xi^2\right)+\frac{\sigma_\xi^2(\bx_*)}{\hat\sigma_\xi^2}\right). 
\end{align}
\end{theorem}
\begin{roadmap}
\emph{Full proof:} \href{https://github.com/tmaciazek/gpnn_synthetic/blob/main/Online_Appendix_1.pdf}{Online Appendix~1}, Section \ref{sec:consistency_app} (Consistency).

\noindent\emph{Strategy.} Use the fact that distance to $m$-th nearest-neighbour, $\epsilon_m$, tends to zero a.s. as $n\to\infty$ to get kernel/Gram matrix limits. Deduce limits of the posterior mean/variance, then plug into the bias-variance decomposition of $MSE$ (the case when $m$ grows with $n$ relies mostly on the key inequalities derived in Lemma~\ref{lemma:K_ineqs} that are based on matrix perturbation theory from \cite{GolubBook}). $CAL$/$NLL$ limits follow by substitution and continuity.
\begin{itemize}
\item Lemma~\ref{lemma:K_ineqs}: controls the convergence of the $K^{-1}\bk^*$ terms to their limits in terms of the (kernel-induced) distance to $m$-th nearest-neighbour, $\epsilon_m$.
  \item Lemma~\ref{lemma:dmax_limit}: using results from $m$-NN theory \citep{Gyorfi2002} we get that $\epsilon_m$, tends to zero a.s. when $n\to \infty$.
  \item Lemma~\ref{lemma:kernel_limits}: limits of Gram matrix and it's inverse as nearest-neighbours converge to the test point.
   \item Lemma~\ref{lemma:bias_variance}: decomposition of $MSE$ into irreducible noise, squared bias and estimator variance.
  \item Lemma~\ref{lemma:unbiased_estimator}: bias and variance limits for $GPnn$/$NNGP$.
\end{itemize}
\end{roadmap}

\begin{corollary}[Universal Pointwise Consistency in Probability]
\label{cor:ptwise_vs_supc}
Assume the conditions of Theorem~\ref{thm:ptwise_consistency}, and let $\hat f_n(\bx_*)$ denote the $GPnn$ or $NNGP$ predictor of the latent regression function $f(\bx_*)$ at test point $\bx_*\in \supp_{\rho_c}(P_\mcX)$ (in the $NNGP$ response model \eqref{eq:nngp_responses} we take $f(\bx_*)=\bt_*^T.\bb+w(\bx_*)$). Consider the squared estimation error $SE$ and the associated (shifted) $\widetilde{MSE}$ defined as 
\begin{equation*}
SE\left(\hat f_n\right):=\left(f(\mcX_*)-\hat f_n(\mcX_*)\right)^2,\quad \widetilde{MSE}\left(\hat f_n\right):=\EE\left[ SE\left(\hat f_n\right)\mid \mcX_*,\bX_n \right].
\end{equation*}
Then Theorem \ref{thm:ptwise_consistency} states that for every $\bx_*\in \supp_{\rho_c}(P_\mcX)$
\[
\widetilde{MSE}\left(\hat f_n\right)\xrightarrow{n\to\infty}0
\qquad\textrm{a.s. with respect to }\bX_n\sim P_\mcX^n,
\]
and hence by Markov's inequality
\[
P\left[SE\left(\hat f_n\right)\geq \epsilon\mid \mcX_*,\bX_n\right]\leq \frac{\EE\left[SE\left(\hat f_n\right)\mid \mcX_*,\bX_n\right]}{\epsilon} = \frac{\widetilde{MSE}\left(\hat f_n\right)}{\epsilon}\xrightarrow{n\to\infty}0 \quad a.s.
\]
for any $\epsilon>0$. Consequently,
\[
SE\left(\hat f_n\right)\xrightarrow{n\to\infty}0
\qquad\textrm{in probability with respect to }(\bX_n,\bY_n)\sim P_{\mcX,\mcY}^n .
\]

Thus, the universal pointwise consistency established in Theorem~\ref{thm:ptwise_consistency} implies pointwise convergence in probability of the $SE$ under $P_{\mcX,\mcY}^n$. This is, however, weaker than the strongly universal pointwise consistency of \citet[Definition~25.1]{Gyorfi2002}, which requires $P\left[SE\left(\hat f_n\right)\xrightarrow{n\to\infty}0\right]=1$.
\end{corollary}

\begin{theorem}[Approximate Universal Consistency]
\label{thm:approx_universal_consistency}
Let $\bX_n$ be a sampling sequence of i.i.d. points from the distribution $P_\mcX$ and $m$ be a fixed number of nearest-neughbours.  Let $\mcX_*\sim P_\mcX$ be a test point. 
Apply the following assumptions:
\begin{itemize}
\item (AC.\ref{a_X}-\ref{a_metr}),
\item  function $f$ in the $GPnn$ response model \eqref{eq:gpnn_responses} satisfies $\|f(\cdot)\|_\infty<B_f<\infty$,
\item functions $t_i$, $i=1,\dots, d_T$ in the $NNGP$ response model \eqref{eq:nngp_responses} satisfy $\|t_i(\cdot)\|_\infty<B_T<\infty$,
\item $\|\sigma_\xi^2(\cdot)\|_\infty<\infty$, where $\sigma_\xi^2(\bx):=\EE\left[\Xi^{2}\mid \mcX=\bx\right]$. 
\end{itemize}
Then we have the following limit for the risk for both $GPnn$ and $NNGP$.
\begin{align}
& \EE_{\mcX_*,\bX_n}\left[MSE(\mcX_*,\bX_n)\right]=\mcR_n+\EE_{\mcX_*}\left[\sigma_\xi^2(\mcX_*)\right]\xrightarrow{n\to\infty} \EE_{\mcX_*}\left[\sigma_\xi^2(\mcX_*)\right] \left(1+\frac{1}{m}\right), 
\end{align}
where $\mcR_n$ is the risk defined in \eqref{def:risk}.
Analogous limits hold for $CAL$ and $NLL$, i.e.,
\begin{align}
\begin{split}
& \EE_{\mcX_*,\bX_n}\left[CAL(\mcX_*,\bX_n)\right]\xrightarrow{n\to\infty}\frac{\EE_{\mcX_*}\left[\sigma_\xi^2(\mcX_*)\right]}{\hat\sigma_\xi^2}\, \left(1+\mcO\left(m^{-2}\right)\right),
\\
& \EE_{\mcX_*,\bX_n}\left[NLL(\mcX_*,\bX_n)\right]\xrightarrow{n\to\infty} \frac{1}{2}\left(\log\left(2\pi\, \hat\sigma_\xi^2\right)+\frac{\EE_{\mcX_*}\left[\sigma_\xi^2(\mcX_*)\right]}{\hat\sigma_\xi^2}+\frac{1}{m}\right)+\mcO\left(m^{-2}\right).
\end{split}
\end{align}
\end{theorem}
\begin{roadmap}
\emph{Full proof:} \href{https://github.com/tmaciazek/gpnn_synthetic/blob/main/Online_Appendix_1.pdf}{Online Appendix~1}, Section \ref{sec:consistency_app} (Consistency).

\noindent\emph{Strategy.} Use Dominated Convergence Theorem (DCT): Theorem~\ref{thm:ptwise_consistency} gives a.s.\ pointwise convergence of $f_n\to 0$ with $f_n:=\left|MSE(\bx_*,\bX_n)-\frac{\sigma_\xi^2(\bx_*)}{m}\right|$. In the proof in Section \ref{sec:consistency_app} we derive uniform bounds on $f_n$ that yield an integrable $g$ (uniformly) dominating $f_n$, implying convergence of $MSE/CAL/NLL$ in expectation when $m$ is fixed.
\end{roadmap}
\begin{remark}[Practical aspects of the fixed-$m$ regime]
In applications one is often given a dataset of fixed size $n$ and chooses $m$ by validation, cross-validation, or computational considerations. Thus the question of whether $m$ should grow with $n$ is not an operational one for a single dataset, but an asymptotic one concerning how the chosen neighbourhood size behaves along a sequence of problems with increasing sample size. From this perspective, the fixed-$m$ and growing-$m$ regimes should be viewed as two asymptotic descriptions of practical tuning behaviour. In particular, if the selected $m$ remains moderate even as $n$ becomes very large, then the fixed-$m$ theory is the more relevant description, and the resulting $1/m$ correction is often negligible for practically meaningful choices such as $m=100$ or larger. If instead the selected $m$ increases with $n$, then the growing-$m$ theory becomes the appropriate asymptotic benchmark. The effect of fixed $m$ on convergence \emph{rates} is discussed separately in Section~\ref{sec:rates}.
\end{remark}
Under the additional assumptions (AR.\ref{aR_iso}) and (AR.\ref{aR_c}) stating that the $GP$ kernel is isotropic and satisfies a H\"{o}lder-like inequality we have exact universal consistency.
\begin{enumerate}[({AR.}1)]
\item The (normalised) GP kernel function is an isotropic and a strictly decreasing function of the Euclidean distance, i.e., 
\[
c(\bx,\bx')\equiv c\left(r\right),\quad r=\left\|\bx-\bx'\right\|_2,\quad c(r_1)< c(r_2)\quad\mathrm{if}\quad r_1> r_2.
\] \label{aR_iso}
\item There exist constants $L_c> 0$ and $0<p\leq 1$ such that the (isotropic and normalised) $GP$ kernel function $c:\, \RR^{d_\mcX}\times \RR^{d_\mcX}\to\RR_{\geq0}$ used in the $GPnn/NNGP$ estimators \eqref{eq:gpnn_unbiased_est} and \eqref{eq:nngp_unbiased_est} is lower bounded as
\[
c(r)\geq 1-L_c\,r^{2p}.
\]
\label{aR_c}
\end{enumerate}
The assumption (AR.\ref{aR_iso}) implies that the kernel-induced metric $\rho_c(\cdotp)$ is equivalent to the the Euclidean metric $\|\cdotp\|_2$. With this assumption in place all of our results will hold also when the nearest neighbours are chosen according to the Euclidean metric instead of the kernel-induced metric. Assumption (AR.\ref{aR_c}) is satisfied by the commonly used kernels from the M\'{a}tern family and by the RBF kernel, see Appendix \ref{appendix:kernel_bounds}. 

\begin{theorem}[Universal Consistency]
\label{thm:universal_consistency}
Let $\bX_n$ be a random sampling sequence of i.i.d. points from the distribution $P_\mcX$ and let $\mcX_*\sim P_\mcX$ be a test point. Let the number of nearest - neighbours $m_n$ grow as $m_n\propto n^\gamma$ with $0<\gamma<1/3$. Apply the following assumptions:
\begin{itemize}
\item there exists $\beta>\frac{2\gamma d_\mcX}{1-3\gamma}$ for which $\EE\left[\|\mcX\|_2^\beta\right]<\infty$ under the probability distribution $P_\mcX$ on $\RR^{d_\mcX}$.
\item (AC.\ref{a_X}-\ref{a_metr}) and (AR.\ref{aR_iso}-\ref{aR_c}),
\item  function $f$ in the $GPnn$ response model \eqref{eq:gpnn_responses} satisfies $\|f(\cdot)\|_\infty\leq B_f<\infty$ for some $B_f>0$,
\item functions $t_i$, $i=1,\dots, d_T$ in the $NNGP$ response model \eqref{eq:nngp_responses} satisfy $\|t_i(\cdot)\|_\infty<B_T<\infty$ for some $B_T>0$,
\item $\|\sigma_\xi^2(\cdot)\|_\infty<\infty$, where $\sigma_\xi^2(\bx):=\EE\left[\Xi^{2}\mid \mcX=\bx\right]$. 
\end{itemize}
Then we have the following limit for the risk of $GPnn$ and $NNGP$.
\begin{align}
& \EE_{\mcX_*,\bX_n}\left[MSE(\mcX_*,\bX_n)\right]\xrightarrow{n\to\infty} \EE_{\mcX_*}\left[\sigma_\xi^2(\mcX_*)\right].
\end{align}
\end{theorem}
\begin{roadmap}
\emph{Full proof:} \href{https://github.com/tmaciazek/gpnn_synthetic/blob/main/Online_Appendix_1.pdf}{Online Appendix~1}, Section \ref{app:universal_consistency}.

\noindent\emph{Strategy.} Decompose $GPnn$ error ($MSE$) into (i) $m$-NN error and (ii) a weight-mismatch term. The $m$-NN regression error tends to zero by $m$-NN universal consistency \citep{Gyorfi2002}, while the mismatch term is handled by splitting the expectation into good/bad events (good: epsilons shrink, bad: probability decays).
\begin{itemize}
  \item Theorem~\ref{thm:nn_consistency} \citep{Gyorfi2002}: $m$-NN regression is universally consistent.
  \item Lemma~\ref{lemma:K_ineqs}: bounds deviation of $GPnn$ weights from uniform in terms of kernel-induced NN distances.
  \item Lemma~\ref{lemma:epsilons_rels}: links different types of kernel-induced distance functions $\epsilon_E,\epsilon_{E,2}$ to single $\epsilon_m$ while AR.\ref{aR_c} links $\epsilon_m$ with the Euclidean distance to the $m$th NN, $d_m$.
  \item Lemma~\ref{lemma:bad_region_bound}: controls the probability of the bad event, $P[d_m\ge R]$, via the known convergence rate of the moments of $d_m$ established in \cite{Gyorfi2002}.
\end{itemize}
\end{roadmap}
\begin{remark}
The universal consistency (Theorem~\ref{thm:universal_consistency}) holds for a much wider class of data distributions than the ones considered in the stronger Theorem~\ref{thm:mse_convergence_rate} (Section \ref{sec:rates}) establishing risk convergence rates. Namely, we have proved universal consistency for any data distribution $P_{\mcX,\mcY}$ which satisfies the moment condition $\EE\left[\|\mcX\|_2^\beta\right]<\infty$ with $\beta>\frac{2\gamma d_\mcX}{1-3\gamma}$ for some $\gamma\in]0,1/3[$, and where responses are generated via bounded regression function(s) and heteroscedastic noise with bounded variance. Theorem ~\ref{thm:mse_convergence_rate}, on the other hand, requires the moment condition from (AR.\ref{aR_X}), H\"{o}lder-continuous and bounded regression function(s), homoscedastic noise and uses $\gamma=\frac{2p}{2p+d_\mcX}$ with $d_{\mcX}>4(\alpha+p)$ which automatically satisfies $\gamma<1/3$. For this choice of $\gamma$ we also have
\[
\frac{2\gamma d_\mcX}{1-3\gamma} = \frac{4d_\mcX\, p}{d_{\mcX}-4p} \leq \frac{4d_\mcX\, (p+\alpha)}{d_{\mcX}-4(p+\alpha)},
\]
thus any $\beta$ satisfying the moment condition (AR.\ref{aR_X}) also satisfies the moment condition of Theorem~\ref{thm:universal_consistency} with $\gamma=\frac{2p}{2p+d_\mcX}$.
\end{remark}
In the experiments with real life datasets we only have access to a fixed training sample $(X_n,\by_n)$ and a set of test data $(X_{\text{test}}, \by_{\text{test}})$ of the size $n_{\text{test}}$. There, we measure the performance of different regression methods using the empirical averages of the above performance metrics over the test data i.e.,
\begin{align}\label{measures_experimental}
\begin{split}
\widehat{MSE}(\by_n) & :=\frac{1}{n_{\text{test}}} \sum_{y_*\in \by_{\text{test}}}se(y_*,\by_n) ,\quad \widehat{CAL}(\by_n) :=\frac{1}{n_{\text{test}}} \sum_{y_*\in \by_{\text{test}}}cal(y_*,\by_n)
\\
\widehat{NLL}(\by_n) & :=\frac{1}{n_{\text{test}}} \sum_{y_*\in \by_{\text{test}}}nll(y_*,\by_n).
\end{split}
\end{align}
The goal is to minimise $\widehat{MSE}$ and $\widehat{NLL}$ and $\left|\widehat{CAL}-1\right|$.

The key feature of the limits from Theorem \ref{thm:ptwise_consistency} and Theorem \ref{thm:universal_consistency} is that (in the leading order in $1/m$) the large-$n$ limits only depend on the estimated noise variance $\hat\sigma_\xi^2$. In fact, the limiting value of $MSE$ does not depend on any of the $GPnn$ hyper-parameters at all. This leads to the following two observations which we subsequently back up with further theoretical and experimental evidence.
\begin{enumerate}[i)]
\item To obtain high quality predictions in the large-$n$ regime it is sufficient to only cheaply estimate the hyperparametres $\hat\sigma_\xi^2$ (the noise variance), $\hat\sigma_f^2$ (the kernelscale) and $\hat\ell$ (the lengthscale). This is because the $MSE$-landscape becomes flat, i.e., highly insensitive to the hyper-parameters. 
\item  In order to improve the $CAL$ and $NLL$ prediction metrics without changing the $MSE$, one needs an extra re-calibration step which adjusts the predictive variance. To this end, we propose a simple \emph{post-hoc} (re)calibration procedure explained below.
\end{enumerate}
\paragraph{A Cheap Hyper-Parameter Estimation Method}

To avoid the high cost of exact GP hyperparameter learning on large training sets, we estimate kernel and noise parameters using a block-diagonal approximation to the full covariance matrix. Concretely, we set aside a small training subset and partition it into $B$ disjoint batches (subsets) of size $n_B$, and assume independence across batches, i.e. we approximate the full $GP$ covariance by a block-diagonal matrix with $B$ dense blocks. We note that this specific choice of hyper-parameter estimation method is not critical -- due to the insensitivity of $GPnn$/$NNGP$ predictive performance to hyper-parameter choice (in massive datasets) other cheap methods could be used instead.

We fit a zero-mean exact $GP$ with the kernel $k$ and a Gaussian likelihood, sharing a single set of hyper-parameters across all blocks: lengthscale $\hat\ell$, kernelscale $\hat\sigma_f^2$, and noise-variance $\hat\sigma_\xi^2$. The approximate log marginal likelihood is then
\begin{equation}\label{eq:likelihood_approx}
\log p(\by\mid \theta)\approx\sum_{b=1}^{B}\log\mathcal{N}\left(\by^{(b)}; 0, K_\theta\left(X^{(b)},X^{(b)}\right)+\hat\sigma_\xi^2 \mathbb{I}\right),
\end{equation}
where $\mathcal{N}$ denotes the density of the normal distribution, $X^{(b)}$ and $\by^{(b)}$ are the batch covariates and responses. This corresponds to replacing the full covariance by
\[
\widetilde{K}_\theta = \mathrm{blockdiag}\left(K_\theta\left(X^{(1)},X^{(1)}\right),\ldots,K_\theta\left(X^{(B)},X^{(B)}\right)\right).
\]

In practice, we optimize this objective by gradient-based maximization of the (summed) exact marginal log-likelihood \eqref{eq:likelihood_approx} computed independently per batch. Within each Adam-optimizer \citep{adam} step, we evaluate the exact $GP$ marginal likelihood on each block and accumulate the loss as the sum of per-block marginal log likelihoods, then backpropagate once and update the shared parameters.  We use Adam for a fixed number of iterations.  After optimization, we read off the learned hyperparameters $\hat\ell$, $\hat\sigma_f^2$, $\hat\sigma_\xi^2$. 

This procedure is ``cheap" because it replaces a single $\mathcal{O}((Bn_B)^3)$ matrix inversion by $B$ independent $\mathcal{O}(n_B^3)$ inversions (and can be evaluated in parallel), while still letting all blocks jointly inform a single global set of kernel parameters. 
\paragraph{The Calibration Procedure} 
The calibration procedure is motivated by the fact that one can simultaneously rescale $\hat\sigma_f^2\to\alpha\hat\sigma_f^2$ and $\hat\sigma_\xi^2\to\alpha\hat\sigma_\xi^2$, with $\alpha>0$, without changing the $GPnn$ or $NNGP$ predictive mean \eqref{eq:gpnn_unbiased_est}, \eqref{eq:nngp_unbiased_est}. Hence such a transformation leaves the $MSE$ unchanged on any test set. To calibrate the predictive distribution, one uses a held-out calibration set $X_\mathcal{C}$ of pairs $(\bx_{*,i},y_{*,i})$, computes the corresponding predictive means and variances $\tilde\mu_i^*$ and ${\sigma_i^*}^2$, and then sets
\begin{equation}\label{eq:calibration}
\hat\sigma_f^2\to\alpha\hat\sigma_f^2, \qquad \hat\sigma_\xi^2\to\alpha\hat\sigma_\xi^2,\qquad \alpha=\frac{1}{|X_\mathcal{C}|}\sum_{i=1}^{|X_\mathcal{C}|}\frac{(y_{*,i}-\tilde\mu_i^*)^2}{{\sigma_i^*}^2}.
\end{equation}
The calibrated hyper-parameters $\alpha\hat\sigma_\xi^2$, $\alpha\hat\sigma_f^2$ achieve perfect calibration on the held-out set $X_\mathcal{C}$ and also minimise the $NLL$ (see \citet{GPnn23} for the proof). Crucially, this calibration also significantly improves the predictive variance of $GPnn$ when deployed on an independent test set -- see Section \ref{sec:real_world}. This argument applies verbatim to $NNGP$, since its predictive variance has exactly the same form as in $GPnn$. Consequently, the same rescaling yields perfect calibration $\widehat{CAL}=1$ and the optimal $\widehat{NLL}$ over all choices of $\alpha$ on the set $X_\mathcal{C}$ \citep{GPnn23}.

In Figure \ref{fig:gpnn_flowchart} we summarise the $GPnn$ workflow, including the above-described cheap hyper-parameter estimation and the calibration step of the predictive variance. Note that a similar idea of decoupling the cheap hyper-parameter estimation from predictions could be applied to $NNGP$ as well. Work by \cite{nngp2} notes that the quality of predictions in $NNGP$ is typically not sensitive to the choice of $\hat\sigma_f^2$ and $\hat\sigma_\xi^2$ and thus proposes to choose those hyper-parameters cheaply via the $K$-fold cross-validation on a grid \citep[see][Algorithm 5]{nngp2}. Our work shows that in the massive-data regime one can go a few steps further and apply cheap estimation to all the hyper-parameters, including the lengthscale $\hat\ell$ and the parameter $\hat\bb$ in $NNGP$. However, the cheap hyper-parameter estimation may not be suitable in situations when one's goal is to not only achieve high quality predictions, but also to estimate the hyper-paramaters accurately from the training data (for instance, due to their physical meaning informing some physical properties of the problem at hand). 
 \begin{figure}[h]
  \centering
 \includegraphics[width=\textwidth]{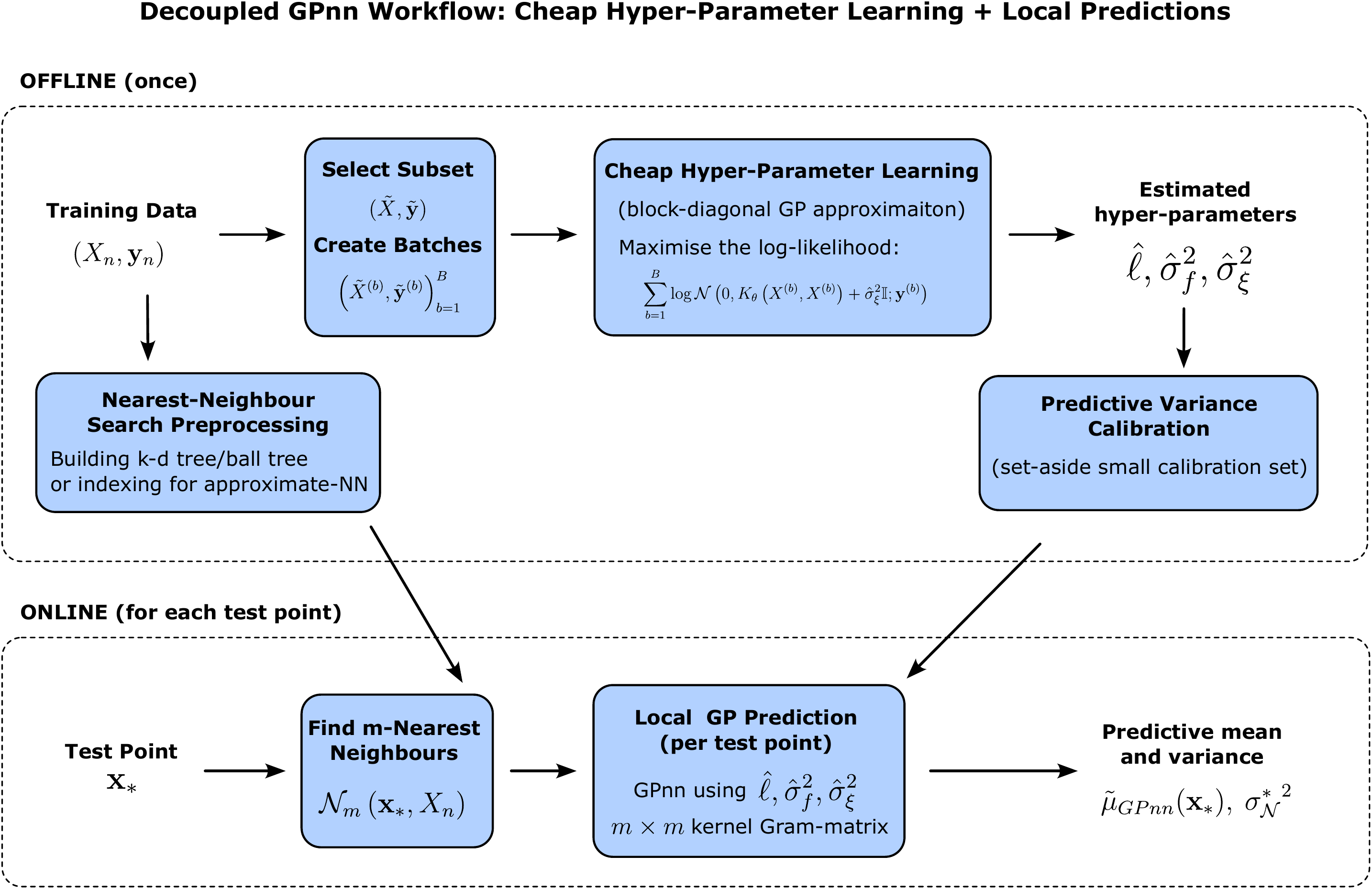}
  \caption{The $GPnn$ workflow, including the above-described cheap hyper-parameter estimation and the calibration step of the predictive variance (see the main text for explanation).}
  \label{fig:gpnn_flowchart}
\end{figure}

\section{Convergence Rates}
\label{sec:rates}
In Theorem \ref{thm:approx_universal_consistency} and Theorem \ref{thm:universal_consistency} we have determined the asymptotic value of the risk when $n\to\infty$. In Theorem \ref{thm:mse_convergence_rate} of this Section, we present the corresponding rate of convergence, and by allowing $m$ to suitably grow with $n$ we show that $GPnn$ and $NNGP$ achieve Stone's optimal convergence rates. This section's results apply to isotropic $GP$ kernels having the H\"{o}lder-like property (AR.\ref{aR_c}-\ref{aR_c2}), responses having the H\"{o}lder property (AR.\ref{aR_f}) and to data distributions/noise models having the properties (AR.\ref{aR_X}) and (AR.\ref{aR_xi}) specified below.
\begin{enumerate}[({AR.}1)]
\setcounter{enumi}{2}
\item The normalised covariance function of the $GP$ sample paths that generate the $NNGP$ responses \eqref{eq:nngp_responses} satisfies
\begin{equation*}
\tilde c\left(\bx,\bx'\right)\geq 1-L_{\tilde c}\, \left\|\bx-\bx'\right\|_2^{2 q_0}, \quad L_{\tilde c}>0.
\end{equation*}
\label{aR_c2}
\item The function $f$ in the $GPnn$ response model \eqref{eq:gpnn_responses} is bounded in absolute value by some constant $\infty>B_f\geq 1$ and is $q$-H\"{o}lder-continuous, i.e., there exist constants $1\leq L_f<\infty$ and $0<q\leq 1$ such that for every $\bx,\bx'$
\[|f(\bx)-f(\bx')|\leq L_f \left\|\bx-\bx'\right\|_2^q.\]
Each function $t_i$, $i=1,\dots,d_T$ in the $NNGP$ response model \eqref{eq:nngp_responses} is bounded and $q_i$-H\"{o}lder continuous, i.e.,
\[\left|t_i(\bx)\right|\leq B_T<\infty,\quad \left|t_i(\bx)-t_i(\bx')\right|\leq L_i\|\bx-\bx'\|_2^{q_i},\quad i\in\{1,\dots,d_T\}\]
with $0<q_i\leq 1$ and $1\leq L_i<\infty$.
 \label{aR_f}
\item There exists $\beta>\frac{4d\,(p+\alpha)}{d-4(p+\alpha)}$ for which $\EE\left[\|\mcX\|_2^\beta\right]<\infty$ under the probability distribution $P_\mcX$ on $\RR^{d_\mcX}$ with $d>4(p+\alpha)$ where $\alpha=\min\{q,p\}$ for $GPnn$ and $\alpha=\min\{q_0,q_1,\dots,q_{d_T},p\}$ for $NNGP$ with $p$ defined in (AR.\ref{aR_c}). \label{aR_X}
\item The noise is homoscedastic, i.e., the noise $\Xi_i$ in $GPnn$ responses \eqref{eq:gpnn_responses} and $NNGP$ responses \eqref{eq:nngp_responses} is i.i.d. from the probability distribution $P_\xi$ with mean zero and fixed variance $\sigma_\xi^2<\infty$.
\label{aR_xi}
\end{enumerate}

Note that the exponent $p$ always refers to the $GP$ regression kernel (which is known in practice) and exponents $q,\, \{q_i\}_{i=0}^{d_T}$ always refer to the generative functions/kernels from the $GPnn/NNGP$ response models (which are often unknown in practice). Assumption (AR.\ref{aR_f}) strengthens the assumption (AC.\ref{a_f}) from Section \ref{sec:consistency}. Assumption (AR.\ref{aR_X}) is standard when deriving analogous convergence rates for the $k$-NN theory \citep{Gyorfi2002, Kohler2006}. This work draws on some results from the $k$-NN theory, so it inherits some of the assumptions. Assumption (AR.\ref{aR_xi}) strengthens the assumption (AC.\ref{a_xi}) from Section \ref{sec:consistency}.

\begin{theorem}[Convergence Rates]
\label{thm:mse_convergence_rate}
Let $n$ be the size of the $GPnn/NNGP$ training set which is i.i.d. sampled from the distribution $P_{\mcX}$ and let the test point be also sampled from $P_{\mcX}$. Let $m$ be the (fixed) number of nearest-neighbours used in $GPnn/NNGP$. Assume (AC.\ref{a_metr}) and (AR.\ref{aR_iso}-\ref{aR_xi}). Define $\alpha := \min\{p,q\}$ for $GPnn$ and $\alpha:=\min\{p,q_0,\allowbreak q_1,\dots,q_{d_T}\}$ for $NNGP$. Then, if $d_\mcX > 4 (\alpha+p)$, we have
\begin{align}
\begin{split}
\mcR_n\leq \frac{\sigma_\xi^2}{m}+A_1\,\left(\frac{m}{n}\right)^{2\alpha/d_\mcX}+A_2\,m\,\left(\frac{m}{n}\right)^{2(\alpha+p)/d_\mcX},
\end{split}
\end{align}
where $\mcR_n$ is the $GPnn/NNGP$ risk defined in \eqref{def:risk} and $A_1,A_2>0$ depend on $p$, $q$, $d_\mcX$, $B_f$, $B_T$, $L_f$, $L_c$, $\sigma_\xi$ and the $GPnn/NNGP$ hyper-paramaters. Taking $m_n = n^{\frac{2p}{2p+d_\mcX}}$ we obtain the following optimal minimax convergence rate.
\begin{equation}
\mcR_n\leq\left(\sigma_\xi^2+A_1+A_2\right)\, n^{-\frac{2\alpha}{2p+d_\mcX}}.
\end{equation}
\end{theorem}
\begin{roadmap}
    \emph{Full proof:} \href{https://github.com/tmaciazek/gpnn_synthetic/blob/main/Online_Appendix_1.pdf}{Online Appendix~1}, Section \ref{app:convergence_rates}.
    
    \noindent\emph{Strategy.} Start from the risk representation $\mcR_n=\EE[MSE]-\sigma_\xi^2$. Apply Theorem~\ref{thm:mse_key_bound} that controls $MSE$ in terms of kernel-induced NN distances and then take expectations, splitting into a bounded good-event region and a tail region via Lemma~\ref{lemma:subset_expectation_bound} and Lemma~\ref{lemma:bad_region_bound}. Control the good-region terms using Lemma~\ref{lemma:dmax_rate} ($d_m$-moment rates) and Lemma~\ref{lemma:epsilons_rates} (with Lemma~\ref{lemma:kohler} and Lemma~\ref{cond_expectation_eps_max_bound} as supporting steps). Finally choose $m_n$ to balance the two leading terms and obtain the stated rate. $NNGP$ is handled analogously to $GPnn$.
\begin{itemize}
  \item Theorem~\ref{thm:mse_key_bound}: deterministic bound on $|MSE-MSE_\infty|$ via kernel-induced NN distances.
  \item Lemma~\ref{lemma:subset_expectation_bound}: split $\EE\left[|MSE-MSE_\infty|\right]$ into good/bad events.
  \item Lemma~\ref{lemma:bad_region_bound}: control bad-event probability bound via NN-distance moments.
  \item Lemma~\ref{lemma:dmax_rate}: rates for the moments of $m$-NN distance $d_m$, building on \cite[Lemma 1]{Kohler2006}.
\end{itemize}
\end{roadmap}
In (geo)spatial modeling applications of $NNGP$ one typically takes $d_{\mcX}\in\{2,3\}$, since $\mcX$ describes spatial coordinates on Earth. \cite{stein1999interpolation} recommends the Mat\'{e}rn class of kernels as a default family for spatial interpolation because its smoothness parameter $\nu$ allows the local differentiability of the random field to be tuned to obtain the best fit for the data at hand. In particular, prior works have used Mat\'{e}rn kernel with $1/2\leq \nu<1$ for modeling forest canopy and biomass \citep{nngp1,nngp2}, modeling land surface temperature from satellite imagery \citep{Heaton2019} and for modeling traffic from spatial measurements \citep{Wu2024-we}. Note that Theorem \ref{thm:mse_convergence_rate} works under the assumption $d>4(\alpha+p)$ ($p=\min\{\nu,1\}$ for Mat\'{e}rn-$\nu$ kernels -- see \href{https://github.com/tmaciazek/gpnn_synthetic/blob/main/Online_Appendix_1.pdf}{Online Appendix~1}, Section \ref{appendix:kernel_bounds}), thus it does not cover these practically important applications of $NNGP$. Indeed, if $1/2\leq \nu,\alpha<1$, then $4(\alpha+p)\geq 4$, thus $d$ must be at least $5$ for Theorem \ref{thm:mse_convergence_rate} to apply. However, Proposition \ref{prop:rate_small_d} shows that, for covariates supported on a convex compact set, both $NNGP$ and $GPnn$ attain Stone's optimal minimax rate \emph{asymptotically} in any dimension. This is especially relevant for (geo)spatial applications, where data are naturally confined to a bounded geographical region, and thus includes the practically important low-dimensional settings not covered by Theorem \ref{thm:mse_convergence_rate}.
\begin{proposition}[Asymptotic Convergence Rates]
\label{prop:rate_small_d}
Let $n$ be the size of the training set which is i.i.d. sampled from the distribution $P_{\mcX}$ and let the test point be also sampled from $P_{\mcX}$.  Define $\alpha$ for $GPnn/NNGP$ as in Theorem \ref{thm:mse_convergence_rate}. Assume (AC.\ref{a_metr}), (AR.\ref{aR_iso}-\ref{aR_xi}) and 
\begin{itemize}
    \item $P_{\mcX}$ is supported on a compact convex set and has density which is smooth and strictly positive.
\end{itemize}
Then taking $m_n = n^{\frac{2p}{2p+d_\mcX}}$ we have for sufficiently large $n$
\[
\mcR_n\leq A\, n^{-\frac{2\alpha}{2p+d_\mcX}}
\]
where $0<A<\infty$ depends on $P_{\mcX}$, $p$, $q$, $d_\mcX$, $B_f$, $B_T$, $L_f$, $L_c$, $\sigma_\xi$ and the $GPnn$ or $NNGP$ hyper-paramaters.
\end{proposition}
\begin{roadmap}
    Proved in \href{https://github.com/tmaciazek/gpnn_synthetic/blob/main/Online_Appendix_1.pdf}{Online Appendix~1}, Section \ref{app:convergence_rates} using techniques from the proof of Theorem \ref{thm:mse_convergence_rate} and nearest-neighbour asymptotics from \cite{evans} -- Lemmas \ref{lemma:dmin_asymp} and \ref{lem:dm_asymp}.
\end{roadmap}

\section{Performance of $GPnn$ on Real World Datasets}
\label{sec:real_world}
We briefly summarize the real-data results from \citet{GPnn23}, which compare $GPnn$ against SVGP \citep{Hensman} and five distributed $GP$ baselines \citep{hinton2002training,cao2014generalized,tresp2000bayesian,deisenroth2015distributed,amalg_SOD}. Full implementation details, dataset preprocessing, and complete results for all methods are given in \citet{GPnn23}. In all experiments, inputs were pre-whitened, responses were standardized using training-set statistics, and all methods used the squared exponential covariance function. Results were averaged over three random $7/9$--$2/9$ train--test splits.

Table~\ref{tab:metrics_best_dist} reports the best of the five distributed methods with respect to $(R)MSE$. Complete results for all baselines and all three predictive criteria are given in \citet{GPnn23}. With the exception of the Bike dataset, $GPnn$ performs best among the reported methods in both $RMSE$ and $NLL$, and is likewise competitive in calibration; see also Figure~\ref{fig:results}. Note that in Song dataset methods varied considerably in calibration (e.g. large calibration values show a tendency to overinflate the predictive variance) despite having similar NLL levels. In the original experiments of \citet{GPnn23}, these gains were achieved with substantially lower training cost than the competing methods, especially on the largest datasets, where the speed-up was particularly pronounced. A non-negligible fraction of this cost comes from calibration, which can be parallelized or omitted if predictive uncertainty is not required.

\begin{adjustbox}{max width=\textwidth, caption={RMSE and NLL results (mean and standard deviation over 3 runs) for the best distributed method (w.r.t. MSE), SVGP and $GPnn$. Adapted from \cite{GPnn23}.}, 
label={tab:metrics_best_dist}, float=table}
\setlength{\tabcolsep}{3pt}
\begin{tabular}{lllllllll}
\toprule
 &  &  & \multicolumn{3}{l}{\textbf{NLL}} & \multicolumn{3}{l}{\textbf{RMSE}} \\
 &  &  & Distributed & GPnn & SVGP & Distributed & GPnn & SVGP \\
Dataset & \(n\) & \(d\) &  &  &  &  &  &  \\
\midrule
Poletele & 4.6e+03 & 19 & 0.0091 ± 0.015 & \bfseries -0.214 ± 0.019 & -0.0667 ± 0.017 & 0.241 ± 0.0033 & \bfseries 0.195 ± 0.0042 & 0.226 ± 0.0059 \\
Bike & 1.4e+04 & 13 & 0.977 ± 0.0057 & 0.953 ± 0.013 & \bfseries 0.93 ± 0.0043 & 0.634 ± 0.004 & 0.624 ± 0.0079 & \bfseries 0.606 ± 0.0033 \\
Protein & 3.6e+04 & 9 & 1.11 ± 0.0051 & \bfseries 1.01 ± 0.0016 & 1.05 ± 0.0059 & 0.733 ± 0.0038 & \bfseries 0.666 ± 0.0014 & 0.688 ± 0.0043 \\
Ctslice & 4.2e+04 & 378 & -0.159 ± 0.052 & \bfseries -1.26 ± 0.01 & 0.467 ± 0.016 & 0.237 ± 0.012 & \bfseries 0.132 ± 0.00062 & 0.384 ± 0.0064 \\
Road3D & 3.4e+05 & 2 & 0.685 ± 0.0041 & \bfseries 0.371 ± 0.004 & 0.608 ± 0.018 & 0.478 ± 0.0023 & \bfseries 0.351 ± 0.0014 & 0.443 ± 0.008 \\
Song & 4.6e+05 & 90 & 1.32 ± 0.0012 & \bfseries 1.18 ± 0.0045 & 1.24 ± 0.0012 & 0.851 ± 6.7e-05 & \bfseries 0.787 ± 0.0045 & 0.834 ± 0.0011 \\
HouseE & 1.6e+06 & 8 & -1.34 ± 0.0013 & \bfseries -1.56 ± 0.0065 & -1.46 ± 0.0046 & 0.0626 ± 5.2e-05 & \bfseries 0.0506 ± 0.00072 & 0.0566 ± 0.00011 \\
\bottomrule
\end{tabular}
\end{adjustbox}
\begin{figure}
  \centering
  \includegraphics[scale=0.32]{\detokenize{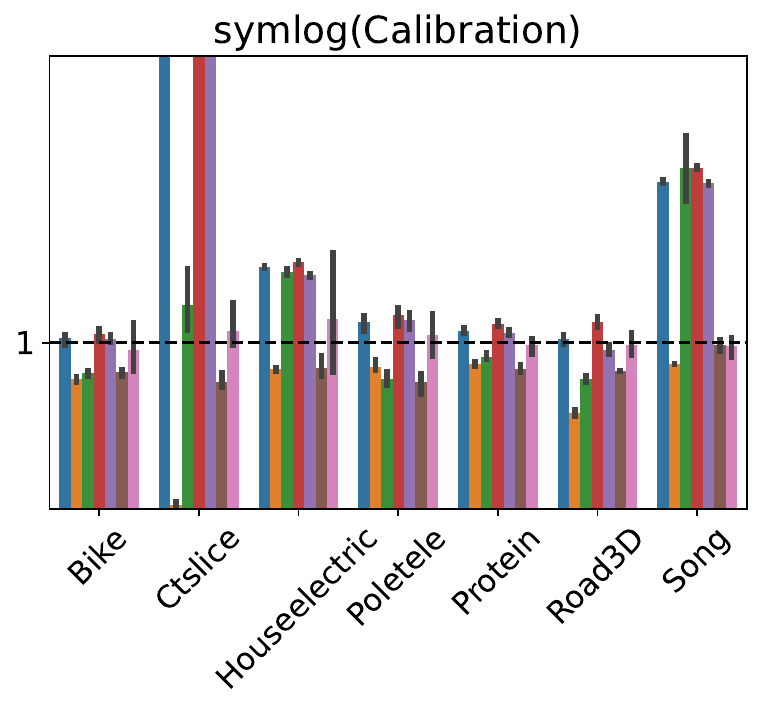}}
  \includegraphics[scale=0.32]{\detokenize{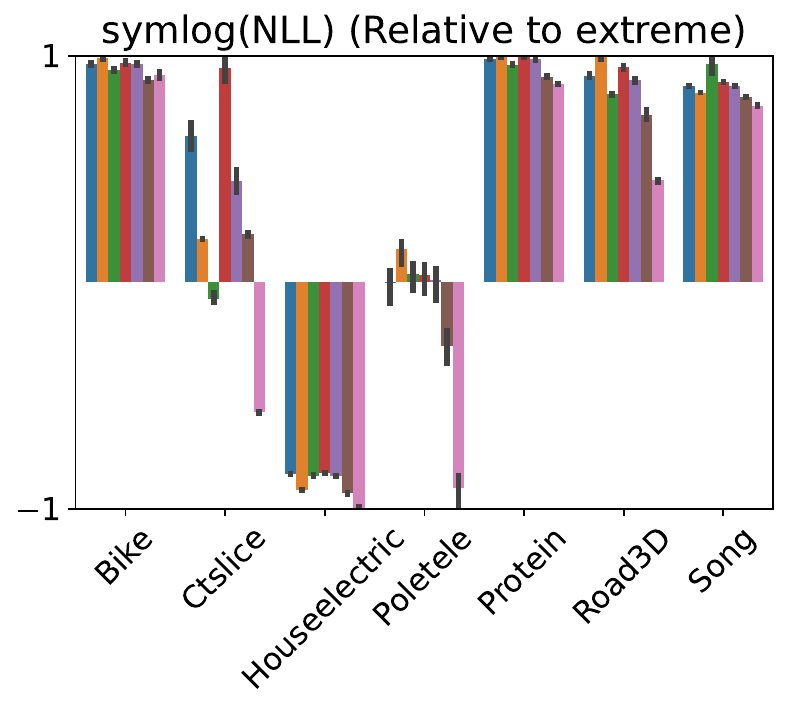}}
  \includegraphics[scale=0.32]{\detokenize{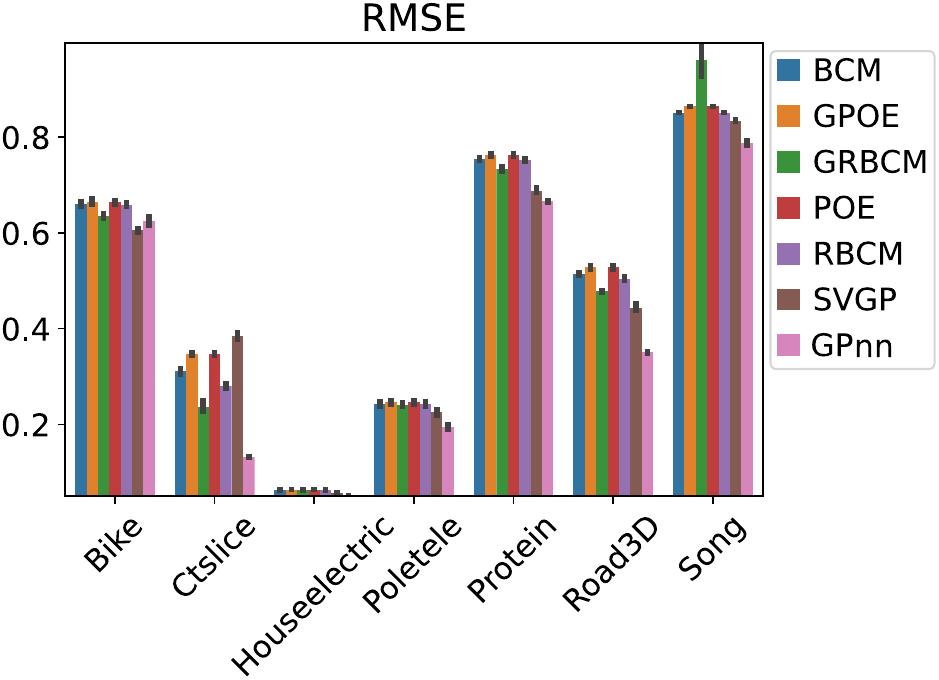}}
  \caption{Experiment results on a suite of UCI datasets. Optimal calibration performance is 1 (indicated by a black dashed line). Lower is better for NLL and RMSE. Y-axis truncated for readability for Calibration due to large values on ``Ctslice''. ``symlog'' is logarithmic axis rescaling applied to the \(y\)-axis on both positive and negative values (``symmetric''). Adapted from \cite{GPnn23}.}
  \label{fig:results}
\end{figure}

\section{Further Evidence of $GPnn/NNGP$ Robustness for Massive Datasets:  Uniform Convergence in the Hyper-Parameter Space and the Vanishing of $MSE$ Derivatives.}

In massive-data regimes, $GP$ hyper-parameter learning is often a computational bottleneck: maximising the (approximate) marginal likelihood typically requires repeated large-matrix inversions and careful tuning, yet our goal is ultimately {\emph{predictive}} accuracy and calibration rather than recovering the exact optimal kernel parameters. The results in this section formalise why $GPnn/NNGP$ remains reliable even when the hyper-parameters $\hat\theta$ are obtained by very cheap, approximate procedures, as observed in \cite{GPnn23,nngp2}. Theorem \ref{thm:uniform_convergence} establishing the \emph{uniform convergence of the MSE over a compact hyper-parameter set} means that, once $n$ is large, the predictive risk of $GPnn$ becomes nearly insensitive to the particular choice of $\hat\theta$ (within a broad, practically relevant range): a coarse estimate, early-stopped optimiser, or subset-based fit yields essentially the same $MSE$ as a carefully optimised one. Complementarily, the \emph{vanishing of risk/$MSE$ derivatives} shows that the risk landscape in $\theta$-space becomes increasingly flat, so the marginal gains from expensive hyper-parameter optimisation diminish rapidly with data size -- small perturbations or estimation error in $\hat\theta$ have a second-order (or negligible) effect on performance. Practically, these properties justify decoupling parameter estimation from prediction: we can allocate minimal compute to obtain a ``reasonable" $\hat\theta$, and rely on the local, nearest-neighbour nature of $GPnn$ and $NNGP$ to deliver stable, well-calibrated predictions at scale without delicate hyper-parameter tuning. Theorems \ref{thm:mse_derivatives_convergence_rate} and \ref{thm:mse_dl_convergence_rate} establish the convergence rates of the risk derivatives. These rates match the risk convergence rate established in Theorem \ref{thm:mse_convergence_rate}. In other words, risk and it's derivatives converge to zero at the same rate (in the matched case of $\alpha=p$).

The results of this Section are proved in \href{https://github.com/tmaciazek/gpnn_synthetic/blob/main/Online_Appendix_1.pdf}{Online Appendix~1} (Section \ref{app:derivatives}). For simplicity, throughout this Section we adapt the homoscedastic noise model from (AR.\ref{aR_xi}), however some of the results can be extended to encompass heteroscedastic noise.

\begin{theorem}[Uniform convergence of MSE in the hyper-parameter space]
\label{thm:uniform_convergence}
Let $X=(\bx_1,\bx_2,\dots)$ be an infinite sequence of i.i.d. points sampled from $P_\mcX$ and denote by $X_n$ its truncation to the first $n$ points. Assume (AC.\ref{a_X}-\ref{a_f}), (AC.\ref{a_metr}), (AR.\ref{aR_xi}) and (AR.\ref{aR_iso}), (AR.\ref{aR_f}). Then, for almost every sampling sequence $X$ and test point $\bx_*\in\supp(P_\mcX)$ and any compact subset $S$ of the hyper-parameters $\Theta=\left(\hat\sigma_\xi^2,\hat\sigma_f^2,\hat\ell\right)\in S\subset \RR_{\geq 0} \times \RR_{>0}\times \RR_{>0}$ we have that
\[ MSE(\bx_*,X_n;\Theta)\xrightarrow{n\to\infty} MSE_\infty(\bx_*;\Theta):= \sigma_\xi^2(\bx_*)\left(1+\frac{1}{m}\right)\]
and this convergence is uniform as a function of $\Theta\in S$.
\end{theorem}
\begin{roadmap}
\emph{Full proof:} \href{https://github.com/tmaciazek/gpnn_synthetic/blob/main/Online_Appendix_1.pdf}{Online Appendix~1}, Section~\ref{sec:mse_uniform}.

\noindent\emph{Strategy.}
Use Theorem~\ref{thm:mse_key_bound} (key deterministic $MSE$ bound) to build a bounding function
$h_\Theta(\bx_*,X_n)$ that is (i) continuous in $\Theta$ on compact $S$ and (ii) satisfies $f_{MSE}=|MSE(\Theta)-MSE_\infty(\Theta)|\le h_\Theta$.
Show that $h_\Theta(\bx_*,X_n)\to 0$ pointwise monotonically in $\Theta$ as $n\to\infty$ (since $d_m,\epsilon_m$ decrease monotonically under this Theorem's assumptions). By Dini's theorem \citep{Rudin} conclude that $h_\Theta(\bx_*, X_n)\xrightarrow{n\to\infty} 0$ uniformly on $S$. Since $f_{MSE}(X_n,\bx_*;\Theta)$ is sandwiched between $h_\Theta(\bx_*, X_{n})$  and the constant zero function it follows that $f_{MSE}(X_n,\bx_*;\Theta)\xrightarrow{n\to\infty} 0$ uniformly on $S$.
\end{roadmap}

\noindent In the remaining part of this section we will use the following shorthand notation for the $MSE$ derivatives. For each $\phi\in\{\hat\sigma_\xi^2,\hat\sigma_f^2,\hat\ell,\hat\bb\}$ we define
\begin{equation}\label{def:derivatives_shorthand}
D_\phi(\mcX_*,\bX_n):=
\begin{cases}
\left|\partial_\phi MSE(\mcX_*,\bX_n)\right|, & \phi\in\{\hat\sigma_\xi^2,\hat\sigma_f^2,\hat\ell\}, \\
\left\|\nabla_{\hat\bb} MSE_{NNGP}(\mcX_*,\bX_n)\right\|_2, & \phi=\hat\bb.
\end{cases}
\end{equation}
\begin{theorem}
\label{thm:ksc_nv_derivatives_consistency}
Assume (AC.\ref{a_X}-\ref{a_f}), (AC.\ref{a_metr}) and (AR.\ref{aR_xi}). 
If the number of nearest neighbours $m$ is fixed, the following limits hold for $GPnn$ and $NNGP$ with probability one (with respect to $\bX_n\sim P_{\mcX}^n$) and for any text point $\bx_*\in \supp_{\rho_c}(P_\mcX)$. 
\begin{equation}\label{eq:mse_nv_ksc_b_derivative_limits}
D_\phi(\bx_*,\bX_n)\xrightarrow{n\to\infty} 0,\quad \phi\in\{\hat\sigma_\xi^2,\hat\sigma_f^2,\hat\bb\}.
\end{equation}
Under the additional assumptions (AR.\ref{aR_iso}), (AR.\ref{aR_f}), the above convergence is uniform on any compact subset $S$ of the hyper-parameters $\Theta=\left(\hat\sigma_\xi^2,\hat\sigma_f^2,\hat\ell,\hat\bb\right)\in S\subset \RR_{\geq 0} \times \RR_{>0}\times \RR_{>0}\times \RR^{d_T}$. Moreover, under (AC.\ref{a_X}-\ref{a_metr}) and the assumptions that
\begin{itemize}
\item  the function $f$ in the $GPnn$ response model \eqref{eq:gpnn_responses} satisfies $\|f(\cdot)\|_\infty<B_f<\infty$,
\item the functions $t_i$, $i=1,\dots, d_T$ in the $NNGP$ response model \eqref{eq:nngp_responses} satisfy $\|t_i(\cdot)\|_\infty<B_T<\infty$,
\end{itemize}
we have that $\EE\left[D_\phi(\mcX_*,\bX_n)\right]\xrightarrow{n\to\infty} 0$ for $(\mcX_*,\bX_n)\sim P_\mcX\otimes P_\mcX^n$ and for each $\phi\in\{\hat\sigma_\xi^2,\hat\sigma_f^2,\hat\bb\}$.
\end{theorem}
\begin{roadmap}
    \emph{Full proof:} \href{https://github.com/tmaciazek/gpnn_synthetic/blob/main/Online_Appendix_1.pdf}{Online Appendix~1}, Section~\ref{sec:mse_derivatives}.

    \noindent\emph{Strategy.}
Use Equations \eqref{eq:dMSE_dphi_expansion}--\eqref{eq:dvar_dphi} to express $\partial_\phi MSE$ via kernel matrices and their derivatives.
Lemma~\ref{lemma:ksc_nv_derivatives_consistency} provides closed-form derivatives and reduces the problem to controlling just the $MSE$-bias term and the two expressions
\[
f(X_{\mcN})-\left(\hat\sigma_\xi^2+m\hat\sigma_f^2\right)K_{\mathcal{N}}^{-1}f(X_{\mcN}), \quad \mathbf{k}^*_{\mathcal{N}} - \left(\hat\sigma_\xi^2+m\hat\sigma_f^2\right)K_{\mathcal{N}}^{-1}\mathbf{k}^*_{\mathcal{N}}.
\]
Lemma~\ref{lemma:kernel_limits} controls the matrix/vector limits, while Theorem~\ref{thm:universal_consistency} takes care of the bias term.
For the uniform-in-$\Theta$ statement, reuse the same bounding idea as in the proof of Theorem~\ref{thm:uniform_convergence}
(applied to the derivative expressions).
\begin{itemize}
  \item Eqns.~\eqref{eq:dMSE_dphi_expansion}--\eqref{eq:dvar_dphi}: derivative expansions.
  \item Lemma~\ref{lemma:ksc_nv_derivatives_consistency} (explicit derivative formulas + expectation limits): gives $\partial_\phi$ of posterior mean/variance for $\phi\in\{\hat\sigma_\xi^2,\hat\sigma_f^2\}$ and $\nabla_{\hat\bb} MSE_{NNGP}$, and shows their expectations vanish.
  \item Lemma~\ref{thm:uniform_convergence_deriv}: proves uniform convergence of $MSE$ derivatives.
  \item Lemma~\ref{lemma:bias_variance} (Bias--variance decomposition): used to control/interpret terms involving bias and variance.
\end{itemize}
\end{roadmap}
\begin{theorem}[Convergence Rates of Derivatives]
\label{thm:mse_derivatives_convergence_rate}
Let $n$ be the size of the training set which is i.i.d. sampled from the distribution $P_{\mcX}$ and let the test point be also sampled from $P_{\mcX}$. Let $m$ be the (fixed) number of nearest-neighbours used in $GPnn/NNGP$. Assume (AC.\ref{a_metr}) and (AR.\ref{aR_iso}-\ref{aR_xi}). In $GPnn$ define $\alpha_\phi := \min\{p,q\}$ for each $\phi\in\{\hat\sigma_\xi^2,\hat\sigma_f^2,\hat\bb\}$. In $NNGP$ define $\overline{q}:=\min\{q_1,\dots,q_{d_T}\}$ and $\alpha_\phi:=\min\{p,q_0,\overline{q}\}$ when $\phi\in\{\hat\sigma_\xi^2,\hat\sigma_f^2\}$ and $\alpha_{\hat\bb}:=\min\{2p,\overline{q}\}$. Then, if $d_\mcX > 4 (\alpha_\phi+p)$, for each $\phi\in\{\hat\sigma_\xi^2,\hat\sigma_f^2,\hat\bb\}$ we have
\begin{align}
\EE\left[D_\phi(\mcX_*,\bX_n)\right]\leq A_1^{(\phi)}\,\left(\frac{m}{n}\right)^{2\alpha_\phi/d_\mcX}+A_2^{(\phi)}\,m\,\left(\frac{m}{n}\right)^{2(\alpha_\phi+p)/d_\mcX},
\end{align}
where $0<A_i^{(\phi)}<\infty$ depend on $p$, $\{q_i\}$, $d_\mcX$, $d_T$, $B_f$, $B_T$, $L_f$, $L_c$, $L_{\tilde c}$, $\sigma_\xi$ and the $GPnn/NNGP$ hyper-paramaters. Taking $m_n = n^{\frac{2p}{2p+d_\mcX}}$ the derivatives tend to zero at the same rates as the (minimax-optimal) risk rate from Stone's theorem, i.e.,
\begin{align}
\EE\left[D_\phi(\mcX_*,\bX_n)\right]\leq\left(A_1^{(\phi)}+A_2^{(\phi)}\right)\, n^{-\frac{2\alpha_\phi}{2p+d_\mcX}}. 
\end{align}
\end{theorem}
\begin{roadmap}
    \emph{Full proof:} \href{https://github.com/tmaciazek/gpnn_synthetic/blob/main/Online_Appendix_1.pdf}{Online Appendix~1}, Section~\ref{sec:mse_derivatives}.

    \noindent\emph{Strategy.}
Bound the two pieces in Equation \eqref{eq:dMSE_dphi_expansion} in terms of kernel-induced distances via Lemma~\ref{lemma:dMSE_dnv_dksc_bounds}
(which uses Lemma~\ref{lemma:mu_mean_var_bound} and Lemma~\ref{lemma:K_ineqs}).
Then take expectations using the good/bad event split (as in the proof of Theorem \ref{thm:mse_convergence_rate}) and plug in the NN-distance and expected kernel-distance rates (Lemma~\ref{lemma:dmax_rate}, Lemma~\ref{lemma:epsilons_rates}),
followed by the choice of $m_n$ that balances the leading terms.
\begin{itemize}
  \item Eqns.~\eqref{eq:dMSE_dphi_expansion}--\eqref{eq:dvar_dphi} (derivative expansions): start point for bounding $|\partial_\phi MSE|$.
  \item Lemma~\ref{lemma:dMSE_dnv_dksc_bounds} (Bounds for $MSE$ derivatives): gives explicit upper bounds on $|\partial_{\hat\sigma_\xi^2} MSE|$ and $|\partial_{\hat\sigma_f^2} MSE|$ in terms of $d_m$ and kernel-induced distances (including $\epsilon_m$).
  \item Lemma~\ref{lemma:ksc_nv_derivatives_consistency}: supplies global covariate-independent boundedness of intermediate $GP$ quantities used to globally bound the derivatives.
  \item Lemma~\ref{lemma:epsilons_rates} and Lemma~\ref{lemma:dmax_rate} turn the deterministic bounds involving $d_m$, $\epsilon_m$ into rates after taking expectations.
\end{itemize}
\end{roadmap}

\begin{remark}[Strong insensitivity of $NNGP$ prediction risk to $\hat\bb$.]
As shown in \href{https://github.com/tmaciazek/gpnn_synthetic/blob/main/Online_Appendix_1.pdf}{Online Appendix~1} (Section~\ref{app:derivatives}), the $MSE$ in $NNGP$ depends on the hyperparameter $\hat\bb$ only through the scalar projection $\bv^{T}(\bb-\hat\bb)$, and its Hessian with respect to $\hat\bb$ is the rank-one matrix $2\,\bv\,\bv^{T}$, independent of $\hat\bb$, where
\[
\bv(\bx_*,X_n):=\bt_*^T-\Gamma\,{{\bk}_{\mcN}^*}^T\, K_{\mcN}^{-1}T_\mcN.
\]
Since $\EE[\|\bv(\mcX_*,\bX_n)\|_2^2]\to 0$ as $n\to\infty$, the expected Hessian norm also vanishes:
\[
\EE\!\left[\left\|\nabla_{\hat\bb}^2 MSE_{NNGP}\right\|_2\right]
=
2\,\EE\!\left[\|\bv(\mcX_*,\bX_n)\|_2^2\right]
\to 0.
\]
Moreover,
\[
\left\|\nabla_{\hat\bb}MSE_{NNGP}\right\|_2
\le
2\|\bv(\bx_*,X_n)\|_2^2\,\|\bb-\hat\bb\|_2,
\]
so the risk landscape becomes asymptotically flat in $\hat\bb$ at both first and second order. In particular, for large $n$ the choice of $\hat\bb$ has negligible effect on the risk, while for finite samples the near-vanishing Hessian makes optimisation over $\hat\bb$ from the risk criterion alone poorly conditioned unless supplemented by an additional criterion, such as regularisation. Furthermore, the $\hat\bb$-gradients vanish faster than those with respect to the other hyperparameters, see Theorem~\ref{thm:mse_derivatives_convergence_rate}. For these reasons, fixing $\hat\bb$ at a default value, such as $\hat\bb=0$, can be a sensible choice in very large-data settings when predictive-risk minimisation is the main objective.
\end{remark}

\subsection{Predictive Risk Landscape with Respect to the Lengthscale}
To present results concerning the vanishing of $MSE$ gradients with respect to the lengthscale hyper-parameter $\hat\ell$, we need to introduce the following new assumption.
\begin{enumerate}[({AD.}1)]
\item The  normalised kernel function $c(\cdotp)$ is isotropic and such that $c(u)$ is differentiable for $u> 0$, the limit $\lim_{u\to0^+} c'(u)$ exists (but may not be finite), and $0\leq c(u)\leq 1$ for all $u\geq 0$, and $c(0)=1$. \label{aD_iso}
\item The normalised kernel function $c(u)$ is differentiable and satisfies for all $u\geq 0$
\[\left|u\frac{d c(u)}{du}\right|\leq B_c,\quad \left|u \frac{d c(u)}{du}\right|\leq L_c'\, u^{2p'}\]
for some $B_c,L_c'\geq 1$, and $0<p'\leq 1$. \label{aD_bnd}
\end{enumerate}
In Appendix \ref{appendix:kernel_bounds} we show that assumptions (AD.\ref{aD_iso}- \ref{aD_bnd}) are satisfied by the squared-exponential kernel and the Mat\'{e}rn family of kernels. For these kernels, we have $p'=p$ where $p$ is defined in (AR.\ref{aR_c}). 

The proofs of Theorem~\ref{thm:l_derivative_consistency} and Theorem~\ref{thm:mse_dl_convergence_rate} below are sketched in \href{https://github.com/tmaciazek/gpnn_synthetic/blob/main/Online_Appendix_1.pdf}{Online Appendix~1}, Section~\ref{sec:mse_derivatives} as a straightforward reapplication of the techniques established in this and previous sections.

\begin{theorem}\label{thm:l_derivative_consistency}
Under the assumptions (AC.\ref{a_X}-\ref{a_f}), (AC.\ref{a_metr}), (AR.\ref{aR_xi}),  (AD.\ref{aD_iso}) and (AR.\ref{aR_c}) the following limit holds for $GPnn$ and $NNGP$ with probability one (with respect to $\bX_n\sim P_{\mcX}^n$) and for any text point $\bx_*\in \supp_{\rho_c}(P_\mcX)$
\begin{equation}
D_{\hat\ell}(\bx_*,\bX_n)\xrightarrow{n\to\infty} 0
\end{equation}
with probability one.
Under the additional assumptions (AR.\ref{aR_iso}) and (AR.\ref{aR_f}), the above convergence is uniform on any compact subset $S$ of the hyper-parameters $\Theta=\left(\hat\sigma_\xi^2,\hat\sigma_f^2,\hat\ell,\hat\bb\right)\in S\subset \RR_{\geq 0} \times \RR_{>0}\times \RR_{>0}\times \RR^{d_T}$. Moreover, under (AC.\ref{a_X_app}-\ref{a_f_app}), (AC.\ref{a_metr_app}), (AR.\ref{aR_xi_app}), (AD.\ref{aD_iso}-\ref{aD_bnd}) and (AR.\ref{aR_c}) and the assumptions that
\begin{itemize}
\item  the function $f$ in the $GPnn$ response model \eqref{eq:gpnn_responses} satisfies $\|f(\cdot)\|_\infty<B_f<\infty$,
\item the functions $t_i$, $i=1,\dots, d_T$ in the $NNGP$ response model \eqref{eq:nngp_responses} satisfy $\|t_i(\cdot)\|_\infty<B_T<\infty$,
\end{itemize}
we have that $\EE\left[D_{\hat\ell}(\mcX_*,\bX_n)\right]\xrightarrow{n\to\infty} 0$ for $(\mcX_*,\bX_n)\sim P_\mcX\otimes P_\mcX^n$.
\end{theorem}

\begin{theorem}
\label{thm:mse_dl_convergence_rate}
Let $n$ be the size of the training set which is i.i.d. sampled from the distribution $P_{\mcX}$ and let the test point be also sampled from $P_{\mcX}$. Let $m$ be the (fixed) number of nearest-neighbours used in $GPnn/NNGP$. Assume (AC.\ref{a_metr}), (AR.\ref{aR_iso}-\ref{aR_xi}) and (AD.\ref{aD_iso}- \ref{aD_bnd}). Then, if $d_\mcX > 4(p'+2p)$ with $p'$ defined in (AD.\ref{aD_bnd}), we have
\begin{align}
\EE\left[D_{\hat\ell}(\mcX_*,\bX_n)\right]\leq \frac{\max\{\hat\ell^{-2p'},1\}}{\hat\ell}A_1\,\left(\frac{m}{n}\right)^{2p'/d_\mcX}+\frac{1}{\hat\ell}A_2\,m^2\,\left(\frac{m}{n}\right)^{2(p'+2p)/d_\mcX},
\end{align}
where $0<A_1,A_2<\infty$ depend on $p$, $\{q_i\}$, $d_\mcX$, $d_T$, $B_f$, $B_T$, $B_c$, $L_f$, $L_c$, $L_c'$, $L_{\tilde c}$, $\sigma_\xi$, $\hat\sigma_\xi$, $\hat\sigma_f$ (but not on $\hat\ell$). Taking $m_n = n^{\frac{2p}{2p+d_\mcX}}$ the derivatives tend to zero at the following rate.
\begin{align}
\EE\left[D_{\hat\ell}(\mcX_*,\bX_n)\right]\leq\frac{1}{\hat\ell}\left(\max\left\{\hat\ell^{-2p'},1\right\}A_1+A_2\right)\, n^{-\frac{2p'}{2p+d_\mcX}}.
\end{align}
\end{theorem}

Our derivative bounds in Theorem \ref{thm:mse_dl_convergence_rate}  yield a direct practical implication for learning the lengthscale: it contains an explicit $1/\hat\ell$ prefactor (up to kernel-dependent constants excluding $\hat\ell$), which implies that the risk/$MSE$ landscape becomes progressively flatter as $\hat\ell$ grows. This flattening is often exacerbated in high ambient dimension. Indeed, for standardised data (i.e., each coordinate has almost unit variance and the coordinates are almost uncorrelated) the typical Euclidean distance $\|\bx-\bx'\|_2$ concentrates at order $\sqrt{d_\mcX}$ (see, e.g., \citet{Aggarwal2001}), and common bandwidth/lengthscale heuristics for isotropic kernels select $\hat\ell$ proportional to a typical (often median) pairwise distance (the ``median heuristic'') \citep{garreau2018,pmlr-v151-meanti22a}. Consequently, one frequently observes that $\hat\ell$ grows like $\sqrt{d_\mcX}$ in practice. In this regime the leading prefactor scales as $1/\hat\ell=\mcO\left(d_\mcX^{-1/2}\right)$, which suppresses gradients in the raw $\hat\ell$-parameter and can make direct optimisation in $\hat\ell$-space increasingly inefficient as $d_\mcX$ grows. A standard remedy (fully consistent with Theorem \ref{thm:mse_dl_convergence_rate}) is to reparameterise in terms of $\log \hat\ell$ and optimise in $(\log \hat\ell)$-space, since $\partial_{\log \hat\ell} MSE = \hat\ell\,\partial_{\hat\ell} MSE$ removes the leading $1/\hat\ell$ prefactor while preserving the location of optima under the change of variables. Since the derivative limits are uniform on every compact subset of the hyper-parameter space, it is practically reasonable to optimise $\log \hat\ell$ within compact, dimension-aware ranges (e.g. $\log \hat\ell\sim \log\sqrt{d_\mcX}$ after data standardisation).

\subsection{Massive-Scale Synthetic Data Experiments}
\label{sec:synthetic_nngp_rates}
In this Section we complement the theory with large-scale synthetic experiments designed to illustrate
i) the convergence rate of the predictive risk,
ii) the flattening of the predictive-risk landscape with respect to the hyper-parameters
$\hat\ell$, $\hat\sigma_\xi^2$ and $\hat\sigma_f^2$,
and 
iii) the convergence rates of the corresponding risk derivatives.

Throughout this section we model the responses according to the $GPnn$ and $NNGP$ models from \eqref{eq:gpnn_responses} and \eqref{eq:nngp_responses}. Predictions are made using the (debiased) $GPnn$ and $NNGP$ predictive means \eqref{eq:gpnn_unbiased_est} and \eqref{eq:nngp_unbiased_est} with the usual hyper-parameters $\Theta=(\hat\sigma_\xi^2,\hat\sigma_f^2,\hat\ell,\hat\bb)$.

\paragraph{Simulation Protocol.}
All simulations are carried out using the locality-based synthetic-data procedure from Algorithm~1 of \cite{GPnn23}. For each training size $n$:
\begin{enumerate}
    \item draw a training set $X_n=\{\bx_i\}_{i=1}^n$ and an independent test set $X_{test}$ from the relevant covariate distribution $P_\mcX$,
    \item for each $\bx_*\in X_{test}$, compute its set $\mcN(\bx_*)$ of $m_n$ nearest neighbours in $X_n$,
    \item in $GPnn$, evaluate the deterministic regression function $f(\bx_*)$ each $\bx_*\in X_{test}$ and $f(\bx)$ for all $\bx$ in its nearest neighbour set $\mcN(\bx_*)$ and add sampled noise; in $NNGP$ sample jointly the local Gaussian latent field vector (in $NNGP$) and then sample the nearest-neighbour responses
    \[
    \left(w(\bx_*),\, \bw_\mcN\right)\sim N\left(0,\sigma_w^2\tilde C_{\mcN\oplus \bx_*}\right),\quad  \by_\mcN\sim N\left(\bw_\mcN,\sigma_\xi^2\bone\right),
    \]
    where $\tilde C_{\mcN\oplus \bx_*}$ is the $(m_n+1)\times (m_n+1)$ Gram matrix formed from the (normalised) correlation function of $w(\cdotp)$ between $\bx_*$ and it's nearest neighbours,
    \item evaluate the $GPnn/NNGP$ predictive mean and variance at $\bx_*$ using only the sampled neighbour responses and the assumed hyper-parameters \(\Theta\),
    \item average the resulting squared errors over the test set and over $N_R$ independent realisations of the training set to obtain the empirical risk as follows
\begin{equation}
\widehat{\mcR}_n(\Theta)
=
\frac{1}{N_R}\sum_{r=1}^{N_R} \frac{1}{|X_{test}|}\sum_{\bx_*\in X_{test}}
\left(g^{(r)}(\bx_*)-\hat f_{n,\Theta}^{(r)}(\bx_*)\right)^2,
\label{eq:mse_hat_exp}
\end{equation}
where $g^{(r)}(\bx_*)=f(\bx_*)$ in $GPnn$ and $g^{(r)}(\bx_*)=\bt(\bx_*)^T.\bb+w^{(r)}(\bx_*)$ in $NNGP$.
\end{enumerate}
This avoids generating an entire size-\(n\) latent Gaussian field for every training set draw while preserving the exact synthetic
$MSE$/risk statistics associated with nearest-neighbour prediction (see \cite{GPnn23} for more explanation). In the experiments, we have used $\left|X_{test}\right|=10^4$, $N_R=5$ training set draws and chosen the nearest-neighbours using an exact search -- for implementation details and code see \url{https://github.com/tmaciazek/gpnn_synthetic}.

\paragraph{Neighbourhood Size Schedule.}
To match Theorem~\ref{thm:mse_convergence_rate} and Proposition~\ref{prop:rate_small_d}, we use
$m_n=\left\lceil C\, n^{\frac{2p}{2p+d_\mcX}}\right\rceil$ 
with a fixed constant \(C\) chosen so that $m=100$ at the maximum $n$ used in the experiments i.e. $n=10^{6}$ when $d_\mcX=2$ and $n=10^{23/2}$ when $d_\mcX\in\{4,8,16\}$. For Mat\'ern-$\nu$ kernels we have $p=\min\{\nu,1\}$ (see \href{https://github.com/tmaciazek/gpnn_synthetic/blob/main/Online_Appendix_1.pdf}{Online Appendix~1}, Section~\ref{appendix:kernel_bounds}).

\paragraph{Slope Estimation.}
To estimate empirical convergence exponents, we fit a least-squares line to the tail of the log--log curve $\log_{10} \widehat{\mcR}_n$ vs. $\log_{10} n$ over the eight largest available values of \(n\).
We compare the fitted slope to the theoretical Stone exponent
\(2\nu/(2\nu+d_\mcX)\).

\paragraph{Illustration of Theorem~\ref{thm:mse_convergence_rate}, Proposition \ref{prop:rate_small_d} and Beyond.} We first consider a $GPnn$ setting where Theorem~\ref{thm:mse_convergence_rate} applies directly. The covariates are sampled from $P_\mcX=N\left(0,\frac{1}{d_{\mcX}}\bone\right)$, and the responses are sampled according to \eqref{eq:gpnn_responses} with $\sigma_\xi^2=0.1$ and the regression function
\[
f_{d_\mcX}(\bx)
=
\tanh\!\left(
\frac{1}{\sqrt{d_\mcX}}\sum_{j=1}^{d_\mcX} \sin\!\left(\sqrt{d_\mcX}\, x_j\right)
+
\frac{1}{\sqrt{d_\mcX/2}}\sum_{j=1}^{d_\mcX/2}\cos\!\left(\sqrt{d_\mcX}\,\bigl(x_{2j-1}+x_{2j}\bigr)\right)
\right).
\]
This regression function was chosen as a bounded, globally Lipschitz, genuinely $d_\mcX$-dimensional nonlinear function combining coordinate-wise oscillations with pairwise interactions. Suitable scaling makes the function have non-trivial variance in all dimensions when $\mcX\sim N\left(0,\frac{1}{d_\mcX}\bone\right)$. The regression kernel is Mat\'{e}rn with $\nu=1$ and $\hat\sigma_f^2=1$, $\hat\ell=0.5$, $\hat\sigma_\xi^2=0.2$. In the notation of Theorem~\ref{thm:mse_convergence_rate} this effectively means $p=1$ (see \href{https://github.com/tmaciazek/gpnn_synthetic/blob/main/Online_Appendix_1.pdf}{Online Appendix~1}, Section~\ref{appendix:kernel_bounds}) which predicts the rate $n^{-\frac{2}{2+d_\mcX}}$. Figure~\ref{fig:risk_rates_gpnn} and Table~\ref{tab:slopes_gpnn} show good agreement with theory as justified by the presented values of $R$-squared showing the goodness of fit \citep{draper1998applied}.

Next, we consider a $NNGP$ setting. Both the latent covariance $\tilde k$ generating the responses and the covariance used in the
$NNGP$ predictor belong to the Mat\'ern family with the same smoothness parameter $\nu$. We choose the experiments so that $\alpha=\nu$, in the notation of Theorem~\ref{thm:mse_convergence_rate}. Thus, the target Stone's minimax exponent is $2\nu/(\nu+d_\mcX)$. Since for Mat\'ern-$\nu$ kernels we have $p=\min\{\nu,1\}$ (see \href{https://github.com/tmaciazek/gpnn_synthetic/blob/main/Online_Appendix_1.pdf}{Online Appendix~1}, Section~\ref{appendix:kernel_bounds}), Proposition~\ref{prop:rate_small_d} predicts the rate $n^{-\frac{\min\{\nu,1\}}{\min\{\nu,1\}+1}}$ vs. Stone's rate $n^{-\frac{\nu}{\nu+1}}$. This matches Stone's minimax-optimal exponent when $\nu<1$. In the experiment we sample the covariates uniformly from unit disk ($d_\mcX=2$). The responses are sampled according to \eqref{eq:nngp_responses} with 
\begin{equation}\label{eq:response_hyperparams}
    \bb=(1,1),\quad \bt(\bx)=(x_1^2,x_2^2),\quad \ell=\sqrt{2},\quad \sigma_w^2=1.0,\quad \sigma_\xi^2=0.1.
\end{equation} 
Figure~\ref{fig:risk_rates_d2} and Table~\ref{tab:slopes_d2} show good agreement with theory.

\begin{figure}[h]
\centering
\begin{minipage}[t]{0.48\linewidth}
\centering
\includegraphics[width=.9\linewidth]{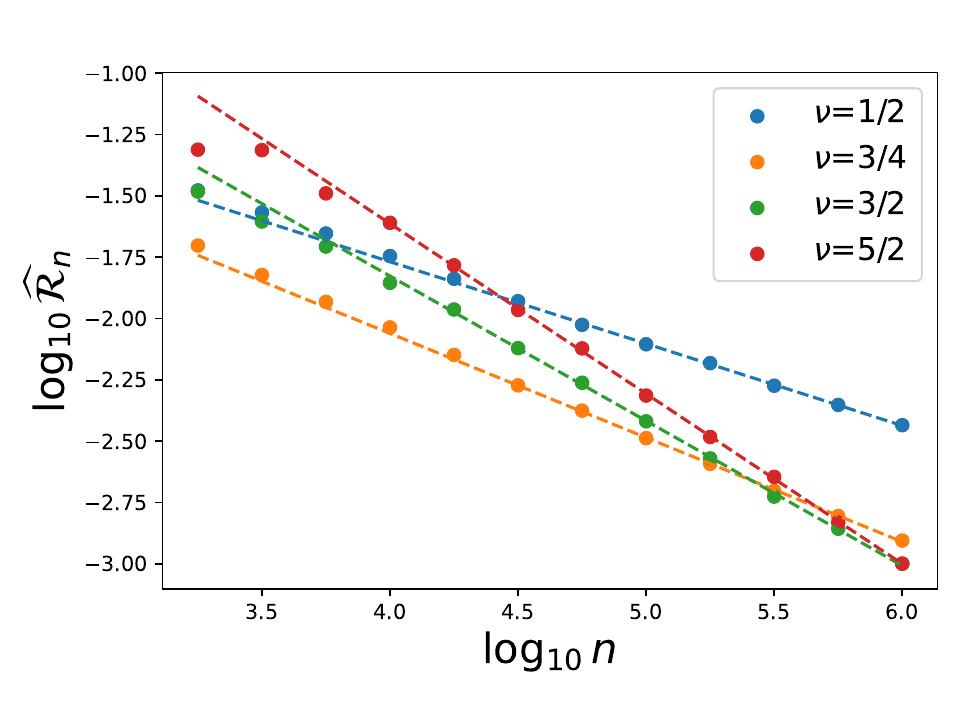}
\captionof{figure}{$NNGP$ regression with $P_\mcX$ uniform on \(d_\mcX=2\) disk.
    Log--log plots of the risk \(\widehat{\mcR}_n\).
    Dashed reference lines show the fitted slopes.}
\label{fig:risk_rates_d2}
\end{minipage}
\hfill
\begin{minipage}[t]{0.48\linewidth}
\centering
\includegraphics[width=.9\linewidth]{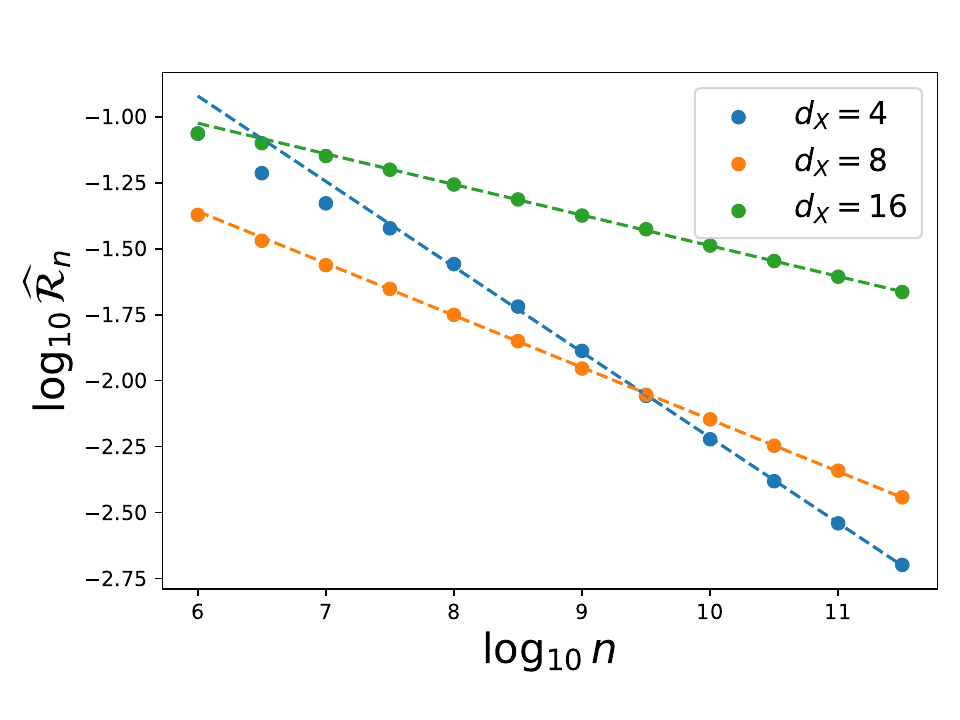}
\captionof{figure}{$GPnn$ regression with Gaussian $P_\mcX$ in different dimensions and Mat\'{e}rn-$1$ kernel function.
    Log--log plots of the risk \(\widehat{\mcR}_n\).
    Dashed reference lines show the fitted slopes.}
\label{fig:risk_rates_gpnn}
\end{minipage}
\end{figure}

\begin{figure}
\centering
\begin{minipage}[t]{0.48\linewidth}
\centering
\begin{tabular}{cccc}
\hline
$\nu$ & Fitted & Stone & $R^2$ \\
\hline
$1/2$ & $0.3339$ & $0.3333\dots$ & $0.9989$ \\
$3/4$ & $0.4246$ & $0.4286\dots$ & $0.9991$ \\
$3/2$ & $0.5902$ & $0.6$ & $0.9995$\\
$5/2$ & $0.6932$ & $0.7143\dots$ & $0.9998$ \\
\hline
\end{tabular}
\captionof{table}{$NNGP$ regression with $P_\mcX$ uniform on \(d_\mcX=2\) disk. Estimated and theoretical (negative) slopes for different values of $\nu$.}
\label{tab:slopes_d2}
\end{minipage}
\hfill
\begin{minipage}[t]{0.48\linewidth}
\centering
\begin{tabular}{cccc}
\hline
$d_{\mcX}$ & Fitted & Stone & $R^2$ \\
\hline
$4$ & $0.3237$ & $0.3333\dots$ & $0.9997$ \\
$8$ & $0.1973$ & $0.2$ & $0.9998$ \\
$16$ & $0.1161$ & $0.1111\dots$ & $0.9998$ \\
\hline
\end{tabular}
\captionof{table}{$GPnn$ regression with Gaussian $P_\mcX$. Estimated and theoretical (negative) slopes for different dimensions ($\nu=1$).}
\label{tab:slopes_gpnn}
\end{minipage}
\end{figure}

We also explore the smoother regime $\nu > 1$ which lies beyond the scope of the present theory but is of clear practical interest. There, we take the neighbourhood size schedule to be $m_n=\left\lceil C\, n^{\frac{2\nu}{2\nu+d_\mcX}}\right\rceil$ with a fixed constant \(C\) chosen so that $m=100$ at the maximum $n$ used in the given series of experiments. Under this neighbourhood size schedule the observed slopes remain close to Stone's minimax rate which provides numerical evidence that $NNGP$ and $GPnn$ continue to exploit higher latent-field smoothness even though the current theory recovers Stone's rates only in the regime $\nu < 1$.

\paragraph{Flattening of the Risk Landscape.}
To illustrate the risk landscape flattening, we consider one-dimensional slices of \(\widehat{R}_n(\Theta)\) as a function of each of the three hyper-parameters with the remaining two held fixed at their true values. According to Theorem \ref{thm:uniform_convergence}, we expect the dependence of empirical risk on each of these quantities to become progressively weaker as \(n\) increases, see Figure~\ref{fig:risk_landscape}.
\begin{figure}[h]
\centering
\includegraphics[width=\linewidth]{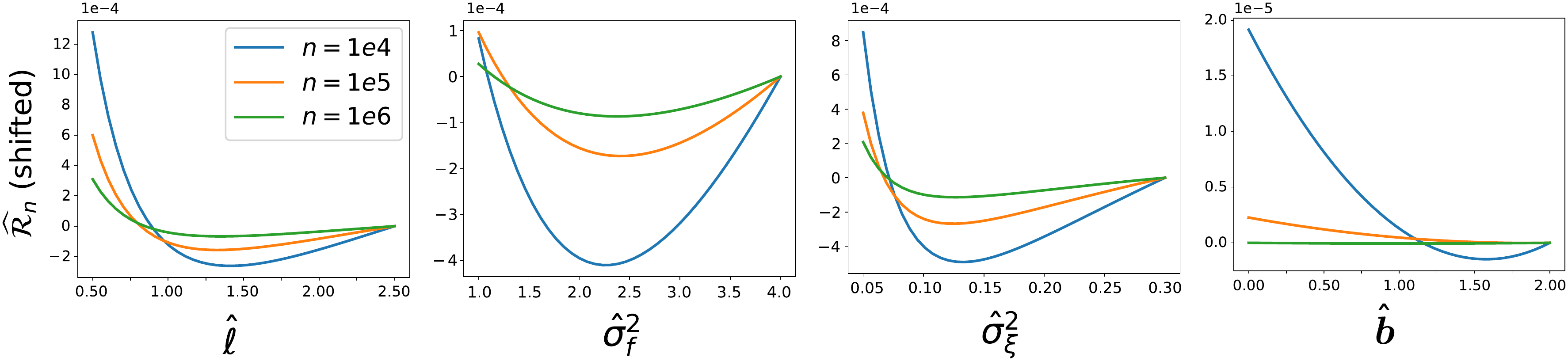}
\captionof{figure}{Risk landscape (shifted) as a function of the hyperparameters and training set size. $NNGP$ regression with $P_\mcX$ uniform on \(d_\mcX=2\) disk and $\nu=1/2$. The parameter $\hat\bb$ is chosen as $\hat\bb=(\hat b,\hat b)$. Note the extreme flatness of $\hat\bb$-landscape.}
\label{fig:risk_landscape}
\end{figure}
We next investigate the vanishing-gradient effect predicted by
Theorem~\ref{thm:mse_derivatives_convergence_rate} and Theorem~\ref{thm:mse_dl_convergence_rate}. We estimate derivative magnitudes by a symmetric finite-difference (five-point stencil) rule applied to the averaged $MSE$. For $\nu<1$ our theory predicts the derivative magnitudes to decay at the same rate as the risk itself, see Fig.~\ref{fig:derivative_rates_d2} and Table~\ref{tab:slopes_derivatives_d2}.

\begin{figure}
\centering
\begin{minipage}[c]{0.45\linewidth}
\centering
\includegraphics[width=.9\linewidth]{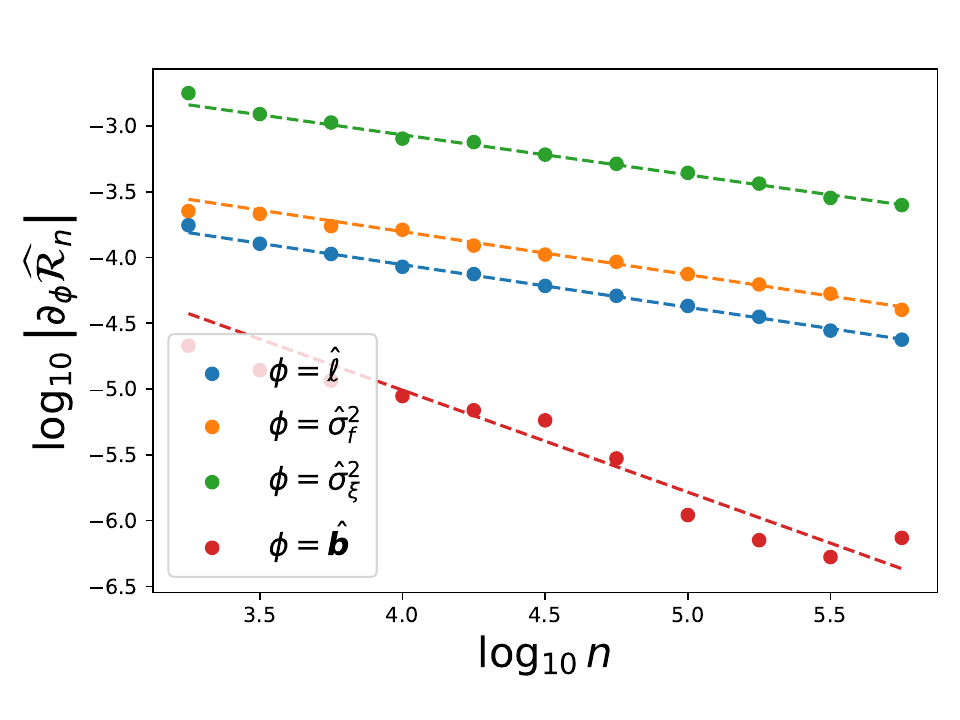}
\captionof{figure}{$NNGP$ regression with $P_\mcX$ uniform on \(d_\mcX=2\) disk.
    Log--log plots of the risk derivatives \(\partial_\phi\widehat{\mcR}_n\) for $\nu=1/2$.
    Dashed reference lines show the fitted slopes.}
\label{fig:derivative_rates_d2}
\end{minipage}
\hfill
\begin{minipage}[c]{0.5\linewidth}
\centering
\begin{tabular}{cccc}
\hline
$\phi$ & Fitted & Theory & $R^2$ \\
\hline
$\hat\ell$ & $0.3239$ & $0.3333\dots$ & $0.9971$ \\
$\hat\sigma_f^2$ & $0.3269$ & $0.3333\dots$ & $0.9927$ \\
$\hat\sigma_\xi^2$ & $0.3041$ & $0.3333\dots$ & $0.9911$ \\
$\hat\bb$ & $0.7754$ & $0.6666\dots$ & $0.9090$ \\
\hline
\end{tabular}
\captionof{table}{Estimated and theoretical (negative) slopes for $\nu=1/2$. \(\partial_{\hat\bb}\widehat{\mcR}_n\) was challenging to fit numerically because of the extreme risk insensitivity.}
\label{tab:slopes_derivatives_d2}
\end{minipage}
\end{figure}

\paragraph{Calibration Effectiveness in $NNGP$.} While the \emph{post-hoc} calibration in \eqref{eq:calibration} has proved highly effective for $GPnn$ on real-world data (see Section~\ref{sec:real_world} and \cite{GPnn23}), its effectiveness for $NNGP$ remains to be established. Here we provide initial supporting evidence in our synthetic-data $NNGP$ experiment. We consider the matched-$\nu$ and mismatched-parameter setting with regression hyperparameters $\nu=1/2$, $\hat\bb=(1/2,1/2)$, $\hat\ell=1.5\sqrt{2}$, $\hat\sigma_f^2=1.5$, and $\hat\sigma_\xi^2=0.2$ (generative response hyperparameters as specified in \eqref{eq:response_hyperparams}), and first compute the empirical $\widehat{CAL}$ and $\widehat{NLL}$ (see \eqref{measures_experimental}). We then recalibrate $\hat\sigma_f^2$ and $\hat\sigma_\xi^2$ on a held-out calibration set of size $n_{\mathcal C}=2000$, which rescales both parameters by a common factor $\alpha$. As shown in Table~\ref{tab:nngp_calibration}, this post-hoc correction improves both measures substantially, driving $\widehat{CAL}$ close to its optimal value $1$ and remarkably reducing $\widehat{NLL}$ when evaluated on an independent test set of the size $n_{test}=8000$.
\begin{table}[h!]
\centering
\renewcommand{\arraystretch}{1.3}
\setlength{\tabcolsep}{3pt}
\begin{tabular}{c | cc | cc | cc}
\hline
$n$ & \multicolumn{2}{c|}{$10^4$} & \multicolumn{2}{c|}{$10^5$} & \multicolumn{2}{c}{$10^6$} \\
& before & after & before & after & before & after \\
\hline
\small{$\widehat{CAL}$} & \footnotesize{$4.90\pm 0.01$} & \footnotesize{$1.036\pm 0.002$} & \footnotesize{$9.84\pm0.01$} & \footnotesize{$1.025\pm 0.001$} & \footnotesize{$20.17\pm0.02$} & \footnotesize{$1.040\pm0.001$} \\
\small{$\widehat{NLL}$} & \footnotesize{$1.494\pm 0.004$} & \footnotesize{$0.355\pm 0.006$} & \footnotesize{$3.563\pm 0.003$} & \footnotesize{$0.2989\pm 0.0006$} & \footnotesize{$8.34\pm0.01$} & \footnotesize{$0.2755\pm0.0003$} \\
\hline
\end{tabular}
\caption{Empirical $\widehat{CAL}$ and $\widehat{NLL}$ in the synthetic-data $NNGP$ experiment before and after the \emph{post-hoc} calibration. Results show excellent effectiveness of the calibration procedure, even in a strongly mismatched hyperparameter setting. The errors are obtained by calculating the STD over $4$ independent realisations of the training set.}
\label{tab:nngp_calibration}
\end{table}

\section{Summary and Conclusions}

This paper studies nearest-neighbour Gaussian process regression in the $GPnn$ and $NNGP$ settings. We characterise the asymptotic behaviour of the main predictive criteria considered here, establish approximate and universal consistency in risk, and derive convergence rates for the predictive $L_2$-risk. In the regime covered by the theory, these rates match Stone's minimax rate with the nearest-neighbour size chosen appropriately.

We also show that the predictive risk becomes asymptotically insensitive to the hyper-parameters: the $MSE$ converges uniformly over compact hyper-parameter sets, and the corresponding risk derivatives vanish asymptotically. This provides a theoretical explanation for the flattening of the risk landscape observed in large-scale experiments.

The theoretical results are supported by real-data and synthetic experiments, which show that $GPnn$ remains competitive in predictive performance and uncertainty quantification while substantially reducing computational cost relative to strong scalable baselines. Overall, the results show that nearest-neighbour GP regression is both computationally scalable and statistically principled, making $GPnn$ and $NNGP$ attractive large-scale alternatives to full Gaussian process regression.


\acks{We would like to thank His Majesty’s Government for fully funding Tomasz Maciazek and for contributing toward Robert Allison’s funding during the course of this work.  We also thank IBM Research and EPSRC for supplying iCase funding for Anthony Stephenson. This work was carried out using the computational facilities of the Advanced Computing Research Centre, University of Bristol - \url{http://www.bristol.ac.uk/acrc/}.}

\bibliography{gpnn}

\appendix

\renewcommand{\thesection}{\Alph{section}}

\makeatletter
\@ifundefined{theHsection}{}{\renewcommand{\theHsection}{\thesection}}
\@addtoreset{theorem}{section}
\renewcommand{\thetheorem}{\thesection.\arabic{theorem}}
\@ifundefined{theHtheorem}{}{\renewcommand{\theHtheorem}{\thetheorem}}
\@ifundefined{c@definition}{}{%
  \@addtoreset{definition}{section}%
  \renewcommand{\thedefinition}{\thesection.\arabic{definition}}%
  \@ifundefined{theHdefinition}{}{\renewcommand{\theHdefinition}{\thedefinition}}%
}
\@ifundefined{c@remark}{}{%
  \@addtoreset{remark}{section}%
  \renewcommand{\theremark}{\thesection.\arabic{remark}}%
  \@ifundefined{theHremark}{}{\renewcommand{\theHremark}{\theremark}}%
}
\@addtoreset{equation}{section}
\renewcommand{\theequation}{\thesection.\arabic{equation}}
\@ifundefined{theHequation}{} {\renewcommand{\theHequation}{\theequation}}
\makeatother

\section{Preliminaries, Notation and Assumptions Recap}
\label{app:preliminaries}
Before starting we recap the main notation, definitions and assumptions.

\paragraph{Notation for Random Variables} We denote the covariate domain space by $\Omega_\mcX\subset\RR^{d_\mcX}$ and a single covariate (random variable) by calligraphic $\mcX$. Similarly, single response variable is denoted by calligraphic $\mcY\in\RR$. The covariate/response distributions are denoted by $P_\mcX$ and $P_\mcY$ and their joint distribution is $P_{\mcX,\mcY}$. The random variables defined as i.i.d. samples of size $n$ of covariate-response pairs are denoted by uppercase boldface letters $(\bX_n,\bY_n)$, where $\bX_n=(\mcX_1,\dots,\mcX_n)$ and $\bY_n=(\mcY_1,\dots,\mcY_n)$. Single data realisations are denoted by lowercase letters. A realisation of $\mcX$ is $\bx\in\RR^{d_\mcX}$ and a relisation of $\mcY$ is $y$. An observed covariate sample is $X_n=(\bx_1,\dots,\bx_n)$ (a matrix of size $n\times d$) and an observed response sample is the vector $\by_n=(y_1,\dots,y_n)$. Then, the regression function can be written as $f(\bx)=\EE[\mcY|\mcX=\bx]$. Similarly, we denote the noise random variable as $\Xi$, it's single realisation as $\xi$ and a sample vector of length $n$ is $\bxi_n$. Any lowercase boldface characters will always denote vectors.

\paragraph{$GPnn$ Response Model.} In $GPnn$ \citep{GPnn23} We assume that the response variables are generated as
\begin{equation}\label{eq:gpnn_responses_app}
\mcY_i = f\left(\mcX_i\right) + \Xi_i,\quad i=1,\dots, n.
\end{equation}

\paragraph{$NNGP$ Response Model.} The responses $\mcY$ are assumed to be generated according to
\begin{equation}\label{eq:nngp_responses_app}
\mcY_i = \bt\left(\mcX_i\right)^T.\bb+w\left(\mcX_i\right)+\Xi_i,
\end{equation}
where $\bb\in\RR^{d_T}$ is the vector of regression coefficients, $\Xi_i$ is the independent and identically distributed noise and $w(\bx)$ is a sample path drawn from a $GP$ with mean-zero and covariance function $\tilde{k}:\, \RR^{d_\mcX}\times \RR^{d_\mcX}\to\RR$. The role of $w(\bx)$ is to model the effect of unknown/unobserved spatially-dependent covariates. 

\paragraph{$GPnn$/$NNGP$ Estimators.}

We fix a continuous symmetric and positive definite kernel function $c:\, \RR^d\times \RR^d\to \RR$ normalised so that $c(\bx,\bx)=1$ and which determines the exact form of the $GPnn$ estimator. Consider a sequence of $n$ training points $X_n=(\bx_1,\dots,\bx_n)$ together with their response values $\by_n=(y_1,\dots,y_n)$, and a test point $\bx_*$. Let $\mcN_m(\bx_*,X_n)$ be the set of $m$-nearest neighbours of $\bx_*$ in $X_n$. Let $X_{\mcN}(\bx_*)=(\bx_{n,1}(\bx_*),\dots,\bx_{n,m}(\bx_*))$ be the sequence of the $m$-nearest neighbours of $\bx_*$ ordered increasingly according to their distance from $\bx_*$ (we assume that ties occur with probability zero) and let $\by_{\mcN}$ be their corresponding responses. Given the  hyperparameters: $\hat\sigma_f^2>0$ (the kernelscale), $\hat\sigma_\xi^2\geq 0$ (the noise variance) and $\hat\ell>0$ (the lengthscale) we define the (shifted) Gram matrix for $m$-nearest neighbours of $\bx_*$ as
\begin{equation}\label{eq:kernel_def_app}
\left[K_\mcN\right]_{ij}:= \hat\sigma_f^2\, c(\bx_{n,i}/\hat\ell,\bx_{n,j}/\hat\ell)+\hat\sigma_\xi^2\,\delta_{ij},\quad \left[\bk_{\mcN}^*\right]_j:=\hat\sigma_f^2\, c(\bx_*/\hat\ell,\bx_{n,j}/\hat\ell),\quad 1\leq i,j\leq m,
\end{equation}
where $\delta_{ij}$ is the Kronecker delta. In $GPnn$ the predicted mean and variance of the distribution of the response $\hat y_*$ at $\bx_*$ is normally are given by the standard $GP$ regression formulae \citep{RW05}. 
\begin{equation}\label{eq:gpnn_est_app}
\mu_{GPnn}={\bk_{\mcN}^*}^T\,K_\mcN^{-1}\,\by_{n,m},\quad {\sigma_\mcN^*}^2=\hat\sigma_\xi^2 +  \hat\sigma_f^2-{\bk_{\mcN}^*}^T\,K_\mcN^{-1}\bk_{\mcN}^*.
\end{equation}
The asymptotically unbiased counterparts of the $GPnn$ estimator reads
\begin{equation}\label{eq:gpnn_unbiased_est_app}
{\tilde\mu}_{GPnn}(\bx_*)=\Gamma\,{\bk_{\mcN}^*}^T\,K_\mcN^{-1}\,\by_{\mcN},\quad \Gamma=\frac{\hat\sigma_\xi^2+m\hat\sigma_f^2}{m\hat\sigma_f^2}.
\end{equation}
The variance of the hyperparameter-conditional predictive distribution in $NNGP$ is the same as the predictive variance in $GPnn$, while the predictive mean is given by the following formula 
\begin{equation}\label{eq:nngp_unbiased_est_app}
{\tilde\mu}_{NNGP}(\bx_*)=\bt_*^T.\hat\bb+\Gamma\, {{k}_{\mcN}^*}^T\, K_{\mcN}^{-1}\left(\by_\mcN-T_\mcN.\hat\bb\right),
\end{equation}
where we have adjusted the version given in \cite{nngp2} by incorporating the factor $\Gamma$ thereby ensuring asymptotic unbiasedness when $m$ is fixed. $T_\mcN$ is the $m\times d_T$-matrix of regressors at the nearest-neighbours $\left(\bt\left(\bx_{n,1}(\bx_*)\right),\dots, \bt\left(\bx_{n,m}(\bx_*)\right)\right)$ and $\bt_*:=\bt(\bx_*)$. Table \ref{tab:predictive_summary_app} summarises the $GPnn$ and $NNGP$ setups described above.
\begin{table}
\begin{center}
\begin{tabular}{ c  | c c c }
 & Response Model & Predictive Mean & Predictive Variance \\ 
 \hline
 $GPnn$ & \footnotesize{$f(\bx) + \xi(\bx)$} &  \footnotesize{$\Gamma\,{\bk_{\mcN}^*}^T\,K_\mcN^{-1}\,\by_{\mcN}$} & \multirow{ 2}{*}{\footnotesize{$\hat\sigma_\xi^2 +  \hat\sigma_f^2-{\bk_{\mcN}^*}^T\,K_\mcN^{-1}\bk_{\mcN}^*$} } \\  
 $NNGP$ &  \footnotesize{$\bt(\bx)^T.\bb+w(\bx)+\xi(\bx)$} &  \footnotesize{$\bt_*^T.\hat\bb+\Gamma\, {{k}_{\mcN}^*}^T\, K_{\mcN}^{-1}\left(\by_\mcN-T_\mcN.\hat\bb\right)$} &  
\end{tabular}
\end{center}
\caption{Summary of response models, predictive mean and variance in $GPnn$ and $NNGP$. See the main text for the explanation of the symbols. The predictive means are corrected by the coefficient $\Gamma=\frac{\hat\sigma_\xi^2+m\hat\sigma_f^2}{m\hat\sigma_f^2}$ making them unbiased when $m$ is fixed and in the limit $n\to\infty$, as opposed to standard formulas used in the literature.}
\label{tab:predictive_summary_app}
\end{table}

\subsection{$L_2$-risk, Universal Consistency and Stone's Optimal Convergence Rates}
In the task of estimating (noiseless) $f(\bx_*)$ in the generative model \eqref{eq:gpnn_responses_app} given noisy data $(X_n,\by_n)$ we denote the estimated value of $f$ at test point $\bx_*$ as $\hat f_n(\bx_*)$, where the subscript $n$ refers to the size of the training dataset. Assume that the training data are i.i.d. samples from the distribution $P_{\mcX,\mcY}$.

The $L_2(P_\mcX)$-risk (which we simply call \textit{risk} throughout the paper) is defined as
\begin{equation}\label{def:risk_app}
\mcR\left(\hat f_n\right):=\EE\left[\int\left(\hat f_n(\bx)-f(\bx)\right)^2dP_\mcX(\bx)\right],
\end{equation}
where the inner integral is taken over the test data given a training sample $X_n,\by_n$ and it can be viewed as the squared $L_2(P_\mcX)$-distance between $\hat f_n$ and $f$. The outer expectation is taken over all the training samples of size $n$ coming from $P_{\mcX,\mcY}^n$. Similarly, we can define an $L_2(P_\mcX)$-risk directly using the observed noisy responses (rather than the exact values of $f$) which is more applicable to the $GPnn$ and $NNGP$ response models \eqref{eq:gpnn_responses_app} and \eqref{eq:nngp_responses_app} as follows.
\begin{equation*}
\mcR_Y\left(\hat f_n\right):=\EE\left[\int\left(\hat f_n(\bx)-y\right)^2dP_{\mcX,\mcY}(\bx,\by)\right].
\end{equation*}
In our noise model specified in the assumption (AC.\ref{a_xi_app}) the above two $L_2(P_\mcX)$-risk measures differ by an additive constant, i.e.,
\[
\mcR_Y\left(\hat f_n\right)=\mcR\left(\hat f_n\right)+\int\sigma_\xi^2\left(\mcX\right)dP_\mcX(\bx),
\]
where $\sigma_\xi^2(\mcX)$ is the variance of the noise variable $\Xi$ at $\mcX$.

We say that the estimator $\hat f_n(\bx_*)$ is \textit{universally consistent} with respect to a family of training data distributions $\mcD$ if it satisfies the following conditions.
\begin{definition}[Universal Consistency]
    A sequence of regression function estimates $(\hat f_n)$ is universally consistent with respect to $\mcD$ if for all distributions $P_{\mcX,\mcY}\in \mcD$ we have 
    \[
    \mcR\left(\hat f_n\right)\xrightarrow{n\to\infty}0.
    \]
\end{definition}
 In this work, we study nearest-neighbour-based estimators which are indexed by $n$ (the training data size) and $m$ (the number of nearest-neighbours). There, we also distinguish the notion of approximate universal consistency.
\begin{definition}[Approximate Universal Consistency]
    A sequence of nearest - neighbour regression function estimates $(\hat f_{n,m})$ is approximately universally consistent with respect to $\mcD$ if for all distributions $P_{\mcX,\mcY}\in \mcD$ we have 
    \[
    \inf_{m\in\NN}\lim_{n\to\infty}\mcR\left(\hat f_{n,m}\right)=0.
    \]
\end{definition}
\cite{Stone_rates} found the best possible minimax rate at which the risk of a universally consistent estimator $\hat f_n$ can tend to zero with $n$. More precisely, denote $\mathcal{D}_q$ the class of distributions of $(\mcX,\mcY)$ where $\mcX$ is uniformly distributed on the unit hypercube $[0,1]^d$ and $\mcY = f(\mcX)+\Xi$ with some $q$-smooth function $f:\,\RR^d\to\RR$ and the noise variable $\Xi$ is drawn from the standard normal distribution independently of $\mcX$. Function $f$ is $q$-smooth if all its partial derivatives of the order $\lfloor q\rfloor$ exist and are $\beta$-H\"{o}lder continuous with $\beta = q - \lfloor q\rfloor$ with respect to the Euclidean metric on $\RR^d$. Stone showed that there exists a positive constant $\mathcal{C}>0$ such that
\[
\lim_{n\to\infty}\,\inf_{\hat f_n}\,\sup_{P\in \mathcal{D}_q} \mathcal{P}_P\left[\int \left(\hat f_n(\bx)-f(\bx)\right)^2dP_\mcX >\mathcal{C}\, n^{-\frac{2q}{2q+d}}\right]=1,
\]
where the outer probability is taken with respect to the training data samples coming from the product distribution $P^n$. This means that the best universally achievable risk cannot decay faster than $\mathcal{O}\left(n^{-\frac{2q}{2q+d}}\right)$.  In this work, we prove that $GPnn$ and $NNGP$ achieve the optimal minimax convergence rates when $0<q\leq 1$ and provide experimental evidence that $GPnn$ can achieve these rates also when $q>1$.

\subsection{Performance Metrics}
Let $\hat f_n$ be equal to ${\tilde\mu}_{GPnn}$ or ${\tilde\mu}_{NNGP}$ defined in \eqref{eq:gpnn_unbiased_est_app} and \eqref{eq:nngp_unbiased_est_app}. Define the following metrics.
\begin{align}
\begin{split}
se(y_*,\by_n) & :=\left(y_* - \hat f_n(\bx_*)\right)^2, \quad cal(y_*,\by_n) := \frac{\left(y_* - \hat f_n(\bx_*)\right)^2}{ {\sigma_\mcN^*}^2},\\
nll(y_*,\by_n) & := \frac{1}{2}\left(\log\left( {\sigma_\mcN^*}^2\right)+\frac{\left(\by_* - \hat f_n(\bx_*)\right)^2}{ {\sigma_\mcN^*}^2}+\log2\pi \right).
\end{split}
\end{align}
We focus on the above performance metrics averaged over the noise component i.e. we treat the training set $\bX_n$ and the test point $\mcX_*$ as given and define the respective conditional expectations over the test response $\mcY_*$ and the training responses $\bY_n$ as follows.
\begin{align}
& MSE := \EE\left[ se(\mcY_*,\bY_n)\mid \mcX_*,\bX_n \right], \label{def:mse_app}
\\ 
& CAL :=  \EE\left[ cal(\mcY_*,\bY_n)\mid \mcX_*,\bX_n \right] = \frac{MSE}{ {\sigma_\mcN^*}^2}, \label{def:cal_app}
\\
& NLL :=  \EE\left[ nll(\mcY_*,\bY_n)\mid \mcX_*,\bX_n  \right] = \frac{1}{2}\left(\log\left( {\sigma_\mcN^*}^2\right)+CAL+\log2\pi \right). \label{def:nll_app}
\end{align}
We use the following shorthand notation for the $MSE$ derivatives. For each $\phi\in\{\hat\sigma_\xi^2,\hat\sigma_f^2,\hat\ell,\hat\bb\}$ we define
\begin{equation}\label{def:derivatives_shorthand_app}
D_\phi(\mcX_*,\bX_n):=
\begin{cases}
\left|\partial_\phi MSE(\mcX_*,\bX_n)\right|, & \phi\in\{\hat\sigma_\xi^2,\hat\sigma_f^2,\hat\ell\}, \\
\left\|\nabla_{\hat\bb} MSE_{NNGP}(\mcX_*,\bX_n)\right\|_2, & \phi=\hat\bb.
\end{cases}
\end{equation}

\subsection{Assumptions Recap}
Below, we list all the assumptions that are used throughout the proofs. Note that the assumption differ between the theorems/proofs, so use the list below as a lookup-list when reading the proofs.

\paragraph{Assumptions Related to Consistency.}
\begin{enumerate}[({AC.}1)]
\item The training covariates $\{\mcX_i\}_{i=1}^n$ and the test covariate $\mcX_*$ are i.i.d. distributed according to the probability measure $P_\mcX$ on $\RR^{d_\mcX}$. \label{a_X_app}
\item The nearest neighbours are chosen according to the kernel-induced metric $\rho_c$. \label{a_nn_app}
\item The function $f$ in the $GPnn$ response model \eqref{eq:gpnn_responses_app} and the functions $t_i$, $i=1,\dots,d_T$ in the $NNGP$ response model \eqref{eq:nngp_responses_app} are continuous almost everywhere with respect to the kernel-induced (pseudo)metric $\rho_c$ and are integrable i.e., they are measurable and satisfy
\[\int f(\bx) dP_\mcX(\bx)<\infty,\quad \int t_i(\bx) dP_\mcX(\bx)<\infty.\]
\label{a_f_app}
\item The noise $\Xi$ is heteroscedastic with mean zero and 
\[
  \EE[\Xi_i \mid \mcX_i] = 0, \qquad
  \EE[\Xi_i^2 \mid \mcX_i] = \sigma_\xi^2(\mcX_i),
\]
for some function $\sigma_\xi^2 : \Omega_\mcX \to \RR_{>0}$ and the noise random variables are uncorrelated given the covariates, i.e.,
\[
\cov\left[\Xi_i,\Xi_j\mid\mcX_i,\mcX_j\right]=0\quad \mathrm{for}\quad i\neq j,\quad \cov\left[\Xi_*,\Xi_i\mid\mcX_*,\mcX_i\right]=0.
\]
In the $NNGP$ response model \eqref{eq:nngp_responses} we also assume that $\{\Xi_i\}\cup\{\Xi_*\}$ are independent of the sample path $w(\cdot)$. We further assume that the variance function $\sigma_\xi^2(\cdot)$ is almost continuous with respect to the kernel metric $\rho_c$ and is an integrable function of $\bx$ i.e.,
\[\int \sigma_\xi^2(\bx)dP_\mcX(\bx)\le \infty.\]
\label{a_xi_app}
\item The covariance function of the $GP$ sample paths generating the $NNGP$ responses \eqref{eq:nngp_responses} satisfies $\tilde k(\bx,\bx)=\sigma_w^2$ for all $\bx\in\RR^{d_\mcX}$. Define $\tilde c(\cdot,\cdot) := \tilde k(\cdot,\cdot)/\sigma_w^2$. The (pseudo)metrics $\rho_{c}$ and $\rho_{\tilde c}$ are equivalent.
\label{a_metr_app}
\end{enumerate} 

\paragraph{Assumptions Related to Convergence Rates.}
\begin{enumerate}[({AR.}1)]
\item The (normalised) GP kernel function is an isotropic and a strictly decreasing function of the Euclidean distance, i.e., 
\[
c(\bx,\bx')\equiv c\left(r\right),\quad r=\left\|\bx-\bx'\right\|_2,\quad c(r_1)< c(r_2)\quad\mathrm{if}\quad r_1> r_2.
\] \label{aR_iso_app}
\item There exist constants $L_c> 0$ and $0<p\leq 1$ such that the (isotropic and normalised) $GP$ kernel function $c:\, \RR^{d_\mcX}\times \RR^{d_\mcX}\to\RR_{\geq0}$ used in the $GPnn/NNGP$ estimators \eqref{eq:gpnn_unbiased_est_app} and \eqref{eq:nngp_unbiased_est_app} is lower bounded as
\[
c(r)\geq 1-L_c\,r^{2p}.
\]
\label{aR_c_app}
\item The normalised covariance function of the $GP$ sample paths that generate the $NNGP$ responses \eqref{eq:nngp_responses_app} satisfies
\begin{equation}
\tilde c\left(\bx,\bx'\right)\geq 1-L_{\tilde c}\, \left\|\bx-\bx'\right\|_2^{2 q_0}, \quad L_{\tilde c}>0.
\end{equation}
\label{aR_c2_app}
\item The function $f$ in the $GPnn$ response model \eqref{eq:gpnn_responses_app} is bounded in absolute value by some constant $\infty>B_f\geq 1$ and is $q$-H\"{o}lder-continuous, i.e., there exist constants $1\leq L_f<\infty$ and $0<q\leq 1$ such that for every $\bx,\bx'$
\[|f(\bx)-f(\bx')|\leq L_f \left\|\bx-\bx'\right\|_2^q.\]
Each function $t_i$, $i=1,\dots,d_T$ in the $NNGP$ response model \eqref{eq:nngp_responses_app} is bounded and $q_i$-H\"{o}lder continuous for, i.e.,
\[\left|t_i(\bx)\right|\leq B_T<\infty,\quad \left|t_i(\bx)-t_i(\bx')\right|\leq L_i\|\bx-\bx'\|_2^{q_i},\quad i\in\{1,\dots,d_T\}\]
with $0<q_i\leq 1$ and $1\leq L_i<\infty$.
 \label{aR_f_app}
\item There exists $\beta>\frac{4d\,(p+\alpha)}{d-4(p+\alpha)}$ for which $\EE\left[\|\mcX\|_2^\beta\right]<\infty$ under the probability distribution $P_\mcX$ on $\RR^{d_\mcX}$ with $d>4(p+\alpha)$ where $\alpha=\min\{q,p\}$ for $GPnn$ and $\alpha=\min\{q_0,q_1,\dots,q_{d_T},p\}$ for $NNGP$ with $p$ defined in (AR.\ref{aR_c}). \label{aR_X_app}
\item The noise is homoscedastic, i.e., the noise $\Xi_i$ in $GPnn$ responses \eqref{eq:gpnn_responses_app} and $NNGP$ responses \eqref{eq:nngp_responses_app} is i.i.d. from the probability distribution $P_\xi$ with mean zero and fixed variance $\sigma_\xi^2<\infty$.
\label{aR_xi_app}
\end{enumerate}

\paragraph{Assumptions Related to $MSE$ Derivatives.}
\begin{enumerate}[({AD.}1)]
\item The  normalised kernel function $c(\cdotp)$ is isotropic and such that $c(u)$ is differentiable for $u> 0$, the limit $\lim_{u\to0^+} c'(u)$ exists (but may not be finite), and $0\leq c(u)\leq 1$ for all $u\geq 0$, and $c(0)=1$. \label{aD_iso_app}
\item The normalised kernel function $c(u)$ is differentiable and satisfies for all $u\geq 0$
\[\left|u\frac{d c(u)}{du}\right|\leq B_c,\quad \left|u \frac{d c(u)}{du}\right|\leq L_c'\, u^{2p'}\]
for some $B_c,L_c'\geq 1$, and $0<p'\leq 1$. \label{aD_bnd_app}
\end{enumerate}
 
 \section{Some Key Matrix Inequalities}\label{appendix:matrix_inequalities}
In this section, we review and provide several generalisations of matrix inequalities relating to the sensitivity of linear equation systems under perturbations which can be found in \citep{GolubBook}. The proofs provided below are almost taken verbatim from the proofs of Lemma~2.6.1 and Theorem~2.6.2 in \citep{GolubBook}. We subsequently apply these results to derive some key inequalities that involve Gram matrices used to prove the main results of this paper.

Below, $||\cdotp ||$ denotes arbitrary matrix norm as well as its compatible column vector norm. The condition number $\kappa(A)$ for  $A\in \RR^{n\times n}$ pertaining to the norm  $||\cdotp ||$ is defined as 
\[\kappa(A):=||A||\, ||A^{-1}||.\]

\begin{lemma}\label{lemma:matrix_ineqs}
Suppose 
\[A\bx = \bb, \quad A\in \RR^{n\times n},\, \mathbf{0}\neq \bb\in\RR^{n\times 1},\]
\[(A+\Delta A)\by = \bb+\Delta\bb,\quad \Delta A\in \RR^{n\times n},\, \Delta \bb\in\RR^{n\times 1},\]
with $||\Delta A||\leq\epsilon_A||A||$ and $||\Delta\bb||\leq \epsilon_b ||\bb||$ for some $\epsilon_A, \epsilon_b> 0$ such that $\epsilon_A\, \kappa(A)<1$. Define
\[r_A:= \epsilon_A\, \kappa(A),\quad r_b:= \epsilon_b\, \kappa(A).\]
Then, $A+\Delta A$ is nonsingular and
\begin{equation}\label{matrix_ineq1}
\frac{||\by||}{||\bx||}\leq \frac{1+r_b}{1-r_A},
\end{equation}
\begin{equation}\label{matrix_ineq2}
\frac{||\by-\bx||}{||\bx||}\leq \frac{r_A+r_b}{1-r_A}.
\end{equation}
\end{lemma}
\begin{proof}
The matrix $A+\Delta A$ is nonsingular due to Theorem 2.3.4 in \cite{GolubBook} and the fact that $||A^{-1}\Delta A||\leq \epsilon_A||A^{-1}||\, ||A||=r_A<1$.

In order to prove the second part of this Lemma, we first note the equality
\[(\bone+A^{-1}\Delta A)\by = \bx+A^{-1}\Delta\bb.\]
Using the above equality and the fact that $||(\bone-F)^{-1}||\leq(1-||F||)^{-1} $ \citep[Lemma~2.3.3]{GolubBook}, we find
\[||\by||\leq \left(1-||A^{-1}\Delta A||\right)^{-1}\,\left(||\bx||+\epsilon_b||A^{-1}||\,||\bb||\right)\leq\frac{||\bx||+\epsilon_b||A^{-1}||\,||\bb||}{1-r_A}.\]
Finally, using the fact that $\epsilon_b ||A^{-1}|| = r_b ||A||^{-1}$ and $||\bb||\leq ||A|| ||\bx||$ we get the inequality \eqref{matrix_ineq1}.

In order to prove inequality \eqref{matrix_ineq2}, we first note that
\[\by-\bx = A^{-1}\Delta\bb-A^{-1}\Delta A\by.\]
Thus,
\[||\by-\bx||\leq \epsilon_b||A^{-1}||\, ||\bb||+\epsilon_A ||A^{-1}||\, ||A||\, ||\by|| = r_b\frac{||\bb||}{||A||}+r_A ||\by||\leq r_b||\bx||+r_A ||\by||.\]
It follows that
\[\frac{||\by-\bx||}{||\bx||}\leq r_b + r_A\frac{||\by||}{||\bx||}\leq r_b + r_A \frac{1+r_b}{1-r_A},\]
where in the last step we applied inequality \eqref{matrix_ineq1}. The result is equivalent to \eqref{matrix_ineq2}.
\end{proof}

\begin{corollary}\label{lemma:matrixT_ineqs}
By taking the transpose of all the equations from Lemma \ref{lemma:matrix_ineqs}, we obtain the following result. Suppose 
\[\bx^T \tilde A = \bb^T, \quad \tilde A\in \RR^{n\times n},\, \mathbf{0}\neq \bb\in\RR^{n\times 1},\]
\[\by^T (\tilde A+\Delta \tilde A) = \bb^T+\Delta\bb^T,\quad \Delta \tilde A\in \RR^{n\times n},\, \Delta \bb\in\RR^{n\times 1},\]
with $||\Delta \tilde A||\leq\tilde\epsilon_A||\tilde A||$ and $||\Delta\bb^T||\leq \tilde\epsilon_b ||\bb^T||$ for some $\tilde\epsilon_A, \tilde\epsilon_b> 0$ such that $\tilde\epsilon_A\, \kappa(\tilde A)<1$. Define
\[\tilde r_A:= \tilde\epsilon_A\, \kappa(\tilde A),\quad \tilde r_b:= \tilde\epsilon_b\, \kappa(\tilde A).\]
Then, $\tilde A+\Delta \tilde A$ is nonsingular and
\begin{equation}\label{matrixT_ineq1}
\frac{||\by^T||}{||\bx^T||}\leq \frac{1+\tilde r_b}{1-\tilde r_A},
\end{equation}
\begin{equation}\label{matrixT_ineq2}
\frac{||\by^T-\bx^T||}{||\bx^T||}\leq \frac{\tilde r_A+\tilde r_b}{1-\tilde r_A}.
\end{equation}
\end{corollary}

Let us next move to applying Lemma \ref{lemma:matrix_ineqs} and Corollary \ref{lemma:matrixT_ineqs} to derive useful inequalities involving Gram matrices.
\begin{lemma}\label{lemma:K_ineqs}
Assume $\hat\sigma_f> 0$ and fix $\bx_*\in\RR^d$, $X_n$ - the training dataset, the kernel function $c(\cdotp,\cdotp)$ and $m$ - the number of nearest neighbours of $\bx_*$ in $X_n$ selected according to the kerne-induced metric $\rho_c$. Denote the nearest-neighbour (shifted) Gram matrix as $K_\mcN$. Assume that $K_\mcN$ is invertible and define $K^\infty_{\mcN}$ as in Lemma \ref{lemma:kernel_limits}. Then, we have
\begin{align}
& \frac{\left\| K_{\mcN}^{-1}\, \mathbf{k}^*_{\mcN}- \hat\sigma_f^2\left(K^\infty_{\mcN}\right)^{-1}\onev\right\|_1}{\left\|\hat\sigma_f^2\left(K^\infty_{\mcN}\right)^{-1}\onev\right\|_1} \leq \frac{\epsilon_m+\epsilon_E}{1-\epsilon_E}, \label{Kk_ineq}
\\
& \frac{\left\| \left(\mathbf{k}^*_{\mcN}\right)^T\,K_{\mcN}^{-1}- \hat\sigma_f^2\,\onev^T\,\left(K^\infty_{\mcN}\right)^{-1}\right\|_1}{\left\|\hat\sigma_f^2\,\onev^T\,\left(K^\infty_{\mcN}\right)^{-1}\right\|_1} \leq \frac{\epsilon_m+\epsilon_E}{1-\epsilon_E}, \label{kTK_ineq}
\end{align}
where  
\begin{align}\label{def:epsilons_E_max}
\epsilon_E := \frac{1}{m}\, \left\|E(\bx_*,X_n)\right\|_1, \quad \epsilon_m := \max_{1\leq i\leq m} \rho_c^2\left(\bx_*,\bx_{n,i}(\bx*)\right).
\end{align}
with $E_{i,j}=\epsilon_{i,j}:=\rho_c^2\left(\bx_{n,i}(\bx*),\bx_{n,j}(\bx*)\right)$, $1\leq i,j\leq m$.
For any function $f:\,\RR^d\to\RR$ that satisfies the $q$-H\"{o}lder condition (AR.\ref{aR_f}) we also have
\begin{equation}\label{Kf_ineq}
\frac{\left\| K_{\mathcal{N}}^{-1}\, f(X)- f\left (\bx_*\right)\left(K^\infty_{\mathcal{N}}\right)^{-1}\onev\right\|_1}{\left\|\left(K^\infty_{\mathcal{N}}\right)^{-1}\onev\right\|_1} \leq B_f \frac{ 2 L_f \min\{d_m^q,1\}+\epsilon_E}{1-\epsilon_E}.
\end{equation}
 What is more,
\begin{equation}\label{kTK2_ineq}
\frac{\left\| \left(\mathbf{k}^*_{\mcN}\right)^T\,K_{\mcN}^{-2}- \hat\sigma_f^2\,\onev^T\,\left(K^\infty_{\mcN}\right)^{-2}\right\|_1}{\left\|\hat\sigma_f^2\,\onev^T\,\left(K^\infty_{\mcN}\right)^{-2}\right\|_1} \leq \frac{\epsilon_m+\epsilon_{E,2}}{1-\epsilon_{E,2}},
\end{equation}
where
\begin{align}
\begin{split}\label{def:epsE2}
\epsilon_{E,2} = \left(\frac{\hat\sigma_f^2}{\hat\sigma_\xi^2+m\hat\sigma_f^2}\right)^2\, \left\|2\frac{\hat\sigma_\xi^2}{\hat\sigma_f^2}\, E+E\,\onev.\onev^T+\onev.\onev^T\,E-E^2\right\|_1.
\end{split}
\end{align}
\end{lemma}
\begin{proof}
The proof of the inequality \eqref{Kk_ineq} follows from the application of Lemma \ref{lemma:matrix_ineqs} with $A = K^\infty_{\mcN}$, $\bb = \lim_{n\to\infty}\mathbf{k}^*_{\mcN}=\hat\sigma_f^2\,\onev$, $A+\Delta A = K_{\mcN}$ and $\bb+\Delta\bb = \mathbf{k}^*_{\mcN}$.

We first calculate the relevant condition number 
\[\kappa(A)=\left\|K^\infty_{\mcN}\right\|_1\, \left\| \left(K^\infty_{\mcN}\right)^{-1} \right\|_1.\]
By a direct calculation using the exact forms of the above matrices from Lemma \ref{lemma:kernel_limits}, we find that
\begin{equation}
\left\|K^\infty_{\mcN}\right\|_1 = \hat\sigma_\xi^2+m\hat\sigma_f^2=\left(\left\| \left(K^\infty_{\mcN}\right)^{-1} \right\|_1\right)^{-1}.
\end{equation}
Thus, $\kappa(A)=1$. To satisfy the conditions of Lemma \ref{lemma:matrix_ineqs}, we set $\epsilon_A = ||\Delta A||_1/||A||_1$ and $\epsilon_b=||\Delta \bb||_1/||\bb||_1$. Let us set $\epsilon_A= r_A$ and $\epsilon_b= r_b$. We have
\[\left[\Delta A\right]_{ij}= \left[K_{\mcN}\right]_{ij} - \left[K^\infty_{\mcN}\right]_{ij} =\hat\sigma_f^2\left(c(\bx_i/\hat\ell,\bx_j/\hat\ell)-1\right) = -\hat\sigma_f^2\epsilon_{i,j},\]
thus we get
\[ \epsilon_A = \frac{\hat\sigma_f^2}{\hat\sigma_\xi^2+m\hat\sigma_f^2}\, \left\|E\right\|_1 \leq  \frac{1}{m}\, \left\|E\right\|_1= \frac{1}{m} \max_{j}\, \sum_{i=1}^m\epsilon_{i,j} <  \frac{1}{m} \max_{j}\, \sum_{i=1}^m 1  = 1. \]
This proves the first part of this Lemma. For the second part, we note that $||\bb||_1 = \hat\sigma_f^2\,||\onev||_1 = m \hat\sigma_f^2$ and  $||\Delta \bb ||_1 = ||\mathbf{k}^*_{\mcN} - \hat\sigma_f^2\,\onev||_1 = \hat\sigma_f^2\sum_{i=0}^m\epsilon_i$, where $\epsilon_i := \rho_c^2\left(\bx_{n,i}(\bx*),\bx*\right)$.
Finally, 
\[\epsilon_b = \frac{1}{m} \sum_{i=0}^m\epsilon_i = \frac{1}{m}\sum_{i=1}^m (1-c(\bx_i/\hat\ell,\bx_*/\hat\ell))  < 1\]
and we note that $\epsilon_b\leq\epsilon_m$.

The proof of the inequality \eqref{kTK_ineq} is fully analogous to the proof of the inequality \eqref{Kk_ineq}. It follows from the application of Corollary \ref{lemma:matrixT_ineqs} with $A = K^\infty_{\mcN}$, $\bb^T =  \lim_{n\to\infty}\left(\mathbf{k}^*_{\mcN}\right)^T=\hat\sigma_f^2\,\onev^T$, $A+\Delta A = K_{\mcN}$ and $\bb^T+\Delta\bb^T =\left(\mathbf{k}^*_{\mcN}\right)^T$.

The proof of \eqref{Kf_ineq} is fully analogous to the proof of \eqref{Kk_ineq} with $A = K^\infty_{\mathcal{N}}$, $\bb = \lim_{n\to\infty}f(X)=f\left (\bx_*\right)\,\onev$, $A+\Delta A = K_{\mathcal{N}}$ and $\bb+\Delta\bb = f(X)$. Lemma \ref{lemma:matrix_ineqs} asserts that 
\[
\frac{\left\| K_{\mathcal{N}}^{-1}\, f(X)- f\left (\bx_*\right)\left(K^\infty_{\mathcal{N}}\right)^{-1}\onev\right\|_1}{|f\left (\bx_*\right)|\left\|\left(K^\infty_{\mathcal{N}}\right)^{-1}\onev\right\|_1} \leq \frac{\epsilon_b+\epsilon_E}{1-\epsilon_E},
\]
where $\epsilon_b = ||\Delta \bb ||_1/||\bb||_1$. Using the H\"{o}lder property and the boundedness of the function $f$, we get 
\begin{gather*}
\epsilon_b =\frac{1}{m |f\left (\bx_*\right)|}  \sum_{i=1}^m |f(\bx_i)-f(\bx_*)|\leq \frac{1}{m |f\left (\bx_*\right)|}\sum_{i=1}^m \min\{L_f\|\bx_i-\bx_*\|_2^q, 2B_f\}  \\
\leq \frac{1}{|f\left (\bx_*\right)|}\ \min\{L_f d_m^q, 2B_f\} \leq  \frac{2L_fB_f}{|f\left (\bx_*\right)|} \min\{d_m^q, 1\} .
\end{gather*}

In order to prove the inequality \eqref{kTK2_ineq} we use Corollary \ref{lemma:matrixT_ineqs} with $\tilde A = \left(K^\infty_{\mcN}\right)^2$, $\bb = \lim_{n\to\infty}\mathbf{k}^*_{\mcN}=\hat\sigma_f^2\,\onev$, $\tilde A+\Delta \tilde A = K_{\mcN}^2$ and $ \bb+\Delta \bb = \mathbf{k}^*_{\mcN}$.
We first calculate the relevant condition number 
\[\kappa(\tilde A)=\left\|\left(K^\infty_{\mcN}\right)^2\right\|_1\, \left\| \left(K^\infty_{\mcN}\right)^{-2} \right\|_1.\]
By a direct calculation using the exact forms of the above matrices using Lemma \ref{lemma:kernel_limits}, we find that
\begin{equation*}
\left\|\left(K^\infty_{\mcN}\right)^2\right\|_1 = \left(\hat\sigma_\xi^2+m\hat\sigma_f^2\right)^2=\left(\left\| \left(K^\infty_{\mcN}\right)^{-2} \right\|_1\right)^{-1}.
\end{equation*}
Thus, $\kappa(\tilde A)=1$. To satisfy the conditions of Corollary \ref{lemma:matrixT_ineqs}, we set $\tilde \epsilon_A = ||\Delta \tilde A||_1/||\tilde A||_1$ and $\tilde \epsilon_b=\left\|\Delta \bb^T\right\|_1/\left\|\bb^T\right\|_1$. Let us set $\tilde \epsilon_A= \tilde r_A =: \epsilon_{E,2}$ and $\tilde \epsilon_b= \tilde r_b$. We have
\begin{align*}
\begin{split}
\left|\left[\Delta \tilde A\right]_{ij}\right|=& \left|\left[K_{\mcN}^2\right]_{ij} - \left[\left(K^\infty_{\mcN}\right)^2\right]_{ij}\right| =2\hat\sigma_\xi^2\left(\hat\sigma_f^2-k_\theta(\bx_i,\bx_j)\right) \\ 
& + \sum_{k=1}^m \left(\hat\sigma_f^4-k_\theta(\bx_i,\bx_k)k_\theta(\bx_j,\bx_k) \right)=2\hat\sigma_\xi^2\hat\sigma_f^2 \,\epsilon_{ij}+ \hat\sigma_f^4\sum_{k=1}^m \left(\epsilon_{ik}+\epsilon_{jk}-\epsilon_{ik}\epsilon_{jk}\right).
\end{split}
\end{align*}
Thus,
\[ \left|\left[\Delta \tilde A\right]_{ij}\right|< 2\hat\sigma_\xi^2\hat\sigma_f^2 + m\hat\sigma_f^4,\]
which implies that 
\[ \epsilon_{E,2} = \frac{1}{\left(\hat\sigma_\xi^2+m\hat\sigma_f^2\right)^2}\, \max_{1\leq i\leq m} \sum_{j=1}^m \left|\left[\Delta \tilde A\right]_{ij}\right|< m \frac{2\hat\sigma_\xi^2\hat\sigma_f^2 + m\hat\sigma_f^4}{\left(\hat\sigma_\xi^2+m\hat\sigma_f^2\right)^2}\leq 1.\]
Finally, note that $\left\|\bb^T\right\|_1 = \hat\sigma_f^2\,\left\|\onev^T\right\|_1 = \hat\sigma_f^2$ and  
\[
||\Delta \bb^T ||_1 = \left\|\left(\mathbf{k}^*_{\mcN}\right)^T - \hat\sigma_f^2\,\onev^T\right\|_1 = \hat\sigma_f^2\max_{1\leq i\leq m}\rho_c^2(\bx_*,\bx_{n,i})\leq \hat\sigma_f^2\, \epsilon_m.
\]
Thus, $\tilde \epsilon_b \leq  \epsilon_m  < 1$.
\end{proof}
\begin{lemma}[Relations between the epsilons.]\label{lemma:epsilons_rels}
Let $\epsilon_{m}$, $\epsilon_E$ and $\epsilon_{E,2}$ be as defined in Equations \eqref{def:epsilons_E_max} and \eqref{def:epsE2} in Lemma \ref{lemma:K_ineqs}. The following bounds hold
\[
\epsilon_E \leq 4\,\epsilon_m,\qquad \epsilon_{E,2}\leq 2\, \epsilon_E.
\]
\end{lemma}
\begin{proof}
Recall the definition of $\epsilon_E$, i.e., $\epsilon_E = \max_{1\leq j\leq m} \sum_{i=1}^m \epsilon_{i,j}$.
The kernel-induced metric satisfies the triangle inequality, i.e., 
\[
\sqrt{\epsilon_{i,j}}=\rho_c(\bx_{n,i}(\bx_*),\bx_{n,j}(\bx_*))\leq \rho_c(\bx_*,\bx_{n,i}(\bx_*)) + \rho_c(\bx_*,\bx_{n,j}(\bx_*))\leq 2\sqrt{\epsilon_{m}}.
\]
By squaring both sides of this inequality we get that $\epsilon_E \leq 4\,\epsilon_m$.

In order to derive the bound for $\epsilon_{E,2}$, we first note that 
\[
\epsilon_{E,2} \leq \frac{1}{m} \frac{\hat\sigma_f^2}{\hat\sigma_\xi^2+m\hat\sigma_f^2}\, \left(2\frac{\hat\sigma_\xi^2}{\hat\sigma_f^2}\, \left\|E\right\|_1+\left\|E\,\onev.\onev^T+\onev.\onev^T\,E-E^2\right\|_1\right).
\]
Next,
\[\left\|E\,\onev.\onev^T+\onev.\onev^T\,E-E^2\right\|_1=\max_{1\leq j\leq m}\sum_{i=1}^m\sum_{k=1}^m\left(\epsilon_{ik}+\epsilon_{jk}- \epsilon_{ik}\epsilon_{jk}\right)\leq m\,\|E\|_1+\sum_{i=1}^m\sum_{k=1}^m\epsilon_{ik},\]
where we have used the facts that
\[\left[E^2\right]_{i,j}=\sum_{k=1}^m\epsilon_{ik}\epsilon_{jk}\geq 0, \quad \max_{1\leq j\leq m}\sum_{i=1}^m\sum_{k=1}^m\epsilon_{jk}=m \max_{1\leq j\leq m}\sum_{k=1}^m\epsilon_{jk}=m\,\|E\|_1.\]
Finally, note that $\sum_i \sum_k\epsilon_{ik}\leq m\,\|E\|_1$. Thus, we obtain
\[
\epsilon_{E,2} \leq \frac{1}{m} \frac{\hat\sigma_f^2}{\hat\sigma_\xi^2+m\hat\sigma_f^2}\, \left(2\frac{\hat\sigma_\xi^2}{\hat\sigma_f^2}\, \left\|E\right\|_1+2m\,\left\|E\right\|_1\right) = \frac{2}{m}\,\left\|E\right\|_1 = 2\, \epsilon_{E}.
\]

\end{proof}

\section{Proving Consistency of $GPnn$ and $NNGP$ Regression}\label{sec:consistency_app}

We first derive the limits of the performance measures defined in Equations \eqref{def:mse_app} - \eqref{def:nll_app} when the training data size $n$ tends to infinity. Firstly, we prove the bias-variance decomposition of the MSE.
\begin{lemma}[Bias-variance decomposition of $MSE$.]\label{lemma:bias_variance} 
Let the test and training covariates $\mcX_*,\mcX_1,\dots,\mcX_n$ be i.i.d. from $P_\mcX$ and let $\bX_n=(\mcX_1,\dots,\mcX_n)$. Assume that the training and test responses are generated as $\mcY_i=g(\mcX_i)+\Xi_i$, where $g$ is a (possibly random) measurable function that is sampled independently of the covariates and $\Xi$ is the random noise which is mean-zero at every location and are uncorrelated given the covariates, i.e.,
\[
\cov\left[\Xi_i,\Xi_j\mid\mcX_i,\mcX_j\right]=0\quad \mathrm{for}\quad i\neq j,\quad \cov\left[\Xi_*,\Xi_i\mid\mcX_*,\mcX_i\right]=0,
\]
and are independent of $g$. Let $\hat \mu=\hat\mu\left(\mcX_*,\bX_n,\by_n\right)$ be an estimator of the test response $\mcY_*$ that depends linearly on the training responses $\by_n$.
We have the following bias-variance decomposition of the corresponding $MSE$
\begin{align}\label{eq:bias_variance}
\begin{split}
MSE := \EE\left[\left(\hat \mu-\mcY_*\right)^2\mid \mcX_*,\bX_n\right]=  \sigma_\xi^2\left(\mcX_*\right)+\mathrm{Bias}^2\left( \mcX_*,\bX_n\right) +  \var\left[\hat\mu-g(\mcX_*)\mid \mcX_*,\bX_n\right],
 \end{split}
\end{align}
where 
\begin{gather*}
 \sigma_\xi^2(\mcX):=\var\left[\Xi\mid \mcX\right],\quad \mathrm{Bias}\left( \mcX_*,\bX_n\right) :=\EE\left[\hat \mu\mid\mcX_*,\bX_n\right]-\EE\left[g(\mcX_*)\mid\mcX_*\right].
\end{gather*}
When applied to $GPnn$ response model \eqref{eq:gpnn_responses_app} we have $g\equiv f$ deterministic, thus
\begin{align}\label{eq:gpnn_bias_variance}
\begin{split}
 \mathrm{Bias}_{GPnn}\left( \mcX_*,\bX_n\right) & = \Gamma\, {\mathbf{k}^*_{\mcN}}^T\, K_{\mcN}^{-1}f(\bX_{\mcN}) - f(\mcX_*)
\\
 \var\left[\tilde\mu_{GPnn}-g\left(\mcX_*\right)\mid \mcX_*,\bX_n\right] & = \var\left[\tilde\mu_{GPnn}\mid \mcX_*,\bX_n\right] = \Gamma^2\, {\mathbf{k}^*_{\mcN}}^T\, K_{\mcN}^{-1}\,\Sigma_{\bxi}\,K_{\mcN}^{-1}\,{\mathbf{k}^*_{\mcN}},
\end{split}
\end{align}
 where $\bX_{\mcN} = \bX_{\mcN}\left(\mcX_*,\bX_n\right)$ is the set of nearest neighbours of $\mcX_*$ in $\bX_n$ and 
 \[
 \Sigma_{\bxi} = \mathrm{diag}\left\{\sigma_\xi^2\left(\mcX_{n,1}\right),\dots,\sigma_\xi^2\left(\mcX_{n,m}\right)\right\}.
 \]
 In the $NNGP$ response model \eqref{eq:nngp_responses_app} we have $g\left(\mcX\right)= \bt\left(\mcX\right)^T.\bb+w\left(\mcX\right)$, thus 
\begin{align}\label{eq:nngp_bias_variance}
\begin{split}
 \mathrm{Bias}_{NNGP}\left( \mcX_*,\bX_n\right) & = \Gamma\,{{k}_{\mcN}^*}^T\, K_{\mcN}^{-1}T_\mcN(\bb-\hat\bb)-\bt_*^T.(\bb-\hat\bb).
\\
\var\left[\tilde\mu_{NNGP}-g(\mcX_*)\mid \mcX_*,\bX_n\right] & =\sigma_w^2+\Gamma^2\, {{k}_{\mcN}^*}^T\, K_{\mcN}^{-1}\left(\tilde K_{\mcN}+\Sigma_\xi\right)K_{\mcN}^{-1} {{k}_{\mcN}^*}
\\
& -2\Gamma\,{{k}_{\mcN}^*}^T\, K_{\mcN}^{-1}\tilde k_\mcN^*,
\end{split}
\end{align}
where $\left[\tilde k_\mcN^*\right]_{i}:=\tilde k\left(\bx_{n,i}\left(\bx_*\right),\bx_*\right)$ and $\left[\tilde K_\mcN\right]_{i,j}:=\tilde k\left(\bx_{n,i}\left(\bx_*\right),\bx_{n,j}\left(\bx_*\right)\right)$, $1\leq i,j\leq m$.
\end{lemma}
\begin{proof}
In the assumed response model, we have
\[
\EE\left[\left(\hat \mu-\mcY_*\right)^2\mid \mcX_*,\bX_n\right] = \sigma_\xi^2\left(\mcX_*\right) + \EE\left[\delta^2\mid \mcX_*,\bX_n\right], \quad \delta:=\hat\mu-g(\mcX_*).
\]
By the definition of variance, we have
\[
 \EE\left[\delta^2\mid \mcX_*,\bX_n\right] =  \bias^2\left( \mcX_*,\bX_n\right) + \var\left[\delta\mid \mcX_*,\bX_n\right]
\]
where $\bias\left( \mcX_*,\bX_n\right):=\EE\left[\delta\mid \mcX_*,\bX_n\right]$. Using the standard identity $\var[A-B]=\var[A]+\var[B]-2\cov[A,B]$, we get
\[
\var\left[\delta\mid \mcX_*,\bX_n\right] = \var\left[\hat\mu\mid\mcX_*,\bX_n\right]+\var\left[g(\mcX_*)\mid\mcX_*,\bX_n\right] -2\cov\left[\hat\mu,g(\mcX_*)\mid \mcX_*,\bX_n\right].
\]
Finally, $\var\left[g(\mcX_*)\mid\mcX_*,\bX_n\right]=\var\left[g(\mcX_*)\mid\mcX_*\right]$, since $g$ is drawn independently of the covariates. This proves Equation \eqref{eq:bias_variance}.

For $GPnn$ we have
\begin{align*}
& \EE\left[\tilde\mu_{GPnn}\mid\mcX_*,\bX_n\right]=\Gamma\,{\mathbf{k}^*_{\mcN}}^T\, K_{\mcN}^{-1}\,\EE\left[\bY_{\mcN}\mid\mcX_*,\bX_n\right] = \Gamma\,{\mathbf{k}^*_{\mcN}}^T\, K_{\mcN}^{-1}f(\bX_{\mcN}),
\\
& \EE\left[\left(\tilde\mu_{GPnn}\right)^2\mid\mcX_*,\bX_n\right] =\Gamma^2\, {\mathbf{k}^*_{\mcN}}^T\, K_{\mcN}^{-1} \, \EE\left[\bY_{\mcN}.\bY_{\mcN}^T\mid\mcX_*,\bX_n\right] \, K_{\mcN}^{-1}\,{\mathbf{k}^*_{\mcN}}.
\end{align*}
Further, 
\begin{align*}
\EE\left[\mcY_{n,i}\mcY_{n,j}\mid\mcX_*,\bX_n\right] & = f(\mcX_{n,i})f(\mcX_{n,j})+\EE\left[2f(\mcX_{n,i})\Xi_{n,j}+\Xi_{n,i}\Xi_{n,j}\mid\mcX_{n,j}\right]
\\
& = f(\mcX_{n,i})f(\mcX_{n,j})+\left[\Sigma_{\bxi}\right]_{ij},
\end{align*}
where in the last equality we have used the noise model assumptions (mean-zero and independent given the nearest-neighbours). 
Thus,
\[
\EE\left[\left(\tilde\mu_{GPnn}\right)^2\mid\mcX_*,\bX_n\right]  =  \left(\EE\left[\tilde\mu_{GPnn}\mid\mcX_*,\bX_n\right]\right)^2 + \Gamma^2\, {\mathbf{k}^*_{\mcN}}^T\, K_{\mcN}^{-1}\,\Sigma_{\bxi}\,K_{\mcN}^{-1}\,{\mathbf{k}^*_{\mcN}}.
\]
Putting together all these above formulae and using the definition of variance yields Equations \eqref{eq:gpnn_bias_variance}.

 In $NNGP$, we have $\EE\left[\bY_{\mcN}\mid\mcX_*,\bX_n\right] =T_{\mcN}\bb$, thus
 \[
\EE\left[\tilde\mu_{NNGP}\mid\mcX_*,\bX_n\right]=\bt_*^T.\hat\bb+\Gamma\, {{k}_{\mcN}^*}^T\, K_{\mcN}^{-1}T_{\mcN}\left(\bb-\hat\bb\right).
 \]
 Thus,
 \begin{align*}
\tilde\mu_{NNGP} - \EE\left[\tilde\mu_{NNGP}\mid\mcX_*,\bX_n\right] = \Gamma\, {{k}_{\mcN}^*}^T\, K_{\mcN}^{-1}\left(\bw\left(\bX_\mcN\right)+\Xi_{\mcN}\right),
 \end{align*}
 and the variance reads
 \begin{align*}
  & \var\left[\tilde\mu_{NNGP} \mid\mcX_*,\bX_n\right] = \Gamma^2\, {{k}_{\mcN}^*}^T\, K_{\mcN}^{-1}\EE\left[\left(\bw\left(\bX_\mcN\right)+\Xi_{\mcN}\right).\left(\bw\left(\bX_\mcN\right)+\Xi_{\mcN}\right)^T\mid\mcX_*,\bX_n\right] \times
  \\
 & \times K_{\mcN}^{-1} {{k}_{\mcN}^*} =  \Gamma^2\, {{k}_{\mcN}^*}^T\, K_{\mcN}^{-1}\EE\left[\bw\left(\bX_\mcN\right).\bw\left(\bX_\mcN\right)^T+\Xi_{\mcN}.\Xi_{\mcN}^T\mid\mcX_*,\bX_n\right]K_{\mcN}^{-1} {{k}_{\mcN}^*}
 \\
 & =  \Gamma^2\, {{k}_{\mcN}^*}^T\, K_{\mcN}^{-1}\left(\tilde K_{\mcN}+\Sigma_\xi\right)K_{\mcN}^{-1} {{k}_{\mcN}^*},
 \end{align*}
 where we have used the fact that the sample path $w(\cdot)$ and the noise variables are independent. This proves Equations \eqref{eq:nngp_bias_variance}. The remaining components of the bias-variance decomposition for $NNGP$ are completely analogous, so we skip them.
\end{proof}
In order to establish the desired $n\to\infty$ limits we will use the following result concerning the shrinking of nearest neighbour sets.
\begin{lemma}{\citep[Lemma~6.1]{Gyorfi2002}}\label{lemma:dmax_limit}
Let $\bX_n=(\mcX_i)_{i=1}^n$ be a sampling sequence of i.i.d. points from the distribution $P_\mcX$ and let $\bx_*\in\supp(P_\mcX)$. Assume that the nearest neighbours are chosen according to the Euclidean distance. Allow the number of nearest neighbours $m$ to change with $n$ so that $\lim_{n\to\infty} m_n/n=0$ (in particular, $m$ can be fixed). Define $d_m$ as the distance of the $m$-th nearest neighbour of $\bx_*$ in $\bX_n$ i.e.,
\[d_m(\bx_*,\bX_n):=\left\|\mcX_{n,m}(\bx_*)-\bx_*\right\|_2,\]
where $\|\cdotp\|_2$ deontes the Euclidean distance in $\RR^d$. Then, $d_m(\bx_*,X_n)\xrightarrow{n\to\infty}0$ with probability one.
\end{lemma}
\begin{remark}\label{remark:epsmax_limit}
Using a fully analogous technique to the one used in the proof of Lemma~6.1 in {\citep[]{Gyorfi2002}}, one can replace the Euclidean metric with the kernel-induced (pseudo)metric $\rho_c$, i.e. choose the nearest-neighbours according to the (pseudo)metric $\rho_c$. In result, we obtain that for every $\bx_*\in\supp_{\rho_c}(P_\mcX)$ we have
\[\epsilon_m(\bx_*,\bX_n):=\rho_c\left(\mcX_{n,m}(\bx_*),\bx_*\right)\xrightarrow{n\to\infty}0\]
\end{remark}
with probability one.
\begin{corollary}\label{coro:limits}
By Remark \ref{remark:epsmax_limit} and the triangle inequality 
\[
\rho_c(\bx_{n,i},\bx_{n,j}) \leq \rho_c(\bx_{n,i},\bx_*)+ \rho_c(\bx_{n,j},\bx_*) \leq 2\rho_c(\bx_{n,m},\bx_*)
\]
we also have $\rho_c(\mcX_{n,i},\mcX_{n,j})\xrightarrow{n\to\infty}0$ with probability one for all $1\leq i,j\leq m$. Consequently, the kernel elements defined in Equation \eqref{eq:kernel_def} have the following limits with probability one
\[
\left[K_\mcN\right]_{ij}\xrightarrow{n\to\infty} \hat\sigma_f^2+\hat\sigma_\xi^2\,\delta_{ij},\quad \left[\bk_{\mcN}^*\right]_j\xrightarrow{n\to\infty}\hat\sigma_f^2,\quad 1\leq i,j\leq m.
\]
\end{corollary}
\begin{lemma}[Gram matrix limits.]\label{lemma:kernel_limits}
Under the assumptions (AC.\ref{a_X_app}-\ref{a_f_app}) with $m\in\NN_{>0}$ fixed and $\bx_*\in\supp_{\rho_c}(P_\mcX)$ fixed the following limits hold with probability one.
\begin{align}
& K_\mcN \xrightarrow{n\to\infty}K_\mcN^\infty:= \hat\sigma_\xi^2\,\bone+ \hat\sigma_f^2\,\onev.\onev^T, \quad \bk_{\mcN}^*\xrightarrow{n\to\infty}\hat\sigma_f^2\,\onev, \label{eq:K_k_imit}
\\
& K_\mcN^{-1}\xrightarrow{n\to\infty}\left(K_\mcN^\infty\right)^{-1}=\frac{1}{\hat\sigma_\xi^2}\left(\bone-\frac{\hat\sigma_f^2}{\hat\sigma_\xi^2+m\hat\sigma_f^2}\,\onev.\onev^T\right), \label{eq:Kinv_imit}
\\
 & f\left(\bX_{\mcN}(\bx_*)\right)\xrightarrow{n\to\infty}f(\bx_*)\,\onev, \label{eq:f_imit}
\end{align}
where $\bone$ is the $m\times m$ identity matrix and $\onev$ is the column vector of ones.
\end{lemma}
\begin{proof}
The limits \eqref{eq:K_k_imit} follow straightforwardly form Corollary \ref{coro:limits}. The fact that $K_\mcN^{-1}\xrightarrow{n\to\infty}\left(K_\mcN^\infty\right)^{-1}$ follows from the continuity of matrix inverse. The RHS of Equation \eqref{eq:Kinv_imit} is calculated using the Sherman–Morrison formula \citep{SM1949}. The limit \eqref{eq:f_imit} follows directly from the assumption (AC.\ref{a_f_app}) stating that $f$ is almost continuous.
\end{proof}
In the next Lemma we show that the estimators defined in Equations \eqref{eq:gpnn_unbiased_est_app} and  \eqref{eq:nngp_unbiased_est_app}  are asymptotically unbiased (given the test point). 
\begin{lemma}[Pointwise limits of bias and variance.]\label{lemma:unbiased_estimator}
Under the assumptions \newline (AC.\ref{a_X_app}-\ref{a_xi_app}) with $m\in\NN_{>0}$ fixed and test point $\bx_*\in\supp_{\rho_c}(P_\mcX)$ fixed the following limit for the predictive noise variance defined in \eqref{eq:gpnn_est_app} holds with probability one.
\begin{align}\label{eq:biased_limits}
\begin{split}
{\sigma_\mcN^*}^2\xrightarrow{n\to\infty}\hat\sigma_\xi^2\,\left(1+\frac{1}{m\Gamma}\right),\quad \Gamma:=\frac{\hat\sigma_\xi^2+m\hat\sigma_f^2}{m\hat\sigma_f^2}.
\end{split}
\end{align}
We also have that the following limits for the estimators $\tilde\mu_{GPnn}$ and  $\tilde\mu_{NNGP}$ defined in \eqref{eq:gpnn_unbiased_est_app}  and \eqref{eq:nngp_unbiased_est_app} hold with probability one for their respective response models.
\begin{align}
\begin{split}\label{eq:gpnn_unbiased_limits}
& \EE\left[\tilde\mu_{GPnn}\mid \mcX_*=\bx_*,\bX_n\right]\xrightarrow{n\to\infty}f(\bx_*), \quad \var\left[\tilde\mu_{GPnn}\mid \mcX_*=\bx_*,\bX_n\right] \xrightarrow{n\to\infty}\frac{\sigma_\xi^2(\bx_*)}{m},
\end{split}
\\
\begin{split}\label{eq:nngp_unbiased_limits}
&  \EE\left[\tilde\mu_{NNGP}\mid \mcX_*=\bx_*,\bX_n\right]\xrightarrow{n\to\infty}\bt^T\left(\bx_*\right).\bb, 
 \\
 & \var\left[\tilde\mu_{NNGP}-\bt^T\left(\bx_*\right).\bb-w(\bx_*)\mid \mcX_*=\bx_*,\bX_n\right] \xrightarrow{n\to\infty}\frac{\sigma_\xi^2(\bx_*)}{m}.
\end{split}
\end{align}
What is more, if $m\equiv m_n$ grows with $n$ such that $\lim_{n\to\infty} m_n/n=0$, we have that with probability one
\[
{\sigma_\mcN^*}^2\xrightarrow{n\to\infty}\hat\sigma_\xi^2
\]
and 
\begin{align}
\begin{split}\label{eq:gpnn_unbiased_limits_m_growing}
& \EE\left[\tilde\mu_{GPnn}\mid \mcX_*=\bx_*,\bX_n\right]\xrightarrow{n\to\infty}f(\bx_*), \quad \var\left[\tilde\mu_{GPnn}\mid \mcX_*=\bx_*,\bX_n\right] \xrightarrow{n\to\infty}0,
\end{split}
\\
\begin{split}\label{eq:nngp_unbiased_limits_m_growing}
&  \EE\left[\tilde\mu_{NNGP}\mid \mcX_*=\bx_*,\bX_n\right]\xrightarrow{n\to\infty}\bt^T\left(\bx_*\right).\bb, 
 \\
 & \var\left[\tilde\mu_{NNGP}-\bt^T\left(\bx_*\right).\bb-w(\bx_*)\mid \mcX_*=\bx_*,\bX_n\right] \xrightarrow{n\to\infty}0.
\end{split}
\end{align}
\end{lemma}
\begin{proof}
Let us start with $GPnn$. By Lemma \ref{lemma:bias_variance}, we have
\begin{align*}
\EE\left[\tilde\mu_{GPnn}\mid \mcX_*=\bx_*,\bX_n\right] & =\Gamma\,{\mathbf{k}^*_{\mcN}}^T\, K_{\mcN}^{-1}f(X_{\mcN}), 
\\
\var\left[\tilde\mu_{GPnn}\mid \mcX_*=\bx_*,\bX_n\right]  & = \Gamma^2\,{\mathbf{k}^*_{\mcN}}^T\, K_{\mcN}^{-1}\,\Sigma_{\bxi}\,K_{\mcN}^{-1}\,{\mathbf{k}^*_{\mcN}}.
\end{align*}

When $m$ is fixed, we insert the limits from Lemma \ref{lemma:kernel_limits} and use the continuity of $f$ from (AC.\ref{a_f}) to get
\[
\Gamma\,{\mathbf{k}^*_{\mcN}}^T\, K_{\mcN}^{-1}f(X_{\mcN}) \to \Gamma\, f(\bx_*)\, \hat\sigma_f^2\, \onev^T\, \left(K_{\mcN}^\infty\right)^{-1}\onev = f(\bx_*)=\EE\left[\mcY\mid \mcX=\bx_*\right],
\]
which shows that the $GPnn$ bias term tends to zero. Because $\sigma_\xi^2(\bx)$ is assumed to be an almost continuous function we have that $\Sigma_{\bxi}\xrightarrow{n\to\infty}\sigma_\xi^2(\bx_*)\bone$ with probability one. Inserting the limits from Lemma \ref{lemma:kernel_limits} to the variance expression yields after some algebra
\begin{equation}\label{eq:var_biased_limit}
\var\left[\tilde\mu_{GPnn}\mid \mcX_*=\bx_*,\bX_n\right]\to \frac{\sigma_\xi(\bx_*)^2}{m}.
\end{equation}
This proves Equations \eqref{eq:gpnn_unbiased_limits}. The limit for ${\sigma_\mcN^*}^2$ is derived in a fully analogous way.

When $m\equiv m_n$ grows with $n$, we decompose ${\mathbf{k}^*_{\mcN}}^T\, K_{\mcN}^{-1} = \hat\sigma_f^2\onev^T\, \left(K_{\mcN}^\infty\right)^{-1} + \Delta^T$ and $f(X_{\mcN})=f(\bx_*)\,\onev+\delta$ to get
\begin{align*}
{\mathbf{k}^*_{\mcN}}^T\, K_{\mcN}^{-1}f(X_{\mcN}) & = \left(\hat\sigma_f^2\onev^T\, \left(K_{\mcN}^\infty\right)^{-1} + \Delta^T\right).\left(f(\bx_*)\,\onev+\delta\right)
\\
& = f(\bx_*)\, \hat\sigma_f^2\, \onev^T\, \left(K_{\mcN}^\infty\right)^{-1}\onev + \hat\sigma_f^2\onev^T\, \left(K_{\mcN}^\infty\right)^{-1} \delta + f(\bx_*)\,\Delta^T.\onev+\Delta^T.\delta
\end{align*}
Using the formula from Equation  \eqref{eq:Kinv_imit} we get that the first term has the limit
\[
\Gamma_n\, f(\bx_*)\, \hat\sigma_f^2\, \onev^T\, \left(K_{\mcN}^\infty\right)^{-1}\onev = f(\bx_*),
\]
where $\Gamma_n$ now depends on $n$ (but tends to one, so the $\Gamma$-correction is inconsequential in the limit). It remains to show that the last three terms tend to zero as $n\to\infty$. First,
\begin{align*}
\Gamma_n\left|\hat\sigma_f^2\onev^T\, \left(K_{\mcN}^\infty\right)^{-1} \delta\right| & \leq \Gamma_n\hat\sigma_f^2\left\|\onev^T\, \left(K_{\mcN}^\infty\right)^{-1}\right\|_1\,\|\delta\|_1=\frac{1}{m}\|\delta\|_1
\\
& \leq \max_{1\leq i\leq m}\left|f\left(\mcX_{n,i}(\bx_*)\right)-f\left(\bx_*\right)\right|\xrightarrow{n\to\infty}0,
\end{align*}
where we have used the formula from Equation \eqref{eq:Kinv_imit} and in the last step we have used Remark \ref{remark:epsmax_limit} together with the continuity of $f$ from (AC.\ref{a_f}). Next,
\begin{align*}
\Gamma_n\left|\Delta^T.\onev\right|\leq m \left\|\Delta^T\right\|_1\leq  \frac{\epsilon_m+\epsilon_E}{1-\epsilon_E}\xrightarrow{n\to\infty}0,
\end{align*}
where we have used: i) Equation \eqref{kTK_ineq} from Lemma \ref{lemma:K_ineqs} to bound $\left\|\Delta^T\right\|_1$, ii) Remark \ref{remark:epsmax_limit} to get $\epsilon_m\xrightarrow{n\to\infty}0$, iii) Corollary \ref{coro:limits} to get $\epsilon_{E}\xrightarrow{n\to\infty}0$. Similarly, we get
\[
\left|\Delta^T.\delta\right|\leq \left\|\Delta^T\right\|_1\,\|\delta\|_1\leq \frac{1}{\Gamma_n}\,\frac{\epsilon_m+\epsilon_E}{1-\epsilon_E}\max_{1\leq i\leq m}\left|f\left(\bx_{n,i}(\bx_*)\right)-f\left(\bx_*\right)\right|\xrightarrow{n\to\infty}0.
\]
This proves the first part of \eqref{eq:gpnn_unbiased_limits_m_growing}. To prove the variance-part of \eqref{eq:gpnn_unbiased_limits_m_growing}, we decompose 
\[
K_{\mcN}^{-1}{\mathbf{k}^*_{\mcN}} = \hat\sigma_f^2\left(K_{\mcN}^\infty\right)^{-1} \onev+ \Delta=\frac{1}{m\Gamma_n}\onev+\Delta,\quad \Sigma_{\bxi}=\sigma_\xi^2(\bx_*)\bone+\delta\Sigma_{\bxi}
\]
 to get
 \begin{align*}
\var\left[\tilde\mu_{GPnn}\mid \mcX_*=\bx_*,\bX_n\right] =& \frac{\sigma_\xi^2(\bx_*)}{m_n}+2\Gamma_n\frac{\sigma_\xi^2(\bx_*)}{m_n}\onev^T.\Delta+\Gamma_n^2\sigma_\xi^2(\bx_*)\Delta^T.\Delta + \frac{1}{m_n^2}\onev^T.\delta\Sigma_{\bxi}\onev
\\
+&\frac{2\Gamma_n}{m_n}\onev^T.\delta\Sigma_{\bxi}\Delta+\Gamma_n^2\Delta^T.\delta\Sigma_{\bxi}\Delta
\end{align*}
Next, we show that all of the above terms tend to zero with probability one as $n\to\infty$. Since $m_n$ grows with $n$, the first term vanishes as $n\to\infty$. Next,
\begin{align*}
&\left|\frac{\Gamma_n}{m_n}\onev^T.\Delta\right|\leq \frac{\Gamma_n}{m_n}\|\onev^T\|_1\,\|\Delta\|_1\leq  \frac{1}{m_n}\frac{\epsilon_m+\epsilon_E}{1-\epsilon_E}\xrightarrow{n\to\infty}0,
\\
&\Gamma_n^2\left|\Delta^T.\Delta\right|\leq\Gamma_n^2\|\Delta^T\|_1\,\|\Delta\|_1\leq \frac{1}{m_n}\frac{\epsilon_m+\epsilon_E}{1-\epsilon_E}\xrightarrow{n\to\infty}0,
\end{align*}
where we have used Equations \eqref{Kk_ineq} and \eqref{kTK_ineq} from Lemma \ref{lemma:K_ineqs} to bound $\left\|\Delta\right\|_1$ and $\left\|\Delta^T\right\|_1$ and Remark \ref{remark:epsmax_limit} with Corollary \ref{coro:limits} to get $\epsilon_m\xrightarrow{n\to\infty}0$ and $\epsilon_{E}\xrightarrow{n\to\infty}0$. The remaining terms tend to zero using the same arguments together with the assumption that $\sigma_\xi^2(\bx)$ is an almost continuous function which implies that
\[
\left\|\delta\Sigma_{\bxi}\right\|_1= \max_{1\leq i\leq m}\left|\sigma_\xi^2(\bx_*)-\sigma_\xi^2(\bx_{n,i})\right|\xrightarrow{n\to\infty}0
\]
with probability one. This proves the second part of \eqref{eq:gpnn_unbiased_limits_m_growing} for $GPnn$. 

Let us next move to proving the bias and variance limits for $NNGP$. By Lemma \ref{lemma:bias_variance}, we have
\begin{align*}
\EE\left[\tilde\mu_{NNGP}\mid \mcX_*=\bx_*,\bX_n\right] & =\bt_*^T.(\bb-\hat\bb)-\Gamma\,{{k}_{\mcN}^*}^T\, K_{\mcN}^{-1}T_\mcN(\bb-\hat\bb), 
\\
\var\left[\tilde\mu_{NNGP}-\bt^T\left(\bx_*\right).\bb-w(\bx_*)\mid \mcX_*=\bx_*,\bX_n\right]  & = \sigma_w^2+\Gamma^2\,{{k}_{\mcN}^*}^T\, K_{\mcN}^{-1}\left(\tilde K_{\mcN}+\Sigma_\xi\right) K_{\mcN}^{-1}{{k}_{\mcN}^*}
\\
& -2\Gamma\,{{k}_{\mcN}^*}^T\, K_{\mcN}^{-1}\tilde k_\mcN^*.
\end{align*}
Decompose the variance into the noise-part and and random field (RF)-part
\[
\var_{noise}:=\Gamma^2\,{{k}_{\mcN}^*}^T\, K_{\mcN}^{-1}\Sigma_\xi K_{\mcN}^{-1}{{k}_{\mcN}^*},\quad  \var_{RF}:=\sigma_w^2+\Gamma^2\,{{k}_{\mcN}^*}^T\, K_{\mcN}^{-1}\tilde K_{\mcN} K_{\mcN}^{-1}{{k}_{\mcN}^*}-2\Gamma\,{{k}_{\mcN}^*}^T\, K_{\mcN}^{-1}\tilde k_\mcN^*.
\]
We recognise that $\EE\left[\tilde\mu_{NNGP}\mid \mcX_*=\bx_*,\bX_n\right] $ and $\var_{noise}$ are effectively bias and variance of $GPnn$ regression with the effective regression function $f_{eff}(\mcX):=\bt\left(\mcX\right)^T.(\bb-\hat\bb)$. Thus, we can directly apply the $GPnn$-results \eqref{eq:gpnn_unbiased_limits} and  \eqref{eq:gpnn_unbiased_limits_m_growing} to show that $\EE\left[\tilde\mu_{NNGP}\mid \mcX_*=\bx_*,\bX_n\right] \to 0$ both when $m$ is fixed and when $m$ grows with $n$. Similarly, we get that $\var_{noise}\to \sigma_\xi^2(\bx_*)$ when $m$ is fixed and $\var_{noise}\to 0$ when $m$ grows with $n$. 

What remains to show is that $\var_{RF}\to0$ with probability one both when $m$ is fixed and when $m$ grows with $n$. This is straightforward to prove when $m$ is fixed. Then, we can plug in the limits from Lemma \ref{lemma:kernel_limits} and use assumption (AC.\ref{a_metr}) together with Remark \ref{remark:epsmax_limit} and Corollary \ref{coro:limits} to obtain that all the nearest-neighbours of $\bx_*$ tend to $\bx_*$ as $n\to\infty$ in both (pseudo)metrics $\rho_{\tilde c}$ and $\rho_{c}$ with probability one. Consequently, the following limits hold with probability one.
\[
K_{\mcN}^{-1}{{k}_{\mcN}^*}\xrightarrow{n\to\infty} \frac{1}{m\Gamma}\onev,\quad \tilde K_\mcN\xrightarrow{n\to\infty} \sigma_w^2\,\onev.\onev^T,\quad \tilde k_\mcN^*\xrightarrow{n\to\infty} \sigma_w^2\,\onev.
\]
This yields
\[
\var_{RF}\xrightarrow{n\to\infty}\sigma_w^2+\frac{\sigma_w^2}{m^2}\onev^T.\left(\onev.\onev^T\right)\onev-2\frac{\sigma_w^2}{m}\onev^T.\onev=0
\]
with probability one.

When $m\equiv m_n$ grows with $n$, we decompose 
\[
\Gamma\,K_{\mcN}^{-1}{\mathbf{k}^*_{\mcN}} =\frac{1}{m_n}\onev+\Delta,\quad \tilde K_\mcN=\sigma_w^2\left(\onev.\onev^T+\tilde\Delta\right),\quad  \tilde k_\mcN^*=\sigma_w^2\left(\onev+\tilde\delta\right).
\]
This yields
\begin{align*}
\frac{\var_{RF}}{\sigma_w^2} = 1+\left(\frac{1}{m_n}\onev^T+\Delta^T\right).\left(\onev.\onev^T+\tilde\Delta\right)\left(\frac{1}{m_n}\onev+\Delta\right)-2\left(\frac{1}{m_n}\onev^T+\Delta^T\right)\left(\onev+\tilde\delta\right)
\\
= \frac{1}{m_n^2}\onev^T.\tilde\Delta \onev+\frac{1}{m_n}\onev^T.\tilde\Delta\Delta+\Delta^T\onev.\onev^T\Delta+\frac{1}{m}\Delta^T.\tilde\Delta\onev+\Delta^T.\tilde\Delta\Delta - 2\Delta^T.\tilde\delta-\frac{2}{m_n}\onev^T.\tilde\delta.
\end{align*}
Thus,
\begin{align*}
\left|\frac{\var_{RF}}{\sigma_w^2} \right|\leq & \frac{1}{m_n}\|\tilde\Delta\|_1+\frac{1}{m_n}\|\tilde\Delta\|_1\|\Delta\|_1+m_n\|\Delta^T\|_1\|\Delta\|_1+\|\Delta^T\|_1\|\tilde\Delta\|_1
\\
+& \|\Delta^T\|_1\|\tilde\Delta\|_1\|\Delta\|_1+2\|\Delta^T\|_1\|\tilde\delta\|_1+\frac{2}{m_n}\|\tilde\delta\|_1.
\end{align*}
Next, define 
\[
\tilde\epsilon_E := \frac{1}{m_n}\max_{1\leq j\leq m_n}\sum_{i=1}^{m_n}\rho_{\tilde c}^2\left(\bx_{n,i}(\bx_*),\bx_{n,j}(\bx_*)\right), \quad \tilde\epsilon_m := \max_{1\leq i\leq m_n} \rho_{\tilde c}^2\left(\bx_*,\bx_{n,i}(\bx*)\right).
\]
Note that 
\[
\frac{1}{m_n}\|\tilde\Delta\|_1=\tilde\epsilon_E\leq 4\tilde\epsilon_m \xrightarrow{n\to\infty}0,\quad \frac{1}{m_n}\|\tilde\delta\|_1\leq\tilde\epsilon_m\xrightarrow{n\to\infty}0
\]
with probability one, where we have used assumption (AC.\ref{a_metr}) together with Remark \ref{remark:epsmax_limit} and Corollary \ref{coro:limits}. By Equations \eqref{Kk_ineq} and \eqref{kTK_ineq} from Lemma \ref{lemma:K_ineqs} we have
\[
\|\Delta\|_1\leq \frac{\epsilon_E+\epsilon_m}{1-\epsilon_E},\quad \|\Delta^T\|_1\leq\frac{1}{m_n} \frac{\epsilon_E+\epsilon_m}{1-\epsilon_E}.
\]
By Remark \ref{remark:epsmax_limit} and Corollary \ref{coro:limits}, we have $\epsilon_m\xrightarrow{n\to\infty}0$ and $\epsilon_E\xrightarrow{n\to\infty}0$. Plugging this into the bound for $\var_{RF}$, we get
\begin{align}\label{eq:varRF_bound}
\left|\frac{\var_{RF}}{\sigma_w^2} \right|\leq \left(2\tilde\epsilon_m+\tilde\epsilon_E\right) \frac{\epsilon_m+1}{1-\epsilon_E}+\left( \frac{\epsilon_E+\epsilon_m}{1-\epsilon_E}\right)^2\left(1+ \tilde\epsilon_E\right)\xrightarrow{n\to\infty}0.
\end{align}
\end{proof}
We are now ready to prove Theorem \ref{thm:ptwise_consistency} which we repeat below for reader's convenience.
\begin{restatedresult}{Theorem~\ref{thm:ptwise_consistency} (Universal Point-Wise Consistency).}
\PointwiseConsistencyStatement
\end{restatedresult}
\begin{proof}
The limit \eqref{ptwise_mse_limit} follows from plugging the limits \eqref{eq:gpnn_unbiased_limits} and \eqref{eq:nngp_unbiased_limits} into the bias-variance decomposition of the MSE from Lemma \ref{lemma:bias_variance}. The limit \eqref{ptwise_cal_limit} follows from plugging the MSE-limit \eqref{ptwise_mse_limit} and the predictive variance limit \eqref{eq:biased_limits} to the definition of the calibration coefficient \eqref{def:cal_app} to obtain
\[ 
CAL(\bx_*,X_n) \xrightarrow{n\to\infty} \,\frac{\sigma_\xi(\bx_*)^2}{\hat\sigma_\xi^2}\,\left(1+\frac{1}{m}\right)\left(1+\frac{\hat\sigma_f^2}{\hat\sigma_\xi^2+m\hat\sigma_f^2}\right)^{-1}.
\]
The final form of the limit \eqref{ptwise_cal_limit} is obtained by expanding the above expression with respect to  powers of $1/m$.

Finally, the limit \eqref{ptwise_nll_limit} follows from plugging the CAL-limit \eqref{ptwise_cal_limit} and the predictive variance limit \eqref{eq:biased_limits} into the definition of $NLL$ \eqref{def:nll_app} and using the continuity of the logarithm. In the final step we have expanded the resulting expressions with respect to  powers of $1/m$.

The limits \eqref{ptwise_mse_cal_limit_mgrow} and \eqref{ptwise_nll_limit_mgrow} are obtained in a fully analogous way using \eqref{eq:gpnn_unbiased_limits_m_growing} and \eqref{eq:nngp_unbiased_limits_m_growing} from Lemma \ref{lemma:bias_variance}.
\end{proof}
Having established the pointwise limits of the performance measures, we are now ready to prove Theorem \ref{thm:approx_universal_consistency} which we repeat below.
\begin{restatedresult}{Theorem~\ref{thm:approx_universal_consistency} (Approximate Universal Consistency).}
\ApproximateUniversalConsistencyStatement
\end{restatedresult}
\begin{proof}
We will prove the desired convergence in expectation using the dominated convergence theorem (DCT) \citep[see][Theorem~1.13]{SteinShakarchi2005}. From Theorem \ref{thm:ptwise_consistency} we know that for both $GPnn$ and $NNGP$ the positive function
\[
f_n(\bx_*,X_\infty):= \left|MSE(\bx_*,X_n) - MSE_\infty(\bx_*)\right|\xrightarrow{n\to\infty} 0, \quad MSE_\infty(\bx_*):=\sigma_\xi^2(\bx_*)\left(1+\frac{1}{m}\right)
\]
for almost every $(\bx_*,X_\infty)$ with respect to the measure $P_\mcX\otimes P_\mcX^{\otimes\infty}$. In the above formula we treat $X_n$ as the first $n$ elements of the infinite sampling sequence $X_\infty$. According to DCT, it suffices to find a measurable function $g(\bx_*,X_\infty)$ such that
\[
\EE\left[g(\mcX_*,\bX_\infty)\right]<\infty,\quad f_n(\bx_*,X_\infty)\leq g(\bx_*,X_\infty)
\]
for every $n$ and almost every $(\bx_*,X_\infty)$ with respect to the measure $P_\mcX\otimes P_\mcX^{\otimes\infty}$. 

Let us first prove the risk limit \eqref{E_mse_limit} for $GPnn$. Define 
\[
B_f:=\|f\|_\infty<\infty,\quad B_\xi:=\|\sigma_\xi^2(\bx)\|_\infty<\infty.
\] 
If $m$ is held fixed, we will show that the function $f_n(\bx_*,X_\infty)$ is upper-bounded by a constant independent of $n$. To see this, note first that
\[
f_n(\bx_*,X_\infty)\leq\left|MSE(\bx_*,X_n)\right| + B_\xi\left(1+\frac{1}{m}\right)\leq \left|MSE(\bx_*,X_n)\right| + 2B_\xi.
\]
The key observation is that $\left|MSE(\bx_*,X_n)\right|$ is bounded. By the bias-variance decomposition of $MSE$ from Lemma \ref{lemma:bias_variance} we have
\begin{align}\label{eq:mse_upper0}
\begin{split}
\left|MSE(\bx_*,X_n)\right|\leq & \left(\frac{\hat\sigma_\xi^2+m\hat\sigma_f^2}{m\hat\sigma_f^2}  \left|{\mathbf{k}^*_{\mcN}}^T\, K_{\mcN}^{-1}f(X_{\mcN})\right|+B_f\right)^2
\\ 
& + B_\xi\, \left(\frac{\hat\sigma_\xi^2+m\hat\sigma_f^2}{m\hat\sigma_f^2}\right)^2\,\left|{\mathbf{k}^*_{\mcN}}^T\, K_{\mcN}^{-2}\,{\mathbf{k}^*_{\mcN}}\right|.
\end{split}
\end{align}
Since the spectrum of $K_{\mcN}$ is lower-bounded by $\hat\sigma_\xi^2$ (recall that $K_{\mcN}$ is the shifted Gram matrix), we have that every eigenvalue of $K_{\mcN}^{-2}$ satisfies
\[\lambda_i\left(K_{\mcN}^{-2}\right)=\frac{1}{\lambda_i\left(K_{\mcN}\right)^2}\leq\frac{1}{\hat\sigma_\xi^2}\frac{1}{\lambda_i\left(K_{\mcN}\right)}=\frac{1}{\hat\sigma_\xi^2}\lambda_i\left(K_{\mcN}^{-1}\right).\]
This means that the matrix $\frac{1}{\hat\sigma_\xi^2}K_{\mcN}^{-1}-K_{\mcN}^{-2}$ is positive semi-definite. Consequently,
\begin{equation}\label{eq:var_upper}
{\mathbf{k}^*_{\mcN}}^T\, K_{\mcN}^{-2}\,{\mathbf{k}^*_{\mcN}}\leq \frac{1}{\hat\sigma_\xi^2}\,{\mathbf{k}^*_{\mcN}}^T\, K_{\mcN}^{-1}\,{\mathbf{k}^*_{\mcN}}\leq \frac{1}{\hat\sigma_\xi^2}k(\bx_*,\bx_*)=\frac{\hat\sigma_f^2}{\hat\sigma_\xi^2},
\end{equation}
where in the second inequality we have used the fact that the $GP$ predictive variance at $\bx_*$ is non-negative, i.e. $k(\bx_*,\bx_*)-{\mathbf{k}^*_{\mcN}}^T\, K_{\mcN}^{-1}\,{\mathbf{k}^*_{\mcN}}\geq 0$.

To bound $\left|{\mathbf{k}^*_{\mcN}}^T\, K_{\mcN}^{-1}f(X_{\mcN})\right|$, we use the submultiplicativity of the $2$-norm as follows.
\[
\left|{\mathbf{k}^*_{\mcN}}^T\, K_{\mcN}^{-1}f(X_{\mcN})\right|=\left\|{\mathbf{k}^*_{\mcN}}^T\, K_{\mcN}^{-1/2}\, K_{\mcN}^{-1/2}f(X_{\mcN})\right\|_2\leq \left\|{\mathbf{k}^*_{\mcN}}^T\, K_{\mcN}^{-1/2}\right\|_2\,\left\|K_{\mcN}^{-1/2}f(X_{\mcN})\right\|_2.
\]
Using the non-negativity of the $GP$ predictive variance we get
\[
\left\|{\mathbf{k}^*_{\mcN}}^T\, K_{\mcN}^{-1/2}\right\|_2^2={\mathbf{k}^*_{\mcN}}^T\, K_{\mcN}^{-1}\,{\mathbf{k}^*_{\mcN}}\leq \hat\sigma_f^2.
\]
By bounding the spectrum of $K_{\mcN}$ from below by $\hat\sigma_\xi^2$ and using the boundedness of $f$ we get
\[
\left\|K_{\mcN}^{-1/2}f(X_{\mcN})\right\|_2\leq \left\|K_{\mcN}^{-1/2}\right\|_2\,\left\|f(X_{\mcN})\right\|_2\leq \sqrt{m}\,\frac{B_f}{\hat\sigma_\xi}.
\]
Plugging all of the above results into \eqref{eq:mse_upper0} we obtain the final upper bound
\begin{align}\label{eq:mse_upper_bound_const}
f_n(\bx_*,X)\leq B_f^2\,\left(\frac{\hat\sigma_\xi^2+\hat\sigma_f^2}{\hat\sigma_f\hat\sigma_\xi}\, \sqrt{m}+1\right)^2
 + B_\xi\, \left(2+\frac{\hat\sigma_f^2}{\hat\sigma_\xi^2}\left(\frac{\hat\sigma_\xi^2+\hat\sigma_f^2}{\hat\sigma_f^2}\right)^2\right).
\end{align}
which is the last ingredient of DCT.

Let next prove the risk limit \eqref{E_mse_limit} for $NNGP$. From Lemma \ref{lemma:bias_variance} we have that
\[
MSE_{NNGP}(\bx_*,X_n)  = MSE_{GPnn}(\bx_*,X_n)+\var_{RF}(\bx_*,X_n),
\]
where
\begin{align*}
& MSE_{GPnn}=\sigma_\xi^2(\bx_*)+\Gamma^2\,{{k}_{\mcN}^*}^T\, K_{\mcN}^{-1}\Sigma_{\bxi}K_{\mcN}^{-1}{k}_{\mcN}^*+ \left(\bt_*^T.(\bb-\hat\bb)-\Gamma\,{{k}_{\mcN}^*}^T\, K_{\mcN}^{-1}T_\mcN(\bb-\hat\bb)\right)^2,
\\
& \var_{RF}=\sigma_w^2+\Gamma^2\,{{k}_{\mcN}^*}^T\, K_{\mcN}^{-1}\tilde K_\mcN K_{\mcN}^{-1}{{k}_{\mcN}^*}-2\Gamma\,{{k}_{\mcN}^*}^T\, K_{\mcN}^{-1}\tilde k_\mcN^*.
\end{align*}
Using the earlier upper bound for $GPnn$ $MSE$ \eqref{eq:mse_upper_bound_const}, we immediately get an upper bound for the $MSE_{GPnn}$-component of $MSE_{NNGP}$. To see this, define $B_T:=\max_{1\leq i\leq d_T}\|t_i\|_\infty<\infty$. By Cauchy-Schwarz we have
\[
\left|\bt(\bx)^T.(\bb-\hat\bb)\right|\leq \|\bt(\bx)\|_2\,\|\bb-\hat\bb\|_2\leq B_T\,\,\|\bb-\hat\bb\|_2.
\]
Then, replacing $B_f$ by $B_T\,\,\|\bb-\hat\bb\|_2$ in the inequality \eqref{eq:mse_upper_bound_const} we get
\begin{equation}\label{eq:gpnn_mse_upper_bound}
MSE_{GPnn}\leq B_T^2\|\bb-\hat\bb\|_2^2\,\left(\frac{\hat\sigma_\xi^2+\hat\sigma_f^2}{\hat\sigma_f\hat\sigma_\xi}\, \sqrt{m}+1\right)^2
 + B_\xi\, \frac{\hat\sigma_f^2}{\hat\sigma_\xi^2}\left(\frac{\hat\sigma_\xi^2+\hat\sigma_f^2}{\hat\sigma_f^2}\right)^2.
\end{equation}
It remains to bound $MSE_{RF}$.
\[
\var_{RF}\leq \sigma_w^2+\Gamma^2\left|{{k}_{\mcN}^*}^T\, K_{\mcN}^{-1}\tilde K_\mcN K_{\mcN}^{-1}{{k}_{\mcN}^*}\right|+2\Gamma\,\left|{{k}_{\mcN}^*}^T\, K_{\mcN}^{-1}\tilde k_\mcN^*\right|.
\]
Further by the submultiplicativity of the $2$-norm,
\[
\left|{{k}_{\mcN}^*}^T\, K_{\mcN}^{-1}\tilde k_\mcN^*\right|\leq \left\|{{k}_{\mcN}^*}^T\, K_{\mcN}^{-1/2}\right\|_2\,\left\|K_{\mcN}^{-1/2}\right\|_2\,\left\|\tilde k_\mcN^*\right\|_2\leq \sigma_w^2\sqrt{m}\,\frac{\hat\sigma_f}{\hat\sigma_\xi}.
\]
Similarly, we get
\begin{align*}
\left|{{k}_{\mcN}^*}^T\, K_{\mcN}^{-1}\tilde K_\mcN K_{\mcN}^{-1}{{k}_{\mcN}^*}\right|\leq & \left\|{{k}_{\mcN}^*}^T\, K_{\mcN}^{-1/2}\right\|_2\left\|K_{\mcN}^{-1/2}\right\|_2\left\|\tilde K_\mcN\right\|_2\left\|K_{\mcN}^{-1/2}\right\|_2\left\|K_{\mcN}^{-1/2}{{k}_{\mcN}^*}\right\|_2
\\
\leq & \sigma_w^2\,m\,\frac{\hat\sigma_f^2}{\hat\sigma_\xi^2},
\end{align*}
where we have used the standard inequality $\left\|\tilde K_\mcN\right\|_2\leq \sigma_w^2\,m$ which stems from the fact that the entries of $\tilde K_\mcN$ are bounded from above by $\sigma_w^2$. Summing up, we have
\begin{equation}\label{eq:nngp_mse_upper_bound}
\var_{RF}\leq\sigma_w^2\left(1+2\Gamma\sqrt{m}\,\frac{\hat\sigma_f}{\hat\sigma_\xi}+\Gamma^2 m\,\frac{\hat\sigma_f^2}{\hat\sigma_\xi^2}\right)=\sigma_w^2\left(1+\sqrt{m}\,\Gamma\,\frac{\hat\sigma_f}{\hat\sigma_\xi}\right)^2.
\end{equation}
 This ends the derivation of the upper bound for $MSE_{NNGP}$ which is independent of $n$ and thus by DCT proves the risk limit \eqref{E_mse_limit} for $NNGP$.
 
 The limits \eqref{cal_nll_limit} are proved in a fully analogous way by using the definitions of $CAL$ \eqref{def:cal_app} and $NLL$ \eqref{def:nll_app} and additionally utilising standard inequality 
 \[
 \hat\sigma_\xi^2\leq {\sigma_\mcN^*}^2\leq \hat\sigma_\xi^2+\hat\sigma_f^2.
 \]
\end{proof}

\subsection{Universal Consistency}
\label{app:universal_consistency}
We start by establishing some preliminary ingredients. The $m_n$-nearest-neighbour ($m_n$-NN) regression function estimate is defined as follows (using notation from Section \ref{sec:preliminaries}).
\[
\mu_{NN}(\bx_*)=\frac{1}{m_n}\sum_{j=1}^{m_n}y_{n,j},
\]
where $y_{n,j}$ is the observed response at the $j$-th nearest neighbour $\bx_{n,j}(\bx_*)$.
\begin{theorem}[{\citet[Theorem~6.1]{Gyorfi2002}}]
\label{thm:nn_consistency}
Let the number of nearest - neighbours $m_n$ grow with $n$ such that $m_n/n\xrightarrow{n\to\infty}0$. Then the $m_n$-NN regression function
estimate is universally consistent for all distributions of $(\mcX,\mcY)$ where nearest-neighbour ties occur with probability zero and $\EE\left[\mcY^2\right]<\infty$.
\end{theorem}

\begin{lemma}\label{lemma:bad_region_bound}
Let $\bX_n$ be a random sampling sequence of $n$ i.i.d. points from the distribution $P_\mcX$ and assume nearest neighbours are selected according to their Euclidean distance from the test point. Let $\mcX_*\sim P_\mcX$ be a test point. Fix $\nu>0$ and assume that there exists $\beta>\frac{2\nu d_\mcX}{d_\mcX-2\nu}$ for which $\EE\left[\|\mcX\|_2^\beta\right]<\infty$ under the probability distribution $P_\mcX$ on $\RR^{d_\mcX}$ with $d_\mcX>2\nu$. Then, for any $0<R\leq 1$ the following inequality holds 
\[
\sqrt{P\left[\min\left\{d_m(\mcX_*,\bX_n),1\right\}\geq R\right]}\leq \frac{1}{R^\nu}\sqrt{c}\,2^{\frac{\nu}{d_\mcX}+\frac{1}{2}}\left(\frac{m}{n}\right)^{\nu/d_\mcX},
\]
where $d_m(\bx_*,X_n)$ is the distance between $\bx_*$ and it's $m$-th nearest neighbour in $X_n$ and the positive constant $c<\infty$ depends on  $d_\mcX,\nu,\beta$ and $\EE\left[\|\mcX\|_2^\beta\right]$.
\end{lemma}
\begin{proof}
 Apply Markov's inequality which states that for any non-negative random variable $U$ and $\lambda> 0$ we have
 \[
 P\left[U\geq \lambda\right]\leq\frac{\EE\left[U\right]}{\lambda}.
 \]
Take $U \equiv \min\{d_m^{2\nu},1\}$ and $\lambda=R^{2\nu}$. This yields
\[
P\left[\min\{d_m,1\}\geq R\right]=P\left[U\geq R^{2\nu} \right]\leq \frac{\EE\left[U\right]}{R^{2\nu}}.\]
Next, we use Lemma \ref{lemma:dmax_rate} applied to $\min\left\{d_m^{2\nu},1\right\}$ to upper bound $\EE\left[U\right]$ as follows. 
\begin{gather*}
\EE\left[U\right]\leq c\,2^{\frac{2\nu}{d_\mcX}+1}\,\left(\frac{m}{n}\right)^{2\nu/d_\mcX}.
\end{gather*}
Taking the square root of both sides, we prove the lemma.
\end{proof}

\begin{restatedresult}{Theorem~\ref{thm:universal_consistency} (Universal Consistency).}
\UniversalConsistencyStatement
\end{restatedresult}
\begin{proof}(for $GPnn$)
At a test point $\bx_*$, we write the $GPnn$ predictor as a weighted sum of the nearest-neighbour responses
\[
\tilde\mu_{GPnn}(\bx_*)=\sum_{j=1}^m w_{n,j}(\bx_*,X_n)\,y_{n,j}=\bw_{n}^T.\by_{n,m},\quad \bw_{n}(\bx_*,X_n)=\Gamma\,K_\mcN^{-1}\,{\bk_{\mcN}^*}.
\]
Define the difference between the $GPnn$ and $m$-NN estimators
\[
D_n(\bx_*):=\tilde\mu_{GPnn}(\bx_*)-\mu_{NN}(\bx_*)=\ba_{n}^T.\by_{n,m},\quad \ba_{n}=\bw_{n}-\frac{1}{m}\onev.
\]
Then, the $GPnn$ squared error at test point $\bx_*$ can be written as
\begin{align*}
\left(f(\bx_*)-\tilde\mu_{GPnn}(\bx_*)\right)^2= & \left(f(\bx_*)-\mu_{NN}(\bx_*)-D_n(\bx_*)\right)^2
\\
\leq & 2\left(f(\bx_*)-\mu_{NN}(\bx_*)\right)^2+2\, D_n(\bx_*)^2,
\end{align*}
where in the last line we have use the fact that for any real $a,b$ we have $(a-b)^2\leq 2a^2+2b^2$. Take expectations (w.r.t. all the training and test data) of both sides to get
\[
\mcR_n^{(GPnn)}=\EE\left[\left(f(\bx_*)-\tilde\mu_{GPnn}(\bx_*)\right)^2\right]\leq 2\EE\left[\left(f(\bx_*)-\mu_{NN}(\bx_*)\right)^2\right]+2\EE\left[D_n(\bx_*)^2\right]
\]
By Theorem \ref{thm:nn_consistency}, we have that the $m$-NN risk tends to zero as the training data size grows
\[
\mcR_n^{(NN)}=\EE\left[\left(f(\bx_*)-\mu_{NN}(\bx_*)\right)^2\right]\xrightarrow{n\to\infty}0.
\]
The remaining part of this proof is devoted to showing that 
\begin{equation}\label{eq:D2_goal}
\EE\left[D_n(\bx_*)^2\right]\xrightarrow{n\to\infty}0,
\end{equation}
which is the last ingredient needed to prove the universal consistency. To this end, we first use the Cauchy-Schwarz inequality $\left(\bu^T.\bv\right)^2\leq \|\bu\|_2^2\,\|\bv\|_2^2$ with 
\[u_j=\sqrt{|a_{n,j}|}\,\mathrm{sign}(a_{n,j}),\quad v_j=y_{n,j}\sqrt{|a_{n,j}|},
\]
to get
\[
D_n^2=\left(\sum_{j=1}^ma_{n,j}\,y_{n,j}\right)^2\leq\left(\sum_{j=1}^m|a_{n,j}|\right)\,\left(\sum_{j=1}^m |a_{n,j}|\,y_{n,j}^2\right)=\|\ba_n\|_1\left(\sum_{j=1}^m |a_{n,j}|\,y_{n,j}^2\right).
\]
Thus, the conditional expectation is bounded as
\begin{equation}\label{eq:D2_cond_exp}
    \EE\left[D_n^2\mid \mcX_*=\bx_*,\bX_n=X_n\right]\leq \left\|\ba_n\right\|_1\left(\sum_{j=1}^m |a_{n,j}|\,\EE\left[\mcY_{n,j}^2\mid \mcX_{n,j}=\bx_{n,j}(\bx_*)\right]\right).
\end{equation}
By the assumption of the boundedness of $f(\cdot)$ and the boundedness of the noise variance
\[
\EE\left[\mcY^2\mid \mcX=\bx\right]\leq 2f(\bx)^2+2\EE\left[\Xi^2\mid \mcX=\bx\right]\leq C_{Y,2}<\infty
\]
for some constant $C_{Y,2}>0$. To derive this, we have used the inequality $(a+b)^2\leq 2a^2+2b^2$ again. Moreover, by the inequality \eqref{Kk_ineq} from Lemma \ref{lemma:K_ineqs} and the identity $\hat\sigma_f^2\left(K_\mcN^\infty\right)^{-1}\onev=\frac{1}{m\Gamma}\onev$, we get that for each $\bx_*,X_n$
\[
\left\|\ba_n(\bx_*,X_n)\right\|_1\leq\frac{\epsilon_m(\bx_*,X_n)+\epsilon_E(\bx_*,X_n)}{1-\epsilon_E(\bx_*,X_n)}.
\]
The functions $\epsilon_m$ and $\epsilon_E$ are defined in Lemma \ref{lemma:K_ineqs}. Plugging this into \eqref{eq:D2_cond_exp} yields
\begin{equation}\label{eq:d_cond_bound}
    \EE\left[D_n^2\mid \mcX_*=\bx_*,\bX_n=X_n\right]\leq C_{Y,2}\,\left\|\ba_n\right\|_1^2\leq C_{Y,2}\, \left(\frac{\epsilon_m+\epsilon_E}{1-\epsilon_E}\right)^2.
\end{equation}
To proceed, we need to take the expectation over the training and test data $\bX_n,\mcX_*$. This requires handling the possible blowup of the above upper bound when $\epsilon_E$ approaches $1$. To this end, we define the good event in the space of training and test data
\[
G_n:=\left\{\min\left\{d_m(\mcX_*,\bX_n),1\right\}<\min\left\{R,1\right\}\right\}, \quad R=\left(\frac{1}{8\max\{L_c,1\}}\right)^{1/2p},
\]
where $L_c$ is defined in (AR.\ref{aR_c_app}). Then, we use the tower property of the expectations and split the expectation of $D_n^2$ as
\[
\EE\left[D_n^2\right]= \EE\left[\EE\left[D_n^2\,\chi_{G_n}\mid  \mcX_*,\bX_n\right]\right]+ \EE\left[\EE\left[D_n^2\,\chi_{G_n^c}\mid  \mcX_*,\bX_n\right]\right],
\]
where $\chi_{\mcS}$ is the indicator function of the set $\mcS$ and $\mcS^c$ is the complement of $\mcS$. By (AR.\ref{aR_c_app}) we have that
\[
\epsilon_m\leq \min\left\{L_c\,d_m^{2p},1\right\}\leq \max\{L_c,1\}\min\{d_m^{2p},1\},
\]
where in the second inequality we have used the fact that for any $a,b\geq 0$ we have $\min\{ab,1\}\leq \max\{a,1\}\min\{b,1\}$. Thus, 
\[
\mathrm{if}\quad \min\{d_m(\bx_*,X_n),1\}<\min\left\{R,1\right\},\quad \mathrm{then}\quad \epsilon_m(\bx_*,X_n)<1/8.
\]
 Furthermore, by Lemma \ref{lemma:epsilons_rels} we get that $\epsilon_E\leq 4\epsilon_m<1/2$, thus $1/(1-\epsilon_E)<2$ and we get from \eqref{eq:d_cond_bound} that
\[
\EE\left[D_n^2\,\chi_{G_n}\right]< 100\,C_{Y,2}\,\EE\left[\epsilon_m^2\,\chi_{G_n}\right]\xrightarrow{n\to\infty}0.
\]
This follows from the dominated convergence theorem since Remark \ref{remark:epsmax_limit} implies that 
\[
\epsilon_m^2\,\chi_{G_n}\xrightarrow{n\to\infty}0\quad a.s.
\]
and $\epsilon_m$ is upper-bounded by $1$.

Finally, we need to show that
\[
\EE\left[D_n^2\,\chi_{G_n^c}\right]\xrightarrow{n\to\infty}0.
\]
By Cauchy-Schwarz inequality we have
\begin{align}\label{eq:Dn_CS}
\begin{split}
    \EE\left[\EE\left[D_n^2\mid  \mcX_*,\bX_n\right]\,\chi_{G_n^c}\right] \leq & \sqrt{\EE\left[\left(\EE\left[D_n^2\mid  \mcX_*,\bX_n\right]\right)^2\right]}\times
    \\
    & \times \sqrt{P\left[\min\left\{d_m(\mcX_*,\bX_n),1\right\}\geq\min\left\{R,1\right\}\right]}
\end{split}
\end{align}
Next, we find a suitable upper bound on $\left\|\ba_n\right\|_1$ as follows
\begin{align*}
\left\|\ba_n\right\|_1\leq & \left\|\bw_n\right\|_1+ \frac{1}{m}\left\|\onev\right\|_1=\Gamma  \left\|K_\mcN^{-1}\,{\bk_{\mcN}^*}\right\|_1+1\leq \sqrt{m}\Gamma  \left\|K_\mcN^{-1}\,{\bk_{\mcN}^*}\right\|_2+1
\\
\leq & \sqrt{m}\Gamma  \left\|K_\mcN^{-1/2}\right\|_2\, \left\|K_\mcN^{-1/2}\,{\bk_{\mcN}^*}\right\|_2+1\leq  \sqrt{m}\Gamma\frac{\hat\sigma_2}{\hat\sigma_\xi}+1\leq \sqrt{m}\left(\Gamma\frac{\hat\sigma_f}{\hat\sigma_\xi}+1\right),
\end{align*}
where in the last line we have used the facts that i) the minimum eigenvalue of $K_\mcN$ is bounded from below by $\hat\sigma_\xi^2$ which implies $\left\|K_\mcN^{-1/2}\right\|_2 \leq 1/\hat\sigma_\xi$ and ii) $\left\|K_\mcN^{-1/2}\,{\bk_{\mcN}^*}\right\|_2\leq \hat\sigma_f$ which follows from the non-negativity of the $GP$ predictive covariance
\[
 \hat\sigma_f^2\geq {\bk_{\mcN}^*}^TK_\mcN^{-1}\,{\bk_{\mcN}^*}=\left\|K_\mcN^{-1/2}\,{\bk_{\mcN}^*}\right\|_2^2.
\]
Plugging this into \eqref{eq:d_cond_bound}, we get
\[
   \EE\left[D_n^2\mid  \mcX_*,\bX_n\right]\leq C_{Y,2}\left(\Gamma\frac{\hat\sigma_f}{\hat\sigma_\xi}+1\right)^2\, m.
\]
Next, we use the above inequality in combination with Lemma \ref{lemma:bad_region_bound} for $\nu>d_{\mcX}\gamma/(1-\gamma)$ which by \eqref{eq:Dn_CS} yields
\begin{align*}
 \EE\left[D_n^2\,\chi_{G_n^c}\right]\leq & C_{Y,2}\left(\Gamma\frac{\hat\sigma_f}{\hat\sigma_\xi}+1\right)^2\,\frac{1}{\min\{R^\nu,1\}}\sqrt{c}\,2^{\frac{\nu}{d_\mcX}+\frac{3}{2}}\, m\, \left(\frac{m}{n}\right)^{\nu/d_\mcX}.
\end{align*}
The condition $\nu>d_{\mcX}\gamma/(1-\gamma)$ ensures that by taking $m = A n^\gamma$ for some $A>0$, the above bound implies that 
\[
 \EE\left[D_n^2\,\chi_{G_n^c}\right]\leq C_{Y,2}\left(\Gamma\frac{\hat\sigma_f}{\hat\sigma_\xi}+1\right)^2\,\frac{1}{\min\{R^\nu,1\}}\sqrt{c}\,2^{\frac{\nu}{d_\mcX}+\frac{3}{2}}\tilde A\, n^{\gamma-\frac{\nu}{d_\mcX}(1-\gamma)}\xrightarrow{n\to\infty}0,
\]
where $\tilde A>0$ depends on $A, \gamma, d_\mcX,\nu$. The above upper bound tends to zero as $n\to\infty$, because $\gamma-\frac{\nu}{d_\mcX}(1-\gamma)<0$ for our choice of $\nu$. Note that the condition $d_{\mcX}>2\nu$ of Lemma \ref{lemma:bad_region_bound} is then guaranteed for any $d_{\mcX}$ by the fact that $\nu>d_{\mcX}\gamma/(1-\gamma)$ and $0<\gamma<1/3$ (this can be straightforwardly verified by substitution). Finally, note that when $\nu>d_{\mcX}\gamma/(1-\gamma)$, then we have
\[
\frac{2\nu d_\mcX}{d_\mcX-2\nu} > \frac{2d_{\mcX}^2\frac{\gamma}{1-\gamma}}{d_{\mcX}-2d_{\mcX}\frac{\gamma}{1-\gamma}} = \frac{2\gamma d_{\mcX}}{1-3\gamma}.
\]
Thus, for any $\beta > \frac{2\gamma d_{\mcX}}{1-3\gamma}$, we can find $\nu>d_{\mcX}\gamma/(1-\gamma)$ which additionally satisfies $\frac{2\nu d_\mcX}{d_\mcX-2\nu}<\beta$, thereby satisfying the moment-condition $\beta>\frac{2\nu d_\mcX}{d_\mcX-2\nu}$ of Lemma \ref{lemma:bad_region_bound}. This finishes the proof.
\end{proof}
\begin{proof}(for $NNGP$) The goal is to prove that
\[
\EE\left[\left(\mu_{NNGP}\left(\mcX_*\right)-g\left(\mcX_*\right)\right)^2\right]\xrightarrow{n\to\infty}0,\quad g\left(\mcX_*\right)=\bt\left(\mcX_*\right)^T.\bb+w\left(\mcX_*\right),
\]
where $g$ is the noise-free part of the $NNGP$ response from Equation \eqref{eq:nngp_responses} and the expectation is over the noise and the random $GP$ sample paths $w$. Then,
\begin{align*}
\left(\mu_{NNGP}\left(\mcX_*\right)-g\left(\mcX_*\right)\right)^2 = & \left(\mu_{NNGP}\left(\mcX_*\right)-\mu_{NN}\left(\mcX_*\right)+\mu_{NN}\left(\mcX_*\right)-g\left(\mcX_*\right)\right)^2
\\
\leq & 2\left(\mu_{NN}\left(\mcX_*\right)-g\left(\mcX_*\right)\right)^2+2\left(\mu_{NNGP}\left(\mcX_*\right)-\mu_{NN}\left(\mcX_*\right)\right)^2.
\end{align*}
Firstly, using $m$-NN universal consistency we show that 
\[
\EE\left[\left(\mu_{NN}\left(\mcX_*\right)-g\left(\mcX_*\right)\right)^2\right]\xrightarrow{n\to\infty}0.
\]
To this end, we decompose $g\left(\mcX_*\right)-\mu_{NN}\left(\mcX_*\right)=A_n+B_n$, where
\begin{gather*}
A_n=\bt\left(\mcX_*\right)^T.\bb-\frac{1}{m}\sum_{j=1}^m\left(\bt\left(\mcX_{n,j}\right)^T.\bb+\Xi_{n,j}\right),\quad B_n=w\left(\mcX_*\right)-\frac{1}{m}\sum_{j=1}^m w\left(\mcX_{n,j}\right).
\end{gather*}
Then,
\[
\EE\left[\left(\mu_{NN}\left(\mcX_*\right)-g\left(\mcX_*\right)\right)^2\right]\leq 2\EE\left[A_n^2\right]+2\EE\left[B_n^2\right].
\]

The term $\EE\left[A_n^2\right]\xrightarrow{n\to\infty}0$ due to the universal consistency of $m$-NN applied to the regression function $f(\bx)=\bt\left(\bx\right)^T.\bb$.

To see that the term $\EE\left[B_n^2\right]\xrightarrow{n\to\infty}0$, note first that $(\sum_j z_j/m)^2\leq (\sum_j z_j^2)/m$ with $z_j=w\left(\mcX_*\right)-w\left(\mcX_{n,j}\right)$ implies the upper bound
\[
B_n^2\leq \frac{1}{m}\sum_{j=1}^m \left(w\left(\mcX_*\right)-w\left(\mcX_{n,j}\right)\right)^2.
\]
Thus,
\[
\EE\left[B_n^2\mid \mcX_*,\bX_n\right]\leq  \frac{2\sigma_w^2}{m}\sum_{j=1}^m \left(1-\tilde c\left(\mcX_*,\mcX_{n,j}\right)\right),
\]
where we have used the fact that
\begin{align*}
\EE\left[ \left(w\left(\mcX_*\right)-w\left(\mcX_{n,j}\right)\right)^2\mid \mcX_*,\bX_n\right]=&\sigma_w^2\left(\tilde c\left(\mcX_*,\mcX_*\right)+\tilde c\left(\mcX_{n,j},\mcX_{n,j}\right)-2\tilde c\left(\mcX_*,\mcX_{n,j}\right)\right)
\\
=& 2\sigma_w^2\left(1-\tilde c\left(\mcX_*,\mcX_{n,j}\right)\right).
\end{align*}
Note that $1-\tilde c\left(\mcX_*,\mcX_{n,j}\right)=\rho_{\tilde c}^2\left(\mcX_*,\mcX_{n,j}\right)$. From Remark \ref{remark:epsmax_limit} we know that $\rho_{c}^2\left(\mcX_*,\mcX_{n,j}\right)\to 0$ with probability one. By assumption (AC.\ref{a_metr_app}) this implies that  $\rho_{\tilde c}^2\left(\mcX_*,\mcX_{n,j}\right)\to 0$ with probability one and thus $\EE\left[B_n^2\mid \mcX_*,\bX_n\right]\to 0$ with probability one. Since $\EE\left[B_n^2\mid \mcX_*,\bX_n\right]\leq 2\sigma_w^2$, dominated convergence theorem implies that $\EE\left[B_n^2\right]\to 0$.

To complete the proof it remains to show that
\[
\EE\left[\left(\mu_{NNGP}\left(\mcX_*\right)-\mu_{NN}\left(\mcX_*\right)\right)^2\right]\xrightarrow{n\to\infty}0.
\]
To this end, we rewrite the $NNGP$-estimator \eqref{eq:nngp_unbiased_est} as
\[
\mu_{NNGP}\left(\mcX_*\right) = \bw_{n}^T.\by_\mcN+\left(\bt\left(\mcX_*\right)^T-\bw_n^TT_\mcN\right)\hat\bb,\quad \bw_{n}=\Gamma\,K_\mcN^{-1}\,{\bk_{\mcN}^*}, \quad \ba_{n}=\bw_{n}-\frac{1}{m}\onev.
\]
Hence 
\[
\mu_{NNGP}\left(\mcX_*\right)-\mu_{NN}\left(\mcX_*\right)=D_n+\tilde B_n,\quad D_n:=\ba_n^T.\by_\mcN, \quad \tilde B_n:=\left(\bt\left(\mcX_*\right)^T-\bw_n^TT_\mcN\right)\hat\bb.
\]
Next, we use the upper bound
\[
\EE\left[\left(\mu_{NNGP}\left(\mcX_*\right)-\mu_{NN}\left(\mcX_*\right)\right)^2\right]\leq 2 \EE\left[D_n^2\right]+2 \EE\left[\tilde B_n^2\right]
\]
We next show that $\EE\left[D_n^2\right]\to 0$ using the methods established in the $GPnn$-part of the proof and that $\EE\left[\tilde B_n^2\right]\to 0$ using continuity of $\bt$ and the fact that $\EE\left[\|\ba_n\|_1^2\right]\to 0$.

By the assumption of the boundedness of $\bt$ and the boundedness of the noise variance we have that the conditional second moment of the $NNGP$ responses is bounded, i.e.,
\begin{align*}
\EE\left[\mcY^2\mid \mcX=\bx\right]=\EE\left[(\bt(\bx)^T.\bb+w(\mcX)+\Xi)^2\mid \mcX=\bx\right] & \leq 3\|\bt(\bx)\|_2^2\,\|\bb\|_2^2+3\sigma_\xi^2(\bx)+3\sigma_w^2
\\ & \leq \tilde C_{Y,2}<\infty
\end{align*}
for some positive constant $\tilde C_{Y,2}$. In the above inequality we have used the fact that $(a+b+c)^2\leq 3a^2+3b^2+3c^2$. As explained in the $GPnn$-part of the proof, the boundedness of $\EE\left[\mcY^2\mid \mcX=\bx\right]$ implies that
\[
 \EE\left[D_n^2\mid \mcX_*=\bx_*,\bX_n=X_n\right]\leq \tilde C_{Y,2}\,\left\|\ba_n\right\|_1^2\leq \tilde C_{Y,2}\, \left(\frac{\epsilon_m+\epsilon_E}{1-\epsilon_E}\right)^2.
\]
Using the method for handling the possible blowup of the above upper bound via the good-bad event split and bounding $\left\|\ba_n\right\|_1$ described in the $GPnn$-part of the proof, we get that $\EE\left[D_n^2\right]\to 0$.

The last part of the proof is to show that $\EE\left[\tilde B_n^2\right]\to 0$. To this end, we use the submultiplicativity of the $2$-nom to get
\[
\tilde B_n^2 = \left\|\hat\bb^T.\left(\bt\left(\mcX_*\right)-T_\mcN^T\bw_n\right)\right\|_2^2\leq \left\|\hat\bb^T\right\|_2^2\,\left\|\bt\left(\mcX_*\right)-T_\mcN^T\bw_n\right\|_2^2
\]
 Next, we split 
 \[
 \bt\left(\mcX_*\right)-T_\mcN^T\bw_n =  \bt\left(\mcX_*\right)-\frac{1}{m}T_\mcN^T\onev + \frac{1}{m}T_\mcN^T\onev-T_\mcN^T\bw_n.
 \]
 Then, 
 \[
\left\|\bt\left(\mcX_*\right)-T_\mcN^T\bw_n\right\|_2^2\leq 2
  \left\| \bt\left(\mcX_*\right)-\frac{1}{m}T_\mcN^T\onev\right\|_2^2+  2\left\| T_\mcN^T\left(\frac{1}{m}\onev-\bw_n\right)\right\|_2^2.
  \]
  Note that
  \[
   \bt\left(\mcX_*\right)-\frac{1}{m}T_\mcN^T\onev =  \bt\left(\mcX_*\right) - \frac{1}{m}\sum_{j=1}^m \bt\left(\mcX_{n,j}\right),
  \]
  thus we can use the universal consistency of $m$-NN applied to each function $t_i(\bx)$, $i=1,\dots,d_T$ to conclude that
  \[
  \EE\left[ \left\| \bt\left(\mcX_*\right)-\frac{1}{m}T_\mcN^T\onev\right\|_2^2\right]\xrightarrow{n\to\infty}0.
  \]
  Finally,
\begin{align*}
\left\| T_\mcN^T\left(\frac{1}{m}\onev-\bw_n\right)\right\|_2 = & \left\| \sum_{j=1}^m\bt\left(\mcX_{n,j}\right)\left(\frac{1}{m}-w_{n,j}\right)\right\|_2\leq \sum_{j=1}^m\left\|\bt\left(\mcX_{n,j}\right)\left(\frac{1}{m}-w_{n,j}\right)\right\|_2
\\
\leq & \sum_{j=1}^m\left\|\bt\left(\mcX_{n,j}\right)\right\|_2\,\left|\left(\frac{1}{m}-w_{n,j}\right)\right|\leq d_TB_T \|\ba_n\|_1,
\end{align*}
where in the last inequality we have used the boundedness of $t_i(\bx)$. Thus, 
\[
\EE\left[\left\| T_\mcN^T\left(\frac{1}{m}\onev-\bw_n\right)\right\|_2^2\right]\leq d_T^2B_T^2\EE\left[\|\ba_n\|_1^2\right]]\xrightarrow{n\to\infty}0,
\]
where we have used the fact that $\EE\left[\|\ba_n\|_1^2\right]]\to0$ which has been explained in the $GPnn$-part of the proof. In summary, this shows that $\EE\left[\tilde B_n^2\right]\to 0$ and finishes the entire proof.
\end{proof}

\section{Convergence Rates of $\EE\left[d_m\right]$, $\EE\left[\epsilon_m\right]$}
One of the key ingredients in the derivation of the convergence rates of the $MSE$ and its derivatives is the knowledge of the convergence rate of the expectation of the functions $d_m(\mcX_*,\bX_n)$ and $\epsilon_m(\mcX_*,\bX_n)$ as $n\rightarrow \infty$. We derive these rates in this section relying following lemma.

\begin{lemma}{\citep[Lemma 1]{Kohler2006}}\label{lemma:kohler}
Assume (AC.\ref{a_X_app}), and that the nearest neighbours are chosen according to the Euclidean metric. Let $\bX_n$ be training data sampled i.i.d. from $P_\mcX$ and let $\mcX_*\sim P_\mcX$. Let $r>0$ and assume that  $d>2r$ and that there exists $\beta> 2r\frac{d_\mcX}{d_\mcX-2r}$ such that $\EE\left[\|\bx\|_2^\beta\right]<\infty$. Define
\[
d_{\min}(\mcX_*,\bX_n):= \min_{\mcX\in\bX_n} \|\mcX-\mcX_*\|_2.
\]
Then,
\[
\EE\left[\min\left\{d_{\min}^{2r},1\right\}\right]\leq c\, n^{-2r/d_\mcX},
\]
where the constant $c>0$ depends on $d_\mcX,r,\beta$ and $\EE\left[\|\mcX\|_2^\beta\right]$.
\end{lemma}

\begin{lemma}\label{lemma:dmax_rate}
Under the same assumptions as in Lemma \ref{lemma:kohler} define $\mcX_{n,j}(\mcX_*,\bX_n)$ as the $j$-th nearest neighbour of $\mcX_*$ in the sample $\bX_n$ (assuming that ties occur with probability zero). Let
\[
d_j(\mcX_*,\bX_n):=\left\|\mcX_{n,j}(\mcX_*,\bX_n)-\mcX_*\right\|_2, \quad \langle d\rangle_{m,r}(\mcX_*,\bX_n) := \frac{1}{m}\sum_{j=1}^m \min\left\{d_j(\mcX_*,\bX_n)^{2r},1\right\}.
\]
Then, we have the following bounds 
\begin{align}
 & \EE\left[ \langle d\rangle_{m,r}\right]\leq c\, \left(\frac{m}{n}\right)^{2r/d_\mcX}, \quad 1<m\leq n, \label{davg_bound}
\\
& \EE\left[\min\left\{d_m^{2r},1\right\}\right] \leq 2^{\frac{2r}{d_\mcX}+1}\, c\, \left(\frac{m}{n}\right)^{2r/d_\mcX}, \quad 1<m\leq n/2. \label{dmax_bound}
\end{align}
\end{lemma}
\begin{proof}
First prove the bound for $\EE\left[ \langle d\rangle_{m,r}\right]$ applying a technique from \citep[Proof of Theorem 6.2]{Gyorfi2002}. Namely, we randomly split the training set $\bX_n$ into $m+1$ disjoint subsets, such that the first $m$ subsets have $\lfloor\frac{n}{m}\rfloor$ elements. Denote by $\tilde\mcX_j$ the nearest neighbour to $\mcX_*$ in the $j$-th subset, $j=1,\dots,m$. Clearly,
\[
\|\mcX_{n,j}-\mcX_*\|_2^{2r}\leq \|\tilde\mcX_{j}-\mcX_*\|_2^{2r},\quad j=1,\dots,m.
\]
Then,
\begin{align*}
\begin{split}
 \EE\left[ \langle d\rangle_{m,r}\right] & = \frac{1}{m}\EE\left[\sum_{j=1}^m  \min\left\{\left\|\mcX_{n,j}-\mcX_*\right\|_2^{2r},1\right\}\right] \leq \frac{1}{m}\EE\left[\sum_{j=1}^m \min\left\{\left\|\tilde\mcX_j-\mcX_*\right\|_2^{2r},1\right\}\right]
 \\
 & = \frac{1}{m}\sum_{j=1}^m \EE\left[\min\left\{\left\|\tilde\mcX_j-\mcX_*\right\|_2^{2r},1\right\}\right] =  \EE\left[\min\left\{\left\|\tilde\mcX_1-\mcX_*\right\|_2^{2r},1\right\}\right]
 \\
 & =  \EE\left[\min\left\{\left\|\mcX_{\lfloor\frac{n}{m}\rfloor,1}-\mcX_*\right\|_2^{2r},1\right\}\right]\leq c\, \left(\frac{n}{m}\right)^{-2r/d_\mcX},
 \end{split}
\end{align*}
where in the last line we have applied Lemma \ref{lemma:kohler}. Finally, we proceed to prove \eqref{dmax_bound}. We have
\begin{align*}
\min\left\{d_m^{2r},1\right\}+ \langle d\rangle_{m,r} & =\frac{1}{m}\left(m \min\left\{d_m^{2r},1\right\} + \sum_{i=1}^m \min\left\{d_i^{2r},1\right\}\right) 
\\
& \leq \frac{1}{m}\left(\sum_{i=1}^{m}\min\left\{d_{i+m}^{2r},1\right\}+ \sum_{i=1}^m \min\left\{d_i^{2r},1\right\}\right) = 2 \langle d\rangle_{2m,r}.
\end{align*}
Using the above inequality in combination with \eqref{davg_bound}, we get \eqref{dmax_bound} as follows.
\[
\EE\left[\min\left\{d_m^{2r},1\right\}\right]\leq \EE\left[\min\left\{d_m^{2r},1\right\}\right]+ \EE\left[\langle d\rangle_{m,r}\right]\leq 2\EE\left[ \langle d\rangle_{2m,r}\right]\leq 2^{\frac{2r}{d_\mcX}+1}\, c\, \left(\frac{m}{n}\right)^{2r/d_\mcX}.
\]
\end{proof}

\begin{lemma}\label{lemma:epsilons_rates}
Let $\bX_n$ be training data sampled i.i.d. from $P_\mcX$ and let $\mcX_*\sim P_\mcX$. Under the assumptions (AC.\ref{a_nn_app}),  (AR.\ref{aR_iso_app}) and (AR.\ref{aR_c_app}) define $\mcX_{n,j}(\mcX_*,\bX_n)$ as the $j$-th nearest neighbour of $\mcX_*$ in the sample $\bX_n$ (assuming that ties occur with probability zero). Define the following distances in terms of the kernel-induced metric $\rho_c$.
\begin{gather*}
\epsilon_{\min}\left(\mcX_*,\bX_n\right):=  \min_{\mcX\in\bX_n} \rho_c^2\left(\mcX,\mcX_*\right),\quad \epsilon_j\left(\mcX_*,\bX_n\right):=\rho_c^2\left(\mcX_{n,j},\mcX_*\right), 
\\
\langle \epsilon\rangle_m\left(\mcX_*,\bX_n\right) := \frac{1}{m}\sum_{j=1}^m \rho_c^2\left(\mcX_{n,j},\mcX_*\right).
\end{gather*}
Assume that  $d>2p$ (with $p$ defined (AR.\ref{aR_c})) in and that there exists $\beta> 2p\frac{d_\mcX}{d_\mcX-2p}$ such that $\EE\left[\left\|\mcX\right\|_2^\beta\right]<\infty$. Then, we have the following bounds
\begin{align}
& \EE\left[\epsilon_{\min}\right]\leq C\, n^{-2p/d_\mcX},
\\
& \EE\left[ \langle \epsilon\rangle_m\right]\leq C\, \left(\frac{m}{n}\right)^{2p/d_\mcX}, \quad 1<m\leq n,
\\
& \EE\left[\epsilon_m\right] \leq 2^{\frac{2p}{d_\mcX}+1}\, C\, \left(\frac{m}{n}\right)^{2p/d_\mcX}, \quad 1<m\leq n/2, \label{epsilon_max_bound}
\end{align}
where 
\[
C = \max\left\{ \frac{L_c}{\hat\ell^{2p}},1\right\}\,c
\] 
with $c$ defined in Lemma \ref{lemma:kohler} and $L_c$ defined in (AR.\ref{aR_c}).
\end{lemma}
\begin{proof}
By the assumption (AR.\ref{aR_iso_app}) the ordering of the nearest neighbour set under the Euclidean metric is the same as the ordering of the nearest neighbour set under the kernel-induced metric. Since $\epsilon_{\min} = 1-c(\mcX_{n,1}/\hat\ell,\mcX_*/\hat\ell)$, by (AR.\ref{aR_c_app}) we have that
\[
\epsilon_{\min}\leq \min\left\{\frac{L_c}{\hat\ell^{2p}}\left\|\mcX_{n,1}- \mcX_*\right\|_2^{2p},1\right\} \leq \max\left\{ \frac{L_c}{\hat\ell^{2p}},1\right\}\, \min\left\{\left\|\mcX_{n,1}- \mcX_*\right\|_2^{2p},1\right\},
\]
where in the second inequality we have used the fact that for any $a,b\geq 0$ we have $\min\{ab,1\}\leq \max\{a,1\}\min\{b,1\}$. Thus, by Lemma \ref{lemma:kohler} we get that
\[
\EE\left[\epsilon_{\min}\right] \leq  \max\left\{ \frac{L_c}{\hat\ell^{2p}},1\right\}\, \EE\left[\min\left\{\left\|\mcX_{n,1}- \mcX_*\right\|_2^{2p},1\right\}\right]\leq C\, n^{-2p/d_\mcX},
\]
where 
\[
C = \max\left\{ \frac{L_c}{\hat\ell^{2p}},1\right\}\,c
\] 
with $c$ defined in Lemma \ref{lemma:kohler}.

Next, we prove the bound for $\EE\left[ \langle \epsilon\rangle_m\right]$ applying a technique from \citep[Proof of Theorem 6.2]{Gyorfi2002}. Namely, we randomly split the training set $X_n$ into $m+1$ disjoint subsets so that the first $m$ subsets contain $\lfloor\frac{n}{m}\rfloor$ elements. Denote by $\tilde\mcX_j$ the nearest neighbour to $\mcX_*$ in the $j$-th subset. Then,
\begin{align*}
\begin{split}
 \EE\left[ \langle \epsilon\rangle_m\right] & = \frac{1}{m}\EE\left[\sum_{i=1}^m \rho_c^2\left(\mcX_{n,i},\mcX_*\right)\right]\leq \frac{1}{m}\EE\left[\sum_{j=1}^m \rho_c^2\left(\tilde \mcX_{j},\mcX_*\right)\right] = \frac{1}{m}\sum_{j=1}^m \EE\left[\rho_c^2\left(\tilde \mcX_{j},\mcX_*\right)\right] 
 \\
 & =  \EE\left[\rho_c^2\left(\tilde \mcX_{1},\mcX_*\right)\right] =  \EE\left[\rho_c^2\left(\mcX_{1,\lfloor\frac{n}{m}\rfloor},\mcX_*\right)\right]\leq C\, \left(\frac{n}{m}\right)^{-2p/d}.
 \end{split}
\end{align*}
Finally, we prove \eqref{epsilon_max_bound}. We have
\begin{align*}
\epsilon_m + \langle \epsilon\rangle_m & =\frac{1}{m}\left(m \epsilon_m + \sum_{i=1}^m \epsilon_i\right) \leq \frac{1}{m}\left(\sum_{i=1}^{m}\epsilon_{i+m}+ \sum_{i=1}^m \epsilon_i\right) = 2 \langle \epsilon\rangle_{2m}.
\end{align*}
Thus,
\[
\EE\left[\epsilon_m\right]\leq \EE\left[\epsilon_m\right]+ \EE\left[\langle \epsilon\rangle_{m}\right]\leq 2\EE\left[ \langle \epsilon\rangle_{2m}\right]\leq 2C\, \left(\frac{n}{2m}\right)^{-2p/d_\mcX}.
\]
\end{proof}

Next, we state some auxiliary results needed for establishing the asymptotic convergence rate in Proposition \ref{prop:rate_small_d}.

\begin{lemma}[Asymptotics of $d_{\min}$ - compact case.]
\label{lemma:dmin_asymp}
Let $\bX_n=(\mcX_1,\dots,\mcX_n)$ i.i.d.\ from $P_{\mcX}$ on
$\mathbb{R}^d$, and let $\mcX_*\sim P_{\mcX}$ be independent of $\bX_n$.
Assume that $P_{\mcX}$ has a density $q$ supported on a compact convex set
$C\subset\mathbb{R}^{d_{\mcX}}$, where $q$ is smooth and strictly positive on $C$. Then for every $r>0$,
\[
n^{2r/d}\,\EE\!\left[\min\left\{d_{\min}(\mcX_*,\bX_n)^{2r},1\right\}\right]
\xrightarrow{n\to\infty}
V_{d_{\mcX}}^{-2r/d_{\mcX}}\Gamma\!\left(1+\frac{2r}{d_{\mcX}}\right)
\int_C q(x)^{\,1-2r/d_{\mcX}}\,dx,
\]
where $V_{d_{\mcX}}$ denotes the volume of the Euclidean unit ball in $\mathbb{R}^{d_\mcX}$.
\end{lemma}
\begin{proof}
We begin with the compact-support asymptotic of \cite{evans} for
nearest-neighbour moments. In the notation of the present paper, it states that
if $\bU_K:=\left(\mcU_1,\dots,\mcU_K\right)$ are i.i.d.\ with density $q$ as above and
\[
\delta_{\min}(\mcU_i,\bU_K):=\min_{j\in\{1,\dots,K\}\setminus\{i\}}\|\mcU_i-\mcU_j\|_2,
\]
then, for every fixed $s>0$,
\[
K^{s/d_{\mcX}}\,\EE\!\left[\delta_{\min}(\mcU_i,\bU_K)^s\right]
\xrightarrow{K\to\infty}
V_{d_{\mcX}}^{-s/d_{\mcX}}\Gamma\!\left(1+\frac{s}{d_{\mcX}}\right)
\int_C q(x)^{\,1-s/d_{\mcX}}\,dx .
\]
Applying this with $s=2r$ and $K=n+1$ yields
\[
(n+1)^{2r/d_{\mcX}}\,\EE\!\left[\delta_{\min}(\mcU_i,\bU_{n+1})^{2r}\right]
\xrightarrow{n\to\infty}
V_{d_{\mcX}}^{-2r/d}\Gamma\!\left(1+\frac{2r}{d}\right)
\int_C q(x)^{\,1-2r/d}\,dx .
\]

We now translate this result to the independent-$\mcX_*$ setup. Let
$\mcU_0\equiv\mcX_*$ and $\mcU_i:=\mcX_i$ for $i=1,\dots,n$. Then
$\mcU_0,\dots,\mcU_n$ are i.i.d.\ from $P_{\mcX}$. Moreover,
\[
\delta_{\min}(\mcU_0,(\mcU_0,\dots,\mcU_n)) = d_{\min}(\mcX_*,\bX_n).
\]
By variable exchange symmetry, the random variables
$\delta_{\min}(\mcU_i,(\mcU_0,\dots,\mcU_n))$, $i=0,\dots,n$, are identically distributed.
Hence
\[
\EE\!\left[d_{\min}(\mcX_*,\bX_n)^{2r}\right]
=
\EE\!\left[\delta_{\min}(Z_0,(Z_0,\dots,Z_n))^{2r}\right]
=
\EE\!\left[\delta_{\min}(Z_i,(Z_0,\dots,Z_n))^{2r}\right]
\]
for any $i=1,\dots,n$. Therefore,
\[
(n+1)^{2r/d}\,\EE\!\left[d_{\min}(\mcX_*,\bX_n)^{2r}\right]
\xrightarrow{n\to\infty}
V_{d_{\mcX}}^{-2r/d}\Gamma\!\left(1+\frac{2r}{d}\right)
\int_C q(x)^{\,1-2r/d}\,dx .
\]
Since $(n+1)^{2r/d}\sim n^{2r/d}$, this is equivalent to
\[
n^{2r/d_{\mcX}}\,\EE\!\left[d_{\min}(\mcX_*,\bX_n)^{2r}\right]
\xrightarrow{n\to\infty}
V_{d_{\mcX}}^{-2r/d_{\mcX}}\Gamma\!\left(1+\frac{2r}{d_{\mcX}}\right)
\int_C q(x)^{\,1-2r/d_{\mcX}}\,dx .
\]

It remains to show that replacing $d_{\min}^{2r}$ with $\min\{d_{\min}^{2r},1\}$ does not affect the asymptotics.
Since $C$ is compact, $D:=\sup_{x,y\in C}\|x-y\|_2<\infty$, and thus
$d_{\min}(\mcX_*,\bX_n)\le D$ almost surely. Furthermore, because $q$ is
continuous and strictly positive on the compact set $C$, and $C$ is compact and
convex, the function
\[
x\mapsto P_{\mcX}[B(x,1)]
\]
is continuous and strictly positive on $C$. Hence
\[
\eta:=\inf_{x\in C} P_{\mcX}[B(x,1)] > 0.
\]
Conditioning on $\mcX_*$ therefore gives
\[
P\left[d_{\min}(\mcX_*,\bX_n)>1 \,\mid\, \mcX_*=x\right]
=
\left[1-P_{\mcX}[B(x,1)]\right]^n
\leq (1-\eta)^n,
\]
and so
\[
P\left[d_{\min}(\mcX_*,\bX_n)>1\right]\leq (1-\eta)^n.
\]
Consequently,
\[
0\le
\EE\!\left[d_{\min}(\mcX_*,\bX_n)^{2r}\right]
-
\EE\!\left[\min\bigl\{d_{\min}(\mcX_*,\bX_n)^{2r},1\bigr\}\right]
\le
D^{2r}(1-\eta)^n.
\]
The right-hand side decays exponentially, hence is negligible compared with
$n^{-2r/d}$. Therefore $d_{\min}^{2r}$ and $\min\{d_{\min}^{2r},1\}$ have the same
asymptotic behaviour, and
\[
n^{2r/d_{\mcX}}\,\EE\!\left[\min\bigl\{d_{\min}(\mcX_*,\bX_n)^{2r},1\bigr\}\right]
\longrightarrow
V_{d_{\mcX}}^{-2r/d_{\mcX}}\Gamma\!\left(1+\frac{2r}{d_{\mcX}}\right)
\int_C q(x)^{\,1-2r/d_{\mcX}}\,dx .
\]
This proves the Lemma.
\end{proof}

\begin{lemma}[Asymptotics of $d_m$ -- compact case.]
\label{lem:dm_asymp}
Let $\bX_n=(\mcX_1,\dots,\mcX_n)$ i.i.d.\ from $P_{\mcX}$ on
$\mathbb{R}^d$, and let $\mcX_*\sim P_{\mcX}$ be independent of $\bX_n$.
Assume that $P_{\mcX}$ has a density $q$ supported on a compact convex set
$C\subset\mathbb{R}^d$, where $q$ is smooth and strictly positive on $C$. Then for every $r>0$, $1<m\leq n/2$, and $n$ large enough we have
\[
\EE\!\left[\min\left\{d_{m}(\mcX_*,\bX_n)^{2r},1\right\}\right]\leq c_1\,\left(\frac{m}{n}\right)^{-2r/d}\,
\]
where $0<c_1<\infty$ depends on $r$, $d_\mcX$ and $P_{\mcX}$.
\end{lemma}
\begin{proof}
We randomly split the training set $\bX_n$ into $m+1$ disjoint subsets, such that the first $m$ subsets have $\lfloor\frac{n}{m}\rfloor$ elements. Denote by $\tilde\mcX_j$ the nearest neighbour to $\mcX_*$ in the $j$-th subset, $j=1,\dots,m$. Clearly,
\[
\|\mcX_{n,j}-\mcX_*\|_2^{2r}\leq \|\tilde\mcX_{j}-\mcX_*\|_2^{2r},\quad j=1,\dots,m.
\]
Then,
\begin{align*}
\begin{split}
 \EE\left[ \langle d\rangle_{m,r}\right] & = \frac{1}{m}\EE\left[\sum_{j=1}^m  \min\left\{\left\|\mcX_{n,j}-\mcX_*\right\|_2^{2r},1\right\}\right] \leq \frac{1}{m}\EE\left[\sum_{j=1}^m \min\left\{\left\|\tilde\mcX_j-\mcX_*\right\|_2^{2r},1\right\}\right]
 \\
 & = \frac{1}{m}\sum_{j=1}^m \EE\left[\min\left\{\left\|\tilde\mcX_j-\mcX_*\right\|_2^{2r},1\right\}\right] =  \EE\left[\min\left\{\left\|\tilde\mcX_1-\mcX_*\right\|_2^{2r},1\right\}\right]
 \\
 & =  \EE\left[\min\left\{\left\|\mcX_{\lfloor\frac{n}{m}\rfloor,1}-\mcX_*\right\|_2^{2r},1\right\}\right].
 \end{split}
\end{align*}
To prove the compact-$P_\mcX$ case we directly apply Lemma \ref{lemma:dmin_asymp} with $n$ replaced by $n/m$ which implies that for $n/m$ large enough
\[
 \EE\left[ \langle d\rangle_{m,r}\right]\leq \EE\left[\min\left\{\left\|\mcX_{\lfloor\frac{n}{m}\rfloor,1}-\mcX_*\right\|_2^{2r},1\right\}\right]\leq c_0\,\left(\frac{n}{m}\right)^{-2r/d},
\]
where $c_0=V_{d_{\mcX}}^{-2r/d_{\mcX}}\Gamma\!\left(1+\frac{2r}{d_{\mcX}}\right)
\int_C q(x)^{\,1-2r/d_{\mcX}}\,dx$. Finally, we have
\begin{align*}
\min\left\{d_m^{2r},1\right\}+ \langle d\rangle_{m,r} & =\frac{1}{m}\left(m \min\left\{d_m^{2r},1\right\} + \sum_{i=1}^m \min\left\{d_i^{2r},1\right\}\right) 
\\
& \leq \frac{1}{m}\left(\sum_{i=1}^{m}\min\left\{d_{i+m}^{2r},1\right\}+ \sum_{i=1}^m \min\left\{d_i^{2r},1\right\}\right) = 2 \langle d\rangle_{2m,r}.
\end{align*}
Using the above inequality in combination with the previous bound for $ \EE\left[ \langle d\rangle_{m,r}\right]$, we get the result of the Lemma as follows.
\[
\EE\left[\min\left\{d_m^{2r},1\right\}\right]\leq \EE\left[\min\left\{d_m^{2r},1\right\}\right]+ \EE\left[\langle d\rangle_{m,r}\right]\leq 2\EE\left[ \langle d\rangle_{2m,r}\right]\leq 2^{\frac{2r}{d_\mcX}+1}\, c_0\, \left(\frac{m}{n}\right)^{2r/d_\mcX}.
\]
\end{proof}

\section{Convergence Rates Proof}
\label{app:convergence_rates}
\begin{lemma}\label{lemma:subset_expectation_bound}
$\Omega\subset\RR^D$ be a probability space with probability measure $P$ and let $S\subsetneq \Omega$ satisfy $0<\mathcal P[S]<1$. Let $g:\ \RR^D\to \RR_{\geq 0}$ be a measurable function such that $c_{g,2}:=\EE\left[g(\mcX)^2\right]<\infty$. Define the conditional expectation
\[
\EE\left[g(\mcX)\mid\mcX\in S\right]:=\frac{1}{\mathcal{P}[S]}\,\int_{S}g\, dP.
\]
Let $S^c:=\Omega-S$. Then, we have
\[
\EE\left[g(\mcX)\right]\leq \EE\left[g(\mcX)\mid\mcX\in S^c\right] + \sqrt{c_{g,2}}\sqrt{\mathcal{P}[S]},
\]
\end{lemma}
\begin{proof}
Decompose the expectation as follows
\[\EE\left[g(\mcX)\right]=\int_{\Omega}g\,dP=\int_{\Omega-S}g\, dP + \int_{S} g\, dP\leq\EE\left[g(\mcX)\mid\mcX\in S^c\right]+\int_{S} g\, dP,\]
where we have used the fact that $\mathcal{P}(S^c)\leq1$ and the definition of the conditional expectation.
What is more, by the Cauchy-Schwarz inequality we have
\[
\int_{S} g\, dP = \EE\left[ g(\mcX) \mathcal{I}_{S}(\mcX)\right]\leq \sqrt{\EE\left[ g(\mcX)^2\right]} \sqrt{ \EE\left[\mathcal{I}_{S}(\mcX)\right]} = \sqrt{c_{g,2}}\sqrt{\mathcal{P}[S]},
\]
where $\mathcal{I}_{S}$ is the indicator function of $S$. Combining the above two inequalities we get the desired result.
\end{proof}

\begin{lemma}\label{cond_expectation_eps_max_bound}
Let $\epsilon_m(\mcX_*,\bX_n)$ and $d_m(\mcX_*,\bX_n)$ be as defined in Lemma \ref{lemma:epsilons_rates} and Lemma \ref{lemma:dmax_rate} and let $\mcX_*\sim P_\mcX$ and $\bX_n\sim P_\mcX^n$. For any $s>0$, $0<R\leq1$ we have
\begin{align}
& \EE\left[\epsilon_m\mid\mcX_*,\epsilon_m< R\right]\leq \EE\left[\epsilon_m\mid\mcX_*\right], \label{eps_max_cond_exp}
 \\
 &  \EE\left[d_m^s\mid \mcX_*,d_m< R\right]\leq \EE\left[\min\left\{d_m^s,1\right\}\mid\mcX_*\right]. \label{d_max_cond_exp}
\end{align}
\end{lemma}
\begin{proof}
Using the definition of conditional expectation we have
\begin{align*}
\begin{split}
 \EE\left[\epsilon_m\mid\mcX_*\right] & = \int \epsilon_m(X_n,\mcX_*) dP_\mcX^n(X_n) = \int_{\epsilon_m< R}\epsilon_m(X_n,\mcX_*) dP_\mcX^n(X_n) 
\\ & + \int_{\epsilon_m\geq R}\epsilon_m(X_n,\mcX_*) dP_\mcX^n(X_n) = \mathcal{P}[\epsilon_m< R\mid \mcX_*]\, \EE\left[\epsilon_m\mid \mcX_*,\epsilon_m< R\right]
\\ &+ \mathcal{P}[\epsilon_m\geq R\mid\mcX_*]\, \EE\left[\epsilon_m\mid \mcX_*,\epsilon_m\geq R\right]
\\& = \EE\left[\epsilon_m\mid \mcX_*,\epsilon_m< R\right]
\\
& +  \mathcal{P}[\epsilon_m\geq R\mid \mcX_*]\, \left(\EE\left[\epsilon_m\mid \mcX_*,\epsilon_m\geq R\right] -  \EE\left[\epsilon_m\mid \mcX_*,\epsilon_m< R\right]\right),
\end{split}
\end{align*}
where in the last line we substituted $\mathcal{P}[\epsilon_m< R\mid \mcX_*]=1-\mathcal{P}[\epsilon_m\geq R\mid \mcX_*]$. Clearly, 
\[
 \EE\left[\epsilon_m\mid\mcX_*,\epsilon_m< R\right]\leq R\leq \EE\left[\epsilon_m\mid \mcX_*,\epsilon_m\geq R\right],
\]
which implies that 
\[
\EE\left[\epsilon_m\mid \mcX_*\right] \geq \EE\left[\epsilon_m\mid \mcX_*,\epsilon_m< R\right].
\] 
The inequality \eqref{d_max_cond_exp} can be derived in a fully analogous way.
\end{proof}

\begin{restatedresult}{Theorem~\ref{thm:mse_convergence_rate} (Convergence Rates).}
\ConvergenceRatesStatement
\end{restatedresult}
\begin{proof}
Let us start with proving the $GPnn$ part of the theorem. The $NNGP$ case is addressed at the end.
Recall from \eqref{def:mse} that for fixed $(\bx_*,X_n)$,
\begin{align*}
MSE(\mcX_*,\bX_n) &
=\EE\left[(\mcY_*-\tilde\mu_{GPnn}(\bx_*))^2\mid \mcX_*,\,\bX_n\right]
\\
& =\sigma_\xi^2+\EE\left[(f(\mcX_*)-\tilde\mu_{GPnn}(\mcX_*))^2\mid \mcX_*,\,\bX_n\right].
\end{align*}
Averaging over $\mcX_*\sim P_\mcX$ and $\bX_n\sim P_\mcX^n$ yields
\[
\mcR_n=\EE\!\left[MSE(\mcX_*,\bX_n)\right]-\sigma_\xi^2,
\]
where $\mcR_n$ is the risk \eqref{def:risk}.
 Define
\[
f_{MSE}(X_n,\bx_*) := \left|MSE(\bx_*,X_n)  - MSE_\infty \right|,\quad MSE_\infty := \sigma_\xi^2\left(1+\frac{1}{m}\right).
\]
Note that
\begin{gather*}
MSE(\bx_*,X_n) = |MSE(\bx_*,X_n) - MSE_\infty + MSE_\infty| \\
 \leq |MSE(\bx_*,X_n) - MSE_\infty| + MSE_\infty = f_{MSE}(\bx_*,X_n) + \sigma_\xi^2\left(1+\frac{1}{m}\right).
\end{gather*}
Taking expectations and subtracting $\sigma_\xi^2$ gives
\begin{equation}\label{eq:risk_vs_fmse}
\mcR_n
\le \frac{\sigma_\xi^2}{m}+\EE\!\left[f_{MSE}(\mcX_*,\bX_n)\right].
\end{equation}
Next, in order to upper bound $\EE\left[f_{MSE}\right]$ we use Lemma \ref{lemma:subset_expectation_bound} for $g \equiv f_{MSE}$ and $S \equiv \Omega_{m,n}(R)$, where 
\[
\Omega_{m,n}(R):=\left\{ (\bx_*,X_n)\mid d_m(\bx_*,X_n)\geq R\right\}, \quad 0<R\leq1.
\]
By Lemma \ref{lemma:subset_expectation_bound},
\begin{equation}\label{eq:mse_cond_ineq}
\EE\left[f_{MSE}\right]\leq \EE\left[f_{MSE}\mid d_m< R\right] + \sqrt{c^{(2)}_{m,n}}\sqrt{P\left[\Omega_{m,n}(R)\right]},
\end{equation}
where $P\left[\Omega_{m,n}(R)\right]$  is the probability of the event $\Omega_{m,n}(R)$ under the probability measure $P_\mcX\otimes P_\mcX^n$ and 
\[
c^{(2)}_{m,n} = \EE\left[f_{MSE}(\mcX_*,\bX_n)^2\right].
\]
Our goal is to show that the terms in inequality \eqref{eq:mse_cond_ineq}  have the following upper bounds when $d>4(\alpha+p)$ with $\alpha=\min\{p,q\}$.
\begin{align}\label{mse_rate_lemma_goal}
\begin{split}
& \sqrt{c^{(2)}_{m,n}}\sqrt{P\left[\Omega_{m,n}(R)\right]}\leq A_2\, m\left(\frac{m}{n}\right)^{2(\alpha+p)/d_\mcX}, \quad A_2>0,
\\
& \EE_{X_n,\bx_*}\left[f_{MSE}(X_n,\bx_*)|d_m< R\right]\leq A_1\,\left(\frac{m}{n}\right)^{2\alpha/d_\mcX}, \quad A_1>0.
\end{split}
\end{align}
Let us start with proving the first statement of \eqref{mse_rate_lemma_goal}. To this end, we first apply Lemma \ref{lemma:bad_region_bound} with $\nu=2(\alpha+p)$ which gives
\[
\sqrt{P\left[\Omega_{m,n}(R)\right]}=\sqrt{P\left[\min\left\{d_m(\mcX_*,\bX_n),1\right\}\geq R\right]}\leq \frac{1}{R^{2(\alpha+p)}}\sqrt{c}\,2^{\frac{2(\alpha+p)}{d_\mcX}+\frac{1}{2}}\left(\frac{m}{n}\right)^{2(\alpha+p)/d_\mcX}.
\]

What is more, $\sqrt{c^{(2)}_{m,n}}$ is bounded, since using the results from the proof of Theorem \ref{thm:approx_universal_consistency} we have that $f_{MSE}(X_n,\bx_*)^2$ is bounded. In particular, Equation \eqref{eq:mse_upper_bound_const} states that 
\[
f_{MSE}(X_n,\bx_*)^2 \leq \left( B_f^2\,\left(\frac{\hat\sigma_\xi^2+\hat\sigma_f^2}{\hat\sigma_f\hat\sigma_\xi}\, \sqrt{m}+1\right)^2
 + B_\xi\, \left(2+\frac{\hat\sigma_f^2}{\hat\sigma_\xi^2}\left(\frac{\hat\sigma_\xi^2+\hat\sigma_f^2}{\hat\sigma_f^2}\right)^2\right)\right)^2.
\]
 Thus, the product $\sqrt{c^{(2)}_{m,n}}\sqrt{\Pi_{m,n}(R)}$ is upper bounded as
 \begin{equation}\label{eq:bad_term_bound}
\sqrt{c^{(2)}_{m,n}}\sqrt{P\left[\Omega_{m,n}(R)\right]} \leq A_2\,m\,\left(\frac{m}{n}\right)^{2(\alpha+p)/d}
 \end{equation}
 with 
 \[
 A_2 =  \frac{1}{R^{2(\alpha+p)}}\sqrt{c}\,2^{\frac{2(\alpha+p)}{d_\mcX}+\frac{1}{2}}\left(B_f^2\,\left(\frac{\hat\sigma_\xi^2+\hat\sigma_f^2}{\hat\sigma_f\hat\sigma_\xi}+1\right)^2
 + B_\xi\, \left(2+\frac{\hat\sigma_f^2}{\hat\sigma_\xi^2}\left(\frac{\hat\sigma_\xi^2+\hat\sigma_f^2}{\hat\sigma_f^2}\right)^2\right)\right)
 \]
 whenever $d > 4(\alpha+p)$. This proves the first statement of \eqref{mse_rate_lemma_goal}.

Let us next move to the proof of the second statement of \eqref{mse_rate_lemma_goal}. We use the fact that $f_{MSE}(X_n,\bx_*)$ has an upper bound given by Theorem \ref{thm:mse_key_bound} which we repeat below for reader's convenience (we put $L_\xi=0$ since $\sigma_\xi^2$ is constant).
\begin{align*}
\begin{split}
f_{MSE}(X_n,\bx_*) \leq &  \Big{(}\left( \left|f(\bx_*)\right|\, + 2 B_f L_f \min\{d_m^{q},1\} \right) \frac{\epsilon_m+\epsilon_E}{1-\epsilon_E} + 2 B_f L_f \min\{d_m^{q},1\}\Big{)}^2
\\
& + \sigma_\xi^2\frac{3}{m}\,\frac{\epsilon_m+\epsilon_E}{(1-\epsilon_E)^2}.
\end{split}
\end{align*}
Next, we assume that $d_m<R$ with
\[R = \min\left\{\hat\ell\,(8L_c)^{-\frac{1}{2p}},1\right\}.\]
By (AR.\ref{aR_c}) this implies that $\epsilon_m<1/8$ which combined with the upper bounds $\epsilon_E\leq 4\epsilon_m$ (see Lemma \ref{lemma:epsilons_rels}) and $\epsilon_m,\epsilon_E\leq 1$ gives
\[\frac{\epsilon_m+\epsilon_E}{1-\epsilon_E}\leq 10\epsilon_m,\quad \frac{\epsilon_m+\epsilon_E}{1-\epsilon_E}\leq \frac{2}{1-4\epsilon_m} \leq 4,\quad \frac{1}{1-\epsilon_E}\leq \frac{1}{1-4\epsilon_m}\leq 2.\]
This in turn allows us to further upper bound $f_{MSE}(X_n,\bx_*)$ as follows.
\begin{align*}
\begin{split}
f_{MSE}(X_n,\bx_*) \leq \left(10 \left( |f(\bx_*)|+ 2 B_f L_f\,R^q \right)\,\epsilon_m  +  2 B_f L_f\,  \min\{d_m^{q},1\}\right)^2
 +60\, \frac{\sigma_\xi^2}{m}\,\epsilon_m.
\end{split}
\end{align*}
Next, we expand the squared term and apply the bounds $d_m^q \leq R^{q}$, $\epsilon_m\leq 1$, $R\leq 1$ and $|f(\bx_*)|\leq B_f$ in suitable places. This yields
\begin{align}\label{eq:mse_indeq_dmax}
\begin{split}
f_{MSE}(X_n,\bx_*)\leq  & 20\left(B_f\left(1+ 2 L_f \right)\left( 5+12 L_f\right)+3\, \frac{\sigma_\xi^2}{m}\right)\, \epsilon_m 
 + (2 B_f L_f)^2\, \min\left\{d_m^{2q},1\right\}.
\end{split}
\end{align}
Next, we evaluate the conditional expectation $\EE\left[* \mid \mcX_*,d_m< R\right]$ of the both sides of the inequality \eqref{eq:mse_indeq_dmax}. Assumption (AR.\ref{aR_iso_app}) implies that $\epsilon_m$ is the squared kernel-metric distance from $\mcX_*$ to it's $m$-th nearest neighbour in $\bX_n$ and $d_m$ is is the Euclidean distance of \emph{the same} $m$-th nearest neighbour from $\mcX_*$. Furthermore, by assumption (AR.\ref{aR_c_app}) we have
\[
\epsilon_{m}\leq \max\left\{\frac{L_c}{\hat\ell^{2p}}\,d_m^{2p},1\right\}\leq \max\left\{\frac{L_c}{\hat\ell^{2p}},1\right\}\,\min\left\{d_m^{2p},1\right\},
\]
where we have used the fact that for any $a,b\geq 0$ we have  $\min\{ab,1\}\leq \max\{a,1\}\min\{b,1\}$. 

By Lemma \ref{cond_expectation_eps_max_bound} we have the inequality
\begin{gather*}
 \EE\left[\min\{d_m^{2q},1\}\mid \mcX_*,d_m< R\right] \leq \EE\left[\min\{d_m^{2q},1\}\mid \mcX_*\right].
\end{gather*}
After plugging the above bounds into inequality \eqref{eq:mse_indeq_dmax} and applying Lemma \ref{lemma:dmax_rate} and Lemma \ref{lemma:epsilons_rates} (after taking the conditional expectation $\EE\left[* \mid d_m< R\right]$ of both sides), we get
\begin{equation}\label{eq:good_term_bound}
\EE\left[f_{MSE}(X_n,\bx_*)\mid d_m< R\right]\leq C_p \left(\frac{m}{n}\right)^{2p/d} + C_q \left(\frac{m}{n}\right)^{2q/d},
\end{equation}
where
\begin{gather*}
C_p =  2^{\frac{2p}{d}+3}\,5\left(B_f\left(1+ 2 L_f \right)\left( 5+12 L_f\right)+3\, \sigma_\xi^2\right)\, \max\left\{ \frac{L_c}{\hat\ell^{2p}},1\right\}\,c, \\
C_q =2^{\frac{2q}{d}+3}\,(B_f L_f)^2 \, c
\end{gather*}
with $c$ defined in Lemma \ref{lemma:kohler}. Combining \eqref{eq:risk_vs_fmse} with \eqref{eq:mse_cond_ineq}, \eqref{eq:bad_term_bound} and \eqref{eq:good_term_bound}
gives
\[
\mcR_n
\le \frac{\sigma_\xi^2}{m}
+A_1\Bigl(\frac{m}{n}\Bigr)^{\frac{2\alpha}{d_\mcX}}
+A_2\,m\Bigl(\frac{m}{n}\Bigr)^{\frac{2(\alpha+p)}{d_\mcX}},
\]
which is \eqref{eq:convergence_limit}. 

\paragraph{$NNGP$ case.}
For $NNGP$, use the bias--variance decomposition (see Lemma~\ref{lemma:bias_variance}) which splits the $NNGP$ MSE
into a $GPnn$-type term plus a random-field term. The $GPnn$-type term is controlled exactly as above (with the relevant
H\"older exponent), while the random-field term
\[
\var_{RF}=\sigma_w^2+\Gamma^2\,{{k}_{\mcN}^*}^T\, K_{\mcN}^{-1}\tilde K_{\mcN} K_{\mcN}^{-1}{{k}_{\mcN}^*}-2\Gamma\,{{k}_{\mcN}^*}^T\, K_{\mcN}^{-1}\tilde k_\mcN^*
\]
is bounded on the good region $\{d_m<R\}$ by using the inequality \eqref{eq:varRF_bound}
\[
\left|\frac{\var_{RF}}{\sigma_w^2} \right|\leq \left(2\tilde\epsilon_m+\tilde\epsilon_E\right) \frac{\epsilon_m+1}{1-\epsilon_E}+\left( \frac{\epsilon_E+\epsilon_m}{1-\epsilon_E}\right)^2\left(1+ \tilde\epsilon_E\right)
\]
derived in the proof of Lemma~\ref{lemma:unbiased_estimator} together with (AR.\ref{aR_c2_app}). This produces contributions of order
$(m/n)^{2q_0/d}$ and $(m/n)^{2q_i/d}$ and hence the overall good-region rate $(m/n)^{2\alpha/d}$ with
$\alpha=\min\{p,q_0,q_1,\dots,q_{d_T}\}$. The bad-region term is treated identically using
Lemma~\ref{lemma:bad_region_bound} with $\nu=2(\alpha+p)$, making use of the upper bound \eqref{eq:gpnn_mse_upper_bound} for the random-field term
\begin{equation*}
\var_{RF}\leq\sigma_w^2\left(1+\sqrt{m}\,\Gamma\,\frac{\hat\sigma_f}{\hat\sigma_\xi}\right)^2\leq m\,\sigma_w^2\left(1+\Gamma\,\frac{\hat\sigma_f}{\hat\sigma_\xi}\right)^2.
\end{equation*}
derived in the proof of Theorem~\ref{thm:approx_universal_consistency}. This completes the proof.
\end{proof}

\begin{restatedresult}{Proposition~\ref{prop:rate_small_d} (Asymptotic Convergence Rates).}
\ConvergenceRatesSmallDStatement
\end{restatedresult}
\begin{proof}
    From the proof of Theorem \ref{thm:mse_convergence_rate} and the proof of Lemma \ref{lemma:bad_region_bound} with $\nu=2(\alpha+p)$ we have that 
\[
\mcR_n
\le \frac{\sigma_\xi^2}{m}
+\tilde A_1 \EE\left[\min\left\{d_m^{2\alpha},1\right\}\right]
+\tilde A_2\,m\,\sqrt{\EE\left[\min\left\{d_m^{4(\alpha+p)},1\right\}\right]}.
\]
Applying Lemma \ref{lem:dm_asymp} twice for $r=\alpha$ and $r=4(\alpha+p)$ separately we get that for $n$ large enough and $1<m<n/2$
\[
\mcR_n
\le \frac{\sigma_\xi^2}{m}
+\tilde A_1 c_{1,1}\,\left(\frac{m}{n}\right)^{-2\alpha/d_\mcX}
+\tilde A_2\sqrt{c_{1,2}}\,m\,\left(\frac{m}{n}\right)^{-2(\alpha+p)/d_\mcX},
\]
where $c_{1,1}$ depends on $d_\mcX$, $P_\mcX$ and $\alpha$ and $c_{1,2}$ depends on $d_\mcX$, $P_\mcX$, $\alpha$ and $p$. Taking $m=n^\frac{2p}{2p+d_\mcX}$ proves the Proposition.
\end{proof}

 \section{A key bound for MSE}
 \begin{lemma}\label{lemma:mu_mean_var_bound}
Assume (AC.\ref{a_X_app}-\ref{a_nn_app}), (AR.\ref{aR_f_app}) and (AC.4*).
\begin{description}
\item[AC.4*] The noise variance $\sigma_\xi^2(\bx)$ is bounded by some constant $B_\xi\geq 1$ and is $s$-H\"{o}lder-continuous, i.e., there exist constants $L_\xi\geq 1$ and $0<s\leq 1$ such that for every $\bx,\bx'$
\[\left|\sigma_\xi^2(\bx)-\sigma_\xi^2(\bx')\right|\leq L_\xi \left\|\bx-\bx'\right\|_2^s.\]
\end{description}
Let $\epsilon_m$ and $d_m$ be defined as in Remark \ref{remark:epsmax_limit} and Lemma \ref{lemma:dmax_limit} respectively and define 
\[\epsilon_E := \frac{1}{m}\, \left\|E\right\|_1\]
with $E$ is the matrix of pairwise distances between the nearest neighbours, i.e., $E_{i,j}=\epsilon_{i,j}:=\rho_c^2\left(\bx_{n,i}(\bx*),\bx_{n,j}(\bx*)\right)$, $1\leq i,j\leq m$. The following inequalities hold.
\begin{align}
\begin{split}
 \left| \EE\left[\tilde\mu_{GPnn}\mid\mcX_*,\bX_n\right]-f(\bx_*) \right|  & \leq \left( \left|f(\bx_*)\right|\, + 2 B_f L_f \min\{d_m^{q},1\} \right) \frac{\epsilon_m+\epsilon_E}{1-\epsilon_E}
 \\
 & +  2 B_f L_f \min\{d_m^{q},1\}, \label{mu_mean_bound}
 \end{split}
 \\
 \begin{split}
 \left| \var\left[\tilde\mu_{GPnn}\mid\mcX_*,\bX_n\right]- \frac{\sigma_\xi^2(\bx_*)}{m}\right| & \leq \frac{2L_\xi B_\xi}{m}\,\min\{d_m^{s},1\}+\frac{3}{m}\,\frac{\epsilon_m+\epsilon_E}{(1-\epsilon_E)^2}\times 
 \\
 & \times \left(\sigma_\xi^2(\bx_*)+2 B_\xi L_\xi\,\min\{d_m^s,1\} \right). \label{mu_var_bound}
\end{split}
\end{align}
\end{lemma}
\begin{proof}
Let us start with the proof of \eqref{mu_mean_bound}. Denote 
\[\boldsymbol{\Delta}:=K_{\mcN}^{-1}\, \mathbf{k}^*_{\mcN}- \hat\sigma_f^2\left(K^\infty_{\mcN}\right)^{-1}\onev\]
 and $\Delta f(X):= f(X) - f(\bx_*)\onev$. Then,
 \begin{gather*}
\EE\left[\tilde\mu_{GPnn}\mid\mcX_*,\bX_n\right]= \Gamma{f(X)}^T\, K_{\mcN}^{-1}\, \mathbf{k}^*_{\mcN} =\Gamma \left(f(\bx_*)\onev+\Delta f(X)\right)^T\, \left(\hat\sigma_f^2\left(K^\infty_{\mcN}\right)^{-1}\onev+ \boldsymbol{\Delta}\right)  \\
 = f(\bx_*)+\Gamma f(\bx_*)\, \onev^T\boldsymbol{\Delta} +\Gamma \hat\sigma_f^2\,\Delta f(X)^T\, \left(K^\infty_{\mcN}\right)^{-1}\onev + \Gamma\Delta f(X)^T\, \boldsymbol{\Delta},
 \end{gather*}
Using the boundedness and the H\"{o}lder property of $f$ (assumption AR.\ref{aR_f_app}), we get
\[ ||\Delta f(X)^T||_1 \leq \min\{L_fd_m^{q},2B_f\}\leq 2 B_f L_f \min\{d_m^{q},1\}.\]
 Furthermore, taking the $1$-norms of the both sides and using the triangle inequality together with the fact that the matrix $1$-norm is submultiplicative, we obtain
 \begin{align*}
 \begin{split}
 \left|\EE\left[\tilde\mu_{GPnn}\mid\mcX_*,\bX_n\right]- f(\bx_*)\right|\leq  & \left( |f(\bx_*)|\, + 2 B_f L_f \min\{d_m^{q},1\} \right)\Gamma \|\boldsymbol{\Delta}\|_1
 \\
&  + \hat\sigma_f^2\, 2 B_f L_f \, \Gamma \left\|  \left(K^\infty_{\mcN}\right)^{-1}\onev \right\|_1\, \min\{d_m^{q},1\}.
 \end{split}
 \end{align*}
 The final result follows from the application of Equation \eqref{Kk_ineq} from Lemma \ref{lemma:K_ineqs} in Appendix \ref{appendix:matrix_inequalities} and by noting that 
 \begin{equation}\label{eq:norms_explicit}
 \left\|\hat\sigma_f^2\left(K^\infty_{\mcN}\right)^{-1}\onev\right\|_1 = \frac{m\hat\sigma_f^2}{\hat\sigma_\xi^2+m\hat\sigma_f^2}=\frac{1}{\Gamma}.
 \end{equation}
Next, we proceed to the proof of \eqref{mu_var_bound}. As shown in the proof of Lemma \ref{lemma:bias_variance} we have
\begin{gather*} 
\var\left[\tilde\mu_{GPnn}\mid\mcX_*,\bX_n\right] =\Gamma^2{\mathbf{k}^*_{\mcN}}^T\, K_{\mcN}^{-1}\Sigma_{\bxi}K_{\mcN}^{-1} \,{\mathbf{k}^*_{\mcN}},
\end{gather*}
where $\Sigma_{\bxi}:=\mathrm{diag}\{\sigma_\xi^2(\bx_{n,1}),\dots, \sigma_\xi^2(\bx_{n,m})\}$. Define $\delta\Sigma_\xi=\Sigma_{\bxi}-\sigma_\xi(\bx_*)^2\bone$. Then, we have
\begin{gather*} 
\frac{1}{\Gamma^2}\var\left[\tilde\mu_{GPnn}\mid\mcX_*,\bX_n\right] = \left( \boldsymbol{\Delta} + \hat\sigma_f^2\left(K^\infty_{\mcN}\right)^{-1}\onev\right)^T\, \left(\delta\Sigma_\xi+\sigma_\xi(\bx_*)^2\bone\right)\,\left( \boldsymbol{\Delta} + \hat\sigma_f^2\left(K^\infty_{\mcN}\right)^{-1}\onev\right) \\
=\frac{1}{\Gamma^2} \mathrm{Var}_{\mcN,\infty}  + 2 \sigma_\xi^2(\bx_*)\, \hat\sigma_f^2\, \onev^T\, \left(K^\infty_{\mcN}\right)^{-1}\, \boldsymbol{\Delta} + \sigma_\xi^2(\bx_*)\, \boldsymbol{\Delta}^T\, \boldsymbol{\Delta}
\\
+\hat\sigma_f^4\, \onev^T \left(K^\infty_{\mcN}\right)^{-1}\delta\Sigma_\xi \left(K^\infty_{\mcN}\right)^{-1}\onev+2\hat\sigma_f^2\, \onev^T\, \left(K^\infty_{\mcN}\right)^{-1}\delta\Sigma_\xi\boldsymbol{\Delta}+\boldsymbol{\Delta}^T\delta\Sigma_\xi \boldsymbol{\Delta},
\end{gather*}
where $\mathrm{Var}_{\mcN,\infty}  = \lim_{n\to\infty}\var\left[\tilde\mu_{GPnn}\mid\mcX_*,\bX_n\right]$ (see Equation \eqref{eq:var_biased_limit} in the proof of Lemma \ref{lemma:unbiased_estimator}). Taking the one-norm of the both sides   we obtain
\begin{gather*}
\frac{1}{\Gamma^2}\left|\var\left[\tilde\mu_{GPnn}\mid\mcX_*,\bX_n\right] - \mathrm{Var}_{\mcN,\infty} \right| \leq 2 \sigma_\xi^2(\bx_*)\, \hat\sigma_f^2\, \left\| \onev^T\, \left(K^\infty_{\mcN}\right)^{-1} \right\|_1 \,  \left\| \boldsymbol{\Delta} \right\|_1
\\
+\sigma_\xi^2(\bx_*)\, \left\| \boldsymbol{\Delta}^T \right\|_1 \,  \left\| \boldsymbol{\Delta} \right\|_1
\\
+\|\delta\Sigma_\xi\|_1\left(\frac{1}{m}\left(\frac{m\hat\sigma_f^2}{\hat\sigma_\xi^2+m\hat\sigma_f^2}\right)^2+2\hat\sigma_f^2\, \left\| \onev^T\, \left(K^\infty_{\mcN}\right)^{-1} \right\|_1 \,  \left\| \boldsymbol{\Delta} \right\|_1+ \left\| \boldsymbol{\Delta}^T \right\|_1 \,  \left\| \boldsymbol{\Delta} \right\|_1\right).
\end{gather*}
Next, we plug in the inequalities from Equations \eqref{Kk_ineq} and \eqref{kTK_ineq} from Lemma \ref{lemma:K_ineqs} in Appendix \ref{appendix:matrix_inequalities}. We also use Equations \eqref{eq:norms_explicit} and the fact that 
\[
 \left\|\hat\sigma_f^2\onev^T\left(K^\infty_{\mcN}\right)^{-1}\right\|_1 = \frac{\hat\sigma_f^2}{\hat\sigma_\xi^2+m\hat\sigma_f^2}, \quad \|\delta\Sigma_\xi\|_1 \leq 2 B_\xi L_\xi\, \min\{d_m^s,1\}.
\]
Then, after some algebra we get the following inequality for the variance of the estimator $\tilde\mu^*_{\mathcal{N}}$.
\begin{align*}
\begin{split}
\left|\var\left[\tilde\mu_{GPnn}\mid\mcX_*,\bX_n\right] -  \frac{\sigma_\xi^2(\bx_*)}{m} \right| & \leq \frac{2 B_\xi L_\xi}{m}\,\min\{d_m^s,1\} + \frac{1}{m}\frac{\epsilon_m+\epsilon_E}{1-\epsilon_E}\times
\\
& \times \left(2+\frac{\epsilon_m+\epsilon_E}{1-\epsilon_E}\, \right)\left(\sigma_\xi^2(\bx_*)+2 B_\xi L_\xi\,\min\{d_m^s,1\} \right).
\end{split}
\end{align*}
The final result is obtained by applying the inequality $\epsilon_E, \epsilon_m\leq 1$ to get 
\[ 2+\frac{\epsilon_m+\epsilon_E}{1-\epsilon_E} \leq 3 \frac{1}{1-\epsilon_E} .\]
\end{proof}
\begin{theorem}[A key bound for $MSE$.]\label{thm:mse_key_bound}
Assume (AC.\ref{a_X_app}-\ref{a_nn_app}), (AR.\ref{aR_f_app}), (AR.\ref{aR_xi_app}) and (AC.4*). Let $d_m$,
 $\epsilon_m$ and $\epsilon_E$ be as defined in Lemma \ref{lemma:dmax_limit}, Remark \ref{remark:epsmax_limit} and Lemma \ref{lemma:mu_mean_var_bound} respectively. Denote
\[MSE_\infty(\bx_*):=\sigma_\xi^2(\bx_*)\left(1+\frac{1}{m}\right).\]
The following bound holds for $GPnn$.
\begin{align*}
\left|MSE(\bx_*,X) - MSE_\infty(\bx_*)\right| \leq \Big{(}\left( \left|f(\bx_*)\right|\, + 2 B_f L_f \min\{d_m^{q},1\} \right) \frac{\epsilon_m+\epsilon_E}{1-\epsilon_E} 
\\
+  2 B_f L_f \min\{d_m^{q},1\}\Big{)}^2
\\
+ \frac{2 B_\xi L_\xi}{m}\,\min\{d_m^{s},1\} +\frac{3}{m}\,\frac{\epsilon_m+\epsilon_E}{(1-\epsilon_E)^2}\left(\sigma_\xi^2(\bx_*)+2 B_\xi L_\xi\, \min\{d_m^{s},1\}\right).
\end{align*}
\end{theorem}
\begin{proof}
Using the bias-variance decomposition of $MSE$ and the triangle inequality for the absolute value we get
\begin{gather*}
\left|MSE(\bx_*,X) - MSE_\infty(\bx_*)\right|  \leq \left(\EE\left[\tilde\mu_{GPnn}\mid\mcX_*,\bX_n\right]-f(\bx_*) \right)^2 
\\
+ \left| \var\left[\tilde\mu_{GPnn}\mid\mcX_*,\bX_n\right]-  \frac{\sigma_\xi^2(\bx_*)}{m}\right|.
\end{gather*}
The final form of the inequality follows from combining the results of Lemma \ref{lemma:mu_mean_var_bound}.
\end{proof}

\section{Uniform flatness of risk landscape and asymptotic vanishing of $MSE$ derivatives}
\label{app:derivatives}

\subsection{The Uniform Convergence of $MSE$ in the Hyper-parameter Space}\label{sec:mse_uniform}

\begin{restatedresult}{Theorem~\ref{thm:uniform_convergence} (Uniform convergence of MSE in the hyper-parameter space).}
\UniformConvergenceStatement
\end{restatedresult}
\begin{proof} (For $GPnn$ -- the $NNGP$ counterpart follows straightforwardly using the same techniques.)
Fix $\bx_*\in\supp(P_\mcX)$ and an (infinite) sampling sequence $X$ such that $X$ is dense $\supp(P_\mcX)$. Fix a compact subset of the hyper-parameter space $S$. Using Theorem \ref{thm:mse_key_bound}, we will construct a non-negative function $h_\Theta(\bx_*,X_n)$ that is continuous as a function of $\Theta$ and that bounds the $MSE$ as follows
\[f_{MSE}(X_n,\bx_*;\Theta):=\left| MSE(\bx_*,X_n; \Theta) - MSE_\infty(\bx_*; \Theta) \right| \leq h_\Theta(\bx_*,X_n)\quad\mathrm{for\ all}\quad \Theta\in S,\]
and that forms a monotonically decreasing sequence of functions of $\Theta$ with respect to $n$, i.e.,
\begin{equation}\label{h_conditions}
h_\Theta(\bx_*, X_{n+1}) \leq h_\Theta(\bx_*, X_n) ,\quad h_\Theta(\bx_*, X_n)\xrightarrow{n\to\infty} 0\quad\mathrm{for\ all}\quad \Theta\in S.
\end{equation}
By Dini's theorem \citep{Rudin} we have that $h_\Theta(\bx_*, X_n)\xrightarrow{n\to\infty} 0$ uniformly on $S$. Since $f_{MSE}(X_n,\bx_*;\Theta)$ is sandwiched between $h_\Theta(\bx_*, X_{n})$  and the constant zero function it follows that $f_{MSE}(X_n,\bx_*;\Theta)\xrightarrow{n\to\infty} 0$ uniformly on $S$.

It remains to construct $h_\Theta(\bx_*, X_n)$. To this end, define 
\begin{align*}
\begin{split}
h_\Theta^{(0)}(\bx_*,X_n):=  \Big{(}\left( \left|f(\bx_*)\right|\, + 2 B_f L_f \min\{d_m^{q},1\} \right) \frac{\epsilon_m+\epsilon_E}{1-\epsilon_E} 
+  2 B_f L_f \min\{d_m^{q},1\}\Big{)}^2
\\
+ \frac{2 B_\xi L_\xi}{m}\,\min\{d_m^{s},1\} +\frac{3}{m}\,\frac{\epsilon_m+\epsilon_E}{(1-\epsilon_E)^2}\left(\sigma_\xi^2(\bx_*)+2 B_\xi L_\xi\, \min\{d_m^{s},1\}\right).
\end{split}
\end{align*}
By Theorem \ref{thm:mse_key_bound} we have that $f_{MSE}(X_n,\bx_*;\Theta)\leq h_\Theta^{(0)}(\bx_*,X_n)$ for all $\Theta\in S$. The function $\epsilon_m(\bx_*,X_n)$ decreases monotonically with $n$ since the nearest neighbours are chosen with respect to the kernel-induced metric and $\epsilon_m(\bx_*,X_n)$ is just the distance between $\bx_*$ and its $m$-th nearest neighbour in $X_n$. However, $\epsilon_E(\bx_*,X_n)$ may not decrease with $n$, so $h_\Theta^{(0)}(\bx_*,X_n)$ may not monotonically decrease with $n$ as well. To fix this, we will find an upper bound for $\epsilon_E(\bx_*,X_n)$ which does decrease monotonically with $n$.

By (AR.\ref{aR_iso}) the nearest neighbours can be equivalently chosen according to the Euclidean metric, thus the nearest-neighbour set is independent of the length scale choice $\hat\ell$ which enters the kernel-induced metric. Thus, we have 
\[ 
\epsilon_m(\bx_*, X_n) = \rho_c^2\left(\frac{d_m(\bx_*, X_n)}{\hat\ell}\right), \quad d_m(\bx_*, X_n) = \|\bx_*-\bx_{n,m}(\bx_*)\|_2.
\]
What is more, by (AR.\ref{aR_iso}) the kernel metric is a strictly increasing function of the Euclidean distance. Thus, $\epsilon_E(\bx_*, X_n)$ is upper bounded by the $\epsilon_E$ calculated for the nearest neighbour configuration where the nearest neighbours are grouped on the antipodal points of the Euclidean ball of the radius $d_m(\bx_*, X_n)$, i.e.
\[
\bx_{n,1}=\dots=\bx_{n,m-1}=2\bx_*-\bx_{n,m}.
\]
In other words,
\[
\epsilon_E(\bx_*, X_n)\leq \tilde\epsilon_E(\bx_*, X_n):=(m-1)\, \rho_c^2\left(\frac{2d_m(\bx_*, X_n)}{\hat\ell}\right).
\]
The so-defined function $\tilde\epsilon_E(\bx_*, X_n)$ is now monotonically decreasing with $n$. Hence, we put 
\begin{align}\label{mse_bound_monotone}
\begin{split}
h_\Theta(\bx_*,X_n):=  \Big{(}\left( \left|f(\bx_*)\right|\, + 2 B_f L_f \min\{d_m^{q},1\} \right) \frac{\epsilon_m+\tilde\epsilon_E}{1-\tilde\epsilon_E} 
+  2 B_f L_f \min\{d_m^{q},1\}\Big{)}^2
\\
+ \frac{2 B_\xi L_\xi}{m}\,\min\{d_m^{s},1\} +\frac{3}{m}\,\frac{\epsilon_m+\tilde\epsilon_E}{(1-\tilde\epsilon_E)^2}\left(\sigma_\xi^2(\bx_*)+2 B_\xi L_\xi\, \min\{d_m^{s},1\}\right).
\end{split}
\end{align}
Clearly, we have
\[
f_{MSE}(X_n,\bx_*;\Theta)\leq h_\Theta^{(0)}(\bx_*,X_n)\leq h_\Theta(\bx_*,X_n). 
\]
The upper bound $h_\Theta(\bx_*,X_n)$ satisfies all the conditions of Dini's theorem: it is monotonically decreasing with $n$ and is also a continuous function of $\Theta$. This is because only the length scale $\hat\ell$ enters the formula \eqref{mse_bound_monotone} explicitly through $\epsilon_m$ and $\tilde\epsilon_E$ which are computed via the kernel metric and the kernel function is assumed to be continuous with respect to its arguments.
\end{proof}

\subsection{The limits and the convergence rates of $MSE$ derivatives}\label{sec:mse_derivatives}
Using the $MSE$ bias-variance expansion formula, we get that for any of the hyperparamteters $\phi \in\{\hat \sigma_\xi^2,\hat\sigma_f^2,\hat\ell\}$ 
\begin{align}
\begin{split}\label{eq:dMSE_dphi_expansion}
\frac{\partial MSE_{GPnn}(\bx_*,X)}{\partial \phi} = &\, 2\, \mathrm{Bias}_{GPnn}\left( \mcX_*,\bX_n\right) \,\frac{\partial \EE\left[\tilde\mu_{GPnn}\mid\mcX_*,\bX_n\right] }{\partial \phi} 
\\
+ & \frac{\partial\var\left[\tilde\mu_{GPnn}\mid \mcX_*,\bX_n\right]}{\partial \phi},
\end{split}
\end{align}
where 
\[
 \mathrm{Bias}_{GPnn}\left( \mcX_*,\bX_n\right) = \Gamma\, {\mathbf{k}^*_{\mcN}}^T\, K_{\mcN}^{-1}f(\bX_{\mcN}) - f(\mcX_*).
\]
For simplicity, throughout this Section we adapt the homoscedastic noise model from (AR.\ref{aR_xi_app}), however some of the results can be extended to encompass heteroscedastic noise. Using the well-known formulas for matrix derivatives, we further obtain
\begin{equation}\label{eq:dmu_dphi}
\frac{1}{\Gamma}\frac{\partial \EE\left[\tilde\mu_{GPnn}\mid \mcX_*,\bX_n\right] }{\partial \phi} = \left(\frac{\partial\mathbf{k}^*_{\mathcal{N}}}{\partial \phi}\right)^T\,K_{\mathcal{N}}^{-1}f(X) -  {\mathbf{k}^*_{\mathcal{N}}}^T \,K_{\mathcal{N}}^{-1}\,\frac{\partial K_{\mathcal{N}}}{\partial \phi} \,K_{\mathcal{N}}^{-1}f(X),
\end{equation}
\begin{gather}\label{eq:dvar_dphi}
\frac{1}{\Gamma^2}\frac{1}{2\sigma_\xi^2}\frac{\partial \mathrm{Var}\left[\tilde\mu_{GPnn}\mid \mcX_*,\bX_n\right]}{\partial \phi}= {\mathbf{k}^*_{\mathcal{N}}}^T\, K_{\mathcal{N}}^{-2}\,\frac{\partial\mathbf{k}^*_{\mathcal{N}}}{\partial \phi}  -{\mathbf{k}^*_{\mathcal{N}}}^T\,K_{\mathcal{N}}^{-2}\,\frac{\partial K_{\mathcal{N}}}{\partial \phi} K_{\mathcal{N}}^{-1}\,{\mathbf{k}^*_{\mathcal{N}}} .
\end{gather}

\begin{lemma}\label{lemma:ksc_nv_derivatives_consistency}
The derivatives of the expected value and the variance of the estimator $\tilde\mu_{GPnn}$ with respect to the noise variance and the kernel scale read
\begin{align}
& \frac{\partial\EE\left[\tilde\mu_{GPnn}\mid \mcX_*,\bX_n\right] }{\partial \left(\hat\sigma_\xi^2\right)} = \frac{1}{m\hat\sigma_f^2}\, {\mathbf{k}^*_{\mathcal{N}}}^T\,K_{\mathcal{N}}^{-1}\left(f(X_{\mcN})-\left(\hat\sigma_\xi^2+m\hat\sigma_f^2\right)K_{\mathcal{N}}^{-1}f(X_{\mcN})\right), \label{eq:dEmu_dnv}
\\
& \frac{\partial \mathrm{Var}\left[\tilde\mu_{GPnn}\mid \mcX_*,\bX_n\right] }{\partial \left(\hat\sigma_\xi^2\right)} = \frac{2\sigma_\xi^2}{m}\frac{\Gamma}{ \hat\sigma_f^2}\, {\mathbf{k}^*_{\mathcal{N}}}^T\,K_{\mathcal{N}}^{-2}\left( \mathbf{k}^*_{\mathcal{N}} - \left(\hat\sigma_\xi^2+m\hat\sigma_f^2\right)K_{\mathcal{N}}^{-1}\mathbf{k}^*_{\mathcal{N}}\right), \label{eq:dVarmu_dnv}
\\
& \frac{\partial\EE\left[\tilde\mu_{GPnn}\mid \mcX_*,\bX_n\right]}{\partial \left(\hat\sigma_f^2\right)} =-\frac{\hat\sigma_\xi^2}{m \hat\sigma_f^4}\, {\mathbf{k}^*_{\mathcal{N}}}^T\,K_{\mathcal{N}}^{-1}\left(f(X_{\mcN})-\left(\hat\sigma_\xi^2+m\hat\sigma_f^2\right)K_{\mathcal{N}}^{-1}f(X_{\mcN})\right), \label{eq:dEmu_dksc}
\\
& \frac{\partial \mathrm{Var}\left[\tilde\mu_{GPnn}\mid \mcX_*,\bX_n\right] }{\partial \left(\hat\sigma_f^2\right)} = - \frac{2\sigma_\xi^2}{m}\frac{\hat\sigma_\xi^2\Gamma}{\hat\sigma_f^4}\, {\mathbf{k}^*_{\mathcal{N}}}^T\,K_{\mathcal{N}}^{-2}\left( \mathbf{k}^*_{\mathcal{N}} - \left(\hat\sigma_\xi^2+m\hat\sigma_f^2\right)K_{\mathcal{N}}^{-1}\mathbf{k}^*_{\mathcal{N}}\right).\label{eq:dVarmu_dksc}
\end{align}
Consequently, under the assumptions (AC.\ref{a_X_app}-\ref{a_f_app}), (AC.\ref{a_metr_app}) and (AR.\ref{aR_xi_app}) we have the following limits holding for every test point $\bx_*\in\supp_{\rho_c}(P_\mcX)$ and almost every sampling sequence $\bX_n\sim P_\mcX^n$.
\begin{align}
\begin{split}\label{eq:mse_nv_ksc_derivative_limits}
\frac{\partial MSE(\bx_*,\bX_n)}{\partial \left(\hat\sigma_\xi^2\right)}\xrightarrow{n\to\infty} 0,\quad \frac{\partial MSE(\bx_*,\bX_n)}{\partial \left(\hat\sigma_f^2\right)}\xrightarrow{n\to\infty} 0,
\end{split}
\\
\begin{split}\label{eq:mse_b_derivative_limit}
\left\|\nabla_{\hat\bb} MSE_{NNGP}(\bx_*,\bX_n)\right\|_2\xrightarrow{n\to\infty} 0.
\end{split}
\end{align}
Moreover, under (AC.\ref{a_X_app}-\ref{a_f_app}), (AC.\ref{a_metr_app}) and (AR.\ref{aR_xi_app}) and the assumptions that
\begin{itemize}
\item  the function $f$ in the $GPnn$ response model \eqref{eq:gpnn_responses} satisfies $\|f(\cdot)\|_\infty<B_f<\infty$,
\item the functions $t_i$, $i=1,\dots, d_T$ in the $NNGP$ response model \eqref{eq:nngp_responses} satisfy $\|t_i(\cdot)\|_\infty<B_T<\infty$,
\end{itemize}
and fixed number of nearest neighbours $m$ we have 
\begin{align}
\begin{split}\label{eq:mse_nv_ksc_exp_derivative_limits}
\EE\left[\frac{\partial MSE(\mcX_*,\bX_n)}{\partial \left(\hat\sigma_\xi^2\right)}\right]\xrightarrow{n\to\infty} 0, \qquad \EE\left[\frac{\partial MSE(\mcX_*,\bX_n)}{\partial \left(\hat\sigma_f^2\right)}\right]\xrightarrow{n\to\infty} 0,
\end{split}
\\
\begin{split}\label{eq:mse_b_exp_derivative_limit}
\EE\left[\left\|\nabla_{\hat\bb} MSE_{NNGP}(\mcX_*,\bX_n)\right\|_2\right]\xrightarrow{n\to\infty} 0.
\end{split}
\end{align}
\end{lemma}
\begin{proof} (For $GPnn$ -- the $NNGP$ counterpart follows straightforwardly using the same techniques.) 
First, note that the derivatives of the kernel elements with respect to the kernel scale parameter read
\[\frac{\partial k(\bx_1,\bx_2)}{\partial(\hat\sigma_f^2)} = \frac{\partial (\hat\sigma_f^2\, c(\bx_1/\hat\ell,\bx_2/\hat\ell))}{\partial\hat\sigma_f^2}=c(\bx_1/\hat\ell,\bx_2/\hat\ell) = \frac{1}{\hat\sigma_f^2}k(\bx_1,\bx_2).\]
Thus, we have the following matrix and vector derivatives
\begin{equation}\label{eq:nv_ksc_matrix_derivatives}
\frac{\partial\mathbf{k}^*_{\mathcal{N}}}{\partial (\hat\sigma_\xi^2)}=0,\quad \frac{\partial\mathbf{k}^*_{\mathcal{N}}}{\partial (\hat\sigma_f^2)}= \frac{1}{\hat\sigma_f^2}\mathbf{k}^*_{\mathcal{N}},\quad \frac{\partial K_{\mathcal{N}}}{\partial (\hat\sigma_\xi^2)} = \bone, \quad  \frac{\partial K_{\mathcal{N}}}{\partial (\hat\sigma_f^2)} =  \frac{1}{\hat\sigma_f^2}\left(K_{\mathcal{N}} - \hat\sigma_\xi^2\bone\right)
\end{equation}
The derivation of the formulas \eqref{eq:dEmu_dnv} - \eqref{eq:dVarmu_dksc} boils down to applying the chain rule and using the generic derivative formulas \eqref{eq:dmu_dphi} and \eqref{eq:dvar_dphi} together with derivatives \eqref{eq:nv_ksc_matrix_derivatives}. The resulting calculation is a straightforward but tedious task, thus we skip its details.

Finally, in order to prove the limits \eqref{eq:mse_nv_ksc_derivative_limits}, note that the expressions in the brackets in the equations \eqref{eq:dEmu_dnv} - \eqref{eq:dVarmu_dksc} vanish as $n\to\infty$ because
\[
\onev-\left(\hat\sigma_\xi^2+m\hat\sigma_f^2\right)\left(K^\infty_{\mathcal{N}}\right)^{-1}\onev =0
\]
due to Lemma \ref{lemma:kernel_limits}. The proof of the limits \eqref{eq:mse_nv_ksc_exp_derivative_limits} stating that the $MSE$ derivatives tend to zero in expectation follows the same lines as the proof of Theorem \ref{thm:universal_consistency}. Namely, using the same methods one can show that the expressions from Equations \eqref{eq:dEmu_dnv}-\eqref{eq:dVarmu_dksc} are bounded from above for any $n$ by the respective constants independent of $n$ and that depend only on the hyper-parameters and $m$. In particular, we have that
\begin{gather*}
\left|{\mathbf{k}^*_{\mathcal{N}}}^T\,K_{\mathcal{N}}^{-1}\left(f(X_{\mcN})-\left(\hat\sigma_\xi^2+m\hat\sigma_f^2\right)K_{\mathcal{N}}^{-1}f(X_{\mcN})\right)\right| \leq \left|{\mathbf{k}^*_{\mathcal{N}}}^T\,K_{\mathcal{N}}^{-1}f(X_{\mcN})\right|+
\\
+\left(\hat\sigma_\xi^2+m\hat\sigma_f^2\right)\,\left|{\mathbf{k}^*_{\mathcal{N}}}^T\,K_{\mathcal{N}}^{-2}f(X_{\mcN})\right| \leq \left\|{\mathbf{k}^*_{\mathcal{N}}}^T\,K_{\mathcal{N}}^{-1/2}\right\|_2\,\left\|K_{\mathcal{N}}^{-1/2}f(X_{\mcN})\right\|_2
\\
+\left(\hat\sigma_\xi^2+m\hat\sigma_f^2\right)\left\|{\mathbf{k}^*_{\mathcal{N}}}^T\,K_{\mathcal{N}}^{-1/2}\right\|_2\,\left\|K_{\mathcal{N}}^{-1}\right\|_2\,\left\|K_{\mathcal{N}}^{-1/2}f(X_{\mcN})\right\|_2
\\
\leq \hat\sigma_f \sqrt{m}\frac{B_f}{\hat\sigma_\xi}+\hat\sigma_f\frac{\hat\sigma_\xi^2+m\hat\sigma_f^2}{\hat\sigma_\xi^2} \sqrt{m}\frac{B_f}{\hat\sigma_\xi}\leq m\sqrt{m}\,\frac{B_f\hat\sigma_f\left(2\hat\sigma_\xi^2+\hat\sigma_f^2\right)}{\hat\sigma_\xi^3},
\end{gather*}
where we have used the facts that
\[
\left\|{\mathbf{k}^*_{\mathcal{N}}}^T\,K_{\mathcal{N}}^{-1/2}\right\|_2\leq \hat\sigma_f,\quad \left\|K_{\mathcal{N}}^{-1/2}f(X_{\mcN})\right\|_2\leq \sqrt{m}\frac{B_f}{\hat\sigma_\xi},\quad \left\|K_{\mathcal{N}}^{-1}\right\|_2\leq\frac{1}{\hat\sigma_\xi^2}.
\]
Similarly, we get
\begin{align*}
\left|{\mathbf{k}^*_{\mathcal{N}}}^T\,K_{\mathcal{N}}^{-2}\left( \mathbf{k}^*_{\mathcal{N}} - \left(\hat\sigma_\xi^2+m\hat\sigma_f^2\right)K_{\mathcal{N}}^{-1}\mathbf{k}^*_{\mathcal{N}}\right)\right| & \leq m\, \frac{\hat\sigma_f^2}{\hat\sigma_\xi^2}\frac{2\hat\sigma_\xi^2+\hat\sigma_f^2}{\hat\sigma_\xi^3}.
\end{align*}
In particular, this yields uniform bounds in $\bx_*,\bX_n$ for the $MSE$-derivatives with the leading $m$-dependence of $\mcO(m)$. This allows us to use the dominated convergence theorem to obtain \eqref{eq:mse_nv_ksc_exp_derivative_limits}.

To prove \eqref{eq:mse_b_derivative_limit} and \eqref{eq:mse_b_exp_derivative_limit}, we use the bias-variance decomposition of $MSE$ in $NNGP$ from Lemma \ref{lemma:bias_variance} and find that the $MSE$ depends only quadratically on $(\hat\bb-\bb)$. Consequently, we have
\[
\nabla_{\hat\bb}MSE_{NNGP} = -2\left(\bv^T.(\bb-\hat\bb)\right)\bv,\quad \bv(\mcX_*,\bX_n):=\bt_*^T-\Gamma\,{{k}_{\mcN}^*}^T\, K_{\mcN}^{-1}T_\mcN.
\] 
and thus $\left\|\nabla_{\hat\bb}MSE_{NNGP}\right\|_2\leq 2\|\bv\|_2^2\,\|\bb-\hat\bb\|_2$. Note that $\|\bv\|_2^2$ can be bounded using the same techniques as the one used in previous proofs to bound the $GPnn$ bias term for the regression function $f\equiv t_i$, since
\[
\|\bv\|_2^2=\sum_{i=1}^{d_T}\left(t_i\left(\mcX_*\right)-\Gamma\,{{k}_{\mcN}^*}^T\, K_{\mcN}^{-1}t_i\left(X_\mcN\right)\right)^2\leq d_T B_T^2\left(1+\frac{\Gamma\hat\sigma_f}{\hat\sigma_\xi}\sqrt{m}\right)^2.
\]
Thus, $\|\bv\|_2^2\to 0$ with probability one as $n\to\infty$ and $\EE\left[\|\bv\|_2^2\right]\to 0$ which proves \eqref{eq:mse_b_derivative_limit} and \eqref{eq:mse_b_exp_derivative_limit}.
\end{proof}

\noindent Let us next move to proving convergence results for the $\hat\ell$-derivative.
\begin{lemma}\label{lemma:kernel_derivative_limit}
Let $k(\bx_1,\bx_2)=\sigma_f^2\, c\left(\|\bx_1-\bx_2\|_2/\hat\ell\right)$ be an isotropic kernel function such that $c(u)$ is differentiable for $u> 0$, the limit $\lim_{u\to0^+} c'(u)$ exists (but may not be finite), and $0\leq c(u)\leq 1$ for all $u\geq 0$, and $c(0)=1$. Assume that the corresponding normalised kernel metric $\rho_c(u)=\sqrt{1-c(u)}$ satisfies the condition
\begin{equation}\label{eq:rho_upper_bound}
\rho_c(u)\leq L\, u^{p},\quad L>0,\quad 0<p\leq 1.
\end{equation}
Then,
\begin{equation}\label{rcprime_limit}
\lim_{u\to 0^+} u\, c'(u) =0. 
\end{equation}
\end{lemma}
\begin{proof}
First, note that $c'(u)=-(\rho^2(u))'$, thus it suffices to show that $\lim_{u\to 0^+} u\, (\rho^2(u))' =0$. What is more,
\[
\lim_{u\to 0^+}\, (u\,\rho^2(u))' = \lim_{u\to 0^+} u\, (\rho^2(u))' + \lim_{u\to 0^+} \rho^2(u) = \lim_{u\to 0^+} u\, (\rho^2(u))' ,
\]
since $\rho(0)=0$. Let $f(u) = u\,\rho^2(u)$ and $g(u)=L^2\, u^{2p+1}$, thus $f(0)=g(0)=0$. Because $\rho(u)$ satisfies \eqref{eq:rho_upper_bound}, we also have $0\leq f(u)\leq g(u)$. This implies the following bound for the right-sided derivative of $f(u)$
\[
f'(0^+)=\lim_{h\to 0^+}\frac{f(h)}{h} \leq \lim_{h\to 0^+}\frac{g(h)}{h} = g'(0^+).
\]
Therefore,
\[
\lim_{u\to 0^+}\, (u\,\rho^2(u))' \leq L^2 \lim_{u\to 0^+}\, (u^{2p+1})' = L^2 (2p+1)\lim_{u\to 0^+}\, u^{2p} = 0. 
\]
We also have that $\lim_{u\to 0^+}\, (u\,\rho^2(u))' \geq 0$, since $\rho'(0)\geq 0$ as $\rho(u)$ achieves its global minimum at $u=0$. This proves \eqref{rcprime_limit}.
\end{proof}

\begin{lemma}\label{lemma:l_derivative_consistency}
Under the assumptions of Lemma \ref{lemma:kernel_derivative_limit} and (AD.\ref{aD_iso_app}) and (AR.\ref{aR_c_app})  we have
\begin{equation}\label{eq:mse_l_derivative_ptwise_limit}
\frac{\partial MSE(\bx_*,X_n)}{\partial \hat\ell}\xrightarrow{n\to\infty} 0
\end{equation}
with probability one. Moreover, under (AC.\ref{a_X_app}-\ref{a_f_app}), (AC.\ref{a_metr_app}), (AR.\ref{aR_xi_app}), (AD.\ref{aD_iso}-\ref{aD_bnd}) and (AR.\ref{aR_c}) and the assumptions that
\begin{itemize}
\item  the function $f$ in the $GPnn$ response model \eqref{eq:gpnn_responses} satisfies $\|f(\cdot)\|_\infty<B_f<\infty$,
\item the functions $t_i$, $i=1,\dots, d_T$ in the $NNGP$ response model \eqref{eq:nngp_responses} satisfy $\|t_i(\cdot)\|_\infty<B_T<\infty$,
\end{itemize}
we have that for $(\mcX_*,\bX_n)\sim P_\mcX\otimes P_\mcX^n$ 
\begin{align}
\label{eq:mse_l_exp_derivative_limit_app}
\EE\left[\left|\frac{\partial MSE(\bx_*,X_n)}{\partial \hat\ell}\right|\right]\xrightarrow{n\to\infty} 0.
\end{align}\end{lemma}
\begin{proof}
(For $GPnn$ -- the $NNGP$ counterpart follows straightforwardly using the same techniques.) Fully analogous to the proof of Lemma \ref{lemma:ksc_nv_derivatives_consistency}. The pointwise limit is shown using Equations \eqref{eq:dMSE_dphi_expansion}-\eqref{eq:dvar_dphi} together with Lemma \ref{lemma:kernel_derivative_limit}. Note that the derivative of an isotropic kernel $k(r)$ with respect to the lengthscale reads
\[\frac{\partial k(r/\hat\ell)}{\partial \hat\ell} = -\frac{1}{\hat\ell^2} r\, k'(r/\hat\ell). \]
Thus, by Lemma \ref{lemma:kernel_derivative_limit} we have $\frac{\partial k(r/\hat\ell)}{\partial \hat\ell}\xrightarrow{n\to\infty} 0$.  The limit in expectation \eqref{eq:mse_l_exp_derivative_limit_app} is shown using the dominated convergence theorem. To this end, we derive an upper bound on the $MSE$ that is independent of $\bx_*,X_n$ using the Equations \eqref{eq:dMSE_dphi_expansion}-\eqref{eq:dvar_dphi} and assumption (AD.\ref{aD_bnd}) together with the bounds used in the proof of Theorem \ref{thm:universal_consistency}.

Starting from the bias-variance expansion \eqref{eq:dMSE_dphi_expansion} with $\phi=\hat\ell$ and taking absolute values, we have
\begin{align}\label{eq:dMSE_dl_abs_bound}
\begin{split}
\left|\frac{\partial MSE_{GPnn}(\mcX_*,\bX_n)}{\partial \hat\ell}\right| = &\, 2\, \left| \Gamma\, {\mathbf{k}^*_{\mcN}}^T\, K_{\mcN}^{-1}f(\bX_{\mcN}) - f(\mcX_*)\right| \,\left|\frac{\partial \EE\left[\tilde\mu_{GPnn}\mid\mcX_*,\bX_n\right] }{\partial \hat\ell} \right|
\\
+ &\left| \frac{\partial\var\left[\tilde\mu_{GPnn}\mid \mcX_*,\bX_n\right]}{\partial \hat\ell}\right|,
\end{split}
\end{align}
We now bound each factor on the right-hand side.

First, we bound $\left| \Gamma\, {\mathbf{k}^*_{\mcN}}^T\, K_{\mcN}^{-1}f(\bX_{\mcN}) - f(\mcX_*)\right|$ using $\left|f(\bx_*)\right|\le B_f$ and the Cauchy-Schwarz bound as follows
\[
\left|{\mathbf{k}^*_{\mcN}}^T\, K_{\mcN}^{-1}f(\bX_{\mcN})\right|
=
\left| \left({\mathbf{k}^*_{\mathcal{N}}}^T K_{\mathcal{N}}^{-1/2}\right)\left(K_{\mathcal{N}}^{-1/2}f(X_{\mathcal N})\right)\right|
\le
\left\|{\mathbf{k}^*_{\mathcal{N}}}^T K_{\mathcal{N}}^{-1/2}\right\|_2
\left\|K_{\mathcal{N}}^{-1/2}f(X_{\mathcal N})\right\|_2.
\]
By the non-negativity of the GP predictive variance (as used in the derivation leading to \eqref{eq:mse_upper_bound_const}) we have
\[
\left\|{\mathbf{k}^*_{\mathcal{N}}}^T\,K_{\mathcal{N}}^{-1/2}\right\|_2 \le \hat\sigma_f,
\qquad
\left\|K_{\mathcal{N}}^{-1/2}f(X_{\mathcal N})\right\|_2\le \sqrt{m}\,\frac{B_f}{\hat\sigma_\xi}.
\]
Therefore,
\begin{equation}\label{eq:bias_factor_bound_slow_m}
\left| \Gamma\, {\mathbf{k}^*_{\mcN}}^T\, K_{\mcN}^{-1}f(\bX_{\mcN}) - f(\mcX_*)\right|
\le B_f +\Gamma\, \hat\sigma_f\,\sqrt{m}\,\frac{B_f}{\hat\sigma_\xi}
=
B_f\left(1+\frac{\Gamma\,\hat\sigma_f}{\hat\sigma_\xi}\sqrt{m}\right).
\end{equation}
Next, we bound $\left|\partial_{\hat\ell}\EE\left[\tilde\mu_{GPnn}\mid\mcX_*,\bX_n\right] \right|$ using using \eqref{eq:dmu_dphi} as follows.
\begin{equation*}
\frac{1}{\Gamma}\left|\frac{\partial \EE\left[\tilde\mu_{GPnn}\mid\mcX_*,\bX_n\right]  }{\partial \hat\ell}\right| \leq \left|\left(\frac{\partial\mathbf{k}^*_{\mathcal{N}}}{\partial \hat\ell}\right)^T\,K_{\mathcal{N}}^{-1}f(X_\mcN)\right| +  \left|{\mathbf{k}^*_{\mathcal{N}}}^T \,K_{\mathcal{N}}^{-1}\,\frac{\partial K_{\mathcal{N}}}{\partial \hat\ell} \,K_{\mathcal{N}}^{-1}f(X_\mcN)\right|.
\end{equation*}
We bound the two RHS-terms separately. Using Cauchy-Schwarz we get
\[
\left|\left(\frac{\partial\mathbf{k}^*_{\mathcal{N}}}{\partial \hat\ell}\right)^T\,K_{\mathcal{N}}^{-1}f(X_{\mathcal N})\right|
=
\left|\left(K_{\mathcal N}^{-1/2}\frac{\partial\mathbf{k}^*_{\mathcal{N}}}{\partial \hat\ell}\right)^T
\left(K_{\mathcal N}^{-1/2}f(X_{\mathcal N})\right)\right|
\le
\left\|K_{\mathcal N}^{-1/2}\frac{\partial\mathbf{k}^*_{\mathcal{N}}}{\partial \hat\ell}\right\|_2
\left\|K_{\mathcal N}^{-1/2}f(X_{\mathcal N})\right\|_2.
\]
By (AD.\ref{aD_bnd}) and the chain rule, for any pair of inputs,
\[
\left|\frac{\partial k(\bx_1,\bx_2)}{\partial \hat\ell}\right|
=
\left|\frac{\hat\sigma_f^2}{\hat\ell}\,u\,c'(u)\right|
\le
\frac{\hat\sigma_f^2 B_c}{\hat\ell},
\qquad u=\frac{\|\bx_1-\bx_2\|_2}{\hat\ell}.
\]
Hence every entry of $\partial_{\hat\ell}\mathbf{k}^*_{\mathcal N}$ has magnitude at most $\hat\sigma_f^2 B_c/\hat\ell$, so
\[
\left\|\frac{\partial\mathbf{k}^*_{\mathcal{N}}}{\partial \hat\ell}\right\|_2 \le \sqrt{m}\,\frac{\hat\sigma_f^2 B_c}{\hat\ell}.
\]
Together with $\|K_{\mathcal N}^{-1/2}\|_2\le 1/\hat\sigma_\xi$ and
$\|K_{\mathcal N}^{-1/2}f(X_{\mathcal N})\|_2\le \sqrt{m}\,B_f/\hat\sigma_\xi$, we obtain
\begin{equation}\label{eq:dEmu_dl_term1_bound_slow_m}
\left|\left(\frac{\partial\mathbf{k}^*_{\mathcal{N}}}{\partial \hat\ell}\right)^T\,K_{\mathcal{N}}^{-1}f(X_{\mathcal N})\right|
\le
\frac{1}{\hat\sigma_\xi}\left(\sqrt{m}\,\frac{\hat\sigma_f^2 B_c}{\hat\ell}\right)\left(\sqrt{m}\,\frac{B_f}{\hat\sigma_\xi}\right)
=
\frac{B_f\hat\sigma_f^2B_c}{\hat\ell\,\hat\sigma_\xi^2}\,m.
\end{equation}
Next, we apply Cauchy--Schwarz again as follows.
\begin{align*}
\left|{\mathbf{k}^*_{\mathcal{N}}}^T \,K_{\mathcal{N}}^{-1}\,\frac{\partial K_{\mathcal{N}}}{\partial \hat\ell} \,K_{\mathcal{N}}^{-1}f(X_{\mathcal N})\right|
&=
\left|\left({\mathbf{k}^*_{\mathcal{N}}}^T K_{\mathcal N}^{-1/2}\right)
\left(K_{\mathcal N}^{-1/2}\frac{\partial K_{\mathcal{N}}}{\partial \hat\ell}K_{\mathcal N}^{-1}f(X_{\mathcal N})\right)\right| \\
&\le
\left\|{\mathbf{k}^*_{\mathcal{N}}}^T K_{\mathcal N}^{-1/2}\right\|_2\;
\left\|K_{\mathcal N}^{-1/2}\frac{\partial K_{\mathcal{N}}}{\partial \hat\ell}K_{\mathcal N}^{-1}f(X_{\mathcal N})\right\|_2.
\end{align*}
Using $\|{\mathbf{k}^*_{\mathcal{N}}}^T K_{\mathcal N}^{-1/2}\|_2\le \hat\sigma_f$ and
\[
\left\|K_{\mathcal N}^{-1/2}\frac{\partial K_{\mathcal{N}}}{\partial \hat\ell}K_{\mathcal N}^{-1}f\right\|_2
\le
\|K_{\mathcal N}^{-1/2}\|_2\left\|\frac{\partial K_{\mathcal{N}}}{\partial \hat\ell}\right\|_2 \|K_{\mathcal N}^{-1}f\|_2,
\quad
\|K_{\mathcal N}^{-1}f\|_2\le \|K_{\mathcal N}^{-1/2}\|_2\|K_{\mathcal N}^{-1/2}f\|_2,
\]
we get
\[
\left|{\mathbf{k}^*_{\mathcal{N}}}^T \,K_{\mathcal{N}}^{-1}\,\frac{\partial K_{\mathcal{N}}}{\partial \hat\ell} \,K_{\mathcal{N}}^{-1}f(X_{\mathcal N})\right|
\le
\hat\sigma_f \cdot \frac{1}{\hat\sigma_\xi}\left\|\frac{\partial K_{\mathcal{N}}}{\partial \hat\ell}\right\|_2
\left(\frac{1}{\hat\sigma_\xi}\cdot \sqrt{m}\,\frac{B_f}{\hat\sigma_\xi}\right)
=
\hat\sigma_f\,\left\|\frac{\partial K_{\mathcal{N}}}{\partial \hat\ell}\right\|_2\,\frac{B_f\sqrt{m}}{\hat\sigma_\xi^3}.
\]
Finally, by (AD.\ref{aD_bnd}) we have the uniform entrywise bound
$\left|\partial_{\hat\ell} k(\bx_i,\bx_j)\right|\le \hat\sigma_f^2B_c/\hat\ell$ for all $i,j$,
so the maximal row sum satisfies
\[
\left\|\frac{\partial K_{\mathcal{N}}}{\partial \hat\ell}\right\|_2
\le
\left\|\frac{\partial K_{\mathcal{N}}}{\partial \hat\ell}\right\|_\infty
\le
m\,\frac{\hat\sigma_f^2B_c}{\hat\ell}.
\]
Hence,
\begin{equation}\label{eq:dEmu_dl_term2_bound_slow_m}
\left|{\mathbf{k}^*_{\mathcal{N}}}^T \,K_{\mathcal{N}}^{-1}\,\frac{\partial K_{\mathcal{N}}}{\partial \hat\ell} \,K_{\mathcal{N}}^{-1}f(X_{\mathcal N})\right|
\le
\hat\sigma_f\left(m\,\frac{\hat\sigma_f^2B_c}{\hat\ell}\right)\frac{B_f\sqrt{m}}{\hat\sigma_\xi^3}
=
\frac{B_f\hat\sigma_f^3B_c}{\hat\ell\,\hat\sigma_\xi^3}\,m^{3/2}.
\end{equation}
Combining \eqref{eq:dEmu_dl_term1_bound_slow_m} and \eqref{eq:dEmu_dl_term2_bound_slow_m} with \eqref{eq:dmu_dl_bound} yields
\begin{equation}\label{eq:dEmu_dl_bound_slow_m}
\left|\frac{\partial \EE\left[\tilde\mu_{GPnn}\mid\mcX_*,\bX_n\right] }{\partial \hat\ell} \right|
\leq
\Gamma\,
\frac{B_fB_c}{\hat\ell}\left(\frac{\hat\sigma_f^2}{\hat\sigma_\xi^2}\,m+\frac{\hat\sigma_f^3}{\hat\sigma_\xi^3}\,m^{3/2}\right).
\end{equation}

As the final step, we bound $|\partial_{\hat\ell}\var\left[\tilde\mu_{GPnn}\mid \mcX_*,\bX_n\right]|$ using \eqref{eq:dvar_dphi} as follows
\[
\frac{1}{\Gamma^2}\frac{1}{2\sigma_\xi^2}
\left|\frac{\partial \var\left[\tilde\mu_{GPnn}\mid \mcX_*,\bX_n\right]}{\partial \hat\ell}\right|
\le
\left|{\mathbf{k}^*_{\mathcal{N}}}^T\, K_{\mathcal{N}}^{-2}\,\frac{\partial\mathbf{k}^*_{\mathcal{N}}}{\partial \hat\ell}\right|
+
\left|{\mathbf{k}^*_{\mathcal{N}}}^T\,K_{\mathcal{N}}^{-2}\,\frac{\partial K_{\mathcal{N}}}{\partial \hat\ell} K_{\mathcal{N}}^{-1}\,{\mathbf{k}^*_{\mathcal{N}}}\right|.
\]
For the first term,
\[
\left|{\mathbf{k}^*_{\mathcal{N}}}^T\, K_{\mathcal{N}}^{-2}\,\frac{\partial\mathbf{k}^*_{\mathcal{N}}}{\partial \hat\ell}\right|
=
\left|(K_{\mathcal N}^{-1}\mathbf k^*_{\mathcal N})^T\left(K_{\mathcal N}^{-1}\frac{\partial\mathbf{k}^*_{\mathcal{N}}}{\partial \hat\ell}\right)\right|
\le
\|K_{\mathcal N}^{-1}\mathbf k^*_{\mathcal N}\|_2\,\left\|K_{\mathcal N}^{-1}\frac{\partial\mathbf{k}^*_{\mathcal{N}}}{\partial \hat\ell}\right\|_2.
\]
Using
\begin{gather*}
\|K_{\mathcal N}^{-1}\mathbf k^*_{\mathcal N}\|_2
\le
\|K_{\mathcal N}^{-1/2}\|_2\,\|K_{\mathcal N}^{-1/2}\mathbf k^*_{\mathcal N}\|_2
\le
\frac{1}{\hat\sigma_\xi}\,\hat\sigma_f,
\\
\left\|K_{\mathcal N}^{-1}\frac{\partial\mathbf{k}^*_{\mathcal{N}}}{\partial \hat\ell}\right\|_2
\le
\|K_{\mathcal N}^{-1}\|_2\,\left\|\frac{\partial\mathbf{k}^*_{\mathcal{N}}}{\partial \hat\ell}\right\|_2
\le
\frac{1}{\hat\sigma_\xi^2}\left(\sqrt{m}\,\frac{\hat\sigma_f^2B_c}{\hat\ell}\right),
\end{gather*}
we obtain
\begin{equation}\label{eq:dVar_dl_term1_bound_slow_m}
\left|{\mathbf{k}^*_{\mathcal{N}}}^T\, K_{\mathcal{N}}^{-2}\,\frac{\partial\mathbf{k}^*_{\mathcal{N}}}{\partial \hat\ell}\right|
\le
\frac{\hat\sigma_f}{\hat\sigma_\xi}\cdot \frac{\hat\sigma_f^2B_c}{\hat\ell\,\hat\sigma_\xi^2}\,\sqrt{m}
=
\frac{\hat\sigma_f^3B_c}{\hat\ell\,\hat\sigma_\xi^3}\,\sqrt{m}.
\end{equation}
For the second term, insert $K_{\mathcal N}^{-1/2}$ and use Cauchy--Schwarz:
\begin{align*}
\left|{\mathbf{k}^*_{\mathcal{N}}}^T\,K_{\mathcal{N}}^{-2}\,\frac{\partial K_{\mathcal{N}}}{\partial \hat\ell} K_{\mathcal{N}}^{-1}\,{\mathbf{k}^*_{\mathcal{N}}}\right|
&=
\left|\left({\mathbf{k}^*_{\mathcal{N}}}^T K_{\mathcal N}^{-1/2}\right)\left(K_{\mathcal N}^{-3/2}\frac{\partial K_{\mathcal{N}}}{\partial \hat\ell} K_{\mathcal{N}}^{-1}\,{\mathbf{k}^*_{\mathcal{N}}}\right)\right| \\
&\le
\left\|{\mathbf{k}^*_{\mathcal{N}}}^T K_{\mathcal N}^{-1/2}\right\|_2\,
\left\|K_{\mathcal N}^{-3/2}\frac{\partial K_{\mathcal{N}}}{\partial \hat\ell} K_{\mathcal{N}}^{-1}\,{\mathbf{k}^*_{\mathcal{N}}}\right\|_2 \\
&\le
\hat\sigma_f\cdot \|K_{\mathcal N}^{-3/2}\|_2\left\|\frac{\partial K_{\mathcal{N}}}{\partial \hat\ell}\right\|_2 \|K_{\mathcal N}^{-1}{\mathbf{k}^*_{\mathcal{N}}}\|_2.
\end{align*}
Using $\|K_{\mathcal N}^{-3/2}\|_2\le 1/\hat\sigma_\xi^3$,
$\|K_{\mathcal N}^{-1}{\mathbf{k}^*_{\mathcal{N}}}\|_2\le \|K_{\mathcal N}^{-1/2}\|_2\|K_{\mathcal N}^{-1/2}{\mathbf{k}^*_{\mathcal{N}}}\|_2\le \hat\sigma_f/\hat\sigma_\xi$,
and $\|\partial_{\hat\ell} K_{\mathcal N}\|_2\le m\hat\sigma_f^2B_c/\hat\ell$, we get
\begin{equation}\label{eq:dVar_dl_term2_bound_slow_m}
\left|{\mathbf{k}^*_{\mathcal{N}}}^T\,K_{\mathcal{N}}^{-2}\,\frac{\partial K_{\mathcal{N}}}{\partial \hat\ell} K_{\mathcal{N}}^{-1}\,{\mathbf{k}^*_{\mathcal{N}}}\right|
\le
\hat\sigma_f\cdot \frac{1}{\hat\sigma_\xi^3}\left(m\frac{\hat\sigma_f^2B_c}{\hat\ell}\right)\frac{\hat\sigma_f}{\hat\sigma_\xi}
=
\frac{\hat\sigma_f^4B_c}{\hat\ell\,\hat\sigma_\xi^4}\,m.
\end{equation}
Combining \eqref{eq:dVar_dl_term1_bound_slow_m} and \eqref{eq:dVar_dl_term2_bound_slow_m} gives
\begin{equation}\label{eq:dVar_dl_bound_slow_m}
\left|\frac{\partial  \var\left[\tilde\mu_{GPnn}\mid \mcX_*,\bX_n\right]}{\partial \hat\ell}\right|
\le
2\sigma_\xi^2\Gamma^2
\frac{B_c}{\hat\ell}\left(\frac{\hat\sigma_f^3}{\hat\sigma_\xi^3}\sqrt{m}+\frac{\hat\sigma_f^4}{\hat\sigma_\xi^4}m\right).
\end{equation}

Substituting \eqref{eq:bias_factor_bound_slow_m}, \eqref{eq:dEmu_dl_bound_slow_m}, and \eqref{eq:dVar_dl_bound_slow_m}
into \eqref{eq:dMSE_dl_abs_bound} yields the following explicit uper bound for
$\left|\partial_{\hat\ell} MSE(\bx_*,X_n)\right|$ that is uniform in $(\bx_*,X_n)$ for fixed hyper-parameters.
\begin{align}
\begin{split}\label{eq:dMSE_dl_uniform_bound}
\left|\frac{\partial MSE_{GPnn}(\mcX_*,\bX_n)}{\partial \hat\ell}\right| \le &\, 2\, \frac{B_f^2B_c\Gamma}{\hat\ell}\left(1+\frac{\Gamma\,\hat\sigma_f}{\hat\sigma_\xi}\sqrt{m}\right)
\left(\frac{\hat\sigma_f^2}{\hat\sigma_\xi^2}\,m+\frac{\hat\sigma_f^3}{\hat\sigma_\xi^3}\,m^{3/2}\right)
\\
+ &2\sigma_\xi^2\frac{B_c\Gamma}{\hat\ell}\left(\frac{\hat\sigma_\xi^2+m\hat\sigma_f^2}{m\hat\sigma_f^2}\right)^2
\left(\frac{\hat\sigma_f^3}{\hat\sigma_\xi^3}\sqrt{m}+\frac{\hat\sigma_f^4}{\hat\sigma_\xi^4}m\right).
\end{split}
\end{align}
In particular, the leading $m$-dependence of the right-hand side is $\mcO(m^2)$.
\end{proof}

Lemma \ref{lemma:dMSE_dnv_dksc_bounds} below proves $MSE$ bounds for $GPnn$ that are a crucial component in proving Theorem \ref{thm:mse_derivatives_convergence_rate} concerning the convergence rates of the derivatives. We skip the derivation of the corresponding bounds for $NNGP$, because they have the same general forms and follow directly from the derivations presented below.

\begin{lemma}[Bounds for $MSE$ derivatives.]\label{lemma:dMSE_dnv_dksc_bounds}
Assume (AC.\ref{a_X_app}-\ref{a_nn_app}) and (AC.\ref{a_metr_app}) and (AR.\ref{aR_iso_app}-\ref{aR_xi_app}). Consider a sequence of sampling points $X_n=(\bx_1,\bx_2,\dots,\bx_n)$ and let $\bx_*\in\mathrm{Support}_{\rho_c}(P_\mcX)$, $\bx_*\notin X_n$ be a test point.
Then,
\begin{align}
\begin{split}\label{eq:dMSE_dnv_bound}
\left|\frac{\partial MSE_{GPnn}(\bx_*,X_n)}{\partial (\hat\sigma_\xi^2)} \right|\leq &  48 B_f^2 L_f^2\,\frac{1}{m\hat\sigma_f^2} \frac{1}{(1-\epsilon_E)^3}\,\min\left\{d_m^{2q},1\right\}
\\
+ & \frac{1}{m\hat\sigma_f^2}\left(24\sigma_\xi^2\frac{1}{1-\epsilon_{E,2}}\frac{1}{1-\epsilon_E}+\frac{2B_f^2}{(1-\epsilon_E)^3}\left( 78 L_f+5\right)\right)\,\epsilon_m
\end{split}
\\
\begin{split}\label{eq:dMSE_dksc_bound}
\left|\frac{\partial MSE_{GPnn}(\bx_*,X_n)}{\partial (\hat\sigma_f^2)} \right|\leq & 48 B_f^2 L_f^2\,\frac{\hat\sigma_\xi^2}{m\hat\sigma_f^2} \frac{1}{(1-\epsilon_E)^3}\,\min\left\{d_m^{2q},1\right\}
\\
+ & \frac{\hat\sigma_\xi^2}{m\hat\sigma_f^2}\left(24\sigma_\xi^2\frac{1}{1-\epsilon_{E,2}}\frac{1}{1-\epsilon_E}+\frac{2B_f^2}{(1-\epsilon_E)^3}\left( 78 L_f+5\right)\right)\,\epsilon_m. 
\end{split}
\\
\begin{split}
\left\|\nabla_{\hat\bb}MSE_{NNGP}\right\|_2\leq  & 2\,d_T\|\bb-\hat\bb\|_2\, \left(\frac{10B_T(1+2L_T)}{1-\epsilon_E}\,\epsilon_m+4B_fL_f\min\{d_m^{\overline q},1\}\right)^2,
\end{split}
\end{align}
where $\overline q:=\min_i\{q_i\}$ and $L_T:=\max_i L_i$ with the relevant constants defined in (AR.\ref{aR_f_app}). Additionally, under (AR.\ref{aR_c_app}) and (AD.\ref{aD_iso_app}-\ref{aD_bnd_app}) we have that the following upper bound holds for every $\bx_*,X_n$.
\begin{align}\label{eq:dMSE_dl_bound}
\begin{split}
 \left|\frac{\partial MSE_{GPnn}(\bx_*,X)}{\partial \hat\ell} \right| & \leq \frac{1}{\hat\ell} \Bigg{(}\frac{6B_f^2B_cL_c'(1+2L_f)}{(1-\epsilon_E)^2} \left(\frac{5(1+2L_f)}{1-\epsilon_E}\,\epsilon_m+2L_f\min\{d_m^q,1\}\right)
\\
& +  \frac{8\sigma_\xi^2\,\hat\sigma_f^2B_cL_c'}{(1-\epsilon_{E})(1-\epsilon_{E,2})}\Bigg{)}\max\left\{\hat\ell^{-2p'},1\right\}\, \min\left\{d_m^{2p'},1 \right\}.
\end{split}
\end{align}
\end{lemma}
\begin{proof}
The proof repeats the techniques that have been used in the previous parts of this paper.

We first prove the bound \eqref{eq:dMSE_dnv_bound}. The proof of the bound \eqref{eq:dMSE_dksc_bound} is almost identical. According to the Equation \eqref{eq:dMSE_dphi_expansion},
\begin{align*}
\begin{split}
\left|\frac{\partial MSE_{GPnn}(\mcX_*,\bX_n)}{\partial \hat\sigma_\xi^2}\right| = &\, 2\, \left| \Gamma\, {\mathbf{k}^*_{\mcN}}^T\, K_{\mcN}^{-1}f(\bX_{\mcN}) - f(\mcX_*)\right| \,\left|\frac{\partial \EE\left[\tilde\mu_{GPnn}\mid\mcX_*,\bX_n\right] }{\partial \hat\sigma_\xi^2} \right|
\\
+ &\left| \frac{\partial\var\left[\tilde\mu_{GPnn}\mid \mcX_*,\bX_n\right]}{\partial \hat\sigma_\xi^2}\right|,
\end{split}
\end{align*}
The expression $ \left| \Gamma\, {\mathbf{k}^*_{\mcN}}^T\, K_{\mcN}^{-1}f(\bX_{\mcN}) - f(\mcX_*)\right|$ can be suitably upper bounded using the inequality \eqref{mu_mean_bound} from Lemma \ref{lemma:mu_mean_var_bound}. 

Next, we use Equations \eqref{eq:dEmu_dnv} and \eqref{eq:dVarmu_dnv} from Lemma \ref{lemma:ksc_nv_derivatives_consistency} in order to find bounds on the relevant derivatives of the expectation and the variance of $\tilde\mu$. To this end, we use the following bounds.
\begin{gather}
\left\|f(X_{\mcN})-\left(\hat\sigma_\xi^2+m\hat\sigma_f^2\right)K_{\mathcal{N}}^{-1}f(X_{\mcN})\right\|_1 =  \Big{\|}f(X_{\mcN})-f(\bx_*)\onev + f(\bx_*)\onev+ \left(\hat\sigma_\xi^2+m\hat\sigma_f^2\right)\times \nonumber
\\
\times \left(K_{\mathcal{N}}^{-1}\, f\left(X_{n,m}\right)- f\left (\bx_*\right)\left(K^\infty_{\mathcal{N}}\right)^{-1}\onev\right) -  \left(\hat\sigma_\xi^2+m\hat\sigma_f^2\right)f\left (\bx_*\right)\left(K^\infty_{\mathcal{N}}\right)^{-1}\onev\Big{\|}_1 \nonumber
\\
= \Big{\|}f(X_{\mcN})-f(\bx_*)\onev+ \left(\hat\sigma_\xi^2+m\hat\sigma_f^2\right) \left(K_{\mathcal{N}}^{-1}\, f\left(X_{\mcN}\right)- f\left (\bx_*\right)\left(K^\infty_{\mathcal{N}}\right)^{-1}\onev\right)\Big{\|}_1 \nonumber
\\
\leq \left\|f(X_{\mcN})-f(\bx_*)\onev\right\|_1 + \left(\hat\sigma_\xi^2+m\hat\sigma_f^2\right) \left\|K_{\mathcal{N}}^{-1}\, f\left(X_{\mcN}\right)- f\left (\bx_*\right)\left(K^\infty_{\mathcal{N}}\right)^{-1}\onev\right\|_1 \nonumber
\\
\leq 2m\,L_fB_f\,\min\{d_m^q,1\} + mB_f \frac{ 2 L_f \min\{d_m^q,1\}+\epsilon_E}{1-\epsilon_E} \label{eq:f_minus_Kf_bound}
\end{gather}
In the last line we have used (AR.\ref{aR_f}) and the inequality \eqref{Kf_ineq} proved in the Appendix \ref{appendix:matrix_inequalities}.
In a similar way, we apply suitable triangle inequalities together with the inequality \eqref{Kk_ineq} proved in the Appendix \ref{appendix:matrix_inequalities} to obtain
\begin{equation}\label{eq:k_minus_Kk_bound}
\left\| \mathbf{k}^*_{\mathcal{N}}-\left(\hat\sigma_\xi^2+m\hat\sigma_f^2\right)K_{\mathcal{N}}^{-1} \mathbf{k}^*_{\mathcal{N}}\right\|_1 \leq m\,\hat\sigma_f^2\,\epsilon_m\left(1+\frac{5}{1-\epsilon_E}\right).
\end{equation}
We can also use the triangle inequality together with the inequality \eqref{kTK_ineq} to derive the following bound
\begin{gather}
\left\|{\mathbf{k}^*_{\mathcal{N}}}^T\,K_{\mathcal{N}}^{-1}\right\|_1\leq
\left\| {\mathbf{k}^*_{\mathcal{N}}}^T\,K_{\mathcal{N}}^{-1}- \hat\sigma_f^2\,\onev^T\,\left(K^\infty_{\mathcal{N}}\right)^{-1}\right\|_1 + \hat\sigma_f^2\left\|\onev^T\,\left(K^\infty_{\mathcal{N}}\right)^{-1}\right\|_1 \nonumber
\\
\leq \hat\sigma_f^2\left\|\onev^T\,\left(K^\infty_{\mathcal{N}}\right)^{-1}\right\|_1\left(1 + \frac{\epsilon_m+\epsilon_E}{1-\epsilon_E}\right) = \frac{\hat\sigma_f^2}{\hat\sigma_\xi^2+m\hat\sigma_f^2} \frac{1+\epsilon_m}{1-\epsilon_E}\leq  \frac{2}{m} \frac{1}{1-\epsilon_E} \label{eq:kK_bound}
\end{gather}
Similarly, we can use the triangle inequality together with the inequality \eqref{kTK2_ineq} to derive the following bound
\begin{equation}\label{eq:kK2_bound}
\left\|{\mathbf{k}^*_{\mathcal{N}}}^T\,K_{\mathcal{N}}^{-2}\right\|_1\leq  \frac{\hat\sigma_f^2}{(\hat\sigma_\xi^2+m\hat\sigma_f^2)^2} \frac{1+\epsilon_m}{1-\epsilon_{E,2}}\leq \frac{2}{m}\frac{1}{\hat\sigma_\xi^2+m\hat\sigma_f^2}\frac{1}{1-\epsilon_{E,2}}.
\end{equation}
Combining the inequalities \eqref{eq:f_minus_Kf_bound} and \eqref{eq:kK_bound} yields the bounds on the derivatives of the expectation of $\tilde\mu$, while combining the inequalities \eqref{eq:k_minus_Kk_bound} and  \eqref{eq:kK2_bound} yields the bounds on the derivatives of the variance of $\tilde\mu$. The final form of the inequality \eqref{eq:dMSE_dnv_bound} follows after some algebra and by applying the bounds $\epsilon_m\leq 1$ and $\min\left\{d_m^{2q},1\right\}\leq 1$ is suitable places.

Next, we move on to the proof of the inequality \eqref{eq:dMSE_dl_bound}. Let us start with the expansion \eqref{eq:dMSE_dphi_expansion} which implies
\begin{align*}
\begin{split}
\left|\frac{\partial MSE_{GPnn}(\mcX_*,\bX_n)}{\partial \hat\ell}\right| = &\, 2\, \left| \Gamma\, {\mathbf{k}^*_{\mcN}}^T\, K_{\mcN}^{-1}f(\bX_{\mcN}) - f(\mcX_*)\right| \,\left|\frac{\partial \EE\left[\tilde\mu_{GPnn}\mid\mcX_*,\bX_n\right] }{\partial \hat\ell} \right|
\\
+ &\left| \frac{\partial\var\left[\tilde\mu_{GPnn}\mid \mcX_*,\bX_n\right]}{\partial \hat\ell}\right|,
\end{split}
\end{align*}
Using Lemma \ref{lemma:mu_mean_var_bound} and some suitable upper bounds on $|f(\bx_*)|\leq B_f$, $\epsilon_m\leq 1$ we have
\begin{equation}\label{eq:bias_upper_bound}
2\, \left| \Gamma\, {\mathbf{k}^*_{\mcN}}^T\, K_{\mcN}^{-1}f(\bX_{\mcN}) - f(\mcX_*)\right| \leq \frac{10B_f(1+2L_f)}{1-\epsilon_E}\,\epsilon_m+4B_fL_f\min\{d_m^q,1\}.
\end{equation}
Next, using \eqref{eq:dmu_dphi} we have
\begin{equation}\label{eq:dmu_dl_bound}
\frac{m\hat\sigma_f^2}{\hat\sigma_\xi^2+m\hat\sigma_f^2}\left|\frac{\partial \EE_{\by|X_n}\left[\tilde\mu^*_{\mathcal{N}}\right] }{\partial \hat\ell}\right| \leq \left|\left(\frac{\partial\mathbf{k}^*_{\mathcal{N}}}{\partial \hat\ell}\right)^T\,K_{\mathcal{N}}^{-1}f(X)\right| +  \left|{\mathbf{k}^*_{\mathcal{N}}}^T \,K_{\mathcal{N}}^{-1}\,\frac{\partial K_{\mathcal{N}}}{\partial \hat\ell} \,K_{\mathcal{N}}^{-1}f(X)\right|.
\end{equation}
The first term in the above sum can be bounded as 
\[
 \left|\left(\frac{\partial\mathbf{k}^*_{\mathcal{N}}}{\partial \hat\ell}\right)^T\,K_{\mathcal{N}}^{-1}f(X)\right| \leq \left\|\left(\frac{\partial\mathbf{k}^*_{\mathcal{N}}}{\partial \hat\ell}\right)^T\right\|_1\, \left\|K_{\mathcal{N}}^{-1}f(X)\right\|_1.
\]
By assumption (AD.\ref{aD_bnd}) combined with the chain rule, we have 
\[
\left\|\left(\frac{\partial\mathbf{k}^*_{\mathcal{N}}}{\partial \hat\ell}\right)^T\right\|_1\leq \frac{\hat\sigma_f^2}{\hat\ell}\max_{1\leq i\leq m}\min\left\{\left|\frac{d_i}{\hat\ell}c'\left(\frac{d_i}{\hat\ell}\right)\right|,B_c\right\} \leq B_c\,L_c'\frac{\hat\sigma_f^2}{\hat\ell}\min\left\{\left(\frac{d_m}{\hat\ell}\right)^{2p'} ,1\right\}
\]
Next, using Equation \eqref{Kf_ineq} from Lemma \ref{lemma:K_ineqs} in Appendix \ref{appendix:matrix_inequalities} we get
\begin{align}\label{eq:Kf_bound}
\begin{split}
 \left\|K_{\mathcal{N}}^{-1}f(X)\right\|_1 & \leq \left\| K_{\mathcal{N}}^{-1}\, f(X)- f\left (\bx_*\right)\left(K^\infty_{\hat\Theta,\mathcal{N}}\right)^{-1}\onev\right\|_1 + \left\|f\left (\bx_*\right)\left(K^\infty_{\hat\Theta,\mathcal{N}}\right)^{-1}\onev\right\|_1 \\
& \leq \frac{m\, B_f}{\hat\sigma_\xi^2+m\hat\sigma_f^2}\frac{1+ 2 L_f \min\{d_m^q,1\}}{1-\epsilon_E}\leq  \frac{m\, B_f}{\hat\sigma_\xi^2+m\hat\sigma_f^2}\frac{1+ 2 L_f}{1-\epsilon_E}.
 \end{split}
\end{align}
Next, we move to the second term of the inequality \eqref{eq:dmu_dl_bound}.
\[
\left|{\mathbf{k}^*_{\mathcal{N}}}^T \,K_{\mathcal{N}}^{-1}\,\frac{\partial K_{\mathcal{N}}}{\partial \hat\ell} \,K_{\mathcal{N}}^{-1}f(X)\right|\leq \left\|{\mathbf{k}^*_{\mathcal{N}}}^T \,K_{\mathcal{N}}^{-1}\right\|_1\, \left\|\frac{\partial K_{\mathcal{N}}}{\partial \hat\ell} \right\|_1\, \left\| K_{\mathcal{N}}^{-1}f(X)\right\|_1.
\]
Using the triangle inequality and Equation \eqref{kTK_ineq} (Lemma \ref{lemma:K_ineqs}, Appendix \ref{appendix:matrix_inequalities}) we get
\begin{align}\label{eq:Kk_bound}
\begin{split}
 \left\|{\mathbf{k}^*_{\mathcal{N}}}^T \,K_{\mathcal{N}}^{-1}\right\|_1 & \leq \left\| \left(\mathbf{k}^*_{\mathcal{N}}\right)^T\,K_{\mathcal{N}}^{-1}- \hat\sigma_f^2\,\onev^T\,\left(K^\infty_{\hat\Theta,\mathcal{N}}\right)^{-1}\right\|_1 + \hat\sigma_f^2\left\|\onev^T\,\left(K^\infty_{\hat\Theta,\mathcal{N}}\right)^{-1}\right\|_1
  \\
&  \leq \frac{\hat\sigma_f^2}{\hat\sigma_\xi^2+m\hat\sigma_f^2} \frac{1+\epsilon_m}{1-\epsilon_E}\leq \frac{\hat\sigma_f^2}{\hat\sigma_\xi^2+m\hat\sigma_f^2} \frac{1+\epsilon_m}{1-\epsilon_E}\leq \frac{2\hat\sigma_f^2}{\hat\sigma_\xi^2+m\hat\sigma_f^2} \frac{1}{1-\epsilon_E}.
 \end{split}
\end{align}
By assumption (AD.\ref{aD_bnd}) and the triangle inequality, we have 
\begin{align}\label{eq:dK_dl_bound}
\begin{split}
& \left\|\frac{\partial K_{\mathcal{N}}}{\partial \hat\ell} \right\|_1  \leq  \frac{\hat\sigma_f^2}{\hat\ell}\max_{1\leq i\leq m} \sum_{j\neq i} \left|\frac{d_{ij}}{\hat\ell}c'\left(\frac{d_{ij}}{\hat\ell}\right)\right| \leq \frac{\hat\sigma_f^2B_cL_c'}{\hat\ell}\max_{1\leq i\leq m} \sum_{j\neq i} \min \left\{\left(\frac{d_{ij}}{\hat\ell}\right)^{2p'},1 \right\}
\\
& \leq \frac{\hat\sigma_f^2B_cL_c'}{\hat\ell}\max_{1\leq i\leq m} \sum_{j\neq i} \min \left\{\left(\frac{d_i+d_j}{\hat\ell}\right)^{2p'},1 \right\}
\leq m \frac{\hat\sigma_f^2B_cL_c'}{\hat\ell}\min \left\{\left(\frac{d_m}{\hat\ell}\right)^{2p'},1 \right\},
\end{split}
\end{align}
where $d_{ij}:=\|\bx_{n,i}-\bx_{n,j}\|_2$ and $d_{i}:=\|\bx_{n,i}-\bx_*\|_2$. After some algebra we finally obtain
\begin{equation}\label{eq:dEmu_dl_final_ineq} 
\left|\frac{\partial \EE\left[\tilde\mu_{GPnn}\mid\mcX_*,\bX_n\right]}{\partial \hat\ell}\right|\leq \frac{1}{\hat\ell}\,\frac{3B_fB_cL_c'(1+2L_f)}{(1-\epsilon_E)^2}\min\left\{\left(\frac{d_m}{\hat\ell}\right)^{2p'} ,1\right\}.
\end{equation}

Let us next move to analysing the variance derivative. Using \eqref{eq:dvar_dphi} we have
\[
\frac{1}{\Gamma^2}\frac{1}{2\sigma_\xi^2}\left|\frac{\partial \mathrm{Var}\left[\tilde\mu_{GPnn}\mid\mcX_*,\bX_n\right]}{\partial \hat\ell}\right|\leq \left|{\mathbf{k}^*_{\mathcal{N}}}^T\, K_{\mathcal{N}}^{-2}\,\frac{\partial\mathbf{k}^*_{\mathcal{N}}}{\partial \hat\ell}\right|  + \left|{\mathbf{k}^*_{\mathcal{N}}}^T\,K_{\mathcal{N}}^{-2}\,\frac{\partial K_{\mathcal{N}}}{\partial \hat\ell} K_{\mathcal{N}}^{-1}\,{\mathbf{k}^*_{\mathcal{N}}}\right| .
\]
The first term of the above inequality can be bounded as
\begin{gather*}
\left|{\mathbf{k}^*_{\mathcal{N}}}^T\, K_{\mathcal{N}}^{-2}\,\frac{\partial\mathbf{k}^*_{\mathcal{N}}}{\partial \hat\ell}\right| \leq \left\|{\mathbf{k}^*_{\mathcal{N}}}^T\, K_{\mathcal{N}}^{-2}\right\|_1\ \left\|\frac{\partial\mathbf{k}^*_{\mathcal{N}}}{\partial \hat\ell}\right\|_1\leq \Bigg{(}\left\| \left(\mathbf{k}^*_{\mathcal{N}}\right)^T\,K_{\mathcal{N}}^{-2}- \hat\sigma_f^2\,\onev^T\,\left(K^\infty_{\hat\Theta,\mathcal{N}}\right)^{-2}\right\|_1 \\ 
+\hat\sigma_f^2 \left\|\onev^T\,\left(K^\infty_{\hat\Theta,\mathcal{N}}\right)^{-2}\right\|_1\Bigg{)}\left\|\frac{\partial\mathbf{k}^*_{\mathcal{N}}}{\partial \hat\ell}\right\|_1.
\end{gather*}
By the application of Equation \ref{kTK2_ineq} (Lemma \ref{lemma:K_ineqs}, Appendix \ref{appendix:matrix_inequalities})  we obtain
\begin{gather*}
\left|{\mathbf{k}^*_{\mathcal{N}}}^T\, K_{\mathcal{N}}^{-2}\,\frac{\partial\mathbf{k}^*_{\mathcal{N}}}{\partial \hat\ell}\right| \leq 2\hat\sigma_f^2 \left\|\onev^T\,\left(K^\infty_{\hat\Theta,\mathcal{N}}\right)^{-2}\right\|_1\,\left\|\frac{\partial\mathbf{k}^*_{\mathcal{N}}}{\partial \hat\ell}\right\|_1\frac{1}{1-\epsilon_{E,2}} \\
\leq \frac{2}{m} \left(\frac{m\hat\sigma_f^2}{\hat\sigma_\xi^2+m\hat\sigma_f^2}\right)^2 B_c\,L_c'\frac{\hat\sigma_f^2}{\hat\ell}\frac{1}{1-\epsilon_{E,2}} \min\left\{\left(\frac{d_m}{\hat\ell}\right)^{2p'} ,1\right\}
\end{gather*}
In a similar way, we get
\begin{gather*}
\left|{\mathbf{k}^*_{\mathcal{N}}}^T\,K_{\mathcal{N}}^{-2}\,\frac{\partial K_{\mathcal{N}}}{\partial \hat\ell} K_{\mathcal{N}}^{-1}\,{\mathbf{k}^*_{\mathcal{N}}}\right|\leq \frac{2\hat\sigma_f^2B_cL_c'}{\hat\ell}\left(\frac{m\hat\sigma_f^2}{\hat\sigma_\xi^2+m\hat\sigma_f^2}\right)^3 \frac{1}{1-\epsilon_{E}} \frac{1}{1-\epsilon_{E,2}} \min \left\{\left(\frac{d_m}{\hat\ell}\right)^{2p'},1 \right\}.
\end{gather*}
Thus,
\begin{equation}\label{eq:dvar_dl_bound}
\left|\frac{\partial \mathrm{Var}_{\by|X_n}\left[\tilde\mu^*_{\mathcal{N}}\right]}{\partial \hat\ell}\right| \leq 8 \frac{\sigma_\xi^2}{\hat\ell}\hat\sigma_f^2B_cL_c'\frac{1}{1-\epsilon_{E}} \frac{1}{1-\epsilon_{E,2}} \min \left\{\left(\frac{d_m}{\hat\ell}\right)^{2p'},1 \right\}.
\end{equation}
The final result \eqref{eq:dMSE_dl_bound} follows from combining the inequalities \eqref{eq:dvar_dl_bound}, \eqref{eq:dEmu_dl_final_ineq} and \eqref{eq:bias_upper_bound}.
\end{proof}

\begin{lemma}[Uniform convergence of MSE derivatives.]\label{thm:uniform_convergence_deriv}
Let $X=(\bx_1,\bx_2,\dots)$ be an infinite sequence of i.i.d. points sampled from $P_\mcX$ and denote by $X_n$ its truncation to the first $n$ points. Assume (AC.\ref{a_X_app}-\ref{a_f_app}), (AC.\ref{a_metr_app}), (AR.\ref{aR_xi_app}) and (AR.\ref{aR_f_app}). Then, for almost every sampling sequence $X$ and test point $\bx_*\in\supp(P_\mcX)$ and any compact subset $S$ of the hyper-parameters $\Theta=\left(\hat\sigma_\xi^2,\hat\sigma_f^2,\hat\ell,\hat\bb\right)\in S\subset \RR_{\geq 0} \times \RR_{>0}\times \RR_{>0}\times\RR^{d_T}$ we have that
\begin{gather}
\left|\frac{\partial MSE(\bx_*,X_n)}{\partial (\hat\sigma_\xi^2)} \right| \xrightarrow{n\to\infty} 0, \quad
\left|\frac{\partial MSE(\bx_*,X_n)}{\partial (\hat\sigma_f^2)} \right| \xrightarrow{n\to\infty} 0, \\
\left\|\nabla_{\hat\bb} MSE_{NNGP}(\bx_*,\bX_n)\right\|_2\xrightarrow{n\to\infty} 0
\end{gather}
and this convergence is uniform as a function of $\Theta\in S$. If additionally (AR.\ref{aR_c}) and (AD.\ref{aD_iso}-\ref{aD_bnd}) hold, we have that
\begin{equation}
\left|\frac{\partial MSE(\bx_*,X)}{\partial \hat\ell} \right| \xrightarrow{n\to\infty} 0
\end{equation}
uniformly on $S$.
\end{lemma}
\begin{proof} (For $GPnn$ -- the $NNGP$ counterpart follows straightforwardly using the same techniques.) 
We follow the same strategy as in the proof of Theorem \ref{thm:uniform_convergence}. Let us consider the derivative with respect to $\partial (\hat\sigma_\xi^2)$. The proofs for the remaining derivatives are fully analogous. Using Equation \eqref{eq:dMSE_dnv_bound} we will construct an upper bound 
\[
\left|\frac{\partial MSE(\bx_*,X_n)}{\partial (\hat\sigma_\xi^2)} \right| \leq h_\Theta(\bx_*,X_n),
\]
where $h_\Theta(\bx_*,X_n)$ is a continuous function of the hyper-parameters $\Theta=(\hat\sigma_\xi^2, \hat\sigma_f^2,\hat\ell)$ and is monotonically decreasing with $n$. As in the proof of Theorem \ref{thm:uniform_convergence} Dini's theorem combined with the sandwich theorem for uniform convergence readily implies that $\partial MSE(\bx_*,X_n)/\partial (\hat\sigma_\xi^2)$ tends to zero uniformly on $S$.

To construct $h_\Theta(\bx_*,X_n)$, we replace $\epsilon_E$ and $\epsilon_{E,2}$ in the RHS of Equation \eqref{eq:dMSE_dnv_bound} with their suitable upper bounds that are monotonically decreasing with $n$. Namely, 
\begin{align*}
\begin{split}
h_\Theta(\bx_*,X_n) = & 48 B_f^2 L_f^2\,\frac{1}{m\hat\sigma_f^2} \frac{1}{(1-\tilde\epsilon_E)^3}\,\min\left\{d_m^{2q},1\right\}
\\
+ & \frac{1}{m\hat\sigma_f^2}\left(24\sigma_\xi^2\frac{1}{1-\tilde\epsilon_{E,2}}\frac{1}{1-\tilde\epsilon_E}+\frac{2B_f^2}{(1-\tilde\epsilon_E)^3}\left( 78 L_f+5\right)\right)\,\epsilon_m
\end{split}
\end{align*}
We have 
\[ 
\epsilon_m(\bx_*, X_n) = \rho_c^2\left(\frac{d_m(\bx_*, X_n)}{\hat\ell}\right), \quad d_m(\bx_*, X_n) = \|\bx_*-\bx_{n,m}(\bx_*)\|_2.
\]
As in the proof of Theorem \ref{thm:uniform_convergence}, $\epsilon_E(\bx_*, X_n)$ is upper bounded by the $\epsilon_E$ calculated for the nearest neighbour configuration where the nearest neighbours are grouped on the antipodal points of the Euclidean ball of the radius $d_m(\bx_*, X_n)$, i.e.
\begin{equation}\label{eq:antipodal_config}
\bx_{n,1}=\dots=\bx_{n,m-1}=2\bx_*-\bx_{n,m}.
\end{equation}
In other words,
\[
\epsilon_E(\bx_*, X_n)\leq \tilde\epsilon_E(\bx_*, X_n):=(m-1)\, \rho_c^2\left(\frac{2d_m(\bx_*, X_n)}{\hat\ell}\right).
\]
It remains to construct suitable $\tilde\epsilon_{E,2}$. Recall the definition of $\epsilon_{E,2}$
\begin{equation}\label{eps_E2_def}
\epsilon_{E,2}=\left(\frac{\hat\sigma_f^2}{\hat\sigma_\xi^2+m\hat\sigma_f^2}\right)^2\, \max_{1\leq i\leq m}\sum_{j=1}^m\left(2\frac{\hat\sigma_\xi^2}{\hat\sigma_f^2} \,\epsilon_{ij} + \sum_{k=1}^m \left(\epsilon_{ik}+\epsilon_{jk}-\epsilon_{ik}\epsilon_{jk}\right)\right).
\end{equation}
Note that each $\epsilon_{ij}(\bx_*, X_n)$ for $i\neq j$ is upped-bounded by $\epsilon_{ij}$ calculated for the configuration of nearest-neighbours described by Equation \eqref{eq:antipodal_config}. Namely,
\[
\epsilon_{ij}\leq \overline\epsilon:=\rho_c^2\left(\frac{2d_m(\bx_*, X_n)}{\hat\ell}\right).
\]
What is more, if $\epsilon_{ij}'\geq\epsilon_{ij}$ for all $i,j$, then 
\[\epsilon_{ik}'+\epsilon_{jk}'-\epsilon_{ik}'\epsilon_{jk}'\geq \epsilon_{ik}+\epsilon_{jk}-\epsilon_{ik}\epsilon_{jk}.\]
Thus, we can upper-bound the RHS of \eqref{eps_E2_def} by replacing every $\epsilon_{ij}$ with $\overline\epsilon$ and setting $i\equiv m$. This leads to 
\[
\epsilon_{E,2} \leq \tilde\epsilon_{E,2}:= \left(\frac{\hat\sigma_f^2}{\hat\sigma_\xi^2+m\hat\sigma_f^2}\right)^2\,  (m-1)\overline\epsilon\left(2\frac{\hat\sigma_\xi^2}{\hat\sigma_f^2}  + m+2-\overline\epsilon\right).
\]
Since $\epsilon_m$, $\tilde\epsilon_{E}$ and $\tilde\epsilon_{E,2}$ are continuous functions of $\Theta$ that are monotonically decreasing with $n$, so is $h_\Theta(\bx_*,X_n)$. 

The remaining $MSE$ derivatives are treated in a fully analogous way.
\end{proof}

\begin{restatedresult}{Theorem~\ref{thm:mse_derivatives_convergence_rate} (Convergence Rates of Derivatives).}
\DerivativesConvergenceRatesStatement
\end{restatedresult}
\begin{proof}(Sketch.)
Recall the shorthand $D_\phi(\mcX_*,\bX_n)$ from \eqref{def:derivatives_shorthand_app}.
We start from the bias--variance derivative identity \eqref{eq:dMSE_dphi_expansion}--\eqref{eq:dvar_dphi}:
for $\phi\in\{\hat\sigma_\xi^2,\hat\sigma_f^2,\hat\bb\}$,
$\partial_\phi MSE$ is a sum of a bias $\times$ $\partial_\phi \EE[\tilde\mu\mid \mcX_*,\bX_n]$ term and a
$\partial_\phi \var[\tilde\mu\mid \mcX_*,\bX_n]$ term.
Inequalities derived in the proof of Lemma~\ref{lemma:ksc_nv_derivatives_consistency} give uniform control of the derivative building blocks in terms of upper bounds proportional to $m$, and
Lemma~\ref{lemma:dMSE_dnv_dksc_bounds} yields a deterministic upper bound for $D_\phi$ in terms of $d_m$ and nearest-neighbour kernel-metric distances (with Lemma~\ref{lemma:epsilons_rels} replacing $\epsilon_E,\epsilon_{E,2}$ by $\epsilon_m$).

To bound $\EE[|D_\phi|]$, use the same good/bad region decomposition as in the risk-rate proof of Theorem \ref{thm:mse_convergence_rate}:
for $0<R\le1$ define the bad event $\Omega_{m,n}(R):=\{(\bx_*,X_n): d_m(\bx_*,X_n)\ge R\}$ and apply
Lemma~\ref{lemma:subset_expectation_bound} (Cauchy--Schwarz splitting) to obtain
\[
\EE[|D_\phi|]\;\le\;\EE[|D_\phi|\mid d_m<R]\;+\;\sqrt{\EE[D_\phi^2]}\,\sqrt{P[\Omega_{m,n}(R)]}.
\]
On $\{d_m<R\}$, plug in the deterministic bound from Lemma~\ref{lemma:dMSE_dnv_dksc_bounds} and use the NN-distance/$\epsilon$ moment rates
(as in Theorem~\ref{thm:mse_convergence_rate}) to get the first term $A_1^{(\phi)}(m/n)^{2\alpha/d_\mcX}$.
Next, control $P[\Omega_{m,n}(R)]$ via Lemma~\ref{lemma:bad_region_bound} and bound $\EE[D_\phi^2]$ by a uniform upper bound proportional to $m$
(from Lemma~\ref{lemma:ksc_nv_derivatives_consistency}), yielding the second term
$A_2^{(\phi)}\,m\,(m/n)^{2(\alpha+p)/d_\mcX}$. Combining the two contributions proves the claim (and choosing $m_n$ gives the stated rate).
\end{proof}

\begin{restatedresult}{Theorem~\ref{thm:mse_dl_convergence_rate}.}
\DlConvergenceRatesStatement
\end{restatedresult}
\begin{proof}(Sketch.)
Use \eqref{eq:dMSE_dphi_expansion}--\eqref{eq:dvar_dphi} with $\phi=\hat\ell$ and the definition $D_{\hat\ell}=|\partial_{\hat\ell} MSE|$ in
\eqref{def:derivatives_shorthand_app}. Bounding $\partial_{\hat\ell} \EE[\tilde\mu\mid \mcX_*,\bX_n]$ and
$\partial_{\hat\ell} \var[\tilde\mu\mid \mcX_*,\bX_n]$ involves bounding
$\partial_{\hat\ell} \mathbf{k}^*_{\mcN}$ and $\partial_{\hat\ell} K_{\mcN}$. This is handled by the kernel bounds in Lemma~\ref{lemma:dMSE_dnv_dksc_bounds}
and in the proof of Lemma~\ref{lemma:l_derivative_consistency}, which extract the explicit $\hat\ell$-prefactors and yield i) deterministic upper bound for $|D_{\hat\ell}|$ \eqref{eq:dMSE_dl_bound} in terms of $d_m$ and nearest-neighbour kernel-metric distances (with Lemma~\ref{lemma:epsilons_rels} replacing $\epsilon_E,\epsilon_{E,2}$ by $\epsilon_m$) and ii) uniform upper bound on $|D_{\hat\ell}|$ \eqref{eq:dMSE_dl_uniform_bound} proportional to $m^2$ .

This allows us to bound  $\EE[|D_{\hat\ell}|]$ using the same good/bad event split based on $\Omega_{m,n}(R):=\{(\bx_*,X_n): d_m\ge R\}$ as in Theorem \ref{thm:mse_convergence_rate}, i.e., 
\[
\EE[|D_{\hat\ell}|]\;\le\;\EE[|D_{\hat\ell}|\mid d_m<R]\;+\;\sqrt{\EE[D_{\hat\ell}^2]}\,\sqrt{P[\Omega_{m,n}(R)]}.
\]
Then, apply the deterministic bound \eqref{eq:dMSE_dl_bound} in terms of moments of euclidean and kernel-metric NN-distances  to obtain the term
$\frac{\max\{\hat\ell^{-2p'},1\}}{\hat\ell}A_1(m/n)^{2p'/d_\mcX}$ in \eqref{eq:dMSE_dl_fixedm_rate}. On $\Omega_{m,n}(R)$ control the tail probability $P[\Omega_{m,n}(R)]$ as in Theorem \ref{thm:mse_convergence_rate} and use the uniform bound \eqref{eq:dMSE_dl_uniform_bound}, which yields the second term $\frac{1}{\hat\ell}A_2\,m^2(m/n)^{2(p'+2p)/d_\mcX}$ and completes the proof sketch.
\end{proof}

\section{Upper bounds on kernel functions and their derivatives}\label{appendix:kernel_bounds}

In this section we show that some popular kernel choices (exponential, squared exponential and Mat\'{e}rn) satisfy the assumption (A2) and the related assumption needed for calculating the convergence rates of $MSE$-derivatives in Section \ref{sec:mse_derivatives}. Namely, we show that there exist positive constants $L_c, \widetilde L_c$ and $p,\tilde p\in ]0,1]$ such that for all $r\geq 0$
\[c(r)\geq 1-L_c\, r^{2p},\quad \left|r c'(r)\right|\leq \widetilde L_c\, r^{2\tilde p}\]
where $c(r)$ is the normalised kernel function, $c(0)=1$. Recall the relevant kernel function definitions
\[c_{\mathrm{Exp}}(r):=e^{-r},\quad c_{\mathrm{SE}}(r):=e^{-r^2/2},\quad c_{M,\nu}(r):= \frac{2}{\Gamma(\nu)}\left(\frac{\sqrt{2\nu}}{2}r\right)^\nu\, \mathcal{K}_\nu(\sqrt{2\nu}\,r),\]
where $\Gamma(\cdotp)$ is the Euler Gamma function and $\mathcal{K}_\nu(\cdotp)$ is the modified Bessel function of the second kind. 
The goal of this Appendix is to prove that the following inequalities hold for any $r\geq 0$.
\begin{align}
& c_{\mathrm{Exp}}(r)\geq 1-r,\quad c_{\mathrm{SE}}(r)\geq 1-\frac{r^2}{2}, \label{eq:exp_se_bounds}
\\
& \left|r\, c_{\mathrm{Exp}}'(r)\right|\leq r,\quad \left|r\, c_{\mathrm{SE}}'(r)\right|\leq r^2, \label{eq:der_exp_se_bounds}
\\
& c_{M,\nu}(r) \geq 1-\frac{\nu}{\nu-1}\frac{r^2}{2},\quad \nu > 1, \label{eq:mat_bound_nu_ge1}
\\
& c_{M,\nu}(r)\geq 1-\nu^\nu\,\Gamma(1-\nu)\,\mathcal{I}_\nu\left(1\right)\,r^{2\nu}, \quad 0<\nu<1, \label{eq:mat_bound_nu_le1}
\\
& c_{M,1}(r)\ge 1-L_\epsilon\,r^{2-\epsilon},\, L_\epsilon=\max_{0\le r\le 1}\frac{1-c_{M,1}(r)}{r^{2-\epsilon}}<\infty \quad\textrm{for\ any}\quad 0<\epsilon<2,\, 0\leq r\leq 1,
\\
& \left|r\,c_{M,\nu}'(r)\right| \leq \nu\frac{\nu+1}{\nu-1}\,\frac{r^2}{2},\quad \nu > 1, \label{eq:der_mat_bound_nu_ge1}
\\
& \left|r\,c_{M,\nu}'(r)\right| \leq \Gamma(1-\nu)\, \nu^\nu\left(\frac{1}{\Gamma(\nu)\, 2^\nu}+\nu\, \mathcal{I}_\nu\left(1\right)\right)\,r^{2\nu},\quad 0<\nu<1, \label{eq:der_mat_bound_nu_le1}
\end{align}
where $\mathcal{I}_\nu(\cdotp)$ denotes the modified Bessel function of the first kind.

\paragraph{Bounds for the Kernels.} The desired lower bounds for the exponential and squared exponential kernel functions \eqref{eq:exp_se_bounds} are readily derived from the standard inequality $e^{u}\geq 1+u$ which holds for any $u\in\RR$. In a similar way, we use the fact that $e^{-u}\leq 1$ for $u\geq 0$ and get the bounds \eqref{eq:der_exp_se_bounds}. Thus, we have that $p=\tilde p=1$ for $c_{\mathrm{Exp}}$ and $p=\tilde p=2$ for $c_{SE}$. 

In order to obtain the lower bounds \eqref{eq:mat_bound_nu_ge1} and \eqref{eq:mat_bound_nu_le1} for the Mat\'{e}rn family, we need to consider two different cases. The first case concerns kernels with $\nu > 1$ and the second case concerns kernels with $0<\nu<1$. When $\nu>1$ we use the following integral representation of the Mat\'{e}rn kernel \citep{TKS18}
\[c_{M,\nu}(r) = \frac{\nu^\nu}{\Gamma(\nu)}\int_{0}^\infty s^{\nu-1}e^{-\nu s}e^{-\frac{r^2}{2s}}.\]
Using the fact that  $e^{-\frac{r^2}{2s}}\geq 1-\frac{r^2}{2s}$ in the above integral we obtain
\begin{equation*}
c_{M,\nu}(r) \geq \frac{\nu^\nu}{\Gamma(\nu)}\int_{0}^\infty s^{\nu-1}e^{-\nu s}\left( 1-\frac{r^2}{2s}\right) = 1-\frac{\nu}{\nu-1}\frac{r^2}{2},\quad \nu > 1.
\end{equation*}
When $0<\nu<1$ we apply a different strategy that uses the following series expansion of $\mathcal{K}_\nu$ that converges for any $z\in\CC$ and $\nu\notin\ZZ$.
\[
\mathcal{K}_\nu(z) = \frac{\Gamma(\nu)\Gamma(1-\nu)}{2}\left(\sum_{k=0}^\infty \frac{1}{\Gamma(k-\nu+1)k!}\left(\frac{z}{2}\right)^{2k-\nu}-\sum_{k=0}^\infty \frac{1}{\Gamma(k+\nu+1)k!}\left(\frac{z}{2}\right)^{2k+\nu}\right).
\]
This implies that
\begin{align}\label{eq:cM_expansion}
\begin{split}
c_{M,\nu}(r)=1+\sum_{k=1}^\infty \frac{\Gamma(1-\nu)}{\Gamma(k-\nu+1)k!}\left(\frac{\sqrt{2\nu}}{2}r\right)^{2k}  \\ 
- \Gamma(1-\nu)\left(\frac{\sqrt{2\nu}}{2}r\right)^{2\nu}\sum_{k=0}^\infty \frac{1}{\Gamma(k+\nu+1)k!}\left(\frac{\sqrt{2\nu}}{2}r\right)^{2k} 
\end{split}
\end{align}
In order to find the desired upper bound for $0<\nu < 1$ we simply bound the first sum by $0$
\begin{equation}\label{eq:sum_bound_mat1}
\sum_{k=1}^\infty \frac{\Gamma(1-\nu)}{\Gamma(k-\nu+1)k!}\left(\frac{\sqrt{2\nu}}{2}r\right)^{2k} \geq 0.
\end{equation}
Next, using the series expansion of the modified Bessel function of the first kind $\mathcal{I}_\nu(z)$ we get that
\[
\sum_{k=0}^\infty \frac{1}{\Gamma(k+\nu+1)k!}\left(\frac{\sqrt{2\nu}}{2}r\right)^{2k} = \left(\frac{\sqrt{2\nu}}{2}r\right)^{-\nu}\mathcal{I}_\nu\left(r \sqrt{2\nu}\right).
\]
The RHS of the above equation is an increasing function of $r$ for $r>0$. In particular, it implies that for any $r<1/\sqrt{2\nu}$
\begin{equation}\label{eq:sum_bound_mat2}
\sum_{k=0}^\infty \frac{1}{\Gamma(k+\nu+1)k!}\left(\frac{\sqrt{2\nu}}{2}r\right)^{2k} \leq 2^{\nu}\,\mathcal{I}_\nu\left(1\right).
\end{equation}
Plugging the bounds \eqref{eq:sum_bound_mat1} and \eqref{eq:sum_bound_mat2} into the expansion \eqref{eq:cM_expansion} we get the upper bound  \eqref{eq:mat_bound_nu_le1} for any $r<1/\sqrt{2\nu}$. However, the bound \eqref{eq:mat_bound_nu_le1} also holds for $r\geq 1/\sqrt{2\nu}$. To see this, it suffices to note that the bound is a decreasing function of $r$ and evaluate the bound at $r = 1/\sqrt{2\nu}$ to see that its value is negative at this point, i.e.
\[2^{-\nu}\,\Gamma(1-\nu)\,\mathcal{I}_\nu\left(1\right) > 1\quad \mathrm{for\ all}\quad 0<\nu<1.\]
 Thus, for $r \geq 1/\sqrt{2\nu}$ the value of the bound stays negative, and hence it is strictly smaller than $c_{M,\nu}(r)$ which is positive for any $r\geq 0$.

For the normalised Mat\'ern kernel with smoothness parameter $\nu=1$, $c(r)=\sqrt{2}\,r\,K_1\!\left(\sqrt{2}\,r\right)$, the small-$r$ expansion gives
\[
1-c(r)
=
r^2\left(\log\frac{\sqrt{2}}{r}-\gamma+\frac12\right)
+O\!\left(r^4\log\frac1r\right),
\qquad r\to 0,
\]
with $\gamma$ the Euler--Mascheroni constant. In particular,
\[
1-c(r)=r^2\log\frac1r+O(r^2)+O\!\left(r^4\log\frac1r\right),
\qquad r\to 0.
\]
Fix any $\epsilon\in(0,2)$ and define, for $r\in(0,1]$,
\[
g_\epsilon(r):=\frac{1-c(r)}{r^{2-\epsilon}}.
\]
Then, as $r\to 0$,
\[
g_\epsilon(r)
=
r^\epsilon\log\frac1r + O(r^\epsilon)+O\!\left(r^{2+\epsilon}\log\frac1r\right)
\to 0.
\]
Hence $g_\epsilon$ extends continuously to $[0,1]$ by setting $g_\epsilon(0):=0$. Since $g_\epsilon$ is continuous on the compact interval $[0,1]$, it attains a finite maximum there. Therefore, for every $\epsilon\in(0,2)$, there exists
\[
L_\epsilon:=\max_{0\le r\le 1}\frac{1-c(r)}{r^{2-\epsilon}}<\infty
\]
such that
\[
c(r)\ge 1-L_\epsilon r^{2-\epsilon}
\qquad\text{for all } r\in[0,1].
\]
 
 \paragraph{Bounds for the Kernel Derivatives} In order to derive bounds \eqref{eq:der_mat_bound_nu_ge1} and \eqref{eq:der_mat_bound_nu_le1} involving derivatives of the Mat\'{e}rn kernel function, we use the following recursive formula for the derivative of the relevant Bessel function.
 \[
 \mathcal{K}_\nu'(u) = -\frac{1}{2}\left( \mathcal{K}_{\nu-1}(u) + \mathcal{K}_{\nu+1}(u)\right).
 \]
 Using the above formula together with the identity $\mathcal{K}_{-\nu}(u) = \mathcal{K}_{\nu}(u)$, one can verify by a straightforward calculation that the following expressions hold.
 \[
 \left|r\,c_{M,\nu}'(r) \right| = \frac{\nu}{\nu-1}\frac{r^2}{2}\, c_{M,\nu-1}\left(r\sqrt{ \frac{\nu}{\nu-1}}\right)+\nu\left(c_{M,\nu+1}\left(r\sqrt{ \frac{\nu}{\nu+1}}\right)-c_{M,\nu}\left(r\right)\right),\quad \nu>1,
 \]
 \begin{align*}
 \left|r\,c_{M,\nu}'(r) \right| =  \frac{\Gamma(1-\nu)}{\Gamma(\nu)}\left(\frac{\nu}{2}\right)^\nu r^{2\nu}\, & c_{M,1-\nu}\left(r\sqrt{ \frac{\nu}{1-\nu}}\right)  \\
& +\nu\left(c_{M,\nu+1}\left(r\sqrt{ \frac{\nu}{\nu+1}}\right)-c_{M,\nu}\left(r\right)\right),\quad 0<\nu<1.
 \end{align*}
 Finally, we obtain the bounds \eqref{eq:der_mat_bound_nu_ge1} and \eqref{eq:der_mat_bound_nu_le1} by plugging into the above expressions the following inequalities:
 \begin{align*}
& c_{M,\nu-1}\left(r\sqrt{ \frac{\nu}{\nu-1}}\right)\leq 1,\quad c_{M,1-\nu}\left(r\sqrt{ \frac{\nu}{1-\nu}}\right)\leq 1,
\\
& c_{M,\nu+1}\left(r\sqrt{ \frac{\nu}{\nu+1}}\right)-c_{M,\nu}\left(r\right) \leq 1 - \left(1-\frac{\nu}{\nu-1}\frac{r^2}{2}\right) = \frac{\nu}{\nu-1}\frac{r^2}{2},\quad \nu >1
\\
& c_{M,\nu+1}\left(r\sqrt{ \frac{\nu}{\nu+1}}\right)-c_{M,\nu}\left(r\right) \leq \nu^\nu\,\Gamma(1-\nu)\,\mathcal{I}_\nu\left(1\right)\,r^{2\nu},\quad 0<\nu <1,
 \end{align*}
 where in the last two inequalities we have applied the bounds \eqref{eq:mat_bound_nu_ge1} and  \eqref{eq:mat_bound_nu_le1} to bound $c_{M,\nu}(r)$ from below.

\end{document}

%% file: theorem_statements.tex
\long\def\PointwiseConsistencyStatement{%
Assume (AC.\ref{a_X}-\ref{a_metr}). If the number of nearest neighbours $m$ is fixed, the following limits hold for $GPnn$ and $NNGP$ with probability one (with respect to $\bX_n\sim P_{\mcX}^n$) and for any test point $\bx_*\in \supp_{\rho_c}(P_\mcX)$ (see Definition \ref{def:supp}).
\begin{align}
& MSE(\bx_*,\bX_n)\xrightarrow{n\to\infty} \sigma_\xi^2(\bx_*)\left(1+\frac{1}{m}\right), \label{ptwise_mse_limit}
\\
& CAL(\bx_*,\bX_n)\xrightarrow{n\to\infty}\frac{\sigma_\xi^2(\bx_*)}{\hat\sigma_\xi^2}\, \left(1+\mcO\left(m^{-2}\right)\right), \label{ptwise_cal_limit}
\\
& NLL(\bx_*,\bX_n)\xrightarrow{n\to\infty} \frac{1}{2}\left(\log\left(2\pi\, \hat\sigma_\xi^2\right)+\frac{\sigma_\xi^2(\bx_*)}{\hat\sigma_\xi^2}+\frac{1}{m}\right)+\mcO\left(m^{-2}\right). \label{ptwise_nll_limit}
\end{align}
What is more, if $m$ grows with $n$ so that $\lim_{n\to\infty} m_n/n=0$, the following limits hold with probability one and for any text point $\bx_*\in \supp_{\rho_c}(P_\mcX)$.
\begin{align}
& MSE(\bx_*,\bX_n)\xrightarrow{n\to\infty} \sigma_\xi^2(\bx_*),\quad CAL(\bx_*,\bX_n)\xrightarrow{n\to\infty}\frac{\sigma_\xi^2(\bx_*)}{\hat\sigma_\xi^2}  \label{ptwise_mse_cal_limit_mgrow}
\\
& NLL(\bx_*,\bX_n)\xrightarrow{n\to\infty} \frac{1}{2}\left(\log\left(2\pi\, \hat\sigma_\xi^2\right)+\frac{\sigma_\xi^2(\bx_*)}{\hat\sigma_\xi^2}\right). \label{ptwise_nll_limit_mgrow}
\end{align}
}

\long\def\ApproximateUniversalConsistencyStatement{%
Let $\bX_n$ be a sampling sequence of i.i.d. points from the distribution $P_\mcX$ and $m$ be a fixed number of nearest-neughbours.  Let $\mcX_*\sim P_\mcX$ be a test point. 
Apply the following assumptions:
\begin{itemize}
\item (AC.\ref{a_X}-\ref{a_metr}),
\item  function $f$ in the $GPnn$ response model \eqref{eq:gpnn_responses} satisfies $\|f(\cdot)\|_\infty<B_f<\infty$,
\item functions $t_i$, $i=1,\dots, d_T$ in the $NNGP$ response model \eqref{eq:nngp_responses} satisfy $\|t_i(\cdot)\|_\infty<B_T<\infty$,
\item $\|\sigma_\xi^2(\cdot)\|_\infty<\infty$, where $\sigma_\xi^2(\bx):=\EE\left[\Xi^{2}\mid \mcX=\bx\right]$. 
\end{itemize}
Then we have the following limit for the risk for both $GPnn$ and $NNGP$.
\begin{align}
& \EE_{\mcX_*,\bX_n}\left[MSE(\mcX_*,\bX_n)\right]=\mcR_n+\EE_{\mcX_*}\left[\sigma_\xi^2(\mcX_*)\right]\xrightarrow{n\to\infty} \EE_{\mcX_*}\left[\sigma_\xi^2(\mcX_*)\right] \left(1+\frac{1}{m}\right), \label{E_mse_limit}
\end{align}
where $\mcR_n$ is the risk defined in \eqref{def:risk}.
Analogous limits hold for $CAL$ and $NLL$, i.e.,
\begin{align}\label{cal_nll_limit}
\begin{split}
& \EE_{\mcX_*,\bX_n}\left[CAL(\mcX_*,\bX_n)\right]\xrightarrow{n\to\infty}\frac{\EE_{\mcX_*}\left[\sigma_\xi^2(\mcX_*)\right]}{\hat\sigma_\xi^2}\, \left(1+\mcO\left(m^{-2}\right)\right),
\\
& \EE_{\mcX_*,\bX_n}\left[NLL(\mcX_*,\bX_n)\right]\xrightarrow{n\to\infty} \frac{1}{2}\left(\log\left(2\pi\, \hat\sigma_\xi^2\right)+\frac{\EE_{\mcX_*}\left[\sigma_\xi^2(\mcX_*)\right]}{\hat\sigma_\xi^2}+\frac{1}{m}\right)+\mcO\left(m^{-2}\right).
\end{split}
\end{align}
}

\long\def\UniversalConsistencyStatement{%
Let $\bX_n$ be a random sampling sequence of i.i.d. points from the distribution $P_\mcX$ and let $\mcX_*\sim P_\mcX$ be a test point. Let the number of nearest - neighbours $m_n$ grow as $m_n\propto n^\gamma$ with $0<\gamma<1/3$. Apply the following assumptions:
\begin{itemize}
\item there exists $\beta>\frac{2\gamma d_\mcX}{1-3\gamma}$ for which $\EE\left[\|\mcX\|_2^\beta\right]<\infty$ under the probability distribution $P_\mcX$ on $\RR^{d_\mcX}$.
\item (AC.\ref{a_X}-\ref{a_metr}) and (AR.\ref{aR_iso}-\ref{aR_c}),
\item  function $f$ in the $GPnn$ response model \eqref{eq:gpnn_responses} satisfies $\|f(\cdot)\|_\infty\leq B_f<\infty$ for some $B_f>0$,
\item functions $t_i$, $i=1,\dots, d_T$ in the $NNGP$ response model \eqref{eq:nngp_responses} satisfy $\|t_i(\cdot)\|_\infty<B_T<\infty$ for some $B_T>0$,
\item $\|\sigma_\xi^2(\cdot)\|_\infty<\infty$, where $\sigma_\xi^2(\bx):=\EE\left[\Xi^{2}\mid \mcX=\bx\right]$. 
\end{itemize}
Then we have the following limit for the risk of $GPnn$ and $NNGP$.
\begin{align}
& \EE_{\mcX_*,\bX_n}\left[MSE(\mcX_*,\bX_n)\right]\xrightarrow{n\to\infty} \EE_{\mcX_*}\left[\sigma_\xi^2(\mcX_*)\right].
\end{align}
}

\long\def\ConvergenceRatesStatement{%
Let $n$ be the size of the $GPnn/NNGP$ training set which is i.i.d. sampled from the distribution $P_{\mcX}$ and let the test point be also sampled from $P_{\mcX}$. Let $m$ be the (fixed) number of nearest-neighbours used in $GPnn/NNGP$. Assume (AC.\ref{a_metr}) and (AR.\ref{aR_iso}-\ref{aR_xi}). Define $\alpha := \min\{p,q\}$ for $GPnn$ and $\alpha:=\min\{p,q_0,\allowbreak q_1,\dots,q_{d_T}\}$ for $NNGP$. Then, if $d_\mcX > 4 (\alpha+p)$, we have
\begin{align}\label{eq:convergence_limit}
\begin{split}
\mcR_n\leq \frac{\sigma_\xi^2}{m}+A_1\,\left(\frac{m}{n}\right)^{2\alpha/d_\mcX}+A_2\,m\,\left(\frac{m}{n}\right)^{2(\alpha+p)/d_\mcX},
\end{split}
\end{align}
where $\mcR_n$ is the $GPnn/NNGP$ risk defined in \eqref{def:risk} and $A_1,A_2>0$ depend on $p$, $q$, $d_\mcX$, $B_f$, $B_T$, $L_f$, $L_c$, $\sigma_\xi$ and the $GPnn/NNGP$ hyper-paramaters. Taking $m_n = n^{\frac{2p}{2p+d_\mcX}}$ we obtain the following optimal minimax convergence rate.
\begin{equation}\label{mse_optimal_rate}
\mcR_n\leq\left(\sigma_\xi^2+A_1+A_2\right)\, n^{-\frac{2\alpha}{2p+d_\mcX}}.
\end{equation}
}

\long\def\ConvergenceRatesSmallDStatement{%
Let $n$ be the size of the training set which is i.i.d. sampled from the distribution $P_{\mcX}$ and let the test point be also sampled from $P_{\mcX}$.  Define $\alpha$ for $GPnn/NNGP$ as in Theorem \ref{thm:mse_convergence_rate}. Assume (AC.\ref{a_metr}), (AR.\ref{aR_iso}-\ref{aR_xi}) and 
\begin{itemize}
    \item $P_{\mcX}$ is supported on a compact convex set and has density which is smooth and strictly positive.
\end{itemize}
Then taking $m_n = n^{\frac{2p}{2p+d_\mcX}}$ we have for sufficiently large $n$
\[
\mcR_n\leq A\, n^{-\frac{2\alpha}{2p+d_\mcX}}
\]
where $0<A<\infty$ depends on $P_{\mcX}$, $p$, $q$, $d_\mcX$, $B_f$, $B_T$, $L_f$, $L_c$, $\sigma_\xi$ and the $GPnn$ or $NNGP$ hyper-paramaters.
}

\long\def\UniformConvergenceStatement{%
Let $X=(\bx_1,\bx_2,\dots)$ be an infinite sequence of i.i.d. points sampled from $P_\mcX$ and denote by $X_n$ its truncation to the first $n$ points. Assume (AC.\ref{a_X}-\ref{a_f}), (AC.\ref{a_metr}), (AR.\ref{aR_xi}) and (AR.\ref{aR_iso}), (AR.\ref{aR_f}). Then, for almost every sampling sequence $X$ and test point $\bx_*\in\supp(P_\mcX)$ and any compact subset $S$ of the hyper-parameters $\Theta=\left(\hat\sigma_\xi^2,\hat\sigma_f^2,\hat\ell\right)\in S\subset \RR_{\geq 0} \times \RR_{>0}\times \RR_{>0}$ we have that
\[ MSE(\bx_*,X_n;\Theta)\xrightarrow{n\to\infty} MSE_\infty(\bx_*;\Theta):= \sigma_\xi^2(\bx_*)\left(1+\frac{1}{m}\right)\]
and this convergence is uniform as a function of $\Theta\in S$.
}

\long\def\DerivativesConvergenceRatesStatement{%
Let $n$ be the size of the training set which is i.i.d. sampled from the distribution $P_{\mcX}$ and let the test point be also sampled from $P_{\mcX}$. Let $m$ be the (fixed) number of nearest-neighbours used in $GPnn/NNGP$. Assume (AC.\ref{a_metr}) and (AR.\ref{aR_iso}-\ref{aR_xi}). In $GPnn$ define $\alpha_\phi := \min\{p,q\}$ for each $\phi\in\{\hat\sigma_\xi^2,\hat\sigma_f^2,\hat\bb\}$. In $NNGP$ define $\overline{q}:=\min\{q_1,\dots,q_{d_T}\}$ and $\alpha_\phi:=\min\{p,q_0,\overline{q}\}$ when $\phi\in\{\hat\sigma_\xi^2,\hat\sigma_f^2\}$ and $\alpha_{\hat\bb}:=\min\{2p,\overline{q}\}$. Then, if $d_\mcX > 4 (\alpha_\phi+p)$, for each $\phi\in\{\hat\sigma_\xi^2,\hat\sigma_f^2,\hat\bb\}$ we have
\begin{align}
\EE\left[D_\phi(\mcX_*,\bX_n)\right]\leq A_1^{(\phi)}\,\left(\frac{m}{n}\right)^{2\alpha_\phi/d_\mcX}+A_2^{(\phi)}\,m\,\left(\frac{m}{n}\right)^{2(\alpha_\phi+p)/d_\mcX},
\end{align}
where $0<A_i^{(\phi)}<\infty$ depend on $p$, $\{q_i\}$, $d_\mcX$, $d_T$, $B_f$, $B_T$, $L_f$, $L_c$, $L_{\tilde c}$, $\sigma_\xi$ and the $GPnn/NNGP$ hyper-paramaters. Taking $m_n = n^{\frac{2p}{2p+d_\mcX}}$ the derivatives tend to zero at the same rates as the (minimax-optimal) risk rate from Stone's theorem, i.e.,
\begin{align}
\EE\left[D_\phi(\mcX_*,\bX_n)\right]\leq\left(A_1^{(\phi)}+A_2^{(\phi)}\right)\, n^{-\frac{2\alpha_\phi}{2p+d_\mcX}}. 
\end{align}
}

\long\def\DlConvergenceRatesStatement{%
Let $n$ be the size of the training set which is i.i.d. sampled from the distribution $P_{\mcX}$ and let the test point be also sampled from $P_{\mcX}$. Let $m$ be the (fixed) number of nearest-neighbours used in $GPnn/NNGP$. Assume (AC.\ref{a_metr}), (AR.\ref{aR_iso}-\ref{aR_xi}) and (AD.\ref{aD_iso}- \ref{aD_bnd}). Then, if $d_\mcX > 4(p'+2p)$, we have
\begin{align}\label{eq:dMSE_dl_fixedm_rate}
\EE\left[D_{\hat\ell}(\mcX_*,\bX_n)\right]\leq \frac{\max\{\hat\ell^{-2p'},1\}}{\hat\ell}A_1\,\left(\frac{m}{n}\right)^{2p'/d_\mcX}+\frac{1}{\hat\ell}A_2\,m^2\,\left(\frac{m}{n}\right)^{2(p'+2p)/d_\mcX},
\end{align}
where $0<A_1,A_2<\infty$ depend on $p$, $\{q_i\}$, $d_\mcX$, $d_T$, $B_f$, $B_T$, $B_c$, $L_f$, $L_c$, $L_c'$, $L_{\tilde c}$, $\sigma_\xi$, $\hat\sigma_\xi$, $\hat\sigma_f$ (but not on $\hat\ell$). Taking $m_n = n^{\frac{2p}{2p+d_\mcX}}$ the derivatives tend to zero at the following rate.
\begin{align}\label{eq:dMSE_dl_rate}
\EE\left[D_{\hat\ell}(\mcX_*,\bX_n)\right]\leq\frac{1}{\hat\ell}\left(\max\left\{\hat\ell^{-2p'},1\right\}A_1+A_2\right)\, n^{-\frac{2p'}{2p+d_\mcX}}.
\end{align}
}